\documentclass{article}

\PassOptionsToPackage{numbers,sort&compress}{natbib}

\usepackage[preprint]{neurips_2025}
\setcitestyle{numbers,square,sort&compress}
\usepackage[utf8]{inputenc} % allow utf-8 input
\usepackage[T1]{fontenc}    % use 8-bit T1 fonts
\usepackage{hyperref}       % hyperlinks
\usepackage{url}            % simple URL typesetting
\usepackage{booktabs}       % professional-quality tables
\usepackage{amsfonts}       % blackboard math symbols
\usepackage{graphicx}       % include graphics
\usepackage{subcaption}     % for subfigure environment
\usepackage{nicefrac}       % compact symbols for 1/2, etc.
\usepackage{microtype}      % microtypography
\usepackage[table]{xcolor}         % colors

\usepackage{amsmath,amsfonts,bm}
\def\vv{{\bm{v}}}

\DeclareMathAlphabet{\mathsfit}{\encodingdefault}{\sfdefault}{m}{sl}
\SetMathAlphabet{\mathsfit}{bold}{\encodingdefault}{\sfdefault}{bx}{n}

\DeclareMathOperator*{\argmin}{arg\,min}

\providecommand{\Bar}[1]{\overline{#1}}

\usepackage{amssymb}
\usepackage{amsthm}
\newtheorem{assumption}{Assumption}
\newtheorem{lemma}{Lemma}

\usepackage{array} 
\usepackage{caption}
\usepackage{siunitx}
\usepackage{makecell} 

\definecolor{rowgray}{gray}{0.90}   
\definecolor{headerblue}{RGB}{70, 130, 180}  

\definecolor{LightGray}{gray}{0.92} 
\definecolor{HighlightCol}{RGB}{220, 230, 245} 
\definecolor{HighlightCell}{RGB}{235, 245, 230}

\title{Harnessing Optimization Dynamics for Curvature-Informed Model Merging}

\author{%
  Pouria Mahdavinia \\
  The Pennsylvania State University \\
  \texttt{pmahdavi@psu.edu}
  \And
  Hamed Mahdavi \\
  The Pennsylvania State University \\
  \texttt{hmm5834@psu.edu} 
  \AND
  Niloofar Mireshghallah \\
  Carnegie Mellon University \\
  \texttt{niloofar@cmu.edu}
  \And
  Mehrdad Mahdavi \\
  The Pennsylvania State University \\
  \texttt{mzm616@psu.edu}
}

\begin{document}

\maketitle

\begin{abstract}
Model merging is an effective post--training strategy for composing capabilities in large language models without joint retraining. We study this in the supervised fine-tuning (SFT) stage, where multiple capability-based SFT checkpoints---spanning math, code, precise instruction following, general instruction following, and knowledge recall---must be consolidated into a single model. We introduce Optimization Trajectory Aware (OTA) Merging, a curvature-aware aggregation that leverages optimizer second-moment statistics as a diagonal curvature proxy to reweight parameter edits and mitigate interference. Complementing OTA, we propose Fast Fisher Grafting (FFG), a curvature-driven task-localization step that sparsifies conflicting or low-importance edits. FFG induces extremely low-rank masks concentrated in early attention query/key projections and token embeddings, exploiting shared curvature across capabilities. We further develop a memory-light compression of the second moments that preserves OTA's effect. Across diverse capability-based SFT checkpoints, OTA+FFG improves merged-model quality over strong weight-space baselines, reduces negative transfer, and remains robust across sparsity levels. Analyses reveal substantial curvature overlap between checkpoints, offering a novel lens on why simple linear merging can be effective in practice. Ablations confirm that FFG is critical for reducing task interference and that the compressed second moments retain the gains of the full formulation. To facilitate reproducibility, we open-source all code, training and evaluation scripts, visualization artifacts, and capability-specific SFT checkpoints at \url{https://github.com/pmahdavi/ota-merge}.
\end{abstract}

\section{Introduction}
\label{sec:introduction}

The advent of large language models (LLMs) has marked a significant milestone in artificial intelligence, providing powerful, generalist foundations for a wide array of human tasks. Fine-tuning these models on specialized data yields expert models with superior performance on targeted domains~\cite{brown2020language}. However, the operational and computational costs of deploying a large, ever-growing suite of these specialized experts are prohibitive. This challenge has catalyzed research into model merging, a promising paradigm for consolidating the capabilities of multiple expert models into a single, multitask entity. An ideal merging process would capture the proficiencies of each expert without the immense cost of retraining or the latency of ensembling, thereby creating a more versatile and efficient model.

Despite the empirical success of various merging techniques, from simple weight averaging to more complex, geometry-aware methods, a foundational question remains largely unanswered: why does model merging work so effectively? A prevailing hypothesis suggests that successful merging is possible because fine-tuned models co-inhabit a single, wide, flat loss basin, facilitating linear mode connectivity~\cite{frankle2020linear}. Yet, this theory does not fully account for the surprising robustness of simple linear averaging, which often remains a competitive baseline against far more sophisticated methods~\cite{yadav2024mattersmerging}. This suggests a critical gap in our understanding of the underlying loss landscape curvature. Without a clear map of this terrain, merging strategies are developed with limited theoretical guidance, relying more on heuristics than on first principles.

This paper introduces a novel, empirically-grounded perspective on the curvature of SFT fine-tuned LLMs. Our central insight is that the second-moment estimates (\texttt{exp\_avg\_sq}) accumulated by adaptive optimizers like Adam~\cite{kingma2014adam} are not merely a training artifact but a powerful, readily available proxy for the diagonal of the Fisher information matrix, and by extension, the curvature of the loss landscape. By analyzing this wealth of information---information that is "free" as a by-product of training---we investigate the curvature structure of various SFT fine-tuned checkpoints and uncover a striking phenomenon: \textbf{independently fine-tuned models converge to basins that exhibit remarkably similar curvature structures}. This structural alignment holds across both attention and feed-forward (FFN) layers, suggesting a shared geometric foundation that explains the surprising efficacy of simple model merging techniques. In essence, linear averaging performs well because the models' underlying curvature are already highly compatible.

This finding reframes the primary challenge in model merging. The problem is not necessarily one of correcting for gross geometric misalignment between models, but rather one of \textbf{mitigating the negative interference caused by noisy, low-saliency parameter updates} introduced during task-specific fine-tuning. These updates, while beneficial for a single task, can conflict and degrade performance when naively combined---a phenomenon explicitly addressed by recent work that seeks to resolve such parameter conflicts~\cite{yadav2023ties}.

Building on this insight, we propose the \textbf{Optimization Trajectory Aware (OTA)} framework. OTA operationalizes our findings through a two-stage process. First, it employs Fast Fisher Grafting, a principled mask selection method that leverages the optimizer states to identify and revert noisy, non-essential parameter updates from each expert model to the base model checkpoint, a process termed originally as \textbf{grafting}~\cite{panigrahi2023task}. As shown in Figure~\ref{fig:intro_visuals} (left), this process reveals that task-specific knowledge is highly localized within the network, creating structured sparsity. Second, OTA aggregates these denoised experts using a curvature-aware merging strategy, which uses the same optimizer states as a preconditioner to intelligently average the remaining parameters. Our results, summarized in Figure~\ref{fig:intro_visuals} (right), demonstrate that this principled, denoising-first approach allows our merged model to consistently outperform established baselines across a diverse set of capabilities.

The primary contributions of this work are:
\begin{enumerate}
   
    \item We reframe the central challenge of model merging from geometric misalignment to \textbf{interference mitigation}, proposing the highly effective OTA framework which uses optimizer states for both grafting and merging.
    \item We introduce \textbf{Fast Fisher Grafting (FFG)}, a computationally efficient and principled method for identifying and reverting low-saliency fine-tuning deltas, which localizes task-specific knowledge within the model. An extensive suite of experiments are conducted to analyze localization patterns of FFG compared to magnitude pruning, revealing emergence of layer-aware sparsity patterns, ranging from aggressive structured column/row sparsity in early, and late query, and key layers, to formation of specialized attention heads in output projection of late layers. 
    \item Furthermore, storing and accessing large second-moment matrices for all model weights requires an amount of storage comparable to the model itself, which hinders the practicality of our methodology. To address this, we propose compressing these matrices using a rank-one, AdaFactor-style approach\cite{shazeer2018adafactor}. We then evaluated this compression's effectiveness, showing that it maintains comparable performance on our model-merging benchmarks. Moreover, our analysis of the second-moment matrices across transformer layers revealed a very low stable rank, which further validates our compression technique.
     \item We provide \textbf{compelling empirical evidence for a new theory of model merging}. Our work demonstrates that models task-independently fine-tuned with SFT develop a shared curvature geometry, which explains the success of linear averaging. Moreover, we show that these models exhibit almost identical curvature structures \textbf{when trained on the same data but with different learning rate schedules}.
\end{enumerate}

Our experiments, which involve merging specialist Llama 3.1 8B models, confirm that the most significant performance gains stem from the principled, saliency-aware denoising enabled by FFG, validating our core hypothesis.

\begin{figure}[!ht]
    \centering
    % Subfigure: Task Localization Heatmap
    \begin{subfigure}[t]{0.44\textwidth}
        \centering
        \includegraphics[width=\textwidth]{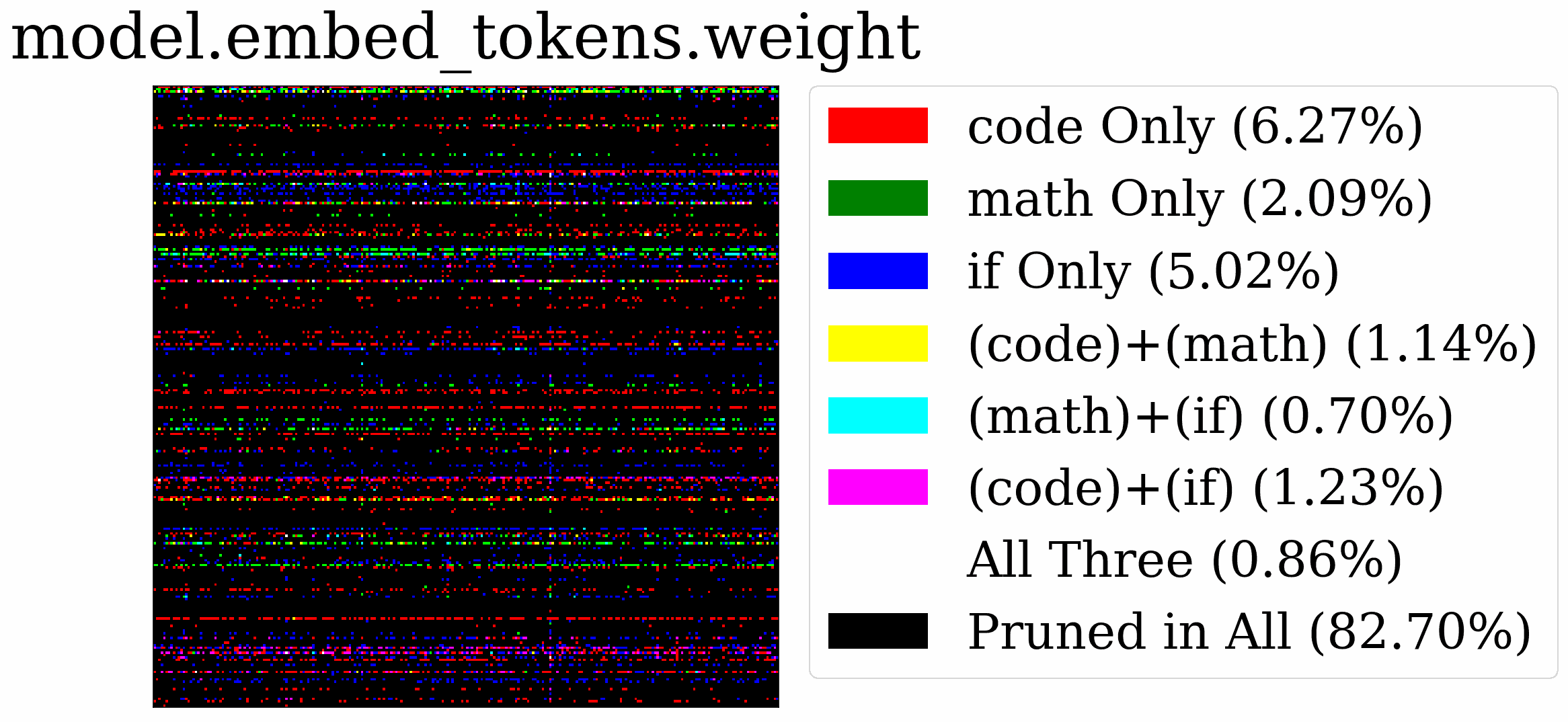}
        \caption{Task localization via FFG}
    \label{subfig:embed}
    \end{subfigure}
    \hfill
    % Subfigure: Radar Chart
    \begin{subfigure}[t]{0.50\textwidth}
        \centering
        \includegraphics[width=\textwidth]{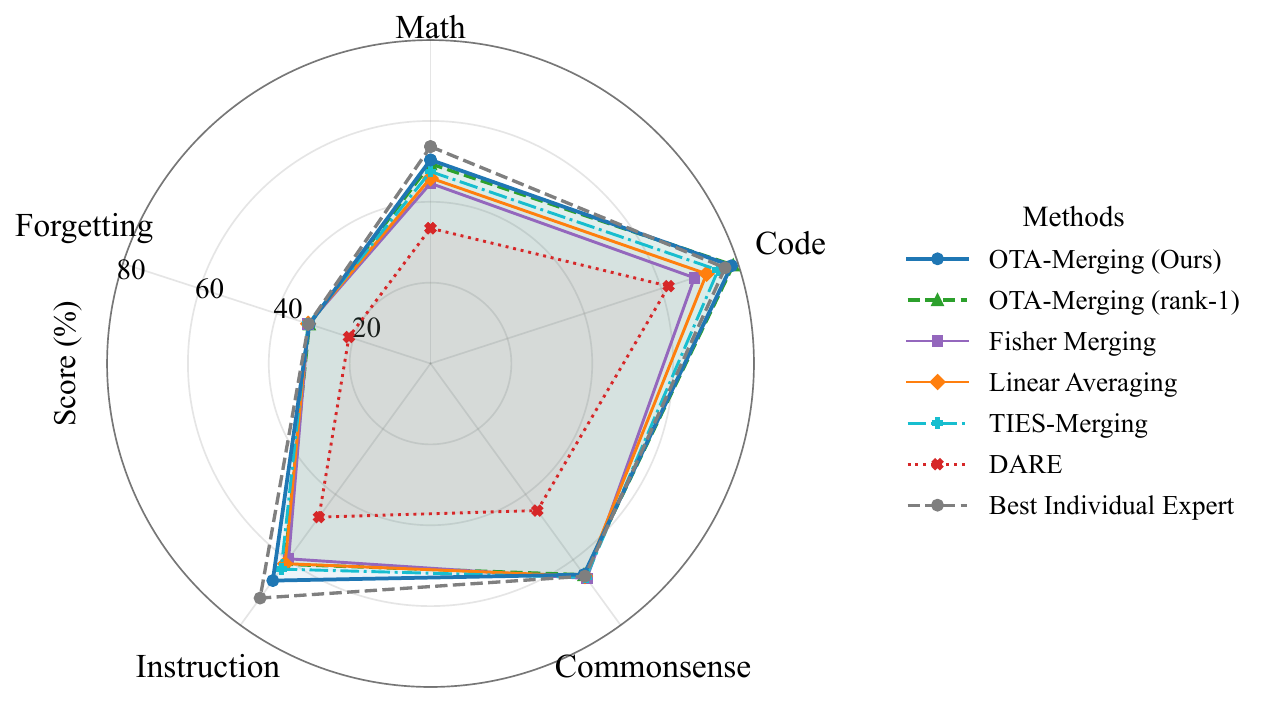}
        \caption{Merging performance comparison}
    \end{subfigure}
    % Main Caption
    \caption{\textbf{Left:} FFG reveals that task-specific knowledge is highly localized. This heatmap shows the FGG mask regions for three expert models (math, code, instructions) in the token embedding layer, demonstrating clear,  low-rank structured sparsity induced by FFG at $40\%$ global density . \textbf{Right:} A capability-based comparison shows that our full OTA method, which combines FFG-based denoising with curvature-aware aggregation, is the top-performing merging technique. The dashed line represents the performance ceiling of the best individual expert for each capability.}
    \label{fig:intro_visuals}
\end{figure}
\section{Related Work}
\label{sec:related}

\paragraph{Weight-Space Model Merging and Composition.}
A rapidly growing literature studies how to combine separately fine-tuned models directly in weight space.
Early work showed that simple weight averaging often improves accuracy and robustness when fine-tuned solutions lie in a shared basin \citep{wortsman2022model}.
\citet{matena2021merging} formalize merging as approximate posterior combination via Fisher-weighted averaging, where (diagonal) Fisher information acts as parameter-wise preconditioner.
Task arithmetic composes behaviors by adding/subtracting task vectors \citep{ilharco2022editing}; its theory and practice were strengthened by \citet{ortizjimenez2023tangent}, who advocate editing in the model’s tangent space.
Permutation alignment methods such as Git Re-Basin expose linear connectivity by matching hidden units before interpolation \citep{ainsworth2023gitrebasin}.
To curb interference, TIES-Merging trims small edits and resolves sign conflicts \citep{yadav2023ties}.
More recently, \citet{tam2024mats} cast merging as solving a linear system in a task-parameter subspace (MaTS), while \citet{huang2024emr} propose a tuning-free, high-performance recipe (EMR) that works across modalities.
Practitioner tooling such as \textsc{MergeKit} has standardized many of these strategies for LLMs \citep{goddard2024mergekit}.
Our approach complements these directions with a curvature-aware, two-stage pipeline: (i) FFG selects/denoises per-parameter edits using the second moments combined with Optimal Brain Surgeon methodology~\cite{hassibi1992obs} , and (ii) curvature-preconditioned aggregation reweights surviving edits during merging.

\paragraph{Parameter-Efficient Composition (Adapters/LoRA).}
A parallel line composes parameter-efficient modules rather than full models.
\citet{pfeiffer2020adapterfusion} fuse multiple adapters at inference time.
For LoRA modules, recent work investigates merging multiple skills into a single model: LoRA Soups advocate learnable concatenation/averaging to compose skills in NLP \citep{prabhakar2024merging}, and \citet{kesim2024merge_loras} study multi-adapter merging for vision.
These methods are orthogonal to ours and can benefit from the same curvature signals for selection or reweighting of adapter deltas.

\paragraph{Pruning and Grafting with Second-Order Signals.}
Classical pruning measured parameter saliency via second-order criteria: Optimal Brain Damage (OBD) uses a diagonal Hessian approximation \citep{lecun1990obd}, and Optimal Brain Surgeon (OBS) leverages full curvature \citep{hassibi1992obs}.
\citet{singh2020woodfisher} provide scalable inverse-Fisher approximations.
For LLMs, one-shot/zero-shot methods such as SparseGPT \citep{frantar2023sparsegpt} and Wanda \citep{sun2023wanda} enable accurate pruning without retraining; movement pruning adapts masks during fine-tuning \citep{sanh2020movement}.
In contrast, our objective is not generic compression: FFG computes a curvature-weighted edit saliency \(s_i = H_{ii}\,\Delta w_i^2\) and grafts by resetting low-saliency coordinates to base weights.
This simultaneously reduces cross-task interference and reveals interpretable task localization.

\paragraph{Mode Connectivity and Alignment.}
Mode-connectivity work shows that the independently trained checkpoints on same data  are often connected by low-loss paths \citep{garipov2018loss}.
After permutation alignment, independently trained networks lie in an approximately convex basin, which explains why linear interpolation/merging can work when models are geometrically aligned \citep{ainsworth2023gitrebasin}.
Our curvature-aware view complements these results: if diagonal curvature morphology is shared across specialists, then linear aggregation with curvature reweighting is particularly effective.

\paragraph{Curvature Proxies from Optimization Dynamics.}
Our work repurposes the readily available second-moment statistics from adaptive optimizers as a proxy for the diagonal Fisher information. Concurrent and independent work by \citet{li2025fishers} compellingly validates this core idea, introducing the "Squisher" and demonstrating its effectiveness as a "for free" replacement for a calculated Fisher across a broad set of applications, including model pruning, continual learning, and a form of Fisher-merging. While our work shares this foundational insight, it diverges significantly in its methodology, application focus, and conceptual contributions. 

First, while \citet{li2025fishers} apply Fisher pruning to the final model weights ($\mathbf{w}^*$) by setting parameters to zero, FFG instead operates on the task vector ($\Delta\mathbf{w} = \mathbf{w}^* - \mathbf{w}_0$) to revert low-saliency updates. This mitigates interference between non-IID experts---a critical step for our SFT merging setting. This denoising role is the cornerstone of our OTA-Merging framework and is a key differentiator from other merging methods. Second, we uniquely employ the second moment proxy as an analytical and interpretability lens. We use it to propose and provide strong empirical evidence for a shared curvature hypothesis, offering a new explanation for the effectiveness of model merging. Furthermore, we leverage FFG as a tool for task localization to understand SFT training regimes, revealing how skills are encoded via structured, role-aware sparsity patterns in the network, a line of analysis not pursued in the concurrent work.

Finally, to address the significant practical issue of storage, we propose and validate an AdaFactor-style rank-1 compression of the second-moment tensor. This reduces the storage overhead significantly, making our approach highly scalable for large models. In summary, while \citet{li2025fishers} establish the broad utility of the optimizer-as-Fisher proxy, our work presents a specialized, end-to-end framework for the challenging SFT merging problem, complete with a novel denoising mechanism, new interpretability insights, and a practical, scalable implementation.

\section{Preliminaries}
\label{sec:prelim}

This section establishes the notation and foundational concepts that underpin our work. We begin by formalizing the SFT setup and the associated notation. We then introduce the Fisher Information Matrix (FIM) as a key tool for understanding the curvature of the loss landscape. Finally, we review how second-order information, approximated by the FIM, is leveraged in established methods for model merging and parameter grafting, setting the stage for our proposed contributions.

\subsection{SFT Setup}

\noindent\textbf{Notation.} We denote matrices with bold capital letters ($\mathbf{A}$), vectors with bold lowercase letters ($\mathbf{v}$), and scalars with regular lowercase letters ($s$). A vector-valued function's $j^\text{th}$ output is denoted as $f^j$. The $i$-th standard basis vector is $\mathbf{e}_i$, and an $n$-dimensional vector of ones is $\mathbf{1}_n$. For any Positive Semi-Definite (PSD) matrix $\mathbf{P} \in \mathbb{R}^{d \times d}$, we define its induced norm on a vector $\mathbf{x} \in \mathbb{R}^{d}$ as $\|\mathbf{x}\|_{\mathbf{P}} = \sqrt{\mathbf{x}^\top \mathbf{P} \mathbf{x}}$.

\noindent\textbf{Learning Setup.} We consider a supervised fine-tuning (SFT) scenario with $T$ distinct tasks from the same base model. For each task $\tau \in \{1, \dots, T\}$, we have a dataset $\mathcal{S}_{\tau} = \{(\mathbf{x}^{\tau}_i , y^{\tau}_i)\}_{i=1}^{|\mathcal{S}_{\tau}|}$, with samples drawn from a true data distribution $D_{\tau}$. We begin with a common pre-trained model architecture, parameterized by $\bm{w}_0 \in \mathbb{R}^d$, which is then fine-tuned for each specific task. The model, $f(\cdot; \bm{w})$, maps an input $\mathbf{x} \in \mathcal{X}$ to a logit vector $\mathbf{z} \in \mathbb{R}^{|V|}$, where $|V|$ is the vocabulary size. These logits parameterize a conditional probability distribution $P(y|\mathbf{x}; \bm{w})$ via the softmax function: $\log p(y|\mathbf{x}; \bm{w}) = \mathbf{e}_y^\top \mathbf{z} - \log \left(\sum_{j=1}^{|V|} \exp(z_j)\right)$.

The objective for each task $\tau$ is to minimize the empirical cross-entropy loss, which approximates the true expected risk over the data distribution $D_{\tau}$:
\[
\mathcal{L}_{\mathcal{S}_\tau}(\bm{w}) = - \mathbb{E}_{(\mathbf{x},y) \sim \mathcal{S}_{\tau}} \left[ \log p(y|\mathbf{x}; \bm{w}) \right].
\]
For notational simplicity, we will drop the task subscript $\tau$ when the context is clear.

\subsection{The Fisher Information Matrix}
A central concept for analyzing the loss landscape is the Fisher Information Matrix (FIM), which measures the sensitivity of the model's output distribution to changes in its parameters $\bm{w}$. It provides a powerful approximation of the loss curvature~\cite{amari1998natural,martens2020new} and is equivalent to the negative Hessian of the log-likelihood, under expectation over model's predictive distributions, $\mathbf{F}(\bm{w})=-\mathbb{E}_{y \sim P(y | \bm{x};\bm{w})}[\nabla^2_{\bm{w}} \log p(y\mid \mathbf{x};\bm{w})]$. In practice, several variants of the FIM are used:

\textbf{True FIM}, $\mathbf{F}(\bm{w})$, is defined over the true data distribution and the model's predictive distribution, making it intractable for deep neural networks:
    \begin{equation}
    \label{eq:true_fim}
    \mathbf{F}(\bm{w}) = \mathbb{E}_{\mathbf{x} \sim D(\mathbf{x}), y \sim P(y|\mathbf{x} ; \bm{w})} \left [ \nabla_{\bm{w}} \log p(y|\mathbf{x} ; \bm{w}) \nabla_{\bm{w}} \log p(y|\mathbf{x} ; \bm{w})^\top \right].
    \end{equation}

\textbf{Expected Empirical FIM}, $\hat{\mathbf{F}}(\bm{w})$, approximates the true FIM by using a finite dataset $\mathcal{S}$ but still requires an expectation over the model's predictions:
    \begin{equation}
    \label{eq:empirical_fim}
    \hat{\mathbf{F}}(\bm{w}) = \frac{1}{|\mathcal{S}|}\sum_{(\mathbf{x}_i, y_i) \in \mathcal{S}} \mathbb{E}_{y \sim P(y|\mathbf{x}_i ; \bm{w})} \left [ \nabla_{\bm{w}} \log p(y|\mathbf{x}_i ; \bm{w}) \nabla_{\bm{w}} \log p(y|\mathbf{x}_i ; \bm{w})^\top \right].
    \end{equation}

\textbf{Observed Empirical FIM}, $\Bar{\mathbf{F}}(\bm{w})$, simplifies this further by replacing the expectation with the observed ground-truth labels from the dataset. This variant, often called the "empirical Fisher," is the most commonly used in practice \cite{martens2015optimizing, matena2022merging} due to its computational advantage by avoiding the need for costly sampling from model's distribution:
    \begin{equation}
    \label{eq:empirical_fim_observed}
    \Bar{\mathbf{F}}(\bm{w}) = \frac{1}{|\mathcal{S}|}\sum_{(\mathbf{x}_i, y_i) \in \mathcal{S}} \left[ \nabla_{\bm{w}} \log p(y_i|\mathbf{x}_i ; \bm{w}) \nabla_{\bm{w}} \log p(y_i|\mathbf{x}_i ; \bm{w})^\top \right].
    \end{equation}

\subsection{Applications of Second-Order Information}

The FIM's ability to capture loss curvature makes it invaluable for a range of model manipulation techniques, from merging diverse experts to compressing a single model.

\noindent\textbf{Preconditioned Model Merging.} The goal of model merging is to combine a set of fine-tuned expert models, $\{\bm{w}_{\tau}^{*}\}_{\tau=1}^T$, into a single, multi-tasked model $\bm{w}_{\text{merged}}$. A generalized approach involves a weighted average in parameter space \cite{tam2023merging}:
\begin{equation}
\label{eq:precond-merge-general}
\bm{w}_{\text{merged}} = \left(\sum_{\tau=1}^T \mathbf{C}_{\tau} \right)^{-1} \left(\sum_{\tau=1}^T \mathbf{C}_{\tau} \bm{w}^*_{\tau}\right),
\end{equation}
where $\mathbf{C}_{\tau}$ are PSD weighting matrices. Different choices for $\mathbf{C}_{\tau}$ yield different merging strategies. For instance, Fisher-weighted averaging \cite{matena2022merging} uses the empirical FIM of each task as $\mathbf{C}_{\tau}$, leveraging the loss landscape geometry to guide the combination process, and~\citet{tam2023merging} leveraged Kronecker factored approximation of empirical fisher~\citep{martens2015kfac} for the choice of $\mathbf{C}_{\tau}$ .

\noindent\textbf{Optimal Brain Damage for Pruning and Grafting.} Second-order information is also fundamental to classic model compression techniques like Optimal Brain Damage (OBD) \cite{lecun1989optimal, hassibi1992second}. OBD identifies and removes parameters with the smallest impact on the loss function. This impact, or "saliency," is estimated via a second-order Taylor expansion. For a model at a local minimum $\bm{w}^*$, a small parameter perturbation $\delta\bm{w}$ changes the loss by $\Delta\mathcal{L}_{\mathcal{S}} \approx \frac{1}{2} (\delta\bm{w})^\top \mathbf{H} (\delta\bm{w})$, where $\mathbf{H}$ is the Hessian, approximated by the FIM. To make this tractable, OBD typically uses only the diagonal of the Hessian. Pruning a parameter $w_i^*$ to zero corresponds to a saliency score of $s_i = \frac{1}{2} \mathbf{H}_{ii} (w_i^*)^2$.

This framework can be repurposed from pruning to \textbf{grafting}. Instead of nullifying parameters, we can selectively revert fine-tuned parameters $\bm{w}^*$ back to their pre-trained state $\bm{w}^0$. The perturbation becomes $\delta w_i = w_i^* - w_i^0$, and the saliency of keeping the fine-tuned update is calculated as $s_i = \frac{1}{2} \mathbf{H}_{ii} (w_i^* - w_i^0)^2$. This score quantifies the importance of the change acquired during fine-tuning, providing a direct link to merging by deciding on a parameter-wise basis whether to retain a specialized update or revert to the base model.

\subsection{Efficiently Estimating Second-Order Information}
A major challenge in using second-order methods is the formidable memory cost of storing the full FIM or Hessian. While generic low-rank approximations like SVD exist, they do not guarantee the preservation of non-negativity, a defining property of these matrices.

\noindent\textbf{Factored Estimators (AdaFactor).} The AdaFactor optimizer \cite{shazeer2018adafactor} introduces a memory-efficient factorization that guarantees non-negativity. For a matrix of squared-gradient Exponential Moving Averages (EMAs) $\mathbf{V} \in \mathbb{R}^{m \times n}$, AdaFactor avoids storing the full $mn$ elements. Instead, it maintains only the moving averages of its row and column sums: $\mathbf{r} = \mathbf{V}\mathbf{1}_n \in \mathbb{R}^m$ and $\mathbf{c}^\top = \mathbf{1}_m^\top \mathbf{V} \in \mathbb{R}^{1 \times n}$. A rank-1, non-negative approximation of the full matrix is then reconstructed as $\hat{\mathbf{V}} = \mathbf{r}\mathbf{c}^\top / (\mathbf{1}_m^\top \mathbf{r})$. This reduces storage from $O(mn)$ to $O(m+n)$ per parameter matrix. The effectiveness of such low-rank approximations is often justified by the concept of stable rank, $r_s(\mathbf{V}) = \|\mathbf{V}\|_F^2 / \|\mathbf{V}\|_2^2$, which measures how well a matrix can be approximated by a low-rank counterpart. While AdaFactor was designed to save memory during training, we propose leveraging its factorization after training to create a highly compressed snapshot of the second-moment matrix, providing nearly storage-free access to valuable curvature information.
\section{The OTA-Merging Framework}
\label{sec:methodology}

We propose OTA Merging, a unified framework designed to merge fine-tuned experts by addressing parameter interference and curvature misalignment in a principled, storage-efficient manner. Our approach is built on a key insight: the second-moment estimates tracked by adaptive optimizers like Adam~\cite{kingma2014adam} can serve as a computationally cheap yet effective proxy for the local curvature of the loss landscape. By leveraging this curvature information, OTA-Merging executes a three-stage process: (1) it identifies and isolates the critical parameters for each task using a novel pruning strategy, FFG; (2) it aggregates these task-specific subnetworks using a curvature-aware weighting scheme; and (3) it employs a compression technique to store the required second moment information with minimal memory overhead.

\subsection{Adam's Second Moment as a Proxy for the Empirical Fisher}
\label{subsec:theory}

Preconditioning-based optimizers, such as Adam~\cite{kingma2014adam} and AdaGrad~\cite{duchi2011adaptive}, scale gradients by a preconditioner matrix that approximates the Fisher Information Matrix (FIM). For a model with parameters $\mathbf{w}$, the update at step $k$ is given by $\mathbf{w}_{k+1} = \mathbf{w}_k - \eta \mathbf{P}_k^{-1} \mathbf{m}_k$, where $\eta$ is the learning rate, $\mathbf{m}_k$ is the first momentum of the gradients, and $\mathbf{P}_k = \operatorname{Diag}(\mathbf{v}_k)$ is a diagonal preconditioner derived from the second moment, $\mathbf{v}_k$.

The second moment, $\mathbf{v}_k$, is typically an exponential moving average (EMA) of element-wise squared gradients: $\mathbf{v}_k = \beta_2 \mathbf{v}_{k-1} + (1-\beta_2) (\nabla \mathcal{L}_{B_k}(\mathbf{w}_k))^{\odot2}$, where $B_k$ is the mini-batch at step $k$. This formulation means that $\mathbf{P}_k$ accumulates information about the diagonal of the empirical FIM over the optimization trajectory. A comprehensive study on the connection between the empirical FIM and the Hessian is provided by~\cite{martens2020new}. Moreover,~\cite{morwani2024new} study the connection of the outer product of mini-batch gradients to the empirical FIM. By leveraging these works, we formalize the connection between the second moment and the Hessian under two mild assumptions, with detailed proofs deferred to Appendix~\ref{sec:appendix_proofs} for the sake of completeness.

\noindent\textbf{Theoretical Justification.}
The core argument rests on the equivalence between the Hessian of the loss function ($\nabla^2 \mathcal{L}_D$) and the Observed Empirical FIM ($\Bar{\mathbf{F}}$) near a local minimum $\mathbf{w}_*$.
\begin{enumerate}
    \item We first assume the network's output is locally linear with respect to its parameters near the end of training (the \textbf{Late NTK Regime}). This allows us to approximate the Hessian with the Generalized Gauss-Newton (GGN) matrix. While neural network training has been shown not to be well-approximated by the NTK at initialization, meaning an aggressive kernel change is necessary for feature learning~\citep{vyas2022limitations}, the kernel has been shown to stabilize near the end of training~\citep{fort2020deep}.
    \item We then assume the model is well-calibrated at convergence (\textbf{Perfect Calibration}), meaning its predictive distribution matches the true data distribution.
\end{enumerate}
Under these assumptions, one can argue that the Hessian is approximately equal to the Observed Empirical FIM: $\nabla^2 \mathcal{L}_{D} (\mathbf{w}_*) \approx \Bar{\mathbf{F}}(\mathbf{w}_*)$. Furthermore, the expectation of the outer product of mini-batch gradients is a scaled version of the FIM: $\mathbb{E}_{B_k \sim D^{|B|}} [ \nabla \mathcal{L}_{B_k} \nabla \mathcal{L}_{B_k}^\top ] = \frac{1}{|B|} \Bar{\mathbf{F}}(\mathbf{w}_*)$. We would like to highlight that the equivalence of the Empirical Fisher information with the Hessian under perfect calibration is a well-known property~\cite{amari1998natural}, and the expectation of the mini-batch gradient outer product is a corollary of Lemma~$8$ in~\citet{morwani2024new}.  

Thus, we provide a strong theoretical argument for our method: the Adam second-moment accumulator, $\mathbf{v}$, on expectation, is a scaled EMA of the diagonal of the FIM. It is therefore a valid and computationally free proxy for the diagonal curvature of the loss landscape, which we can harness for both parameter selection and model merging.

\subsection{Component 1: Parameter Selection with FFG}
\label{subsec:ffg}

To mitigate destructive interference when merging, we first identify a subnetwork within each expert that maintains the fine-tuning performance of the full model. Inspired by Optimal Brain Damage~\cite{lecun1989optimal}, we score the saliency of each parameter's change from its pre-trained state $\mathbf{w}_0$. The saliency of a parameter change $\Delta w_{\tau,i} = w_{\tau,i}^* - w_{0,i}$ for task $\tau$ is defined by its contribution to the loss, approximated by a second-order Taylor expansion: $s_{\tau,i} = \frac{1}{2} \mathbf{H}_{ii} (\Delta w_{\tau,i})^2$.

Calculating the Hessian $\mathbf{H}$ is infeasible for large models. However, based on the theory in Section~\ref{subsec:theory}, the second-moment estimate $\mathbf{v}_\tau$ from the Adam optimizer serves as an effective, training-free proxy for the diagonal of the FIM, which in turn approximates the Hessian. This insight leads to our FFG saliency score:
\begin{equation}
s_{\tau,i} = (\Delta w_{\tau,i})^2 \cdot v_{\tau,i}.
\end{equation}
For each expert $\tau$, we compute this score for every parameter in its task vector $\Delta\mathbf{w}_\tau = \mathbf{w}^*_\tau - \mathbf{w}_0$. We then generate a binary mask $\mathbf{m}_\tau$ by preserving only the top-$k$ parameters with the highest saliency scores, where $k$ is set by a sparsity ratio $\rho$. Instead of pruning parameters to zero, we graft by reverting the non-selected parameters back to their $\mathbf{w}_0$ values. The resulting pruned task vector is thus $\Delta\mathbf{w}'_\tau = \mathbf{m}_\tau \circ \Delta\mathbf{w}_\tau$.

\subsection{Component 2: Curvature-Aware Aggregation}
\label{subsec:ota_aggregation}

After identifying the essential subnetwork for each expert, we must aggregate them in a way that respects the curvature of the loss landscape. Inspired by preconditioned model merging methods discussed in Section~\ref{sec:prelim}, we can achieve this by solving for a merged parameter vector that is minimally distant from each of the pruned task vectors, where distance is measured in a space warped by the curvature. Let $\mathbf{P}^*_{\tau, \text{Adam}} = \operatorname{Diag}(\sqrt{\mathbf{v}_{\tau}^{\ast}} + \epsilon)$ be the diagonal preconditioning matrix derived from Adam's second-moment estimates for expert $\tau$. The merged model is the solution to the following optimization problem:
\begin{equation}
\mathbf{w}_{\text{merged}} = \mathbf{w}_0 + \argmin_{\Delta\mathbf{w}} \sum_{\tau=1}^T \| \Delta\mathbf{w} - \Delta\mathbf{w}'_{\tau} \|^2_{\mathbf{P}^{*}_{\tau, \text{Adam}}},
\end{equation}
This objective has a closed-form solution, yielding a pre-conditioned average of the pruned task vectors:
\begin{equation}
\label{eq:ota_ffg_base}
\mathbf{w}^{\text{OTA}}_{\text{merged}} = \mathbf{w}_0 + \left( \sum_{\tau=1}^T \mathbf{P}^{*}_{\tau, \text{Adam}} \right)^{-1} \left(\sum_{\tau=1}^T \mathbf{P}^{*}_{\tau, \text{Adam}} (\mathbf{m}_\tau \circ \Delta\mathbf{w}_\tau) \right).
\end{equation}

\subsection{Component 3: Memory-Efficient Preconditioner Compression}
\label{subsec:compression}

A practical challenge is that storing the full second-moment tensor $\mathbf{v}_\tau$ for each expert doubles the storage cost. To overcome this, we adopt a compression strategy inspired by AdaFactor~\cite{shazeer2018adafactor}. For any large weight matrix, instead of storing the full $\mathbf{v}_\tau$, we only store the moving averages of its row-wise and column-wise sums. We can then reconstruct a non-negative, rank-1 approximation of the second-moment tensor, $\hat{\mathbf{v}}_\tau$, from these compressed statistics at runtime.

This low-rank approximation is then used to form a compressed preconditioner, $\hat{\mathbf{P}}^{*}_{\tau} = \operatorname{Diag}(\sqrt{\hat{\mathbf{v}}_{\tau}} + \epsilon)$, which replaces its full-rank counterpart in both the FFG saliency calculation (Section~\ref{subsec:ffg}) and the OTA aggregation formula (Eq.~\ref{eq:ota_ffg_base}). The final, memory-efficient OTA-Merging update is:
\begin{equation}
\label{eq:ota_ffg_compressed}
\mathbf{w}^{\text{OTA}}_{\text{merged}} = \mathbf{w}_0 + \left( \sum_{\tau=1}^T \hat{\mathbf{P}}^{*}_{\tau} \right)^{-1} \left(\sum_{\tau=1}^T \hat{\mathbf{P}}^{*}_{\tau} (\mathbf{m}_\tau \circ \Delta\mathbf{w}_\tau) \right).
\end{equation}
This unified equation elegantly demonstrates how OTA first determines \textit{what} to merge via the FFG mask $\mathbf{m}_\tau$ and then decides \textit{how} to merge using the compressed, curvature-aware preconditioner $\hat{\mathbf{P}}^{*}_{\tau}$, forming a complete and scalable framework.
\section{Experiments}
\label{sec:experiments}

First, we conduct a series of experiments to validate our framework, demonstrating that combining FFG masks with OTA for aggregation yields a superior merging solution. Our central hypothesis is that this two-stage process—first identifying task-critical parameters via FFG, then aggregating them using our curvature-aware method—outperforms established baselines and each of its constituent components.

We structure our investigation to answer the following questions:
\begin{enumerate}
    \item How does our full OTA-FFG method compare to state-of-the-art merging techniques and key ablations on a diverse suite of benchmarks?
    \item How does FFG's mask selection mechanism compare to standard baselines like magnitude pruning, and what is the performance impact across different sparsity ratios?
    \item What kind of sparsity patterns does FFG induce, and do they reveal interpretable, task-specific structures within the network?
    \item What is the empirical justification for our compression strategy?
\end{enumerate}

Second, we use the second-moment connection to curvature as a lens to study the loss landscape properties of SFT fine-tuned models at the end of training. The comprehensive comparison between curvature proxies of different models shows that models fine-tuned on completely different capabilities surprisingly share an already similar curvature. This is demonstrated through extensive visualization of side-by-side comparisons and element-wise max-min ratio curvature heatmaps for all transformer layers. Furthermore, models trained on the same data but with different hyperparameter settings (e.g., a different LR scheduler) appear to have an almost identical curvature structure. This observation provides strong evidence for the effectiveness of the simple linear averaging method, as it shows that the task vectors are implicitly curvature-aligned for most of the parameters.

\subsection{Experimental Setup}
\label{subsec:exp_setup}

\paragraph{Models, Tasks, and Training.}
Our experiments use \texttt{meta-llama/Meta-Llama-3.1-8B} as the base model. To create a realistic merging scenario, we fine-tune five SFT models on distinct, capability-aligned subsets of the \texttt{allenai/tulu-3-sft-mixture} dataset \cite{lambert2024t}. These capabilities, summarized in Table~\ref{tab:specialist_datasets}, include mathematics, coding, general instruction following, knowledge recall, and precise instruction following. This setup creates a well-posed aggregation problem where each expert localizes a complementary skill.

All models are fine-tuned using full-parameter SFT via the \texttt{LLaMA-Factory} library \cite{zheng2024llamafactory}. Crucially for our method, we use the AdamW optimizer and save the complete optimizer state, including the exponential moving average of squared gradients (\texttt{exp\_avg\_sq}), which serves as our preconditioning tensor and curvature proxy.

\begin{table}[htbp]
\centering
\small
\setlength{\tabcolsep}{4pt}
\rowcolors{2}{rowgray}{white}
\caption{SFT model capabilities and representative fine-tuning datasets from the \texttt{tulu-3-sft-mixture}.}
\vspace{0.1em} 
\label{tab:specialist_datasets}
\begin{tabular}{@{}>{\raggedright\arraybackslash}p{3.0cm} >{\raggedright\arraybackslash}p{8.0cm}@{}} 
\rowcolor{headerblue} 
\textcolor{white}{\textbf{Capability}} & \textcolor{white}{\textbf{Representative Datasets}} \\
\toprule
Mathematics Reasoning & \texttt{Tülu 3 Persona MATH}, \texttt{OpenMathInstruct 2}, \texttt{NuminaMath-TIR} \\
Coding & \texttt{Tülu 3 Persona Python}, \texttt{Evol CodeAlpaca} \\
General Instruction Following & \texttt{WildChat (GPT-4 subset)}, \texttt{OpenAssistant}, \texttt{No Robots} \\
Knowledge Recall & \texttt{FLAN v2}, \texttt{SciRIFF}, \texttt{TableGPT} \\
Precise Instruction Following & \texttt{Tülu 3 Persona IF} \\
\bottomrule
\end{tabular}
\end{table}

\paragraph{Methods Under Comparison.}
We evaluate our proposed method and its ablations against a suite of strong baselines implemented in \texttt{MergeKit}~\cite{goddard2024arcee}:
\begin{itemize}
    \item \textbf{OTA with FFG (Ours)}: The full two-stage method. Our initial step is to sparsify each SFT model's task vector using FFG. We tune the sparsity ratio to ensure that the fine-tuning performance of each expert model is preserved for its corresponding capabilities, while making the task vector as sparse as possible. We then merge the resulting sparsified task vectors using OTA's curvature-aware aggregation.
    \item \textbf{OTA (Aggregation Only)}: An ablation that omits the FFG selection stage and performs curvature-aware aggregation on the full task vectors.
    \item \textbf{FFG-TA (Selection Only)}: An ablation study applying FFG to task vectors and then merging them with simple Task Arithmetic (linear averaging).
    \item \textbf{Baselines}: We compare against widely-used methods, including simple \textbf{Linear Averaging}~\cite{wortsman2022model}, \textbf{TIES-Merging}~\cite{yadav2023ties}, \textbf{DARE}~\cite{language_leyu_2023}, \textbf{Breadcrumbs}~\cite{Davari2024model}, and preconditioned \textbf{Fisher Merging}~\cite{matena2022merging}.
\end{itemize}
Notably, for all methods under comparison, the sparsity ratio is tuned on a per-expert basis, whether using FFG or magnitude pruning on task vectors. We observed that tuning a fixed sparsity ratio for all experts made the performance of both OTA and TIES no better than that of linear merging.
\paragraph{Evaluation Suite.}
We evaluate all merged models on a diverse set of benchmarks using the T\"{u}lu-3 evaluation suite via the OLMES toolkit \cite{lambert2024t}, ensuring a rigorous and reproducible assessment. The suite includes: \textbf{HumanEval(+)}~\cite{chen2021evaluating, liu2024correct} for coding, \textbf{GSM8K}~\cite{cobbe2021gsm8k} and \textbf{MATH}~\cite{hendrycks2021math} for mathematical reasoning, \textbf{IFEval}~\cite{zhou2023ifeval} for instruction following, \textbf{BBH (CoT)}~\cite{suzgun2022challenging} for general reasoning, \textbf{DROP}~\cite{dua2019drop} for reading comprehension, and \textbf{PopQA}~\cite{mallen2023trust} for knowledge recall.

\subsection{Main Results: Merging Performance}

\paragraph{SFT Models Localize Distinct Capabilities.}
First, we establish a baseline by analyzing the performance of individual SFT models (Table~\ref{tab:specialist_baselines}). As expected, each expert excels on benchmarks aligned with its training data: the Math specialist outperforms all others on \texttt{MATH} (0.316), and the Coding specialist dominates \texttt{HumanEval} (0.788). Conversely, the near-zero scores of non-coding specialists on \texttt{HumanEval} underscore the non-IID nature of the training data and highlight the core challenge of merging: combining these complementary but isolated skills without destructive interference. Moreover, the SFT experts appear to either outperform or closely match the performance of the Tulu-3 SFT checkpoint (the multi-task SFT tuned model), with the exception of coding capability. For coding, the Tulu models have considerably higher performance on HumanEval and HumanEval+ compared to the SFT coding model. This suggests that this task benefits the most from multi-task learning, as code-only filtered subset of the Tulu SFT mixture was unable to retain the performance of the multi-instructed models.

\begin{table}[ht]
\centering
\caption{Performance of individual SFT experts and the multi-task Tulu-3-8B SFT model on their respective benchmarks. Best performance per row is shown in bold. Scores are reported in $[0,1]$.}
\label{tab:specialist_baselines}
\small
\sisetup{detect-weight=true,detect-family=true,round-mode=places,round-precision=3}
\setlength{\tabcolsep}{4pt}
\begin{tabular}{l *{5}{S[table-format=1.3]} S[table-format=1.3]}
\toprule
\rowcolor{LightGray}
\multicolumn{1}{l}{Benchmark} & \multicolumn{5}{c}{Specialists} & \multicolumn{1}{c}{Tulu-3 SFT} \\
\cmidrule(lr){2-6}\cmidrule(lr){7-7}
& {General} & {Knowledge} & {Math} & {Precise IF} & {Coding} & { } \\
\midrule
BBH-CoT & 0.671 & 0.634 & 0.650 & 0.628 & 0.635 & {\bfseries 0.688} \\
HumanEval & 0.000 & 0.000 & 0.700 & 0.688 & 0.788 & {\bfseries 0.866} \\
HumanEval+ & 0.000 & 0.000 & 0.659 & 0.632 & 0.744 & {\bfseries 0.805} \\
DROP & 0.571 & {\bfseries 0.629} & 0.583 & 0.586 & 0.552 & 0.616 \\
GSM8K & 0.575 & 0.589 & 0.757 & 0.594 & 0.561 & {\bfseries 0.767} \\
IFEval & 0.516 & 0.538 & 0.257 & {\bfseries 0.717} & 0.425 & 0.715 \\
MATH & 0.170 & 0.171 & {\bfseries 0.316} & 0.199 & 0.181 & 0.290 \\
POPQA & {\bfseries 0.317} & 0.301 & 0.296 & 0.307 & 0.301 & 0.295 \\
\bottomrule
\end{tabular}
\end{table}

\paragraph{OTA with FFG Achieves State-of-the-Art Merging Performance.}
The main results in Table~\ref{tab:benchmark_results} confirm our core hypothesis. We compared the performance of several methods for merging the five SFT checkpoints discussed previously. The table shows the performance of each merging method (rows) on a specific capability (columns). A capability's performance is measured by averaging the scores from the benchmarks assigned to it in our evaluation suite. Specifically: Math performance is the average of the GSM8K and MATH benchmarks; Code is the average of HumanEval and HumanEval+; Commonsense is the average of BBH and Drop; Instruction-Following is measured by IFEval; and Forgetting is measured by PopQA. Our full method, OTA, achieves the highest average score (0.582) across all merging techniques, outperforming strong baselines like TIES (0.565). The ablation studies clearly show that the most significant gains come from FFG's saliency-based task vector sparsification. The FFG-TA (Selection Only) ablation, which simply averages FFG-pruned task vectors, already achieves a strong 0.560 average. This is substantially better than OTA (Aggregation Only) (0.536), which uses our curvature preconditioning on unpruned task vectors. This result strongly supports our thesis that the primary obstacle in merging non-IID experts is parameter interference, which FFG effectively mitigates by acting as a denoiser. The poor performance of DARE (0.417) further reinforces that naive, random pruning is detrimental; a saliency-aware method is essential. 

\begin{table}[ht]
\caption{Performance comparison of merging methods. The best-performing merge method in each column is highlighted. The "Average" score is the unweighted mean across the five capability metrics.}
\label{tab:benchmark_results}
\centering
\small
\sisetup{detect-weight=true,detect-family=true,round-mode=places,round-precision=3}
\setlength{\tabcolsep}{5pt}
\begin{tabular}{l S[table-format=1.3] S[table-format=1.3] S[table-format=1.3] S[table-format=1.3] S[table-format=1.3] S[table-format=1.3]}
\toprule
\rowcolor{LightGray}
Model & {Math} & {Code} & {Commonsense} & {Instruction} & {Forgetting} & {Average} \\
\midrule
DARE & 0.335 & 0.619 & 0.450 & 0.470 & 0.212 & 0.417 \\
Breadcrumbs & 0.453 & 0.722 & 0.547 & 0.529 & 0.260 & 0.502 \\
Fisher & 0.446 & 0.686 & 0.657 & 0.597 & 0.318 & 0.541 \\
Linear & 0.459 & 0.718 & 0.650 & 0.612 & 0.318 & 0.551 \\
TIES & 0.475 & 0.748 & \cellcolor{HighlightCell}\bfseries 0.654 & 0.629 & \cellcolor{HighlightCell}\bfseries 0.318 & 0.565 \\
OTA (w Linear) & 0.458 & 0.771 & 0.650 & 0.601 & 0.318 & 0.560 \\
OTA (wo FFG) & 0.458 & 0.660 & 0.654 & 0.590 & 0.318 & 0.536 \\
OTA (rank1) & 0.494 & \cellcolor{HighlightCell}\bfseries 0.787 & 0.646 & 0.614 & 0.315 & 0.571 \\
OTA & \cellcolor{HighlightCell}\bfseries 0.504 & 0.783 & 0.645 & \cellcolor{HighlightCell}\bfseries 0.664 & 0.315 & \cellcolor{HighlightCell}\bfseries 0.582 \\
\midrule
Tulu3-SFT & 0.528 & 0.835 & 0.650 & 0.715 & 0.295 & 0.605 \\
\bottomrule
\end{tabular}
\end{table}

\subsection{Deep Dive: Analysis of the FFG Stage}
\label{sec:ffg-analysis}
\paragraph{FFG Consistently Outperforms Magnitude Pruning.}
To validate the FFG mask selection mechanism, we compare it directly against magnitude pruning across a range of density ratios (Figure~\ref{fig:ffg_sparsity_performance}). We apply FFG and magnitude pruning on task vectors of math, code, and precise-if SFT models, and evaluate each expert on its corresponding benchmark. FFG consistently matches or outperforms magnitude pruning, with the largest gains observed in high-sparsity regimes (1--10\% density). For instance, on IFEval, FFG yields a +0.10 to +0.16 absolute improvement at 1--5\% density. On the Code benchmark (HumanEval), FFG at 20\% density (0.834) even surpasses the dense SFT model (0.788), suggesting that FFG has a regularizing effect by removing noisy, low-saliency updates and thereby improving generalization. A similar pattern is observed for the math SFT model on the MATH benchmark, where at 40\% density FFG achieves 32.52\%, compared to the full math SFT performance of 31.6\%.

The ability of FFG to compress task vectors to much higher sparsity levels while still maintaining, or even improving, fine-tuning performance further motivates an analysis of its underlying subnetwork selection mechanism. Hence, in the subsequent section, we provide a comprehensive empirical study to better understand this mechanism.

\begin{figure}[ht]
    \centering
    \begin{subfigure}[b]{0.4\textwidth}
        \centering
        \includegraphics[width=\textwidth]{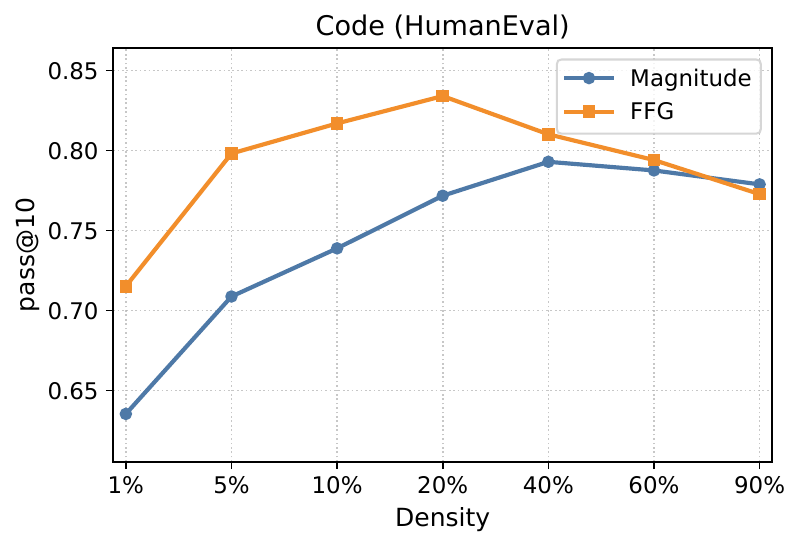}
    \end{subfigure}
    \quad  % Small horizontal space between figures
    \begin{subfigure}[b]{0.4\textwidth}
        \centering
        \includegraphics[width=\textwidth]{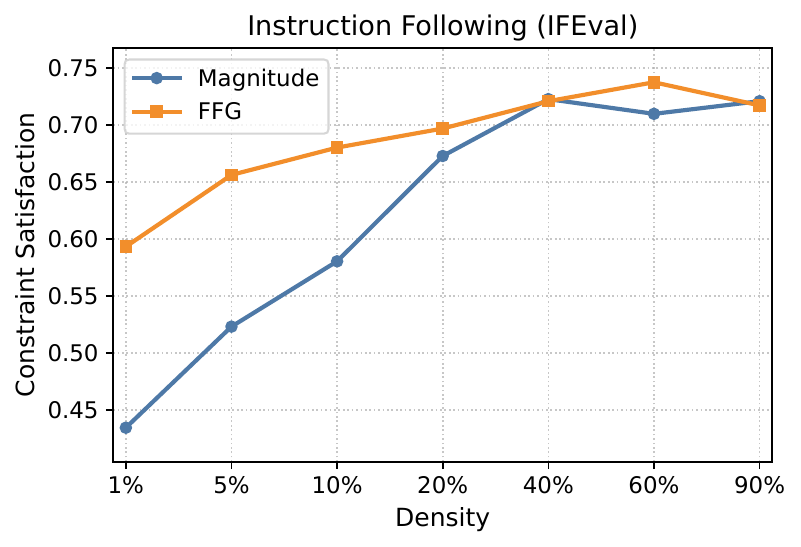}
    \end{subfigure}

    \begin{subfigure}[b]{0.4\textwidth}
        \centering
        \includegraphics[width=\textwidth]{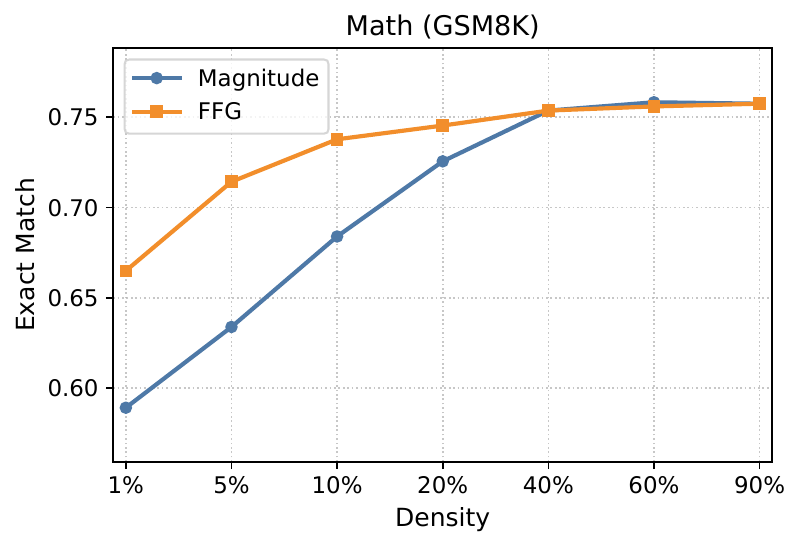}
    \end{subfigure}
    \quad  % Small horizontal space
    \begin{subfigure}[b]{0.4\textwidth}
        \centering
        \includegraphics[width=\textwidth]{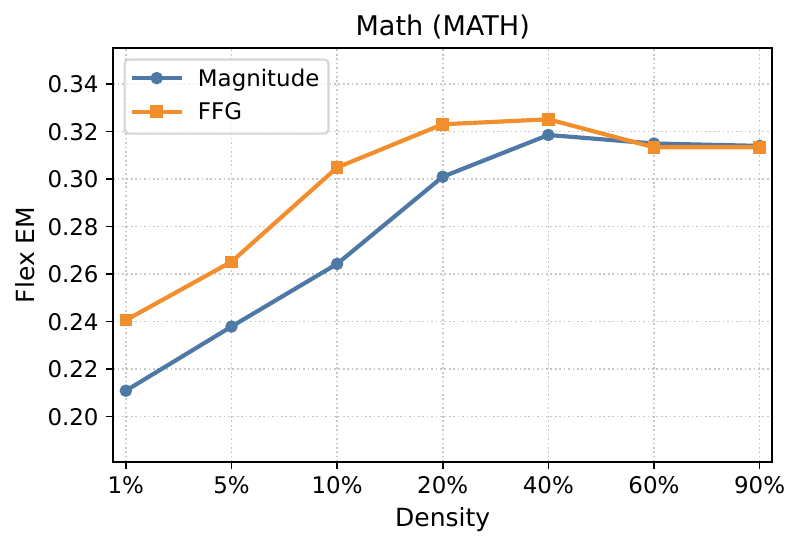}
    \end{subfigure}
    \caption{FFG vs. magnitude pruning across varying density ratios. FFG consistently outperforms, especially at lower densities (1-10\%), highlighting its superior ability to identify salient parameters.}
    \label{fig:ffg_sparsity_performance}
\end{figure}

\begin{figure}[!ht]
    \centering
    \begin{subfigure}[b]{0.49\textwidth}
        \includegraphics[width=\textwidth]{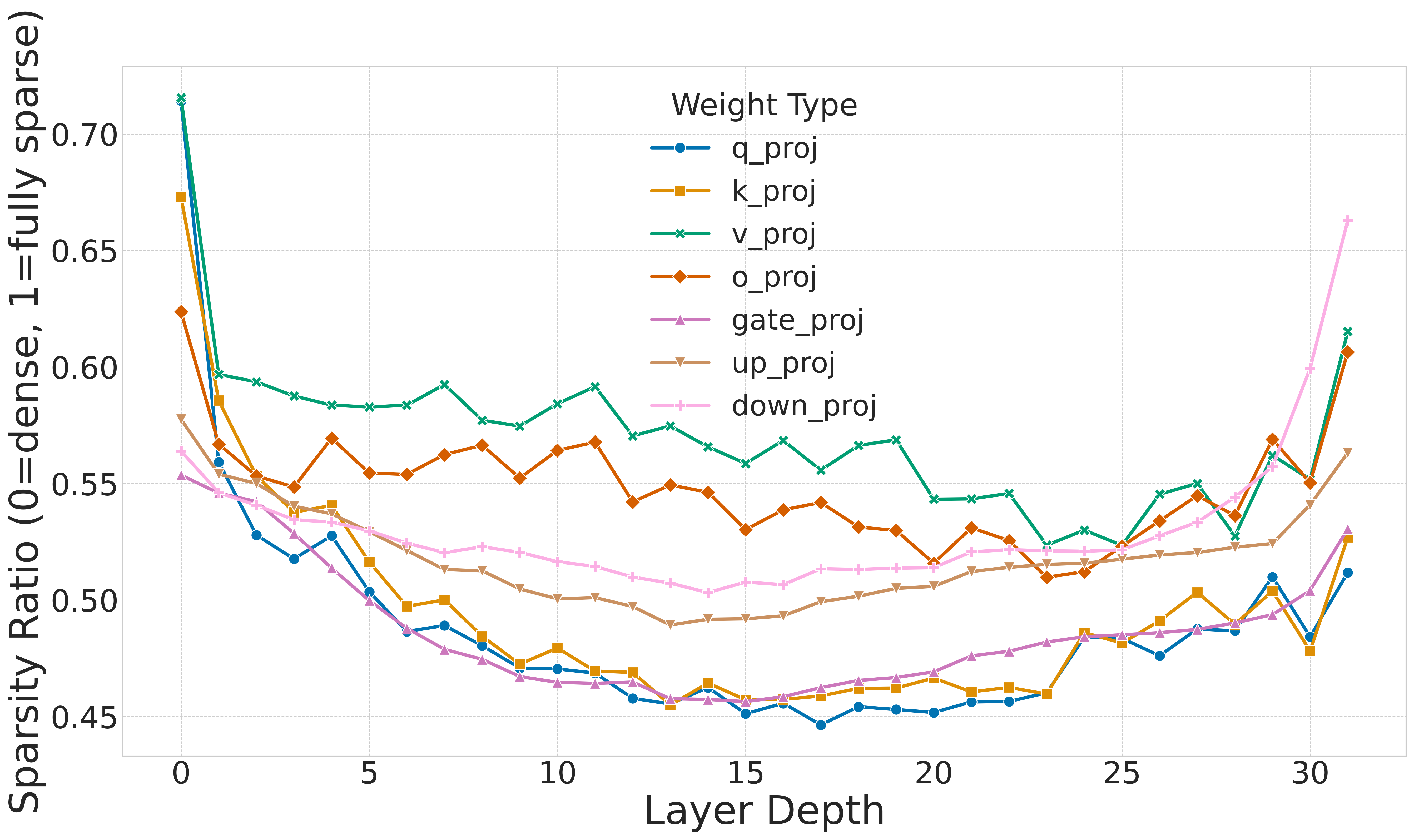}
        \caption{Magnitude Pruning (Math Expert)}
    \end{subfigure}
    \hfill
    \begin{subfigure}[b]{0.49\textwidth}
        \includegraphics[width=\textwidth]{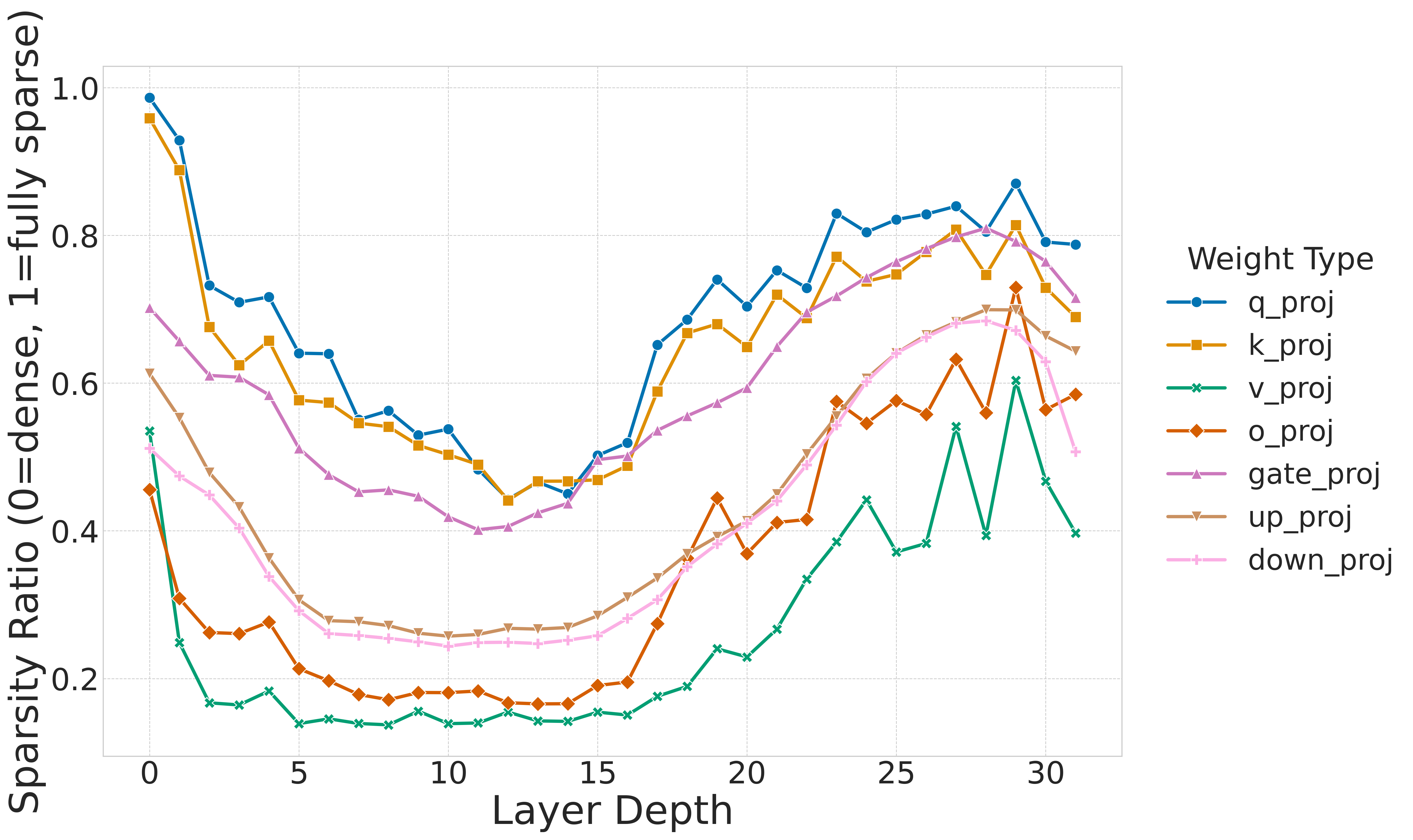}
        \caption{FFG Pruning (Math Expert)}
    \end{subfigure}
    \begin{subfigure}[b]{0.49\textwidth}
        \includegraphics[width=\textwidth]{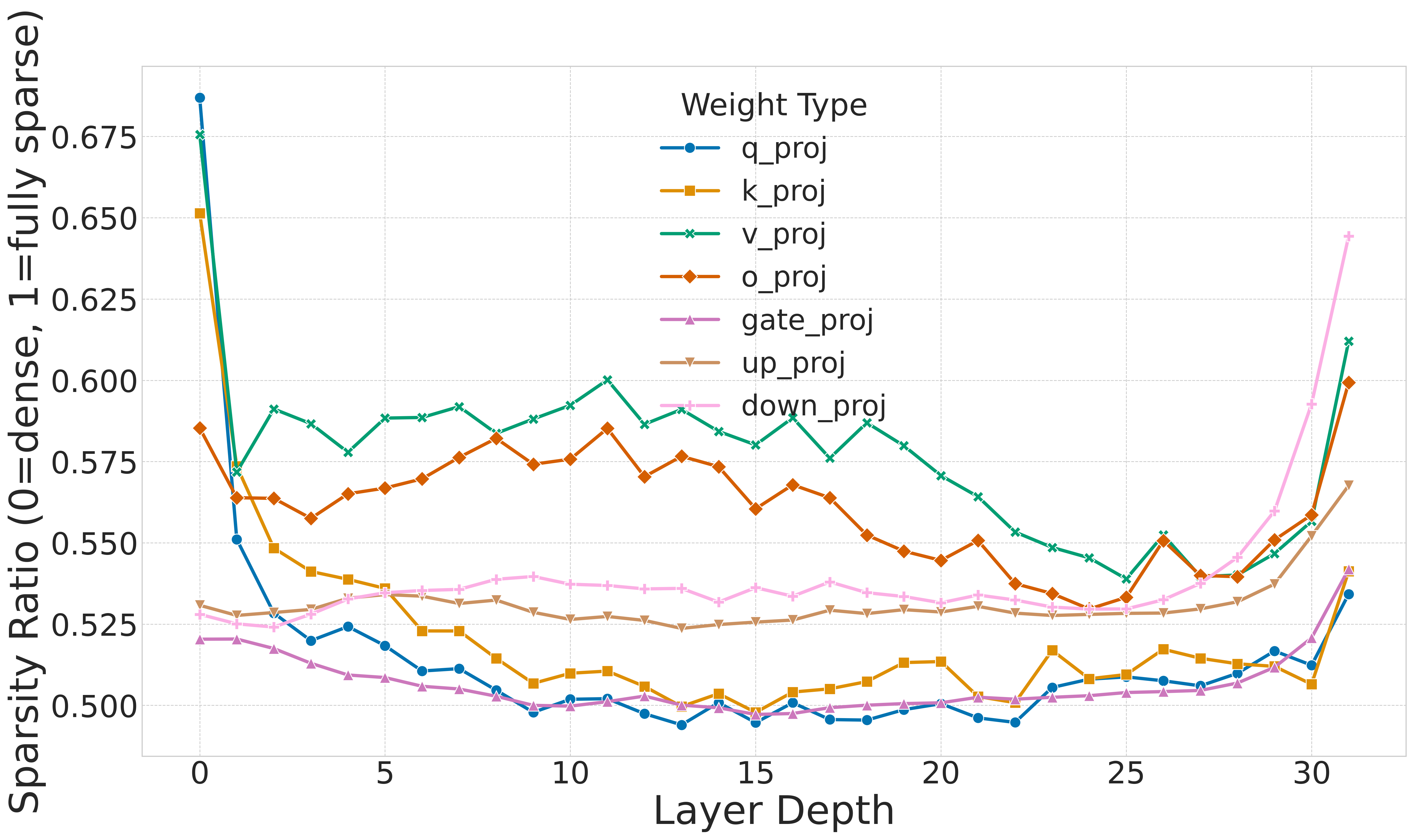}
        \caption{Magnitude Pruning (Code Expert)}
    \end{subfigure}
    \hfill
    \begin{subfigure}[b]{0.49\textwidth}
        \includegraphics[width=\textwidth]{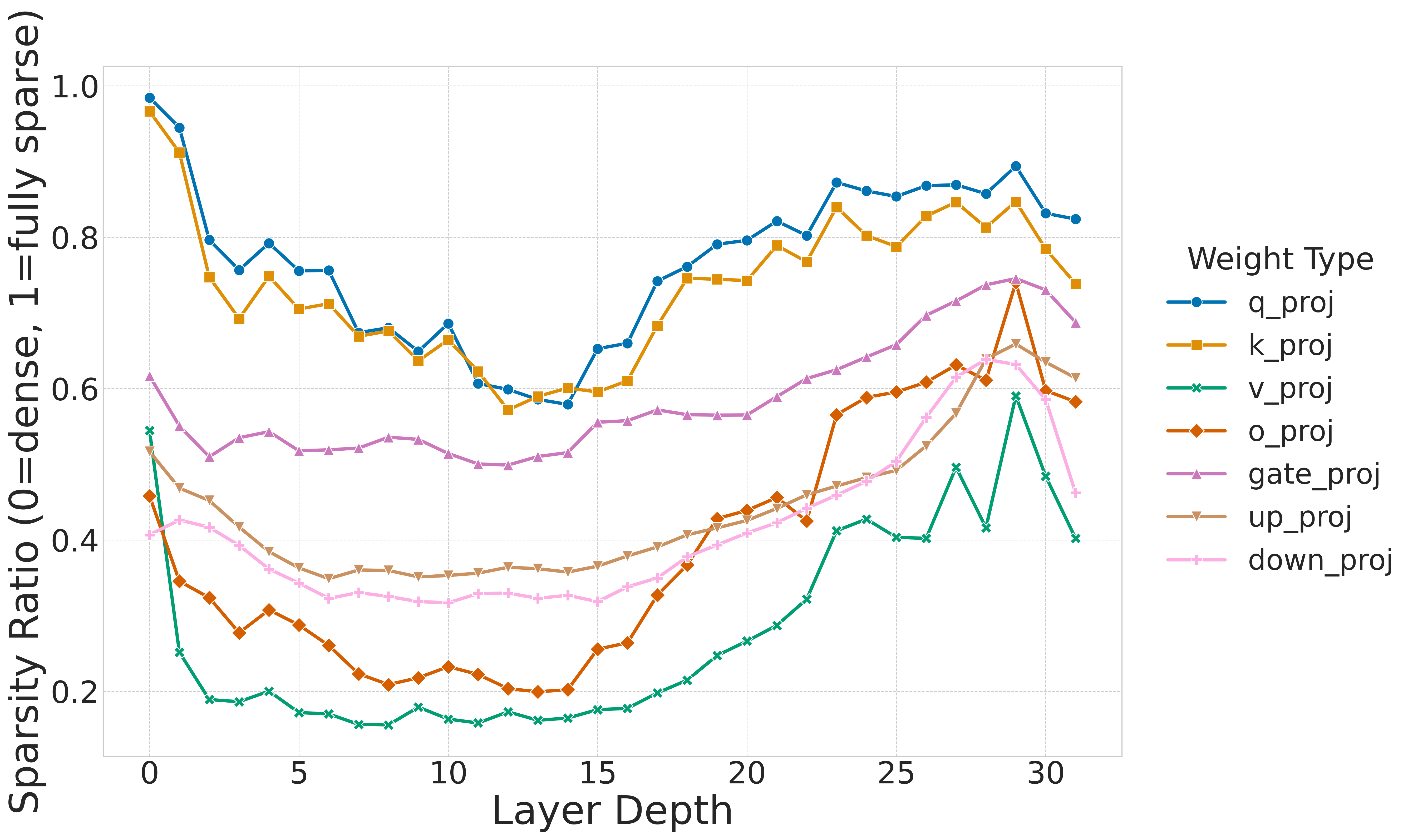}
        \caption{FFG Pruning (Code Expert)}
    \end{subfigure}
    \caption{\textbf{Layer-wise density distribution at a global 40\% task vector pruning density.} FFG (right) exhibits strong, emergent role-aware pruning, aggressively sparsifying query/key weights while preserving value/output/FFN weights. Magnitude pruning (left) is far more uniform and less structured.}
    \label{fig:local_sparsity_distributions}
\end{figure}
\subsubsection{Analyzing the Underlying Mechanism of FFG}
To investigate the distinct mechanisms of magnitude pruning and FFG, we set a global density ratio of $40\%$ to generate sub-network masks for SFT models trained on math, code, and \texttt{precise-if} tasks. Notably, both methods maintain their fine-tuning performance at this budget. Our analysis proceeds in two stages. First, we analyze the resulting density patterns across individual transformer layers. Since both methods rank parameters globally, this fixed global ratio induces different layer-wise density ratios. Second, to further understand these patterns, we visualized the sub-networks selected by each method by weight type and layer depth.

Our first analysis is summarized in Figure~\ref{fig:local_sparsity_distributions}, which plots the density distributions for the math and code SFT models for both magnitude pruning and FFG across all layers and different weight types (e.g., query, key, and value projections). A distinct U-shaped sparsity pattern emerges for each transformer weight type for both FFG and magnitude pruning across layer depths, where the early and late layers are more aggressively pruned, while the middle layers remain the densest. Query and key projections share similar sparsity patterns for each method, and the same property is observed for the up- and down-projection weights.

However, FFG shows a much more diverse range of sparsity ratios across both layer depth and weight type. For instance, in the math SFT model, the sparsity induced by FFG ranges from $99\%$ in the first layer's query projection to $18\%$ in the fifth layer's value projection. On the other hand, the sparsity induced by magnitude pruning ranges from $72\%$ in the first layer's value projection to $45\%$ in the layer 15 query weight.

Furthermore, each method shows a distinct pattern of density ranking across weight types. Magnitude pruning shows that the value and output projections contain the least significant magnitude updates in the task vector (resulting in higher sparsity), while the most significant updates occur in the query and key layers (resulting in lower sparsity compared to other weight types). \textbf{Surprisingly}, FFG aggressively reverts this trend. FFG induces much higher sparsity in the query and key weight types, with even more aggressive pruning towards the initial and late layers when compared to magnitude pruning (e.g., $99\%$ sparsity for the query and key weights of the first layer in FFG versus $70\%$ for these layers in magnitude pruning). On the other hand, the value and output projections are pushed to be much denser, especially in the middle layers (e.g., $18\%$ sparsity for the value projection at layer 5 for FFG, versus $58\%$ for magnitude pruning). The same pattern holds for the down- and up-projection weight types, where FFG exhibits a much more aggressive U-shape compared to magnitude pruning, and the middle layers are pushed to be denser.

Motivated by the aggressive sparsification of FFG on the early and late query/key layers, we conducted an additional experiment to further study the structure of the induced sparsity. We computed the FFG and magnitude pruning density for each row and column of each weight type and visualized the results for the math SFT model by plotting a histogram of the row-wise and column-wise sparsity. Figure~\ref{fig:ffg-q-hist} in the appendix shows the corresponding histograms, for which we chose the query layers 0, 1, 29, and 30, as they maintain the highest sparsity compared to other layer depths and weight types. We observed that FFG exhibits highly structured sparsity, where over $85\%$ of the query rows (corresponding to output features) for all four query layers mentioned are set entirely to zero. A \textbf{low-rank} property is the immediate consequence of this observation. On the other hand, such a property is not observed for magnitude pruning (see Figure~\ref{fig:mag-q-hist}), and we, therefore, conjecture that induced low-rank sparsity is an implicit property of FFG.

\paragraph{FFG shows an implicit layer-wise and role-wise grafting mechanism}
Overall, our analysis of the density distribution patterns reveals novel insights into the task localization of FFG across weight types and layer depths and strongly supports its implicit layer-wise density allocation. FFG aggressively sparsifies the early and late query/key layers, even at a moderate global density ratio (e.g., $40\%$), and imposes a low-rank structure on their task vectors. On the other hand, it allocates most of the global density to the value and projection weight types in the middle layers (approximately twice the density allocated to these weight types compared to magnitude pruning). This implicit density allocation mechanism aligns well with our understanding of SFT training paradigms, where the query and key layers of task vectors were shown to be extremely low-rank, as presented in the seminal work LoRA~\citep{hu2022lora} and further studied theoretically in~\cite{tarzanagh2023transformers}.

\subsection{SFT Task Localization through FFG lens}
\label{sec:localize}
Grounded on FFG's implicit layer-wise and role-wise grafting, and its significantly superior performance for task vector pruning compared to magnitude pruning, we leverage FFG as a lens for SFT task localization within the model. This is achieved through a comprehensive analysis and comparison of FFG masks/sub-networks, with visualizations within each layer and weight type of transformers across different capability SFT models (math, code, and precise-if), as shown in Figure~\ref{fig:3way_ffg_attention_layers}. The FFG global density ratio fixed at $40\%$ across experts. These visualizations illustrate mask localization, where each element's color represents its status across the three expert tasks: dark for elements pruned from all three, white for elements selected by all three, and other colors for elements shared by a subset of the experts. For clarity, we downsampled each weight matrix (e.g., from $4096 \times 4096$ to $256 \times 256$) using uniform row and column subsampling with an adaptive stride based on the weight matrix's size. Note that we chose layers 1, 15, and 30 as representative layers and will analyze and discuss only the patterns observed consistently across layer depth for the same weight type. The complete set of our heatmap visualizations and artifacts, categorized by layer depth and weight type, is available in the project's GitHub repository to support our claims and ensure completeness.

Moreover, the legends in Figure~\ref{fig:3way_ffg_attention_layers} provide the computed overlap across different SFT models' masks, which we use as a second signal in our task visualization analysis. Furthermore, it is important to note that for all heatmaps presented in this paper, excluding the embedding layer and language model heads, the columns represent the weights of a single input feature connected to all output features, while the rows represent all the weights connected to a single output feature. Therefore, a colored dense or sparse row in the mask visualization indicates that the weights connected to the corresponding output feature are densely or sparsely utilizing input features. Conversely, a dense colored column signifies that the corresponding input feature is heavily updated during Supervised Fine-Tuning (SFT), while a dark column represents an input feature that is not used at all.  

\paragraph{SFT Updates Use a Shared, Extremely Sparse Subset of Embeddings' Features Across Tasks.}
As shown in Section~\ref{sec:ffg-analysis}, FFG introduces the most sparsity in the first two and last two layers of the transformer. The mask visualizations for these layers are predominantly dark, with strong row- and column-wise patterns, which supports the low-rank structure of the FFG mask in these areas. Most interestingly, the query and key weights in the first layer show extremely high overlap in pruned regions across the three tasks (e.g., a 97.8\% shared pruned region). We observed that all three masks select an extremely sparse and nearly identical set of input features (a column-wise mask pattern) for the query and key matrices of the first layer, with almost no dense output features. This strongly suggests that SFT updates only a very sparse subset of embedding features. This observation is reasonable, as we expect the early layers to be focused on general language understanding, an ability largely acquired during pre-training.

\paragraph{Task-Specific Dominance in Attention vs. FFN Layers.}
A comprehensive analysis of mask overlaps reveals consistent density ranking patterns across different weight types and transformer layer depths for each SFT model. The Math SFT model consistently dominates the query and key attention layers across all layer depths, often having twice the parameter density compared to the Precise-IF SFT model. Similarly, the Code SFT model consistently shows higher density in the query and key layers than the Precise-IF model, although the difference is less pronounced. These patterns are most significant in the middle layers of the network. Conversely, the Precise-IF and Code SFT models dominate the FFN layers (up, down, and gate projections). This dominance over the Math SFT is clear in the first two layers, vanishes until layer 22, and then re-emerges aggressively from layers 23 to 31. In these later layers, the Precise-IF SFT is significantly denser than both the Code and Math SFTs, with the Code SFT being slightly denser than the Math SFT. The value and output projections exhibit a similar ranking, where the Math SFT is slightly denser than the Code SFT, which in turn is slightly denser than the Precise-IF SFT. These observations seem natural. We expect attention layers to play a more critical role in mathematical reasoning, which requires understanding more complicated patterns between tokens. On the other hand, precise instruction-following may not introduce such complex analogical patterns and might instead rely more on the richer feature extraction capabilities of the FFN layers.

\paragraph{Formation of Specialized and General-Purpose Attention Heads in Layers 17-31.}
The value and output projections reveal another interesting property. From layers 1 to 16, we observed maximum dense mask overlap across the expert models, and the heatmaps appear mostly random with no clear visual pattern. However, from layers 17 to 31, two distinct regions emerge within the masks. The first region maintains high overlap across the experts, while the second region contains almost no overlap. This phenomenon is illustrated for the value and output projections of layer 30 in Figure~\ref{fig:3way_ffg_attention_layers}. It is worth noting that this two-region behavior also appears to some extent in layers 1 and 2 before vanishing until it re-emerges at layer 17. Since a set of subsequent columns in the output projection represents the aggregated feature set from a specific attention head, the non-overlapping regions in layers 17-31 provide strong evidence for the formation of task-specialized heads alongside more general-purpose heads. This claim is further supported by the patterns in the query and key layers for this same range (17-31), where two distinct regions also exist: one with almost no overlap and another with extremely high overlap.  

\begin{figure}[!ht]
    \centering
    % Layer 1, 15, 30 from original for brevity
    \begin{subfigure}[b]{0.24\textwidth}\includegraphics[width=\textwidth]{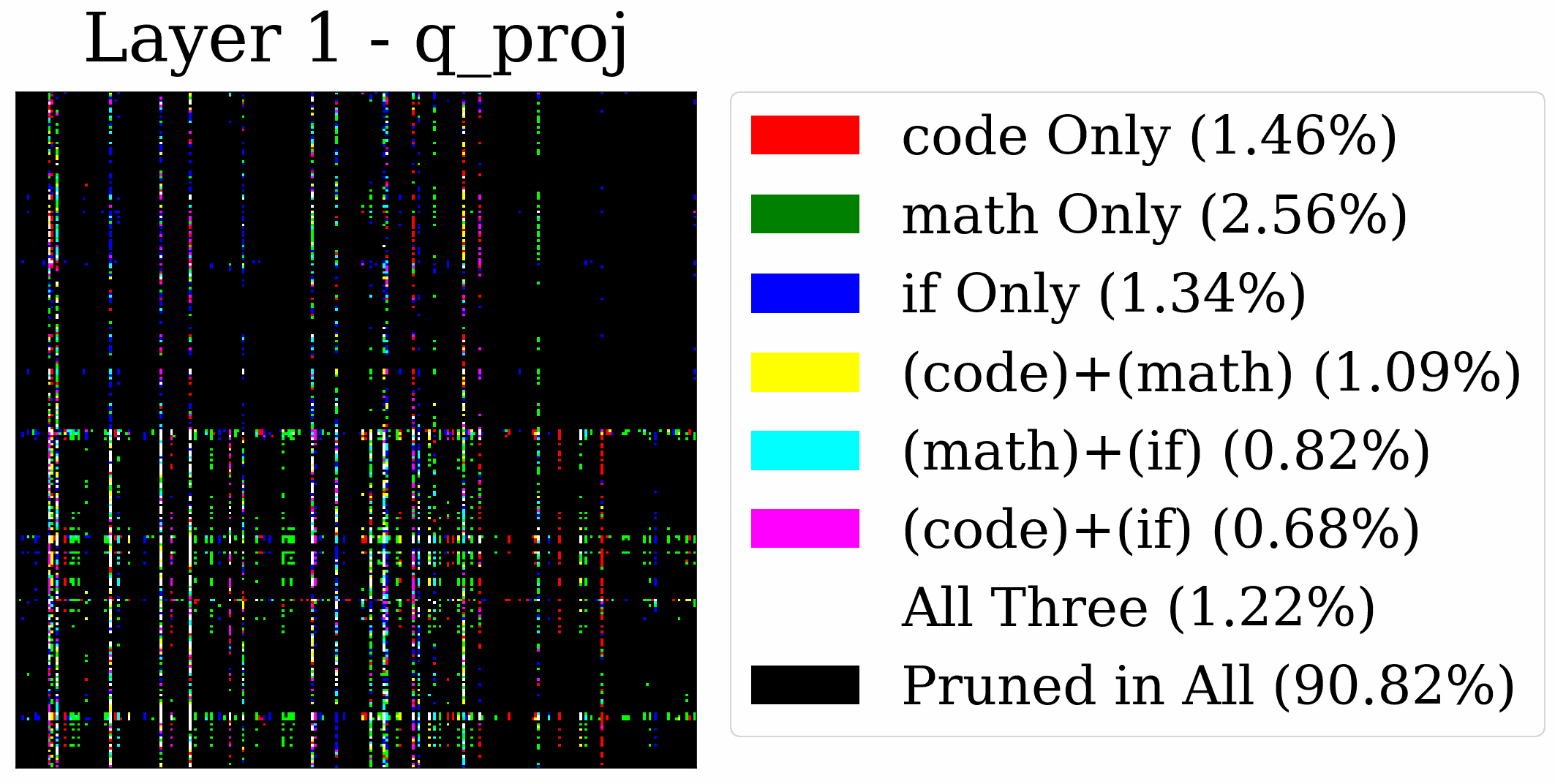}\end{subfigure}\hfill
    \begin{subfigure}[b]{0.24\textwidth}\includegraphics[width=\textwidth]{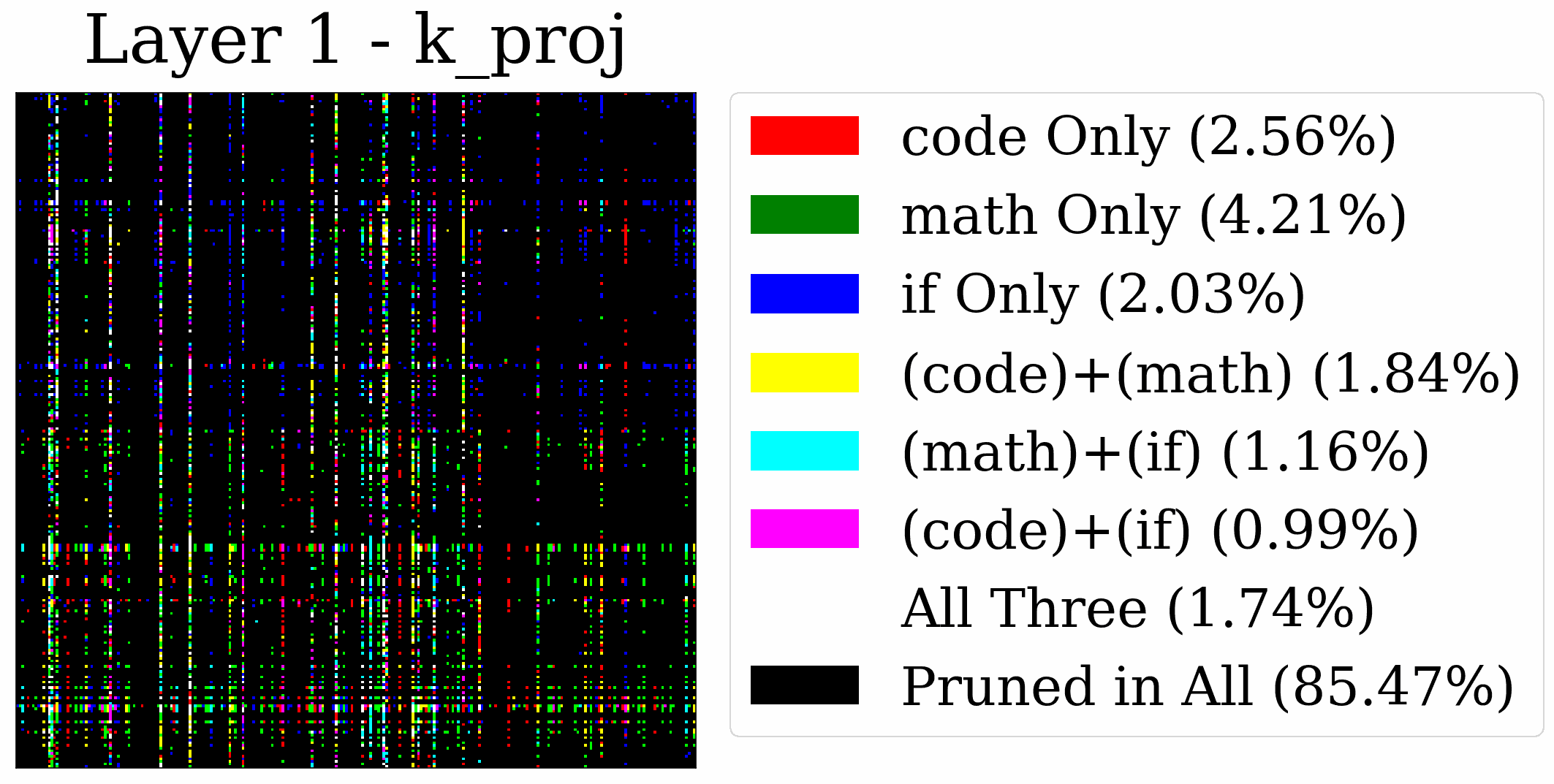}\end{subfigure}\hfill
    \begin{subfigure}[b]{0.24\textwidth}\includegraphics[width=\textwidth]{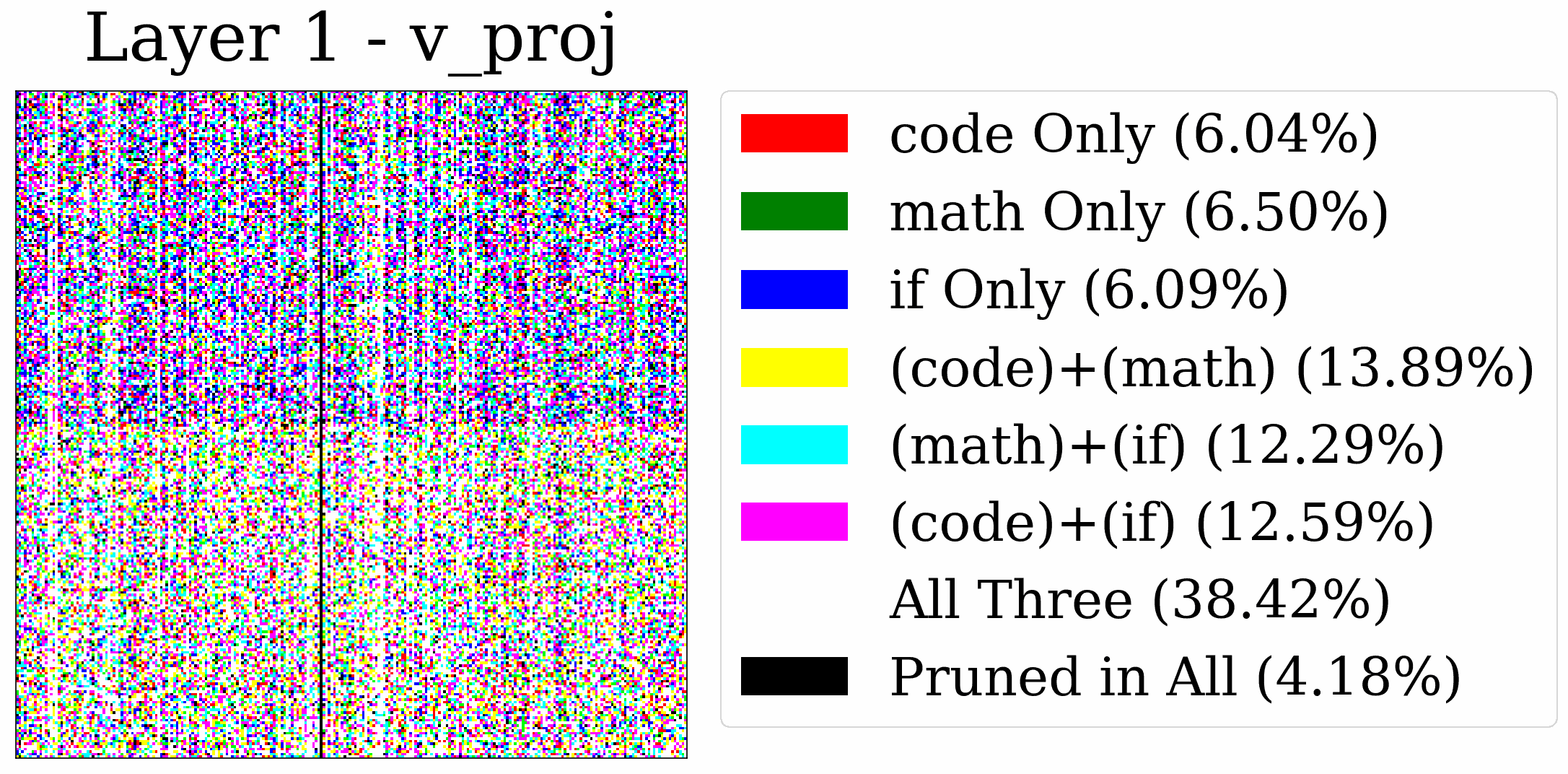}\end{subfigure}\hfill
    \begin{subfigure}[b]{0.24\textwidth}\includegraphics[width=\textwidth]{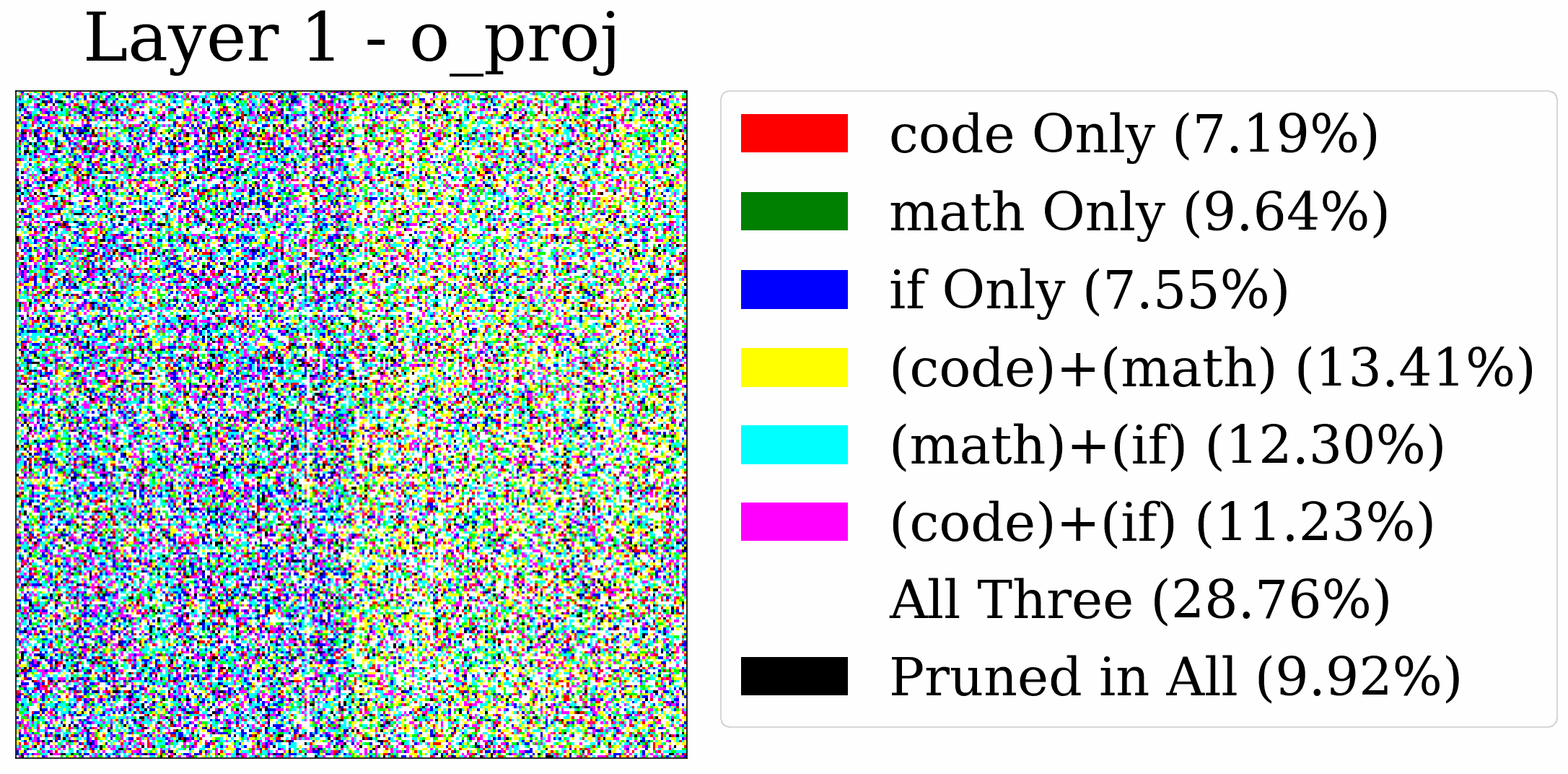}\end{subfigure}
    \begin{subfigure}[b]{0.24\textwidth}\includegraphics[width=\textwidth]{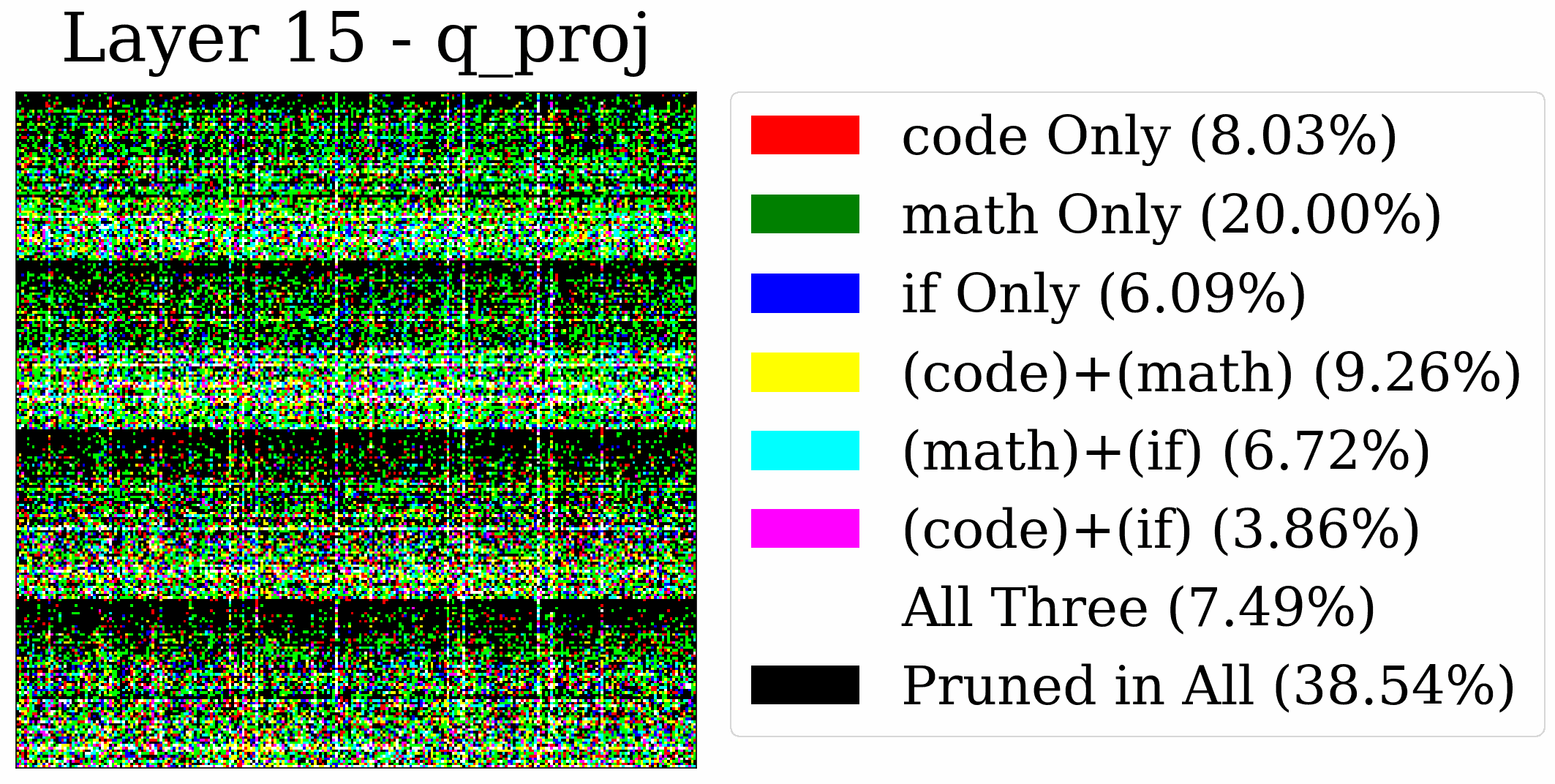}\end{subfigure}\hfill
    \begin{subfigure}[b]{0.24\textwidth}\includegraphics[width=\textwidth]{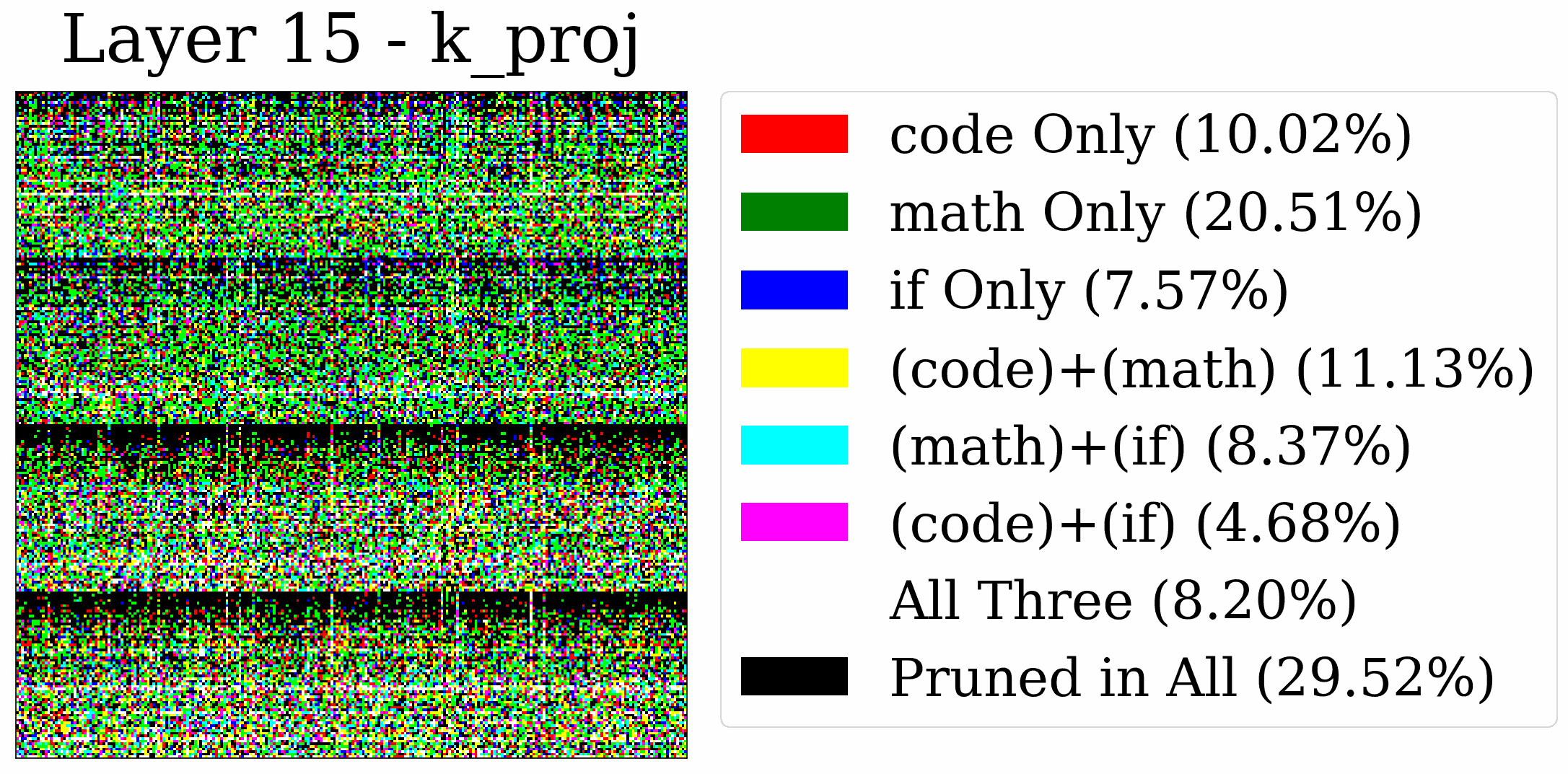}\end{subfigure}\hfill
    \begin{subfigure}[b]{0.24\textwidth}\includegraphics[width=\textwidth]{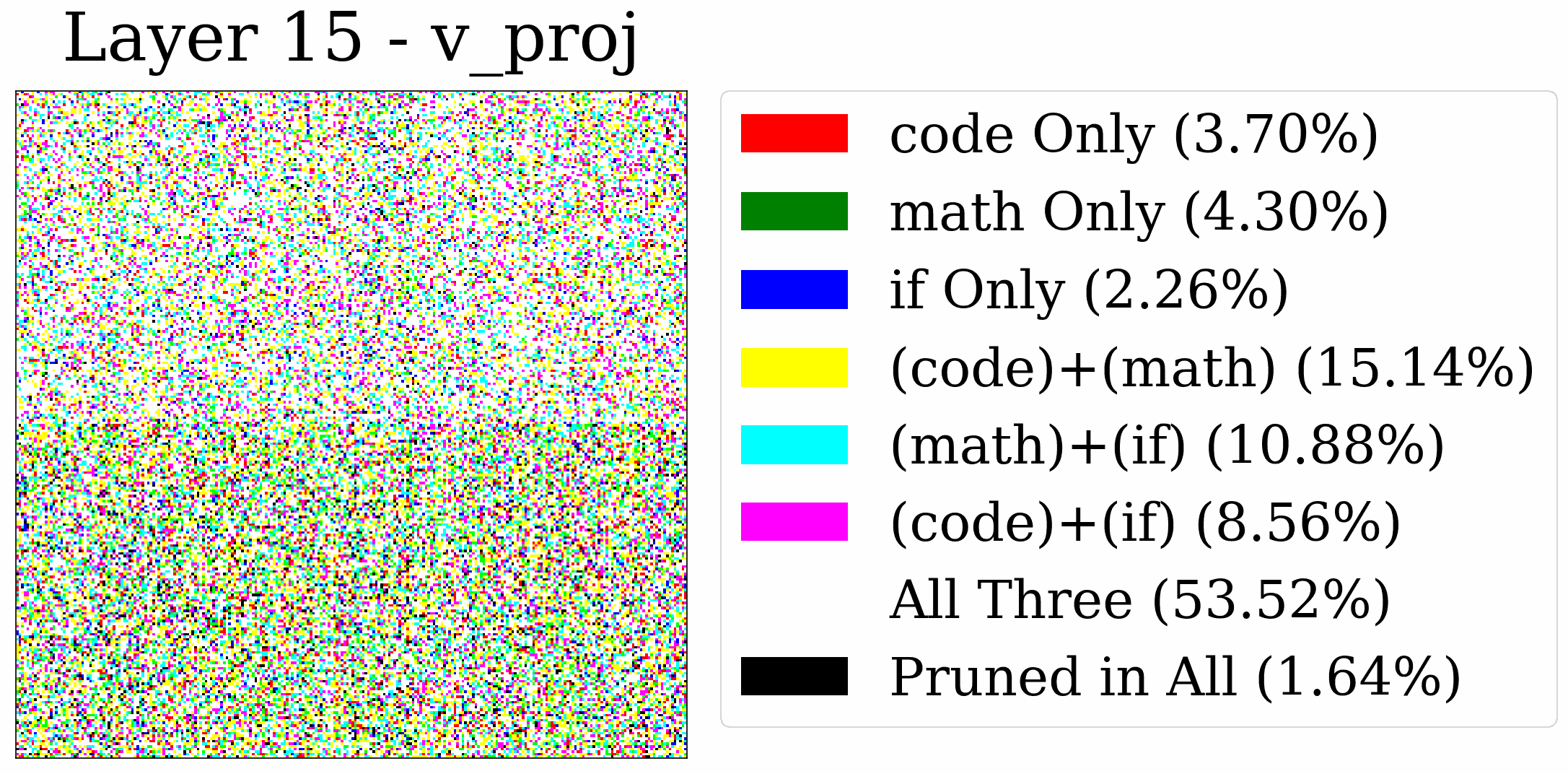}\end{subfigure}\hfill
    \begin{subfigure}[b]{0.24\textwidth}\includegraphics[width=\textwidth]{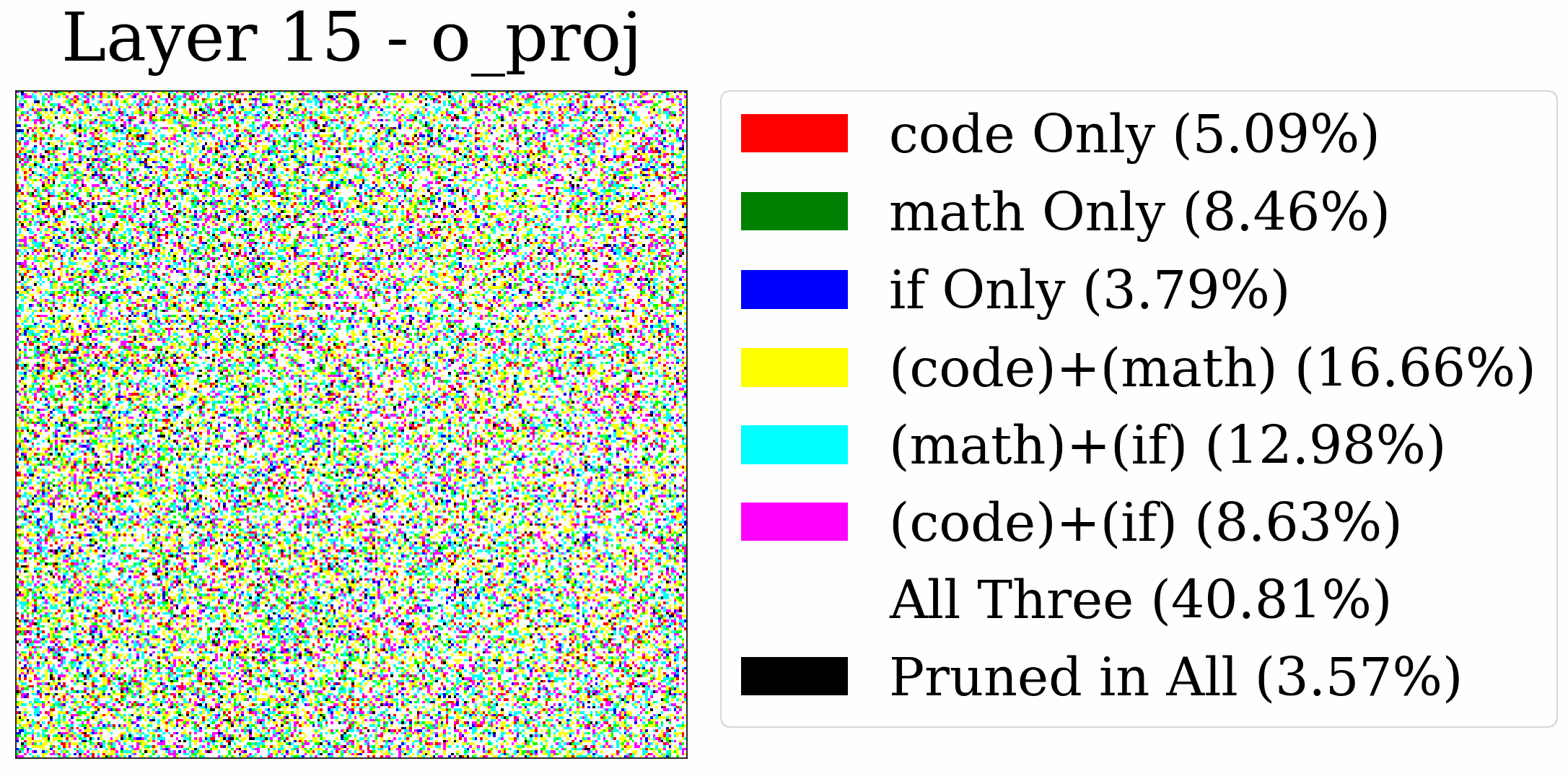}\end{subfigure}
    \begin{subfigure}[b]{0.24\textwidth}\includegraphics[width=\textwidth]{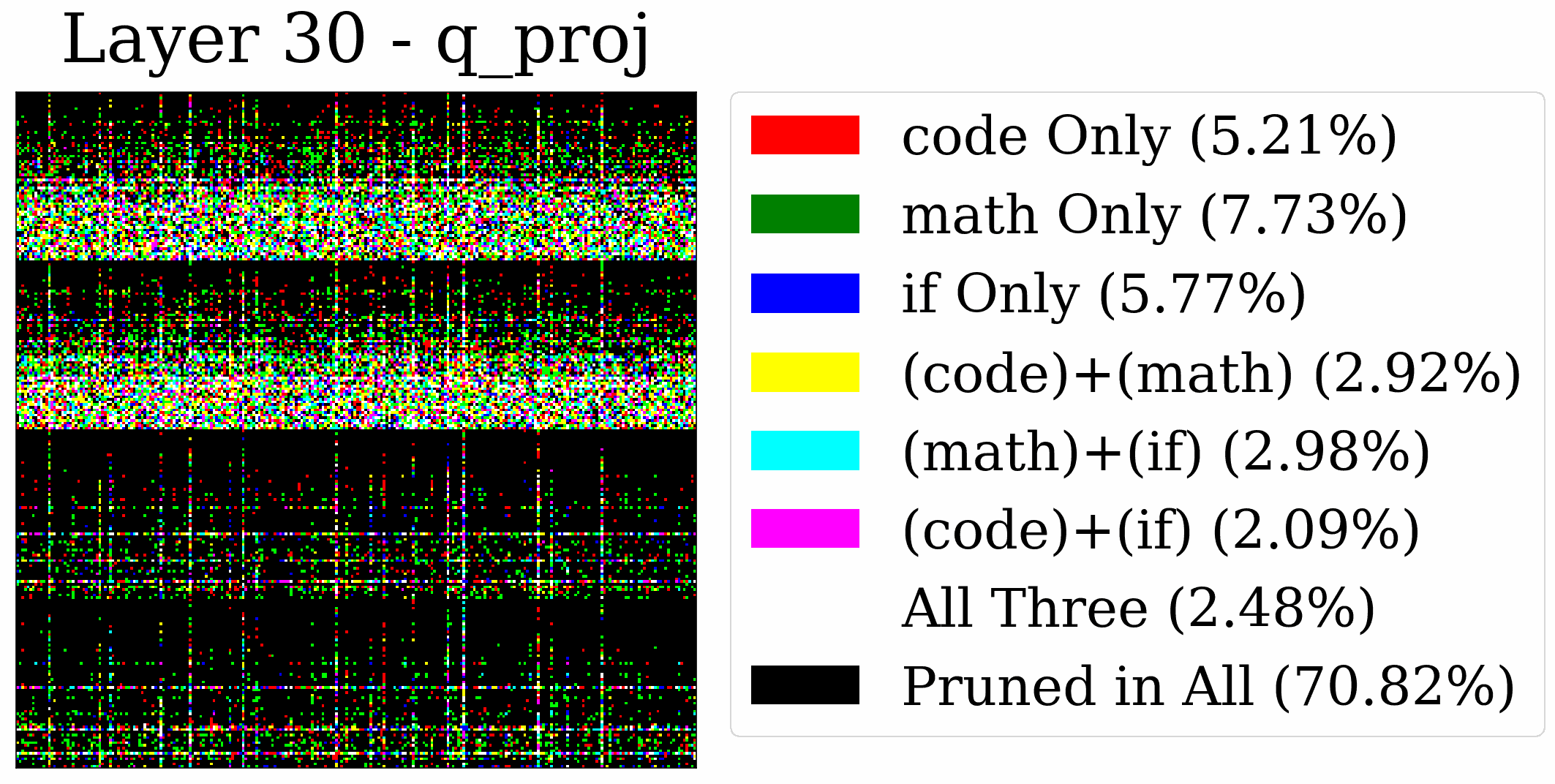}\end{subfigure}\hfill
    \begin{subfigure}[b]{0.24\textwidth}\includegraphics[width=\textwidth]{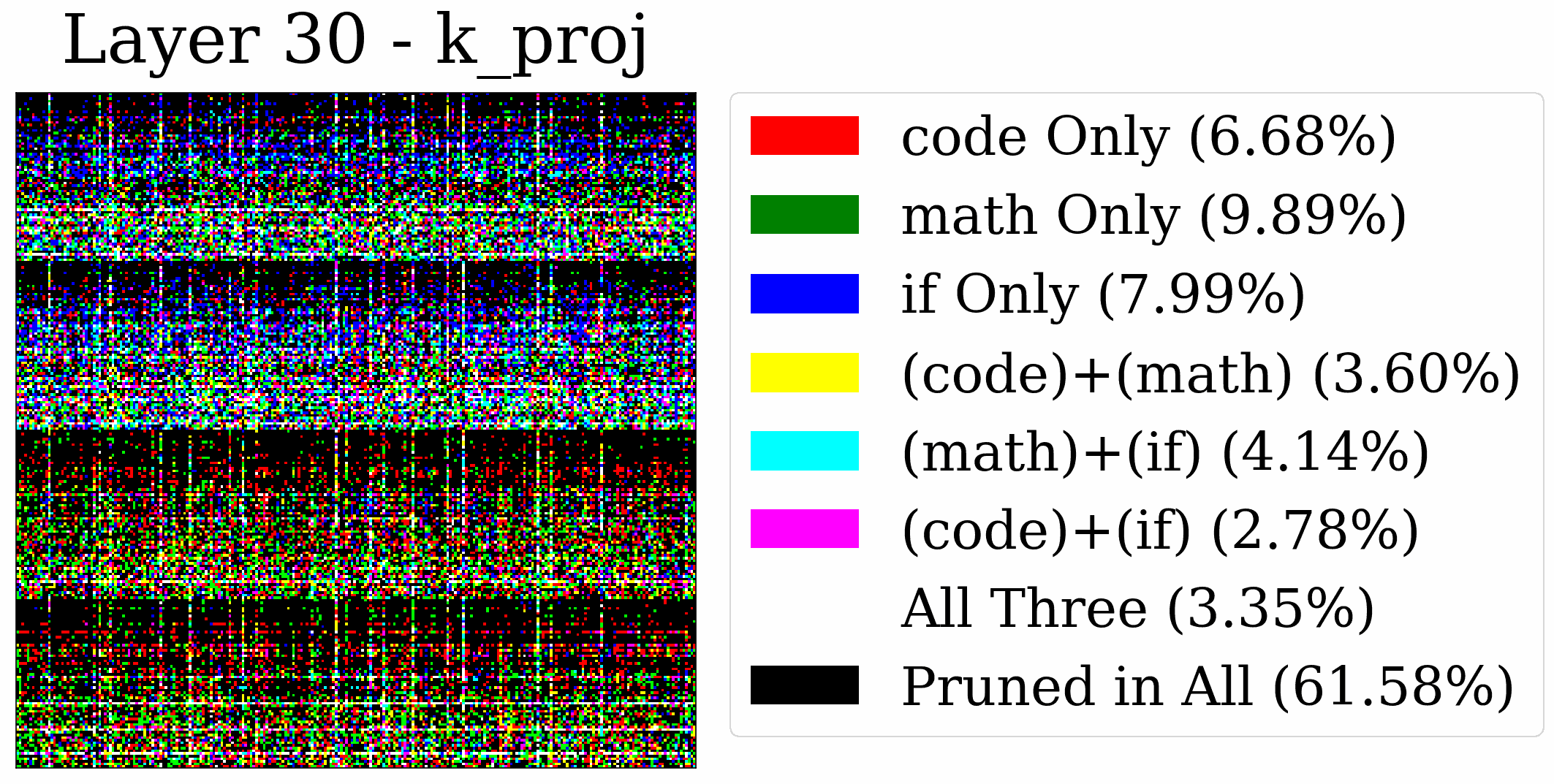}\end{subfigure}\hfill
    \begin{subfigure}[b]{0.24\textwidth}\includegraphics[width=\textwidth]{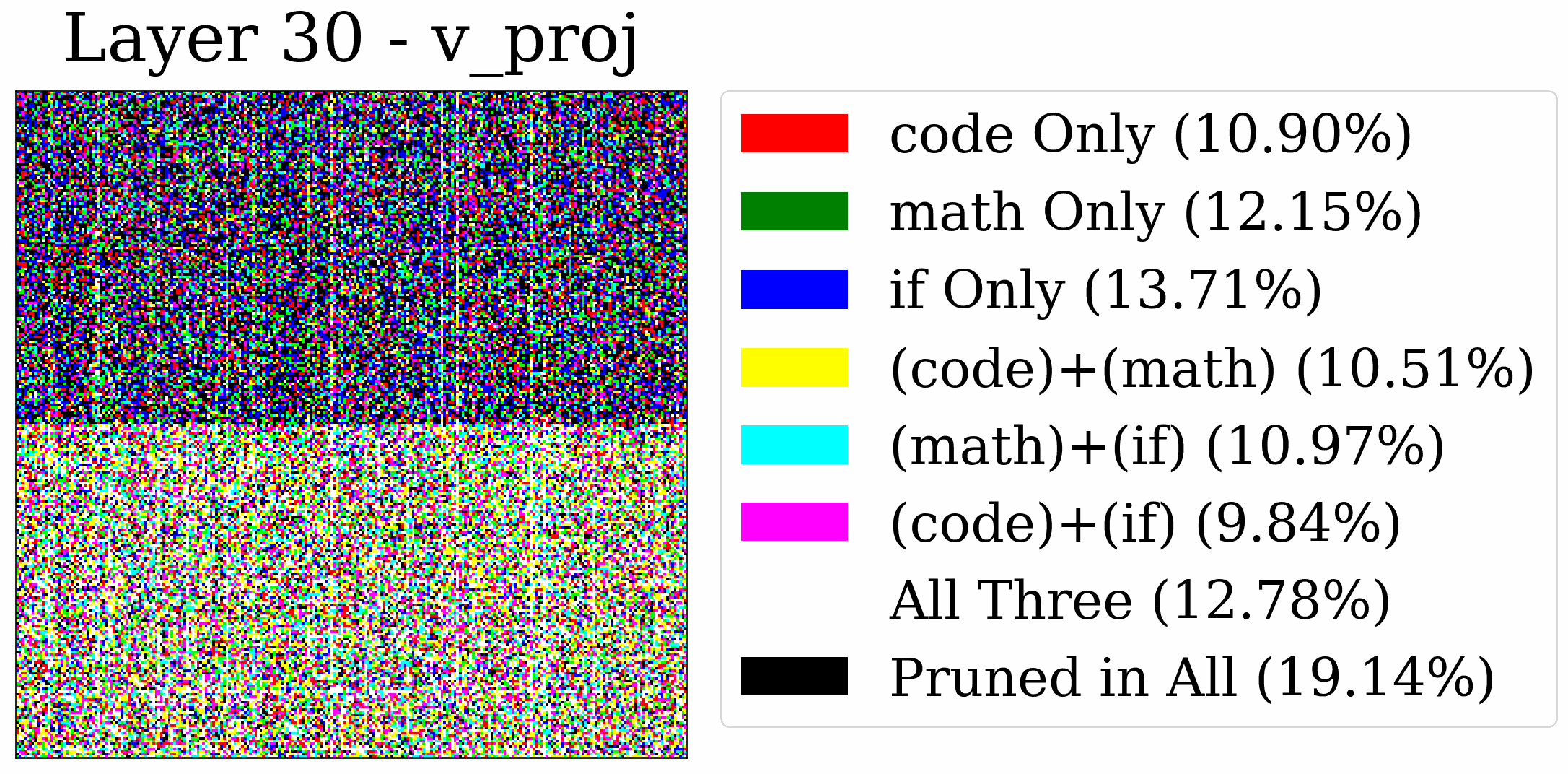}\end{subfigure}\hfill
    \begin{subfigure}[b]{0.24\textwidth}\includegraphics[width=\textwidth]{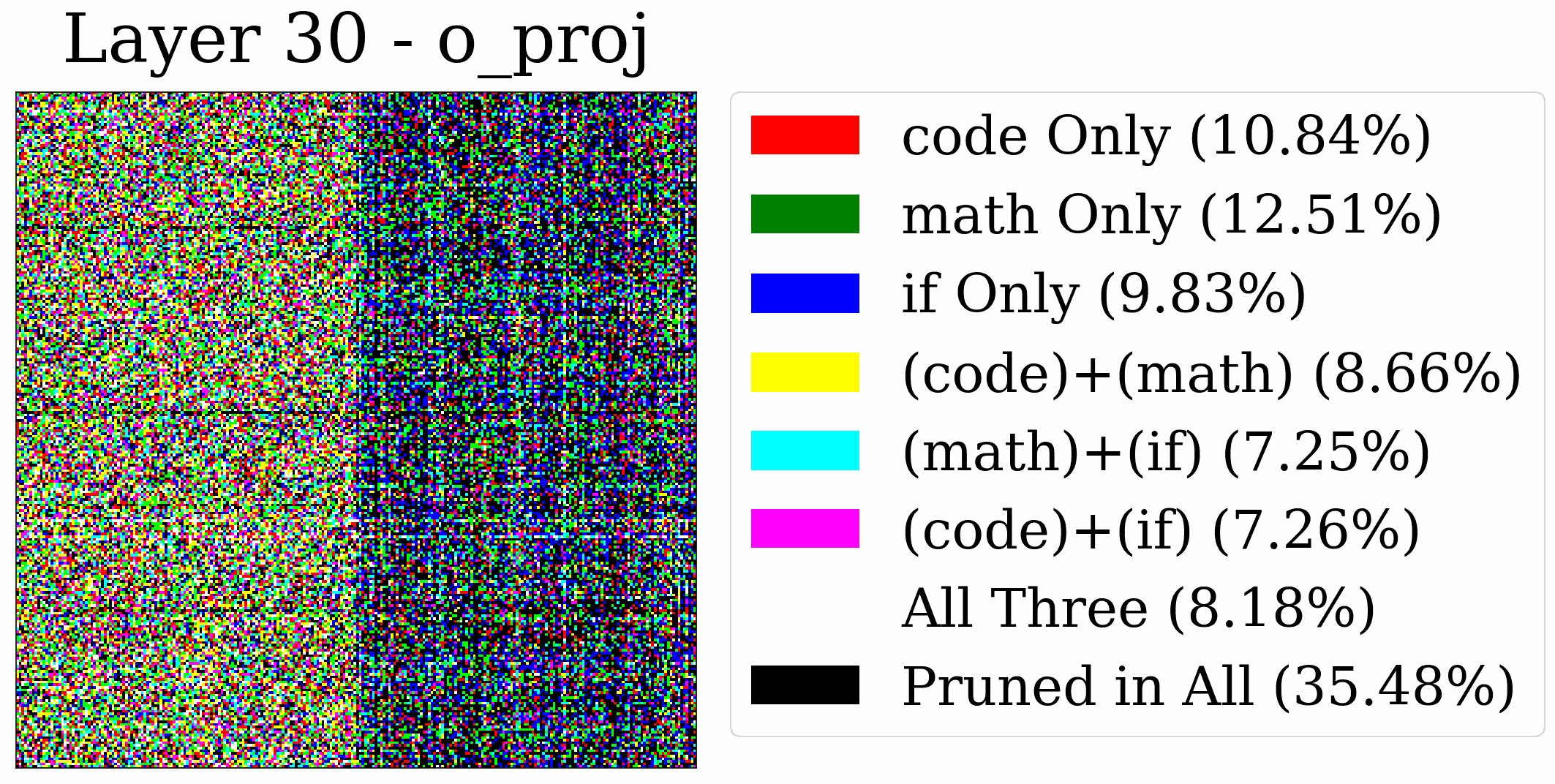}\end{subfigure}
    \caption{3-way FFG comparison for attention components across layers 1, 15, and 30 (top to bottom). Columns show $W_q, W_k, W_v, W_o$. RGB channels represent Code (red), Math (green), and Instruction-Following (blue).}
    \label{fig:3way_ffg_attention_layers}
\end{figure}

\subsection{Analysis of Curvature and Rank Structure}

\paragraph{Second Moments Have Low Stable Rank, Justifying Compression.}
A core assumption of our AdaFactor-style compression is that the second-moment tensors, $\mathbf{v}_\tau$, are inherently low-rank, allowing for efficient compression. We validate this by computing the stable rank of the $\mathbf{v}_\tau$ matrices for the Math and Code experts (Figure~\ref{fig:stable-rank}). Across all layers and for both SFT models, the stable rank is surprisingly low (below $1.3$), confirming that the second-moment matrices are highly compressible and that a rank-1 approximation can capture a significant fraction of their energy. This provides a solid empirical justification for our memory-efficient variant of OTA, which aggressively reduces the required storage for second-moment matrices (from 29.9 GB to 12.6 MB for Llama 3.1 8b under fp32 precision) with a minor drop in model merging performance (from 0.582 to 0.571 average score), as detailed in Table~\ref{tab:benchmark_results}.

\begin{figure}[ht]
  \centering
  \subfloat[Math Expert]{%
    \includegraphics[width=0.48\linewidth]{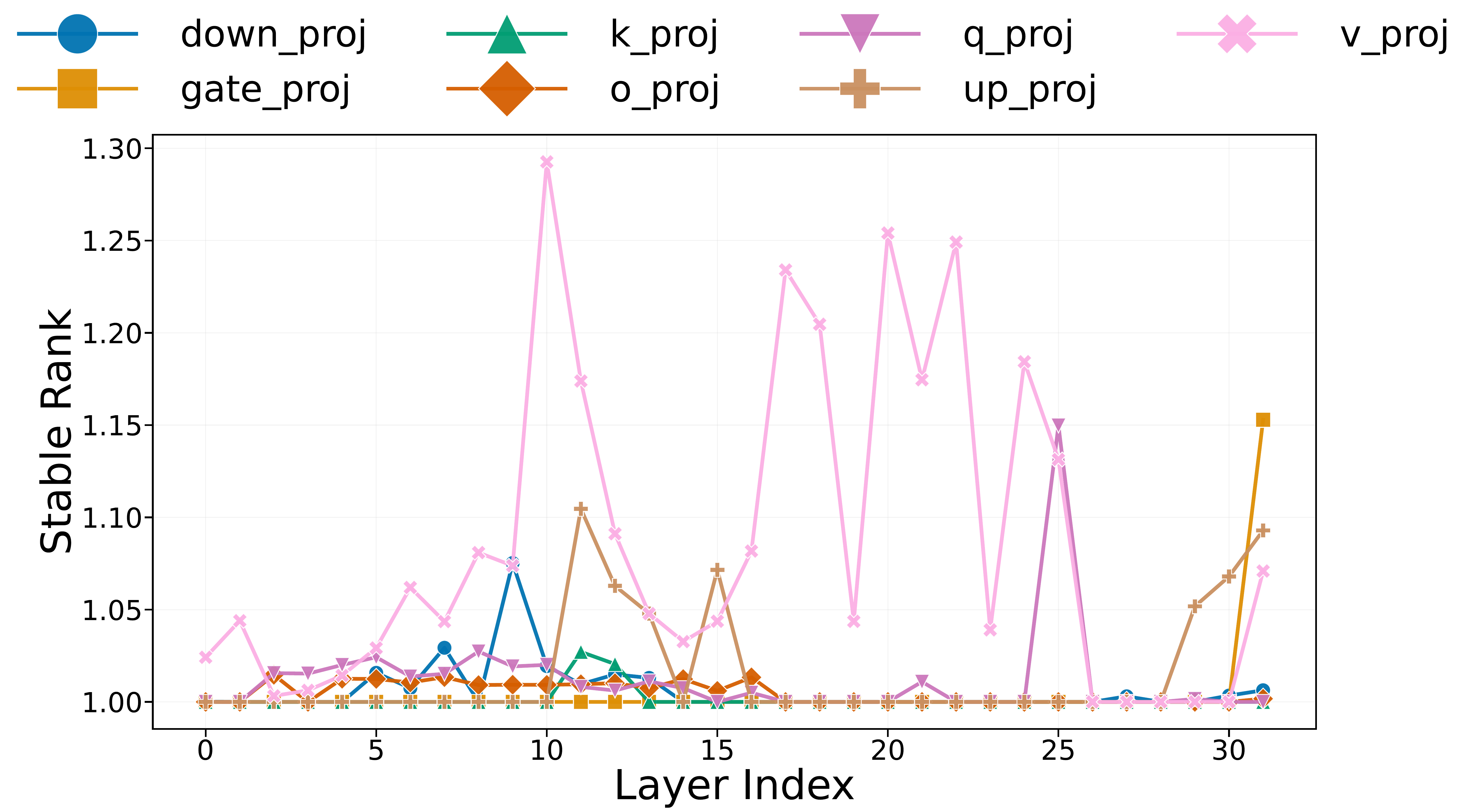}
    \label{fig:rank-left}
  }
  \hfill
  \subfloat[Code Expert]{%
    \includegraphics[width=0.48\linewidth]{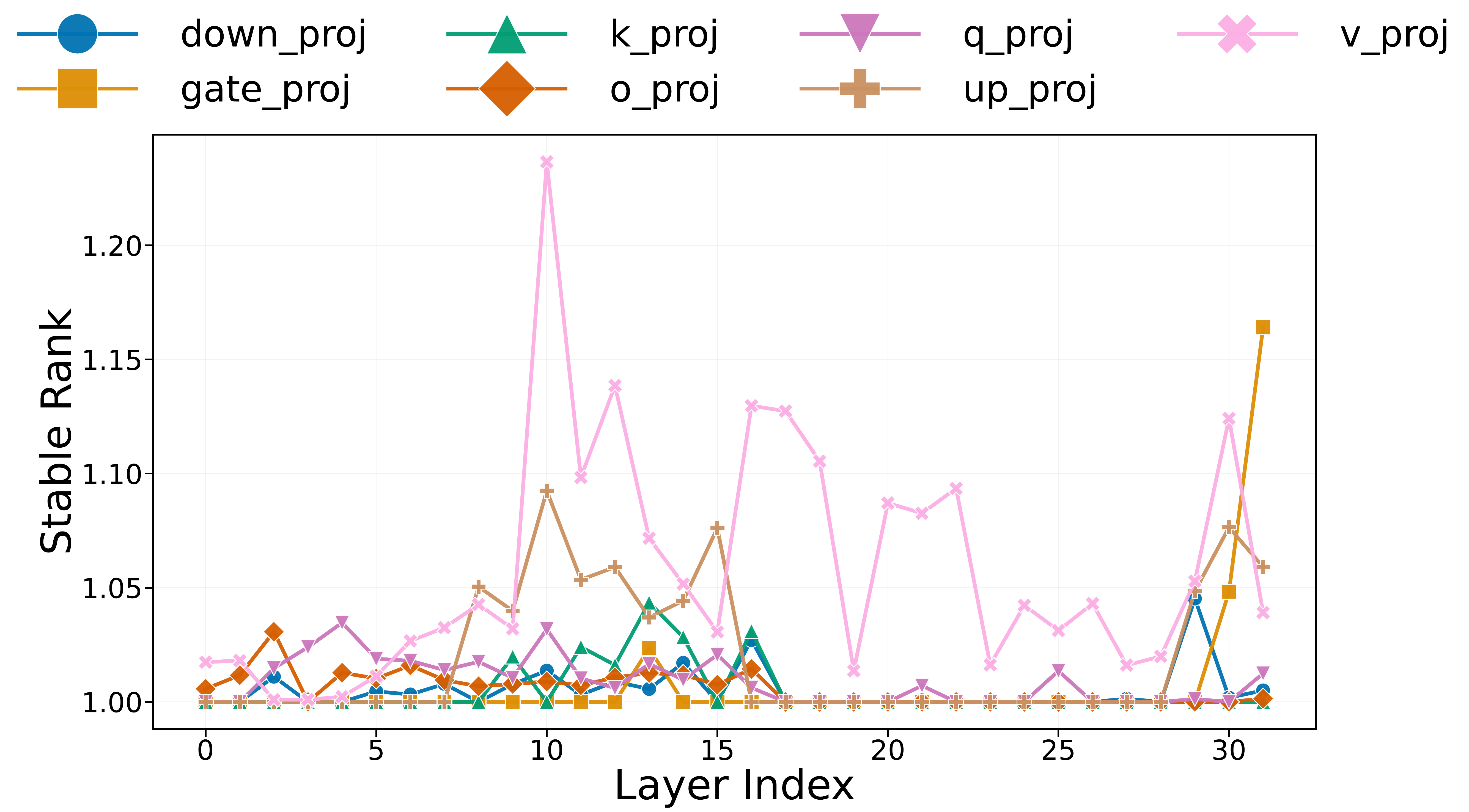}
    \label{fig:rank-right}
  }
  \caption{Stable rank analysis of the second-moment matrices ($\vv_\tau$). The consistently low stable rank across all layers validates our use of AdaFactor for compression.}
  \label{fig:stable-rank}
\end{figure}

\paragraph{Visualizing Shared Curvature Geometry.}
In this section, we leverage the connection between the second moment and diagonal curvature to study how curvature differs across SFT models. We use the same subsampling strategy for the heatmaps that was used to visualize task localization, but we apply it to the second-moment matrices instead of the grafting masks. The curvature heatmap comparisons across experts for each layer depth and weight type provide an empirical visualization suite to study our central conjecture: that SFT models fine-tuned from the same base converge to basins with highly similar curvature geometry.

\begin{figure}[ht]
  \centering
  \newcommand{\subfigwidth}{0.24\textwidth}
  % Row 1: Math model
  \begin{subfigure}[b]{\subfigwidth}\includegraphics[width=\textwidth]{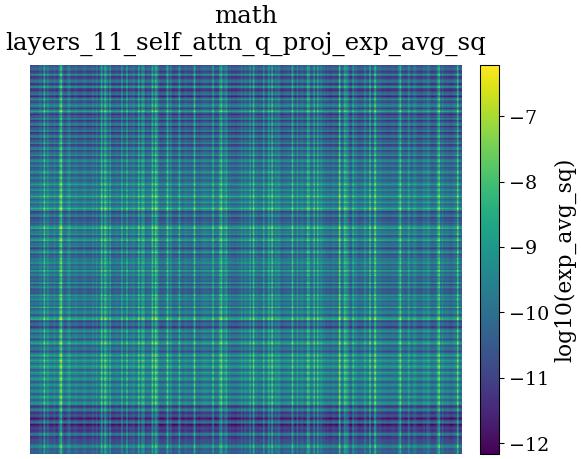}\end{subfigure}\hfill
  \begin{subfigure}[b]{\subfigwidth}\includegraphics[width=\textwidth]{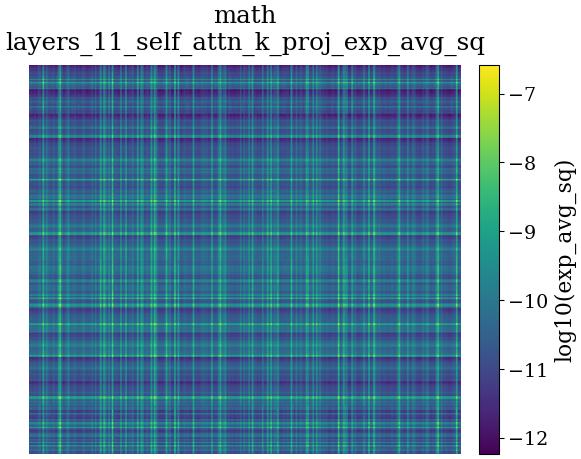}\end{subfigure}\hfill
  \begin{subfigure}[b]{\subfigwidth}\includegraphics[width=\textwidth]{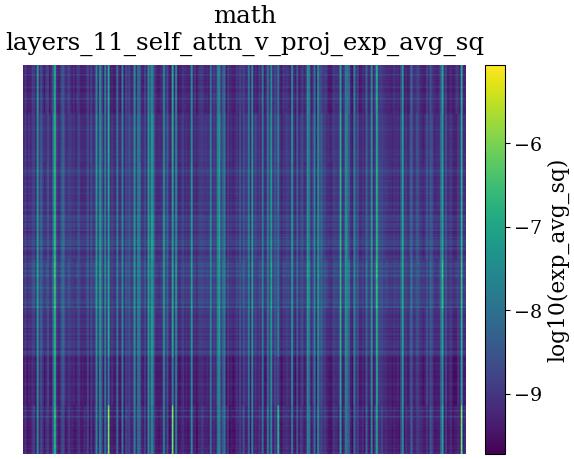}\end{subfigure}\hfill
  \begin{subfigure}[b]{\subfigwidth}\includegraphics[width=\textwidth]{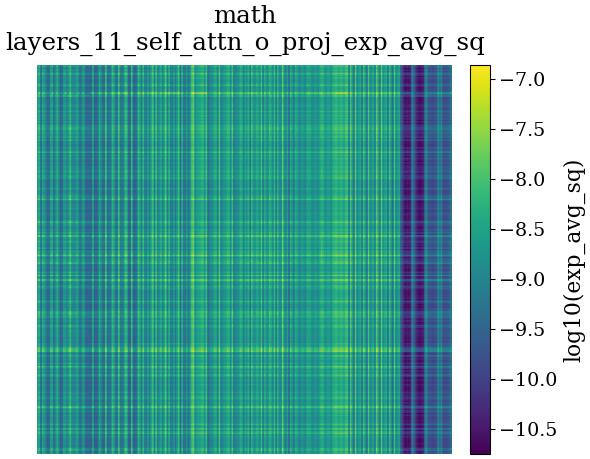}\end{subfigure}
  % Row 2: Code model
  \begin{subfigure}[b]{\subfigwidth}\includegraphics[width=\textwidth]{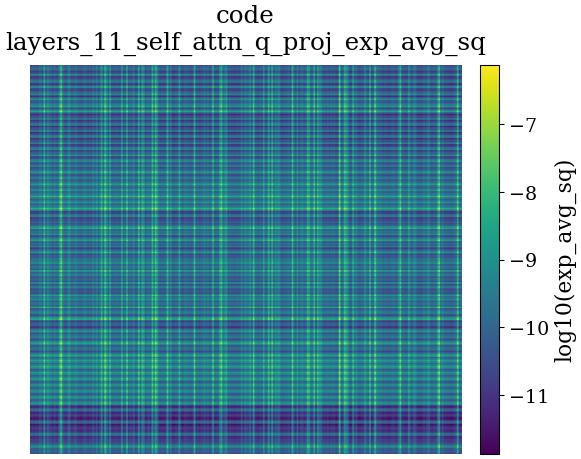}\end{subfigure}\hfill
  \begin{subfigure}[b]{\subfigwidth}\includegraphics[width=\textwidth]{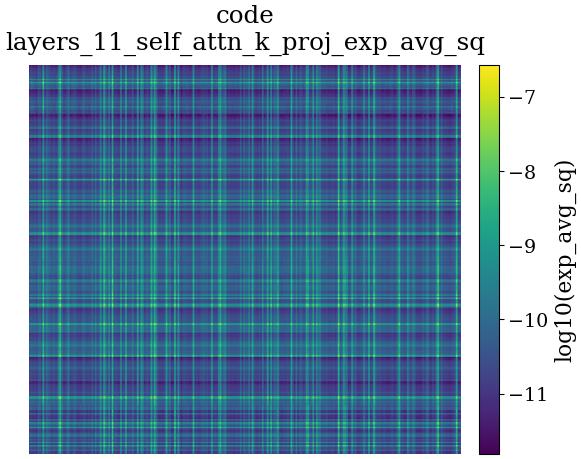}\end{subfigure}\hfill
  \begin{subfigure}[b]{\subfigwidth}\includegraphics[width=\textwidth]{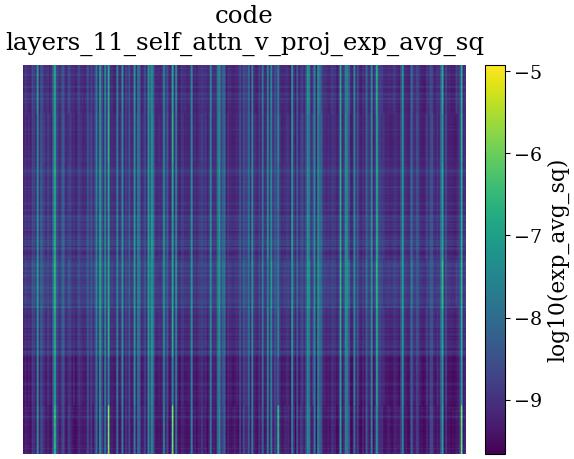}\end{subfigure}\hfill
  \begin{subfigure}[b]{\subfigwidth}\includegraphics[width=\textwidth]{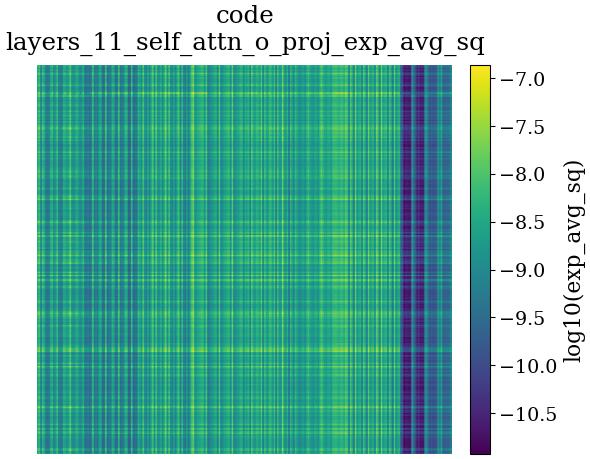}\end{subfigure}
  % Row 3: Max-Min ratio
  \begin{subfigure}[b]{\subfigwidth}\includegraphics[width=\textwidth]{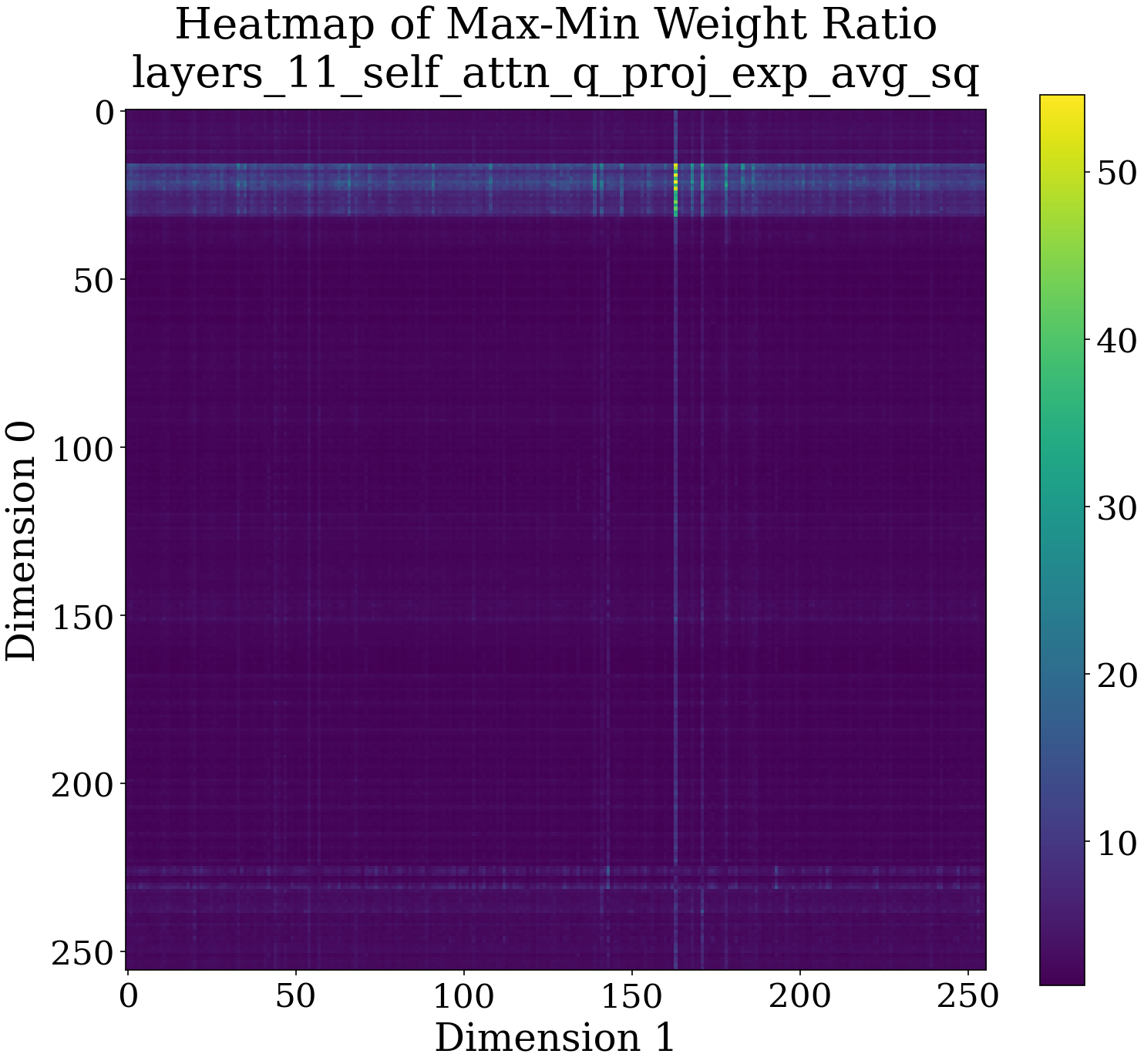}\end{subfigure}\hfill
  \begin{subfigure}[b]{\subfigwidth}\includegraphics[width=\textwidth]{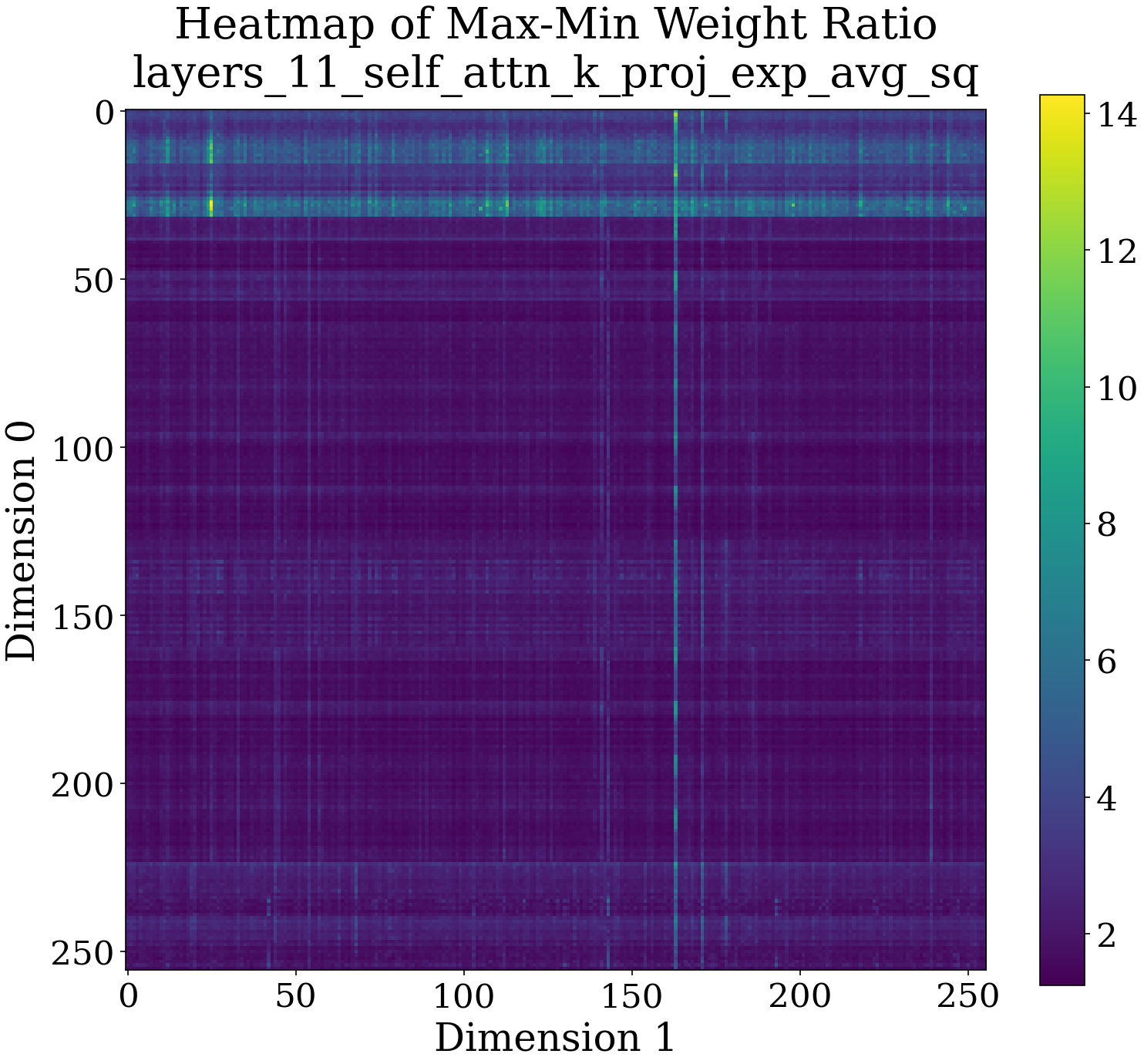}\end{subfigure}\hfill
  \begin{subfigure}[b]{\subfigwidth}\includegraphics[width=\textwidth]{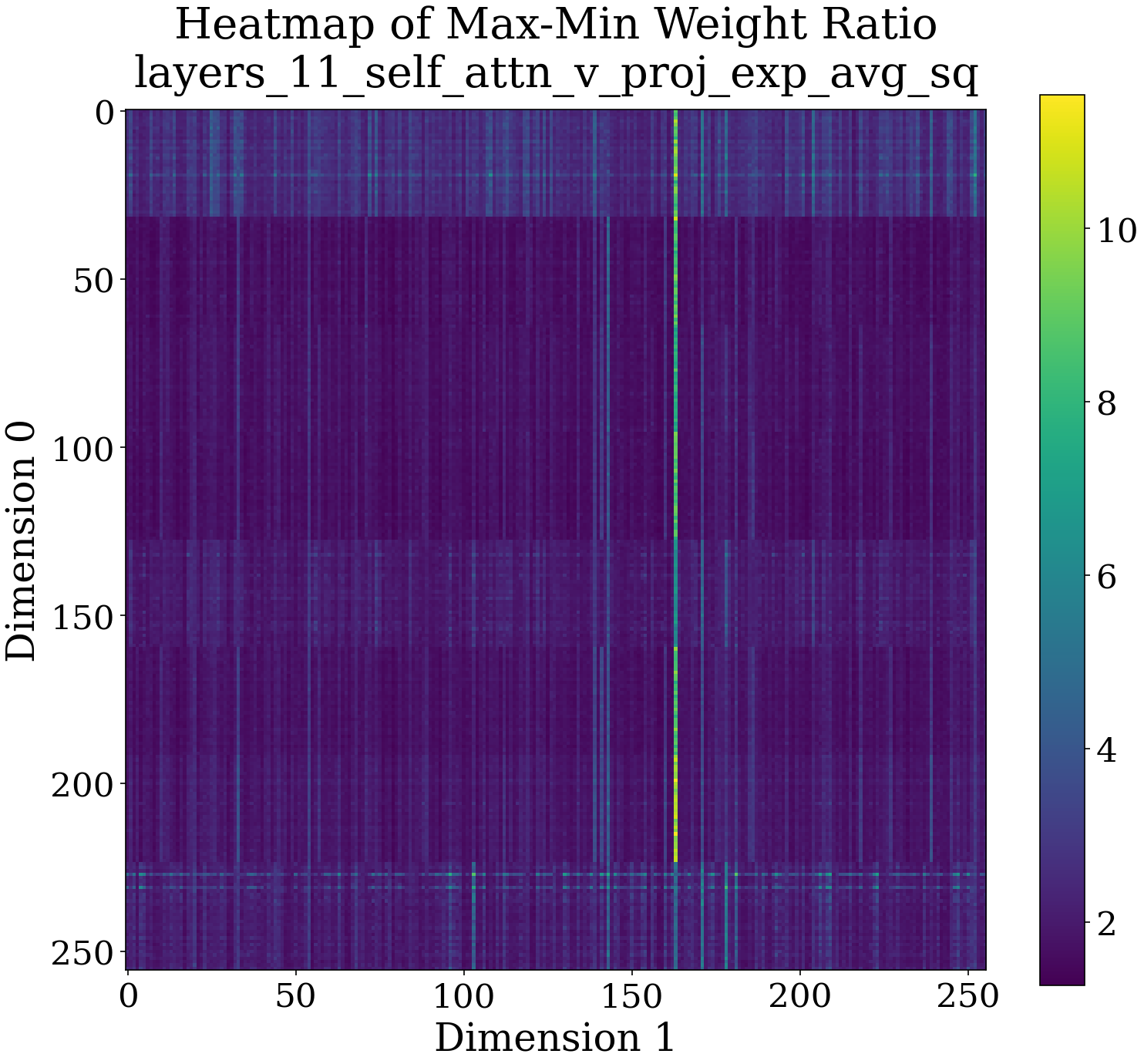}\end{subfigure}\hfill
  \begin{subfigure}[b]{\subfigwidth}\includegraphics[width=\textwidth]{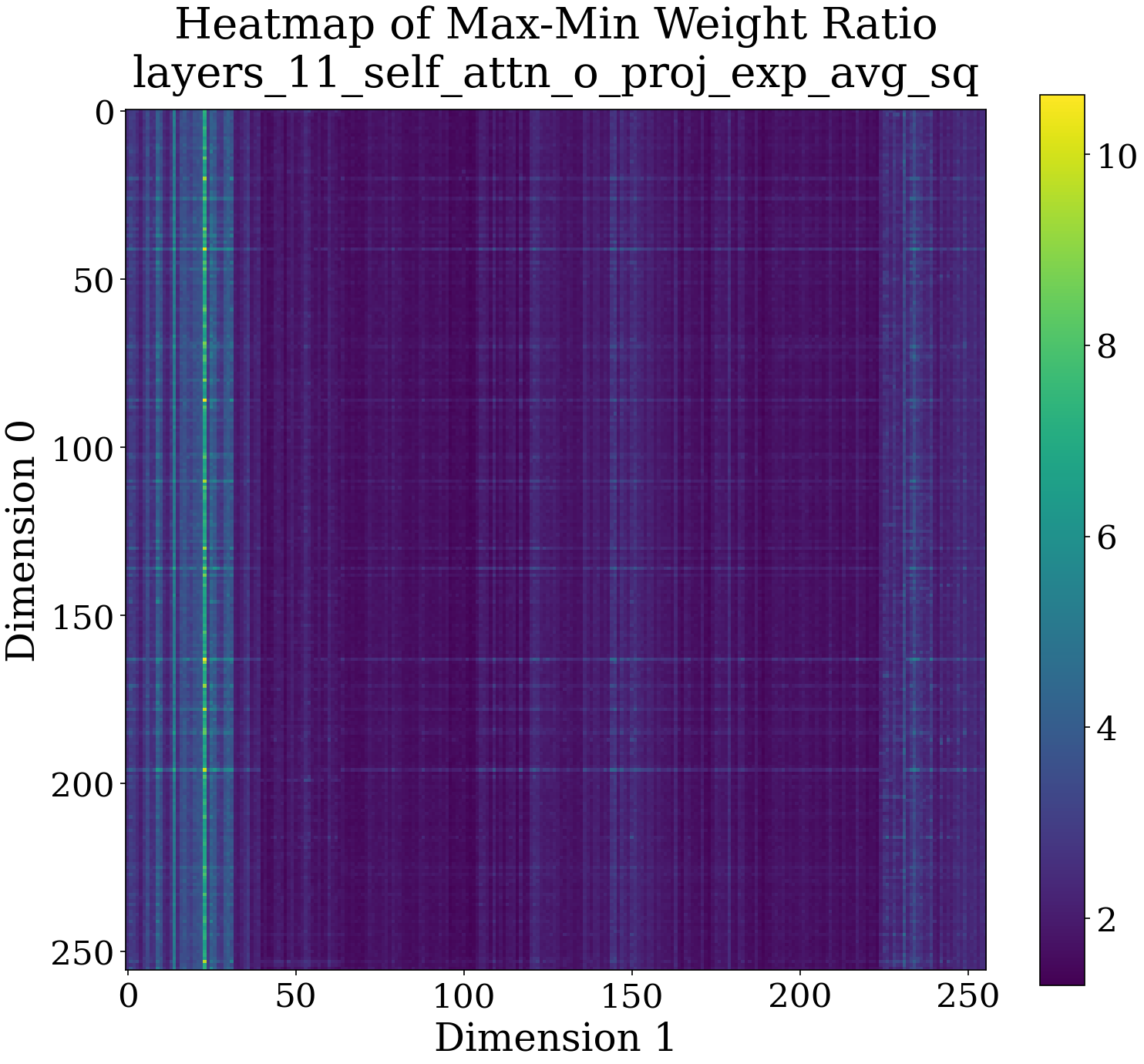}\end{subfigure}
  \caption{\textbf{Shared curvature geometry in attention layer 11.} The top two rows show the Math and Code SFT models, respectively. The striking consistency in structural patterns validates our shared curvature hypothesis. The bottom row shows the max-min ratio across all models, highlighting regions of highest variance (light color) where task specialization is most pronounced.}
\label{fig:attn_shared_curvature_code_vs_math}
\end{figure}

In Figure~\ref{fig:attn_shared_curvature_code_vs_math}, we visualize the log-scaled heatmaps of the square root of the \texttt{exp\_avg\_sq} tensor for the attention weights of two distinct SFT checkpoints, Math and Code, alongside the max-to-min ratio of their diagonal curvatures. We only report patterns that exist consistently across all layers; layer $11$ is shown here as a representative example. The complete set of curvature comparison heatmaps is available in the provided GitHub repository. We first observe that the diagonal curvature has a clearly visible row-wise and column-wise structure, matching the observations in the mask visualization in Section~\ref{sec:localize}. The column-wise band can be interpreted as an \textbf{input feature curvature} for all weights connected to a given input neuron. Similarly, the row-wise band represents an \textbf{output feature curvature} for all weights connected to a given output neuron.

\paragraph{How to interpret these heatmaps, and what is the takeaway?}
We use our notions of input and output curvature to analyze and compare the curvature scales across models. We observed a very high overlap between the subsets of input features that have the largest input curvatures for each model, across all weight types and layers. This can be seen in the strong column-wise bands shared at the exact same positions (row and column indices) for each weight type. A similar property is observed for high-curvature output features. To further confirm this, we also relied on linear max-min ratio heatmaps (as seen in Figure~\ref{fig:attn_shared_curvature_code_vs_math}), where we observed dark heatmaps across all layers and depths. This indicates that the max-min ratio remains orders of magnitude smaller than the curvature's max-min range within each heatmap. For instance, in the layer 11 projection v-layer, we observe that each expert has a curvature max-min ratio of around $10^4$, while the element-wise max-min ratio between the two heatmaps remains consistently below $5$ for almost all elements. Overall, our results strongly show a significant match between the input and output feature curvatures across models, revealing a similar (though not identical) curvature structure across different training checkpoints.

 This shared geometry explains why simple linear merging often performs so well---the models are already geometrically aligned, meaning that linear merging is \textbf{implicitly} curvature-aware, or Fisher-optimal in some sense. Our merging benchmark results further shed light on this, as we observed that Fisher merging does not provide a gain over linear merging (Table~\ref{tab:benchmark_results}), a finding consistent with reports in~\cite{tiesmerging_yadav_2023,yadav2024mattersmerging}. While the overall structure is shared, the max-min ratio plots in the bottom row reveal subtle but important variations, pinpointing specific parameter groups where task-specific fine-tuning induced the most significant changes in the loss landscape's curvature. This shared foundation, with localized, high-variance differences, is precisely what motivates our two-stage approach: FFG isolates these task-specific parameters before a simple aggregation is performed on the shared structure.

\begin{figure}[ht]
  \centering
  \newcommand{\subfigwidthcode}{0.24\textwidth}
  % Row 1: Code model (Cosine LR)
  \begin{subfigure}[b]{\subfigwidthcode}\includegraphics[width=\textwidth]{Figures/curvature_heatmaps/layers_11_self_attn_q_proj_exp_avg_sq_model_1_code_weights_heatmap.jpg}\end{subfigure}\hfill
  \begin{subfigure}[b]{\subfigwidthcode}\includegraphics[width=\textwidth]{Figures/curvature_heatmaps/layers_11_self_attn_k_proj_exp_avg_sq_model_1_code_weights_heatmap.jpg}\end{subfigure}\hfill
  \begin{subfigure}[b]{\subfigwidthcode}\includegraphics[width=\textwidth]{Figures/curvature_heatmaps/layers_11_self_attn_v_proj_exp_avg_sq_model_1_code_weights_heatmap.jpg}\end{subfigure}\hfill
  \begin{subfigure}[b]{\subfigwidthcode}\includegraphics[width=\textwidth]{Figures/curvature_heatmaps/layers_11_self_attn_o_proj_exp_avg_sq_model_1_code_weights_heatmap.jpg}\end{subfigure}
  % Row 2: Code model (WSD LR)
  \begin{subfigure}[b]{\subfigwidthcode}\includegraphics[width=\textwidth]{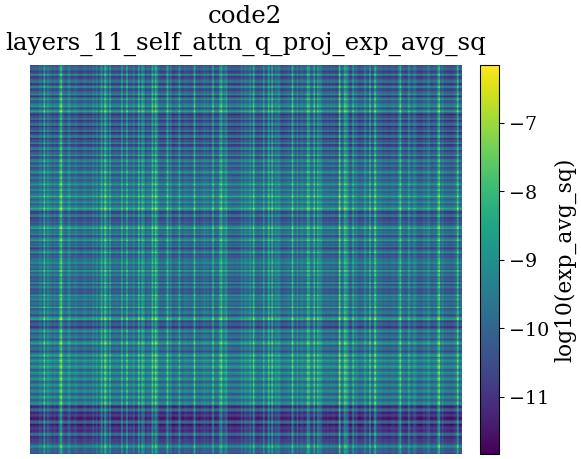}\end{subfigure}\hfill
  \begin{subfigure}[b]{\subfigwidthcode}\includegraphics[width=\textwidth]{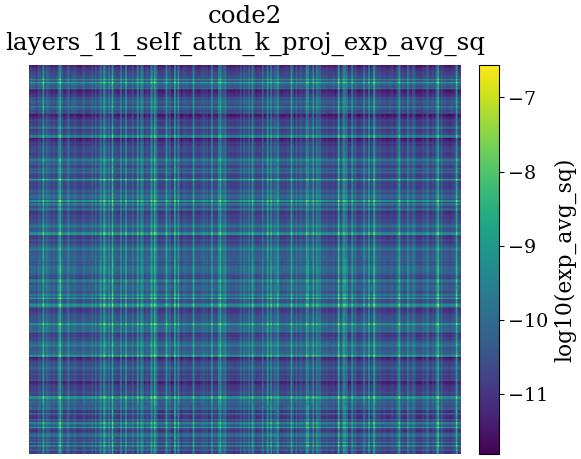}\end{subfigure}\hfill
  \begin{subfigure}[b]{\subfigwidthcode}\includegraphics[width=\textwidth]{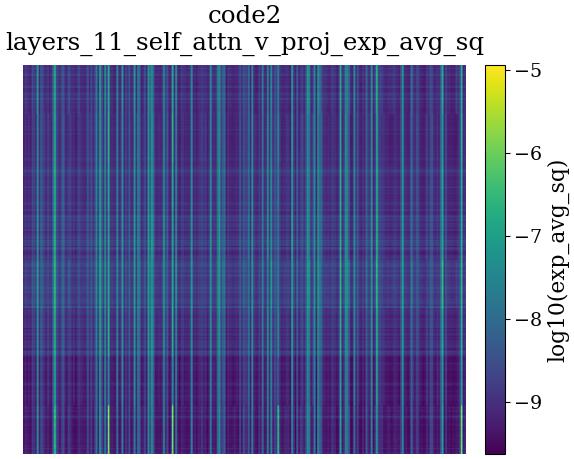}\end{subfigure}\hfill
  \begin{subfigure}[b]{\subfigwidthcode}\includegraphics[width=\textwidth]{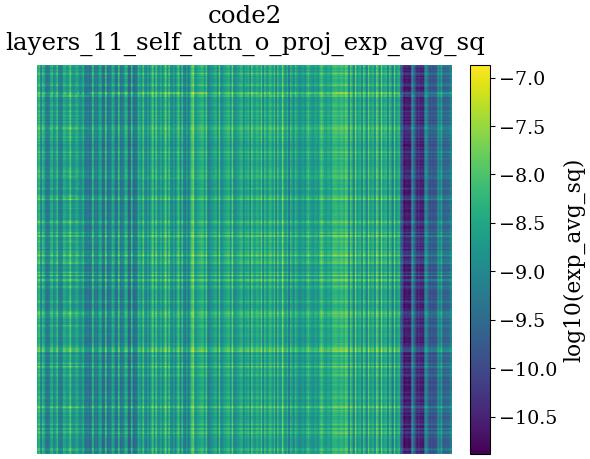}\end{subfigure}
  % Row 3: Max-Min ratio
  \begin{subfigure}[b]{\subfigwidthcode}\includegraphics[width=\textwidth]{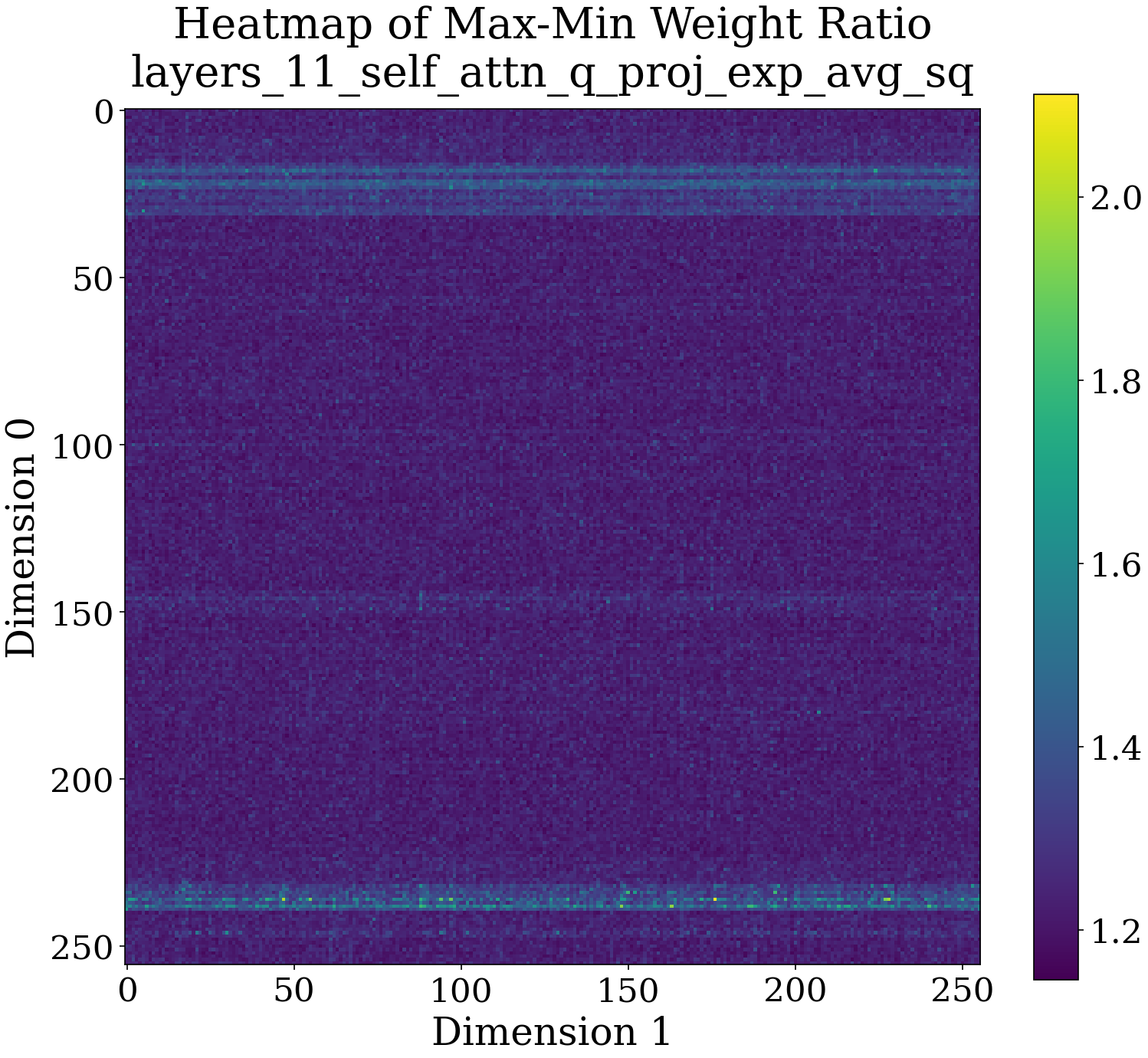}\end{subfigure}\hfill
  \begin{subfigure}[b]{\subfigwidthcode}\includegraphics[width=\textwidth]{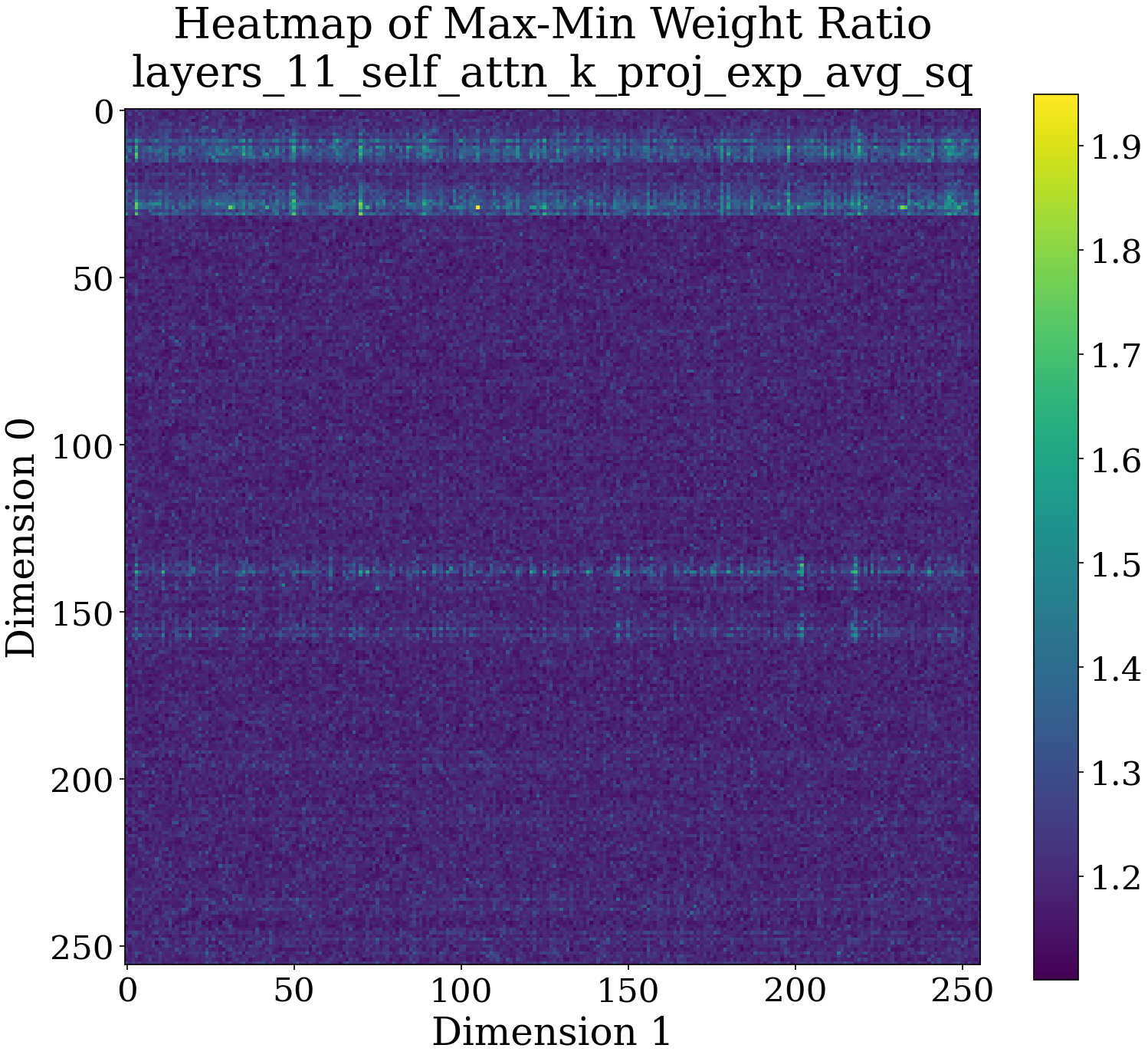}\end{subfigure}\hfill
  \begin{subfigure}[b]{\subfigwidthcode}\includegraphics[width=\textwidth]{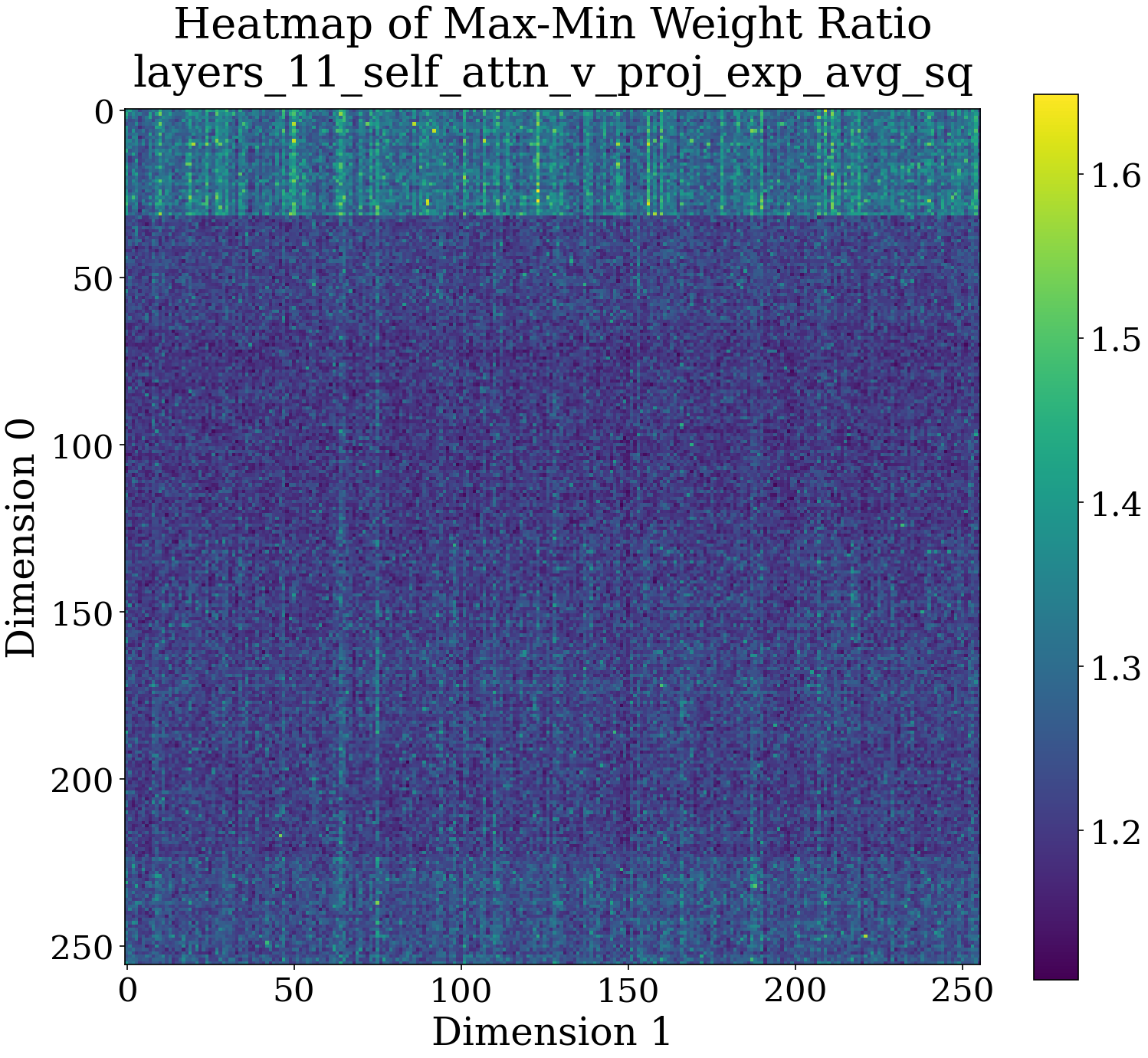}\end{subfigure}\hfill
  \begin{subfigure}[b]{\subfigwidthcode}\includegraphics[width=\textwidth]{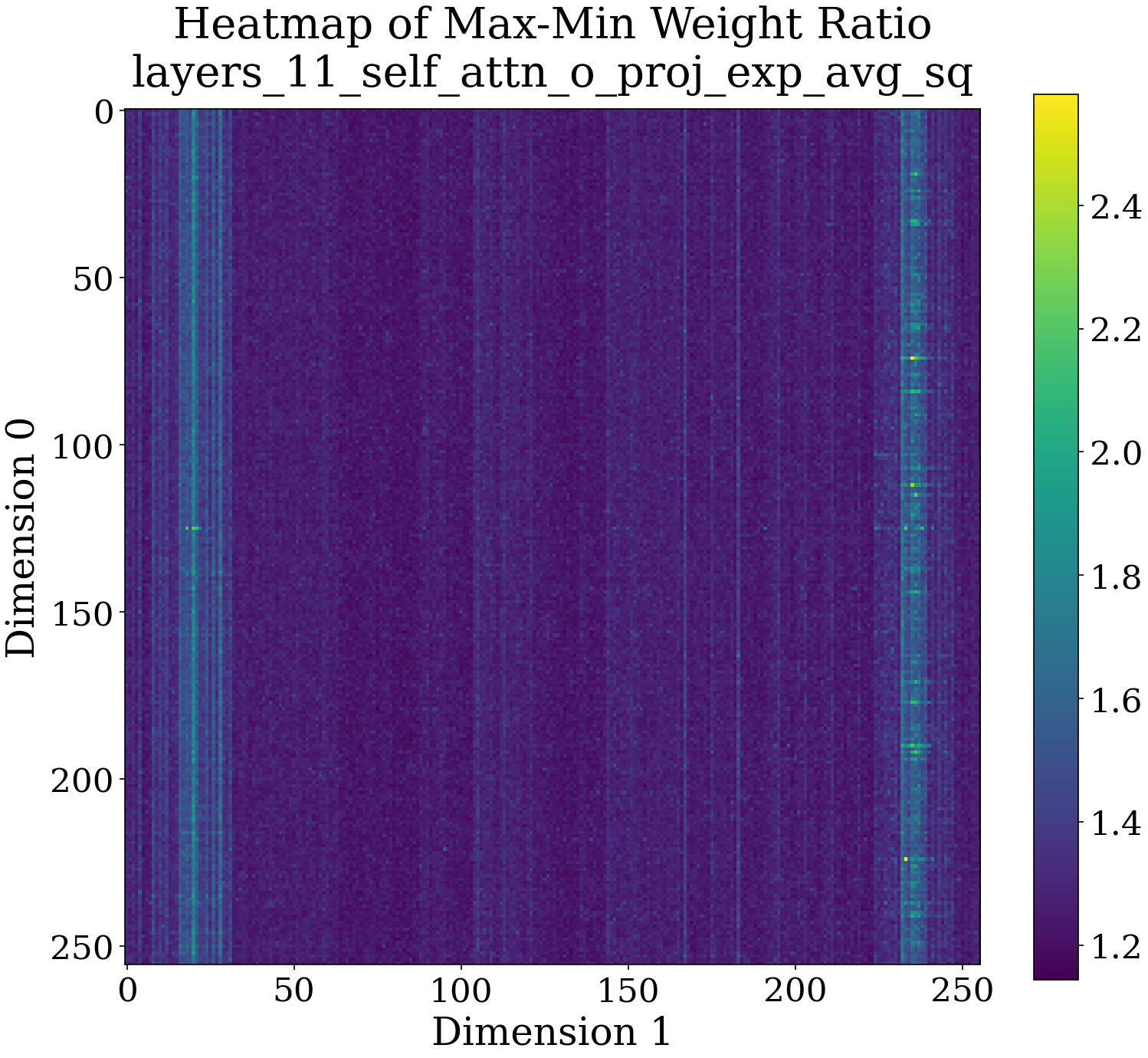}\end{subfigure}
  \caption{\textbf{Control Experiment: Shared curvature in attention layer 11 for two Code models.} The top row shows a Code model trained with a Cosine LR schedule, and the second row shows a Code model trained on the same data with a WSD schedule. The structural similarity is nearly perfect. The bottom row shows the max-min ratio is consistently close to 1 (dark color), indicating minimal geometric deviation.}
  \label{fig:attn_shared_curvature_code_vs_code2}
\end{figure}

 To validate that these observed differences in are indeed task-specific and not mere training artifacts, we conducted a control experiment comparing two Code specialists trained on the same data but with different learning rate schedulers (Cosine vs. WSD). As shown in Figure~\ref{fig:attn_shared_curvature_code_vs_code2}, the resulting curvature patterns are virtually indistinguishable. This is quantitatively confirmed by the max-min ratio plots, which are an order of magnitude smaller than in the Code vs. Math comparison and are consistently close to one. This result provides strong evidence that our curvature-based analysis accurately captures meaningful, task-driven geometric differences, reinforcing the theoretical foundation of OTA-Merging.

\section{Conclusion}
\label{sec:conclusion}
In this work, we introduced OTA Merging, a scalable and effective framework for consolidating specialized models by harnessing the rich, yet often discarded, second-moment statistics from the Adam optimizer. We demonstrated that this optimization history serves as a powerful and computationally efficient proxy for the local loss landscape curvature.

Our approach is twofold: first, FFG leverages this curvature information to act as a principled "denoiser," identifying and reverting noisy parameter updates to isolate the essential knowledge of each expert. Second, a curvature-aware aggregation scheme merges these denoised experts in a way that respects their underlying geometry. This methodology is motivated by our central discovery that independently fine-tuned models exhibit a remarkable geometric consensus, shifting the primary challenge of merging from alignment to interference mitigation.

Our experiments, which consolidated five distinct capabilities into a single Llama 3.1 8B model, validate this perspective. The OTA-merged model achieved state-of-the-art performance, and ablation studies confirmed that the saliency-aware pruning from FFG was the most critical factor for success. Furthermore, we showed that FFG is a potent analytical tool, revealing structured, role-aware sparsity patterns that offer new insights into task localization. By treating the optimization trajectory as a valuable asset, OTA-Merging provides a new, robust paradigm for efficient model composition, paving the way for future explorations into more complex curvature approximations and their application across different model composition techniques.

\section*{Acknowledgements}
This work was partially supported by NSF CAREER Award \#2239374.

\clearpage
\bibliographystyle{unsrtnat}  % <-- instead of unsrt
\bibliography{ref}

\newpage
\appendix
\appendix

\section{Theoretical Justifications and Proofs}
\label{sec:appendix_proofs}

This section provides a detailed derivation of the theoretical insights presented in Section~\ref{subsec:theory}, which establish Adam's second-moment accumulator, $\mathbf{v}$, as a principled proxy for the diagonal of the empirical Fisher Information Matrix (FIM).

\subsection{Proof of Equivalence between Hessian and Empirical FIM}

Our first result connects the Hessian of the loss function to the Observed Empirical FIM at a fine-tuned model's optimal parameters, $\mathbf{w}_*$. The Hessian for a loss $\mathcal{L}_{D}(\mathbf{w})$ over a dataset $D$ is given by:
\begin{equation}
\label{eq:grad_hessian_general_appendix}
\nabla^2 \mathcal{L}_{D} (\mathbf{w}) = \mathbb{E}_{(\mathbf{x},y) \sim D} \left[ \underbrace{\mathbf{J}_f(\mathbf{w}, \mathbf{x})^\top \nabla^2_{\mathbf{z}} \ell(y, \mathbf{z}) \mathbf{J}_f(\mathbf{w} , \mathbf{x})}_{\text{Generalized Gauss-Newton (GGN) term}} \right] + \mathbb{E}_{(\mathbf{x},y) \sim D} \left[ \underbrace{\sum_{j=1}^{C} \left( \nabla_{z^j} \ell(y, \mathbf{z}) \right) \nabla^2_{\mathbf{w}} f^j(\mathbf{w}, \mathbf{x})}_{\text{Second-order term}} \right],
\end{equation}
where $\mathbf{z} = f(\mathbf{w}, \mathbf{x})$ are the model's logits, $\ell$ is the per-sample loss (e.g., negative log-likelihood), and $\mathbf{J}_f$ is the Jacobian of the network function $f$ with respect to the parameters $\mathbf{w}$. To simplify this expression, we rely on two standard assumptions.

\begin{assumption}[Late NTK Locality]
\label{assum:ggn_appendix}
Near an optimal set of parameters $\mathbf{w}_*$, the second-order term in Equation~\eqref{eq:grad_hessian_general_appendix}, which depends on the curvature of the network function $f$ itself, is negligible. This implies local linearity: $f(\mathbf{x};\mathbf{w}_* + \bm{\delta}) \approx f(\mathbf{x};\mathbf{w}_*) + \mathbf{J}_f \bm{\delta} $ for small perturbations $\bm{\delta}$.
\end{assumption}

This assumption allows us to approximate the Hessian using only the GGN term:
\begin{equation}
\label{ref:GGN-EmpFisher1_appendix}
\nabla^2 \mathcal{L}_{D} (\mathbf{w}_*) \approx \mathbb{E}_{(\mathbf{x},y) \sim D} \left[ \mathbf{J}_f(\mathbf{w}_*, \mathbf{x})^\top \nabla^2_{\mathbf{z}} \ell(y, f(\mathbf{x};\mathbf{w}_*)) \mathbf{J}_f(\mathbf{w}_* , \mathbf{x}) \right].
\end{equation}

\begin{assumption}[Perfect Calibration at Fine-Tuned Checkpoints]
\label{assum:calib_appendix}
At the optimal parameters $\mathbf{w}_*$, the model is perfectly calibrated, meaning its predictive distribution matches the true conditional data distribution for any given input $\mathbf{x}$: $p(y|\mathbf{x}; \mathbf{w}_*) = d(y|\mathbf{x})$.
\end{assumption}

With these assumptions, we can now state and prove the main lemma.

\begin{lemma}
\label{lemma:hessian_to_fim_appendix}
Under Assumptions~\ref{assum:ggn_appendix} and~\ref{assum:calib_appendix}, the Hessian of the loss at $\mathbf{w}_*$ is approximately equal to the Observed Empirical FIM, $\Bar{\mathbf{F}}(\mathbf{w}_*)$.
\end{lemma}
\begin{proof}
We begin by analyzing the inner term of the GGN in Equation~\eqref{ref:GGN-EmpFisher1_appendix}, which is the expectation of the Hessian of the negative log-likelihood $\ell = -\log p$ with respect to the true data distribution $y \sim d(y|\mathbf{x})$. Using the identity for the Hessian of the log-likelihood, we have:
\begin{align*}
\mathbb{E}_{y \sim d(y|\mathbf{x})} [-\nabla^2_{\mathbf{z}} \log p(y|\mathbf{z})] &= \mathbb{E}_{y \sim d(y|\mathbf{x})} \left[ -\frac{\nabla^2_{\mathbf{z}} p(y|\mathbf{z})}{p(y|\mathbf{z})} + \frac{(\nabla_{\mathbf{z}} p(y|\mathbf{z}))(\nabla_{\mathbf{z}} p(y|\mathbf{z}))^\top}{p(y|\mathbf{z})^2} \right] \\
&= -\mathbb{E}_{y \sim p(y|\mathbf{z})} \left[ \frac{\nabla^2_{\mathbf{z}} p(y|\mathbf{z})}{p(y|\mathbf{z})} \right] + \mathbb{E}_{y \sim d(y|\mathbf{x})} \left[ (\nabla_{\mathbf{z}} \log p(y|\mathbf{z}))(\nabla_{\mathbf{z}} \log p(y|\mathbf{z}))^\top \right]
\end{align*}
In the second line, we invoke Assumption~\ref{assum:calib_appendix} ($d=p$) to switch the distribution of the expectation for the first term. This first term vanishes because the expectation of the score's gradient is zero under standard regularity conditions:
\[
\mathbb{E}_{y \sim p(y|\mathbf{z})} [\nabla^2_{\mathbf{z}} \log p(y|\mathbf{z})] = \int \nabla^2_{\mathbf{z}} p(y|\mathbf{z}) dy = \nabla^2_{\mathbf{z}} \int p(y|\mathbf{z}) dy = \nabla^2_{\mathbf{z}}(1) = \mathbf{0}.
\]
This leaves only the second term, which is the definition of the Fisher information matrix of the logits. Substituting this result back into Equation~\eqref{ref:GGN-EmpFisher1_appendix}, we arrive at the main result:
\begin{equation}
\label{eq:hessian_fim_equivalence_appendix}
\begin{split}
\nabla^2 \mathcal{L}_{D} (\mathbf{w}_*) &\approx \mathbb{E}_{(\mathbf{x},y) \sim D} \left[ \mathbf{J}_f^\top \left( \mathbb{E}_{y \sim d(y|\mathbf{x})} [(\nabla_{\mathbf{z}} \log p)(\nabla_{\mathbf{z}} \log p)^\top] \right) \mathbf{J}_f \right] \\
&= \mathbb{E}_{(\mathbf{x},y) \sim D} \left[ \left( \mathbf{J}_f^\top \nabla_{\mathbf{z}} \log p \right) \left( \mathbf{J}_f^\top \nabla_{\mathbf{z}} \log p \right)^\top \right] \\
&= \mathbb{E}_{(\mathbf{x},y) \sim D} \left[ (\nabla_{\mathbf{w}} \log p(y|\mathbf{x};\mathbf{w}_*)) (\nabla_{\mathbf{w}} \log p(y|\mathbf{x};\mathbf{w}_*))^\top \right] = \Bar{\mathbf{F}}(\mathbf{w}_*).
\end{split}
\end{equation}
Thus, the Hessian of the loss at an optimal point $\mathbf{w}_*$ is approximately equal to the Observed Empirical FIM.
\end{proof}

\subsection{Proof of Relation between Mini-Batch Second Moment and FIM}
Next, we show how the second moment of mini-batch gradients relates to the Empirical FIM defined above.

\begin{lemma}
\label{lemma:fim-to-batch-fim_appendix}
Let $\nabla \mathcal{L}_{B_k} (\mathbf{w}) = \frac{1}{|B|}\sum_{(\mathbf{x},y) \in B_k} \nabla \mathcal{L}_{\mathbf{x},y}(\mathbf{w})$ be the gradient for a mini-batch $B_k$ of size $|B|$ sampled from the data distribution $D$. Under Assumption~\ref{assum:calib_appendix}, the expectation of the outer product of this mini-batch gradient is a scaled version of the FIM:
\begin{equation}
  \mathbb{E}_{B_k \sim D^{|B|}} \left[ \nabla \mathcal{L}_{B_k} (\mathbf{w}_*) \nabla \mathcal{L}_{B_k} (\mathbf{w}_*)^\top \right] = \frac{1}{|B|} \Bar{\mathbf{F}}(\mathbf{w}_*).
\end{equation}
\end{lemma}
\begin{proof}
We decompose the expectation of the outer product:
\begin{align*}
    \mathbb{E}_{B_k \sim D^{|B|}} \left[ \nabla \mathcal{L}_{B_k} \nabla \mathcal{L}_{B_k}^\top \right] &= \mathbb{E} \left[ \left(\frac{1}{|B|}\sum_{i=1}^{|B|} \nabla \mathcal{L}_{\mathbf{x}_i,y_i}\right) \left(\frac{1}{|B|}\sum_{j=1}^{|B|} \nabla \mathcal{L}_{\mathbf{x}_j,y_j}\right)^\top \right] \\
    &= \frac{1}{|B|^2} \sum_{i,j} \mathbb{E} \left[ \nabla \mathcal{L}_{\mathbf{x}_i,y_i} \nabla \mathcal{L}_{\mathbf{x}_j,y_j}^\top \right] \\
    &= \frac{1}{|B|^2}  \sum_{i=1}^{|B|} \mathbb{E}_{\mathbf{x}_i,y_i \sim D}[\nabla \mathcal{L}_{\mathbf{x}_i,y_i} \nabla \mathcal{L}_{\mathbf{x}_i,y_i}^\top] + \frac{1}{|B|^2} \sum_{i\neq j} \mathbb{E} \left[ \nabla \mathcal{L}_{\mathbf{x}_i,y_i} \nabla \mathcal{L}_{\mathbf{x}_j,y_j}^\top \right].
\end{align*}
Since the samples in the mini-batch are i.i.d., the expectation of the cross-terms ($i \neq j$) decouples:
\[
\mathbb{E} \left[ \nabla \mathcal{L}_{\mathbf{x}_i,y_i} \nabla \mathcal{L}_{\mathbf{x}_j,y_j}^\top \right] = \mathbb{E}_{\mathbf{x}_i,y_i \sim D} [\nabla \mathcal{L}_{\mathbf{x}_i,y_i}] \mathbb{E}_{\mathbf{x}_j,y_j \sim D} [\nabla \mathcal{L}_{\mathbf{x}_j,y_j}]^\top.
\]
At the optimum $\mathbf{w}_*$, the expected gradient over the true data distribution is zero. This follows from Assumption~\ref{assum:calib_appendix} and the fact that the expectation of the score function is zero:
\[
\mathbb{E}_{(\mathbf{x},y) \sim D} \left[ \nabla_{\mathbf{w}} \mathcal{L}_{\mathbf{x},y} (\mathbf{w}_*) \right] = \mathbb{E}_{\mathbf{x}} \left[ \mathbf{J}_f^\top \mathbb{E}_{y \sim p(y|\mathbf{x})} [ \nabla_{\mathbf{z}} \log p(y|\mathbf{z}) ] \right] = \mathbf{0}.
\]
Therefore, all the cross-terms ($i \neq j$) in the decomposition vanish. We are left with only the diagonal terms of the sum:
\begin{align*}
\mathbb{E}_{B_k \sim D^{|B|}} \left[ \nabla \mathcal{L}_{B_k} \nabla \mathcal{L}_{B_k}^\top \right] &= \frac{1}{|B|^2} \sum_{i=1}^{|B|} \mathbb{E} \left[ \nabla \mathcal{L}_{\mathbf{x}_i,y_i} \nabla \mathcal{L}_{\mathbf{x}_i,y_i}^\top \right] \\
&= \frac{|B|}{|B|^2} \mathbb{E}_{(\mathbf{x},y) \sim D} \left[ \nabla \mathcal{L}_{\mathbf{x},y}(\mathbf{w}_*) \nabla \mathcal{L}_{\mathbf{x},y}(\mathbf{w}_*)^\top \right] \\
&= \frac{1}{|B|} \Bar{\mathbf{F}}(\mathbf{w}_*).
\end{align*}
This completes the proof, showing that the second moment of the mini-batch gradient is, on expectation, a scaled version of the full Empirical FIM.
\end{proof}

\section{Additional Information on Training SFT Models}

\subsection{Gradient, and Loss Visualizations}
\label{app:grad-norm}

Figure~\ref{fig:grad-norms} shows the training gradient norm trajectories $\|\nabla \mathcal{L}\|$ for each of our five specialist LLMs. Tasks like \texttt{knowledge\_recall} and \texttt{coding} exhibit high gradient norms, whereas \texttt{math\_reasoning} and \texttt{precise\_if} remain much lower.

\begin{figure}[ht]
  \centering
  \includegraphics[scale=0.7]{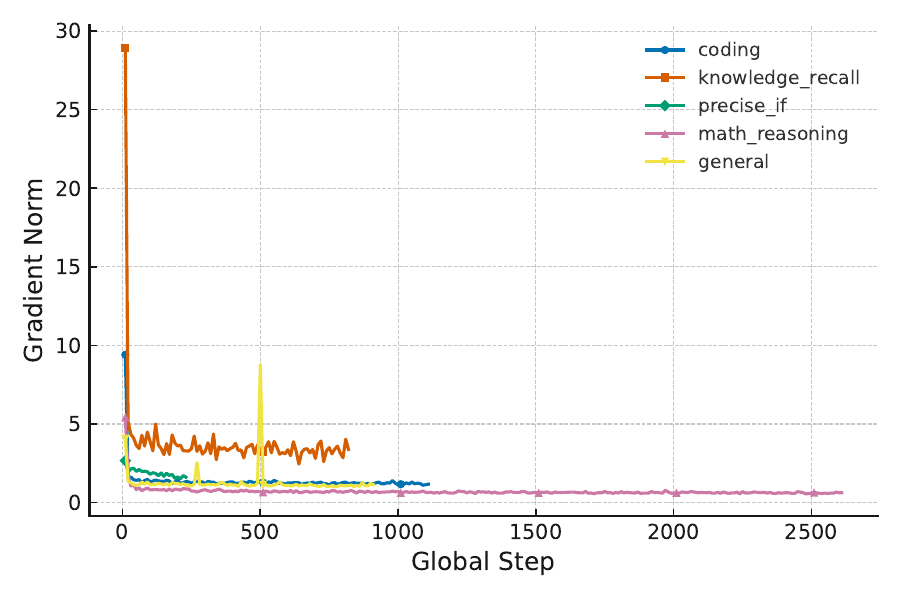}
  \caption{Gradient norm trajectories over global training steps for each fine-tuned SFT models}
  \label{fig:grad-norms}
\end{figure}

Moreover, we visualize the optimization trajectories of each SFT model on distinct tasks as indicated in Figure~\ref{fig:training-losses}. 

\begin{figure}[htbp]
  \centering
  \includegraphics[width=0.9\linewidth]{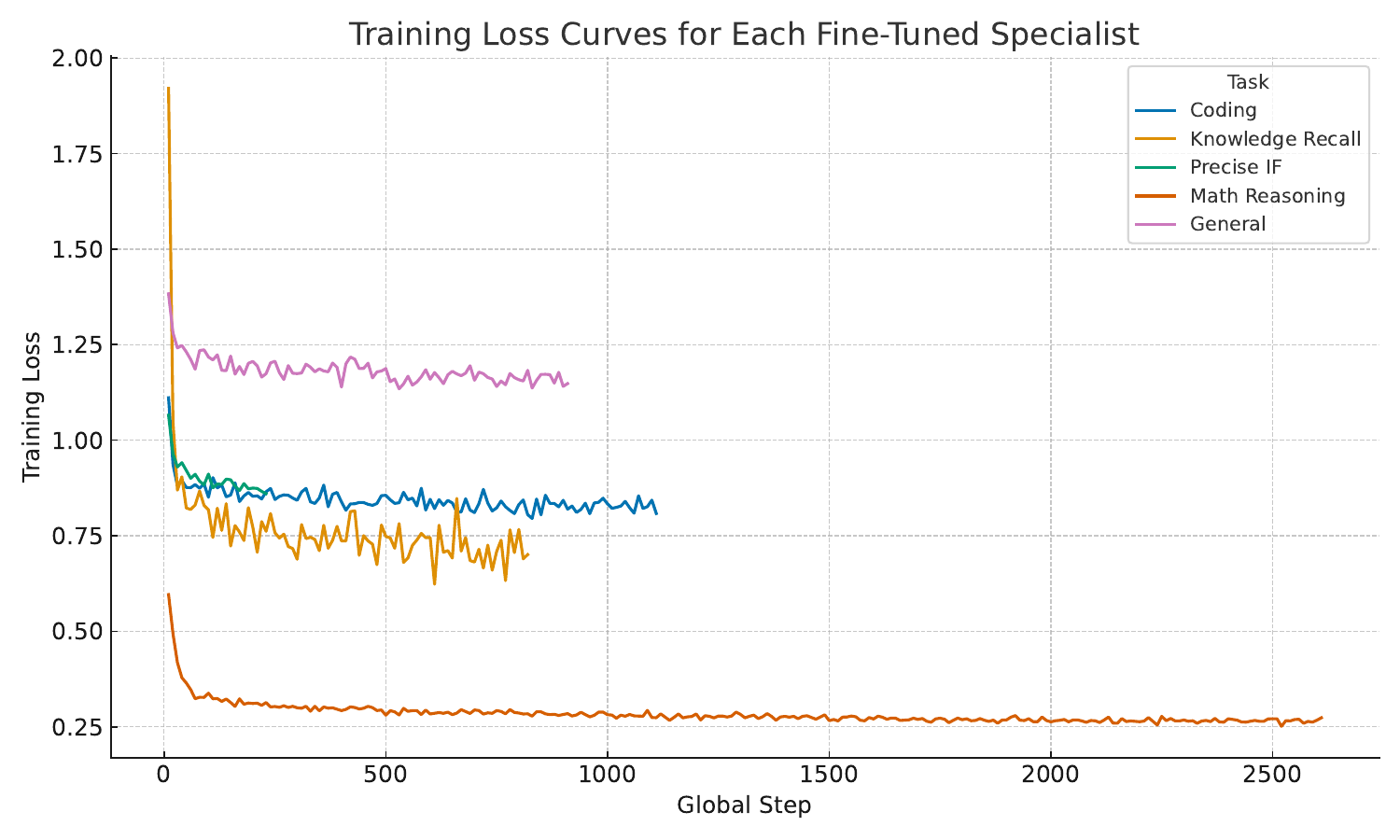}
  \caption{Training loss trajectories across global steps for each SFT model}
  \label{fig:training-losses}
\end{figure}

\subsection{Experiments Details for Training SFT Checkpoints}
\label{sec:sft-d}

The models were trained from the \texttt{Meta-Llama-3.1--8B} base model using full-parameter Supervised Fine-Tuning (SFT) and AdamW optimizer. The second-moment statistics (\texttt{exp\_avg\_sq}) of the optimizer were checkpointed and later used as curvature proxies in OTA-Merging.

We leveraged full post-training stack provided by \textbf{LLaMA-Factory}~\cite{zheng2024llamafactory}. This unified infrastructure handled supervised fine-tuning (SFT), tokenizer alignment, and checkpoint conversion. 

Our configuration closely followed the hyperparameter recipe from the T\"{u}lu-3~\cite{lambert2024t}, with slight task-specific adjustments. Fine-tuning was performed on two NVIDIA A100 GPUs (80GB) per task. We used gradient accumulation of 32 and a micro-batch size of 2 per device, yielding an effective global batch size of 128. All models were trained with a learning rate of \(5 \times 10^{-6}\), except for Precise Instruction Following (IF), which used \(1 \times 10^{-5}\) to encourage faster convergence.

\begin{table}[h]
\centering
\caption{Fine-tuning hyperparameters used for training specialist models.}
\vspace{0.3em}
\label{tab:training-hparams}
\begin{tabular}{@{}lcc@{}}
\toprule
\textbf{Hyperparameter} & \textbf{Value} & \textbf{Notes} \\
\midrule
Base Model & LLaMA 3.1--8B & \\
Micro Batch Size & 2 & Per GPU \\
Gradient Accumulation & 32 & \\
Effective Batch Size & 128 & Across GPUs \\
Max Token Length & 4096 & \\
Learning Rate & \(5 \times 10^{-6}\) & \(1 \times 10^{-5}\) for Precise IF \\
Learning Rate Schedule & Linear & \\
Warmup Ratio & 0.03 & \\
Epochs & 1 & \\
Post-training Stack & LLaMA-Factory & Full stack used~\cite{zheng2024llamafactory} \\
\bottomrule
\end{tabular}
\end{table}

\subsection{Training Dataset Curation}
\label{app:dataset-curation}

Our SFT models were fine-tuned on curated subsets of the \texttt{allenai/tulu-3-sft-mixture} dataset, retrieved from the Hugging Face \texttt{datasets} library. This dataset aggregates instruction-following conversations from a wide range of sources, and we mapped these sources to specific capability categories for training the OTA-Merging specialists.

Each training example consists of a structured conversation in the \texttt{messages} field---a list of dictionaries that capture the dialogue turns between roles: \texttt{system}, \texttt{user}, and \texttt{assistant}. This structure enables role-specific formatting for instruction tuning.

\paragraph{Task Categories and Source Mapping.}
The data sources used for fine-tuning were grouped into the following categories:

\begin{itemize}
    \item \textbf{General Instruction}: \texttt{wildchat}, \texttt{oasst1\_converted}, \texttt{no\_robots}, etc.
    \item \textbf{Knowledge Recall}: \texttt{flan\_v2}, \texttt{sciriff}, \texttt{table\_gpt}
    \item \textbf{Mathematical Reasoning}: \texttt{persona\_math}, \texttt{numinamath}, \texttt{open\_math}
    \item \textbf{Coding}: \texttt{codealpaca}, \texttt{persona\_code}
    \item \textbf{Precise Instruction Following (Precise IF)}: \texttt{persona\_ifdata}
\end{itemize}

\paragraph{Dataset Statistics.}
The number of examples in each category used for fine-tuning is summarized below:

\begin{table}[htbp]
\centering
\caption{Number of examples per task category in \texttt{tulu-3-sft-mixture}.}
\label{tab:dataset-category-stats}
\vspace{0.3em}
\begin{tabular}{@{}lcc@{}}
\toprule
\textbf{Category} & \textbf{Examples} & \textbf{Percentage} \\
\midrule
Mathematical Reasoning & 334,252 & 39.70\% \\
Coding & 142,275 & 16.89\% \\
General Instruction & 116,871 & 13.89\% \\
Knowledge Recall & 104,982 & 12.46\% \\
Precise Instruction Following & 29,980 & 3.56\% \\
\midrule
\textbf{Total} & 728,360 & 100.00\% \\
\bottomrule
\end{tabular}
\end{table}

\paragraph{Message Formatting.}
Each \texttt{messages} list was rendered into a flattened training string using a standardized Jinja2 template consistent with LLaMA-Factory's post-training stack. This template inserts role-specific delimiters and appends \texttt{<eos\_token>} after assistant responses.

For example, the following JSON input:

\begin{verbatim}
[
  {"role": "system", "content": "System prompt."},
  {"role": "user", "content": "User question."},
  {"role": "assistant", "content": "Assistant answer."}
]
\end{verbatim}

is rendered as:

\begin{verbatim}
<|system|>
System prompt.
<|user|>
User question.
<|assistant|>
Assistant answer.<eos_token>
\end{verbatim}

This formatting ensures consistency across training samples and compatibility with instruction-tuned decoding patterns adopted from Tulu~\cite{lambert2024t}.

\section{Magnitude-Based Grafting Visualizations}
\label{sec:ffg_viz_mag}
To complement the main-text FFG visualizations, we provide analogous 3-way magnitude-based grids across layers 1, 2, 15, 18, 29, and 30. RGB channels encode Code (red), Math (green), and Instruction-Following (blue) experts.

\begin{figure}[!ht]
    \centering
    % Layer 1
    \begin{subfigure}[b]{0.24\textwidth}
        \includegraphics[width=\textwidth]{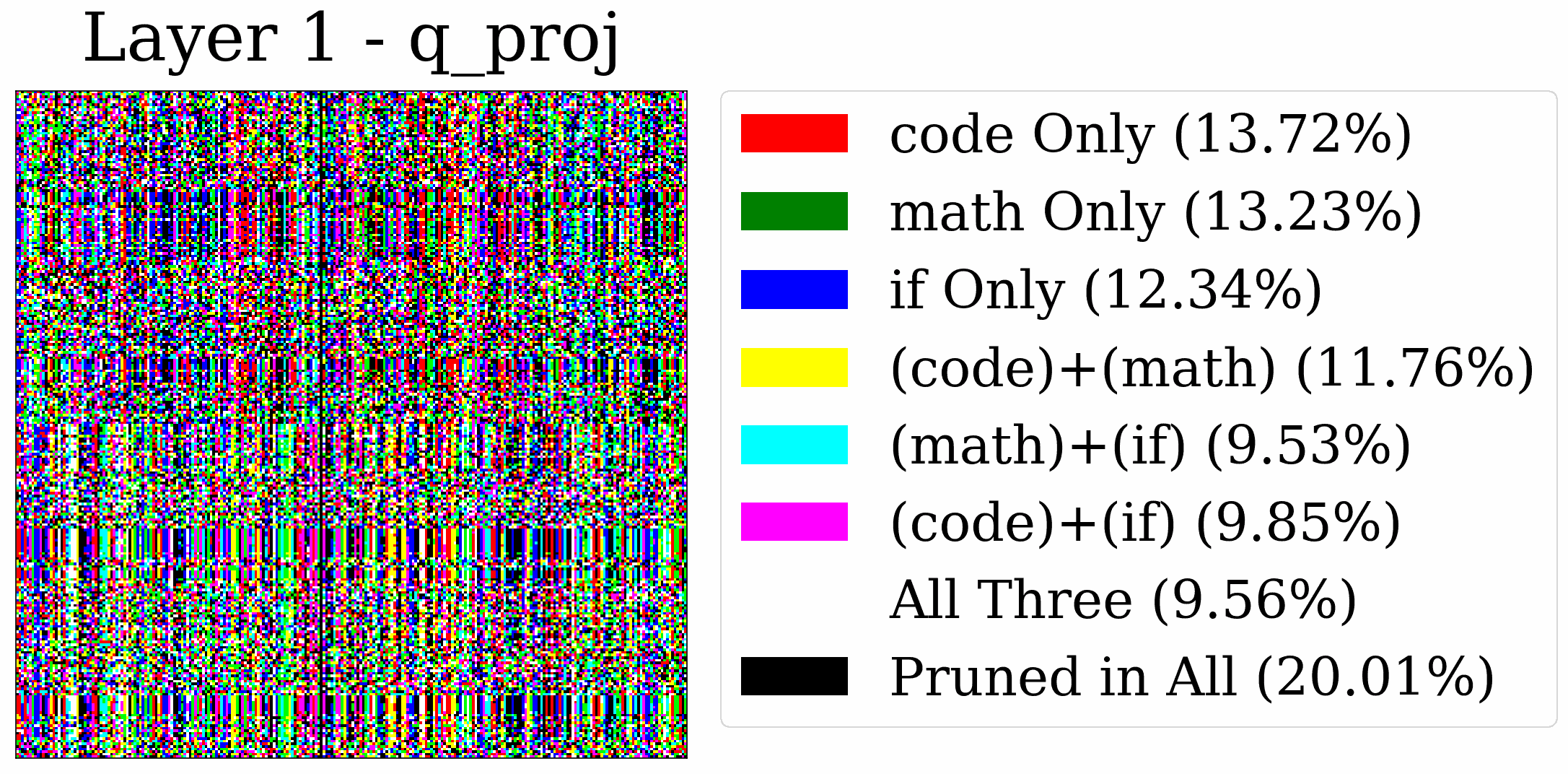}
    \end{subfigure}
    \hfill
    \begin{subfigure}[b]{0.24\textwidth}
        \includegraphics[width=\textwidth]{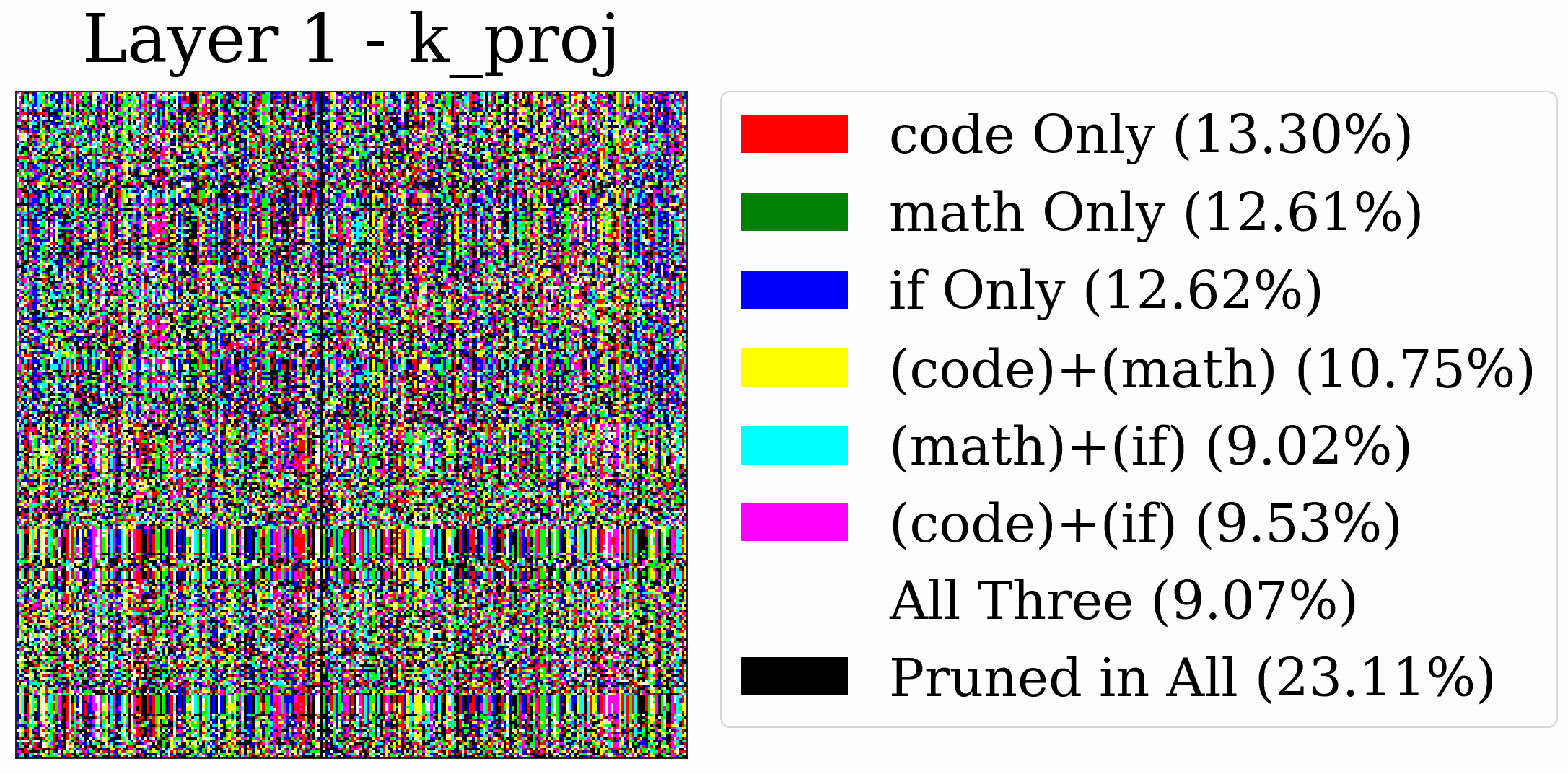}
    \end{subfigure}
    \hfill
    \begin{subfigure}[b]{0.24\textwidth}
        \includegraphics[width=\textwidth]{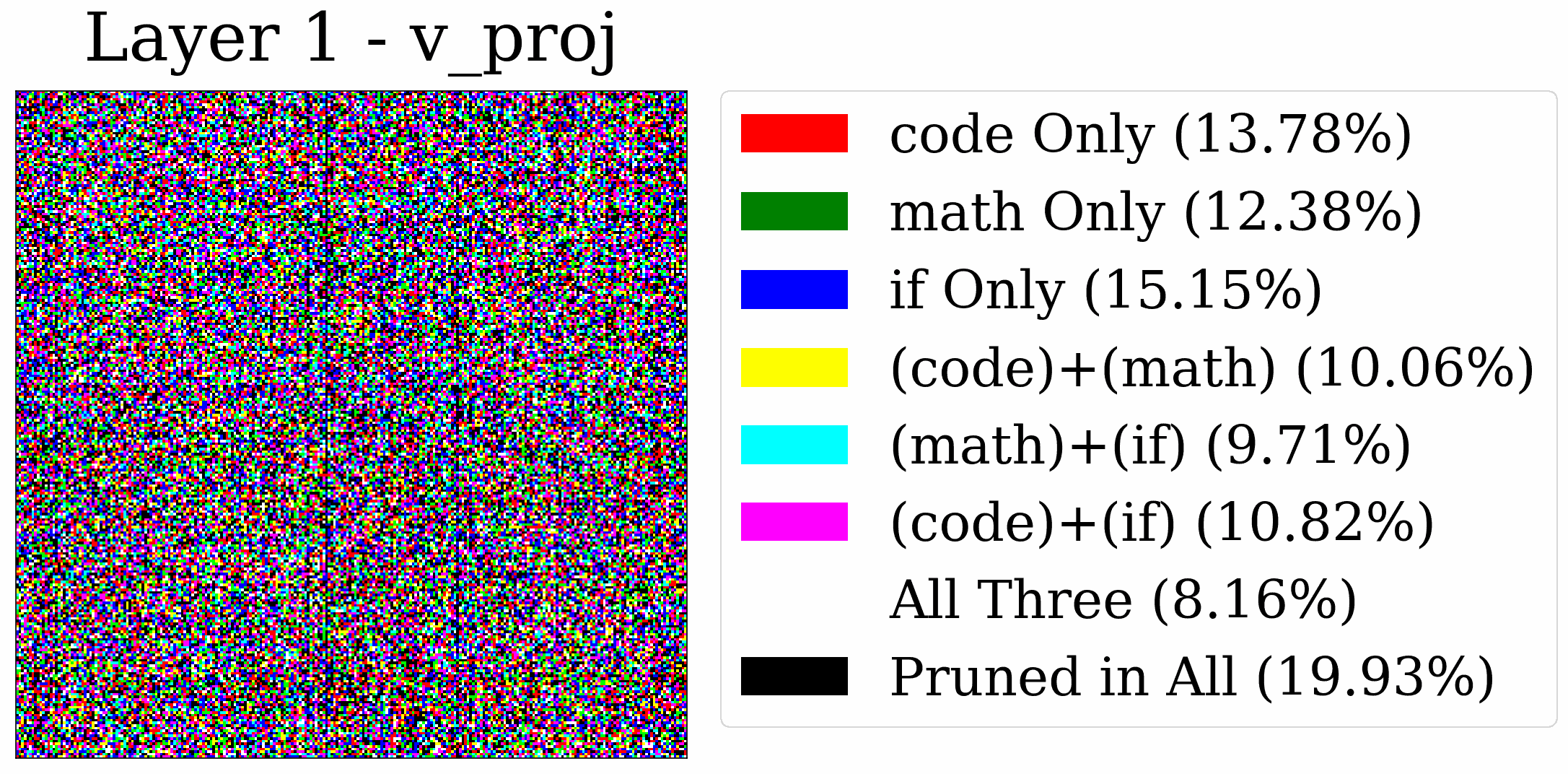}
    \end{subfigure}
    \hfill
    \begin{subfigure}[b]{0.24\textwidth}
        \includegraphics[width=\textwidth]{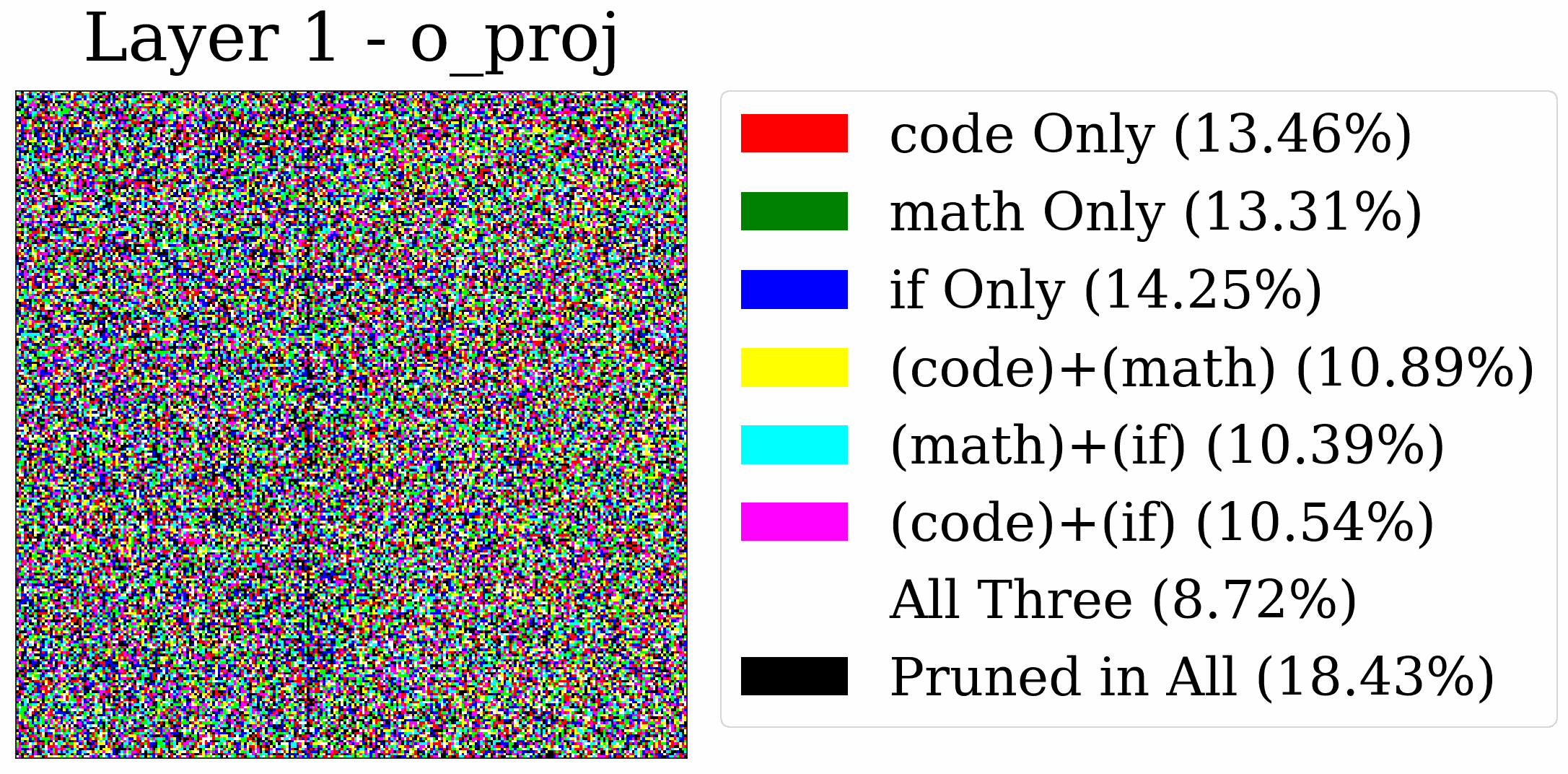}
    \end{subfigure}
    % Layer 2
    \begin{subfigure}[b]{0.24\textwidth}
        \includegraphics[width=\textwidth]{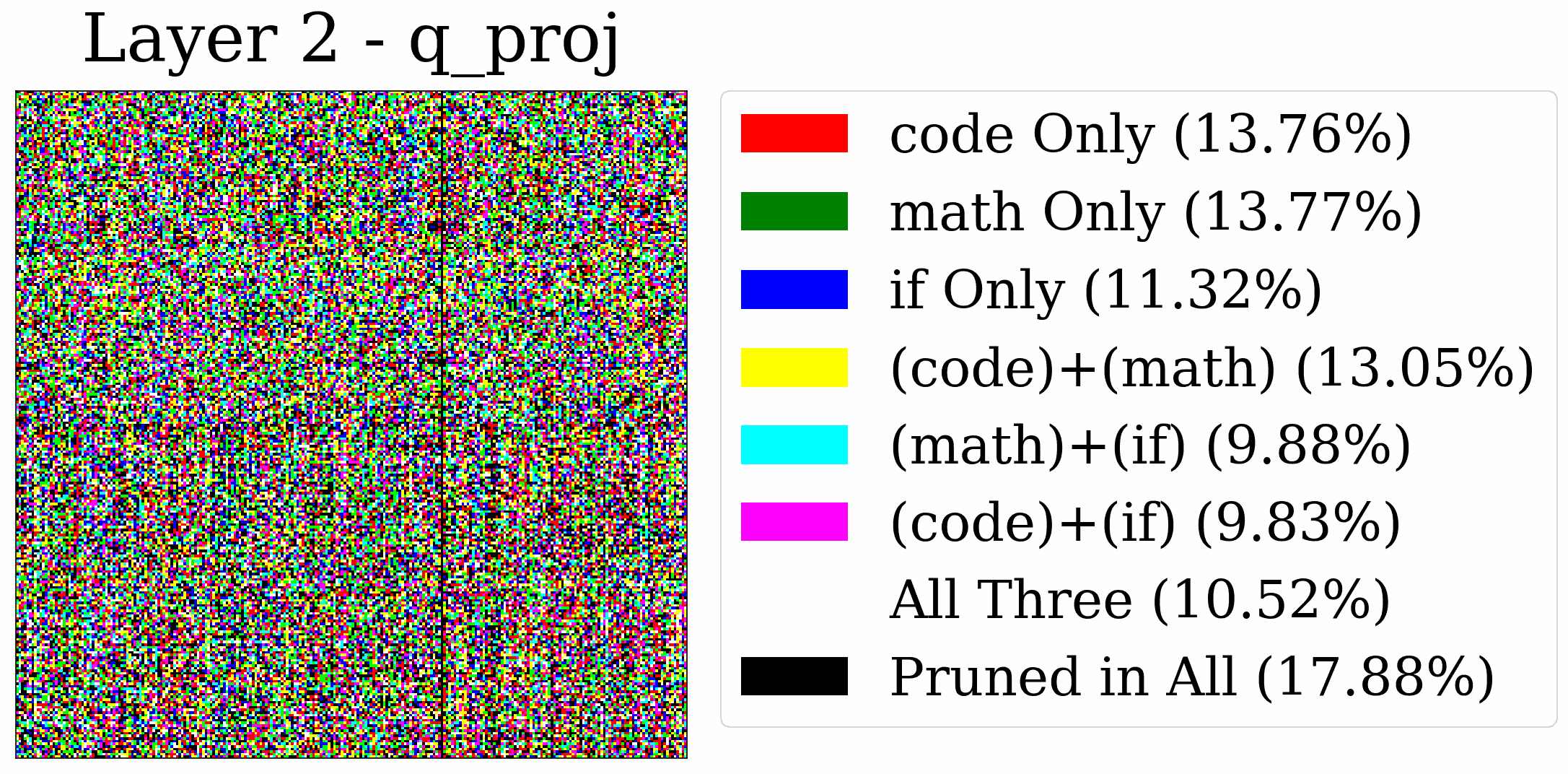}
    \end{subfigure}
    \hfill
    \begin{subfigure}[b]{0.24\textwidth}
        \includegraphics[width=\textwidth]{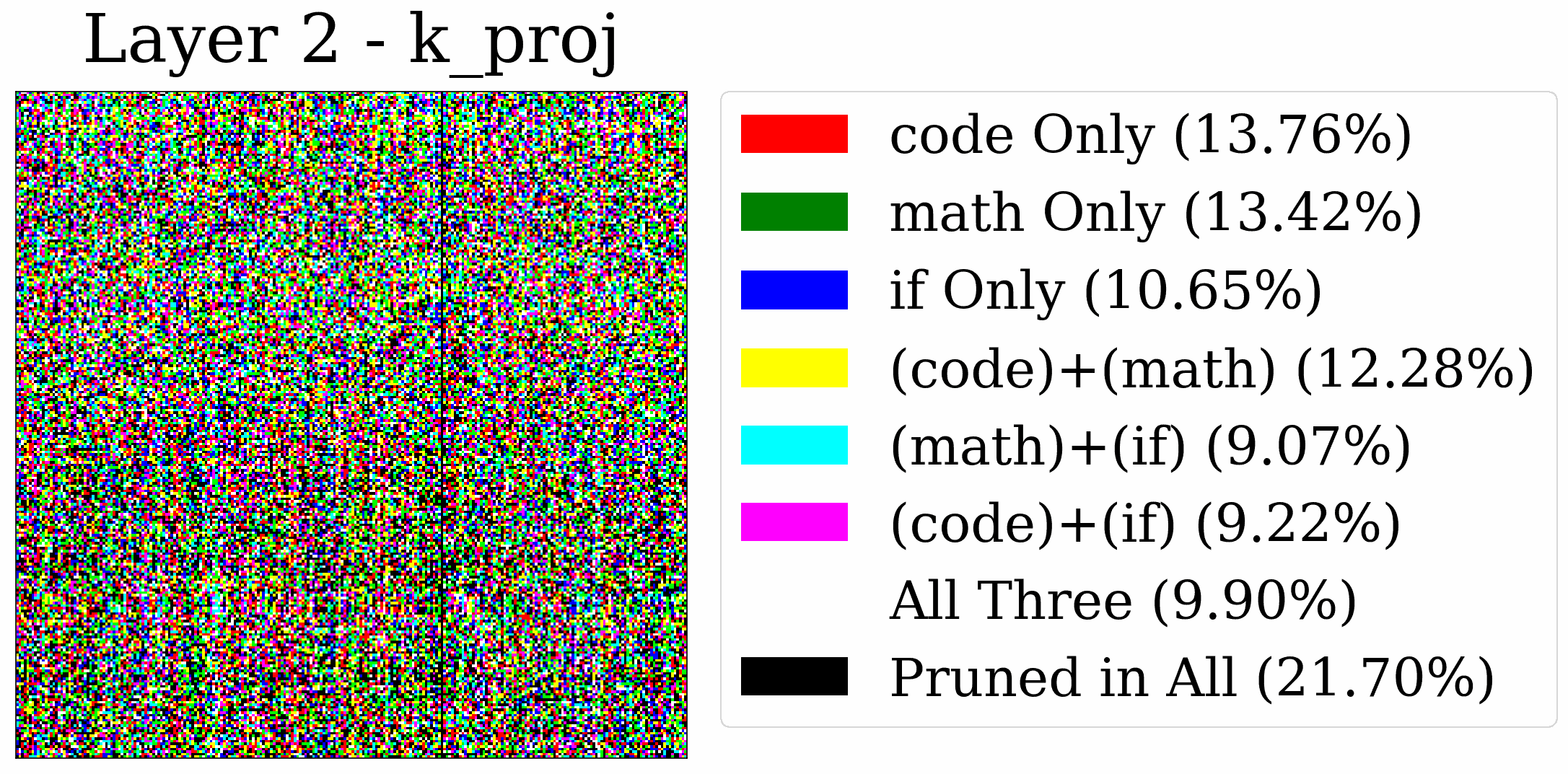}
    \end{subfigure}
    \hfill
    \begin{subfigure}[b]{0.24\textwidth}
        \includegraphics[width=\textwidth]{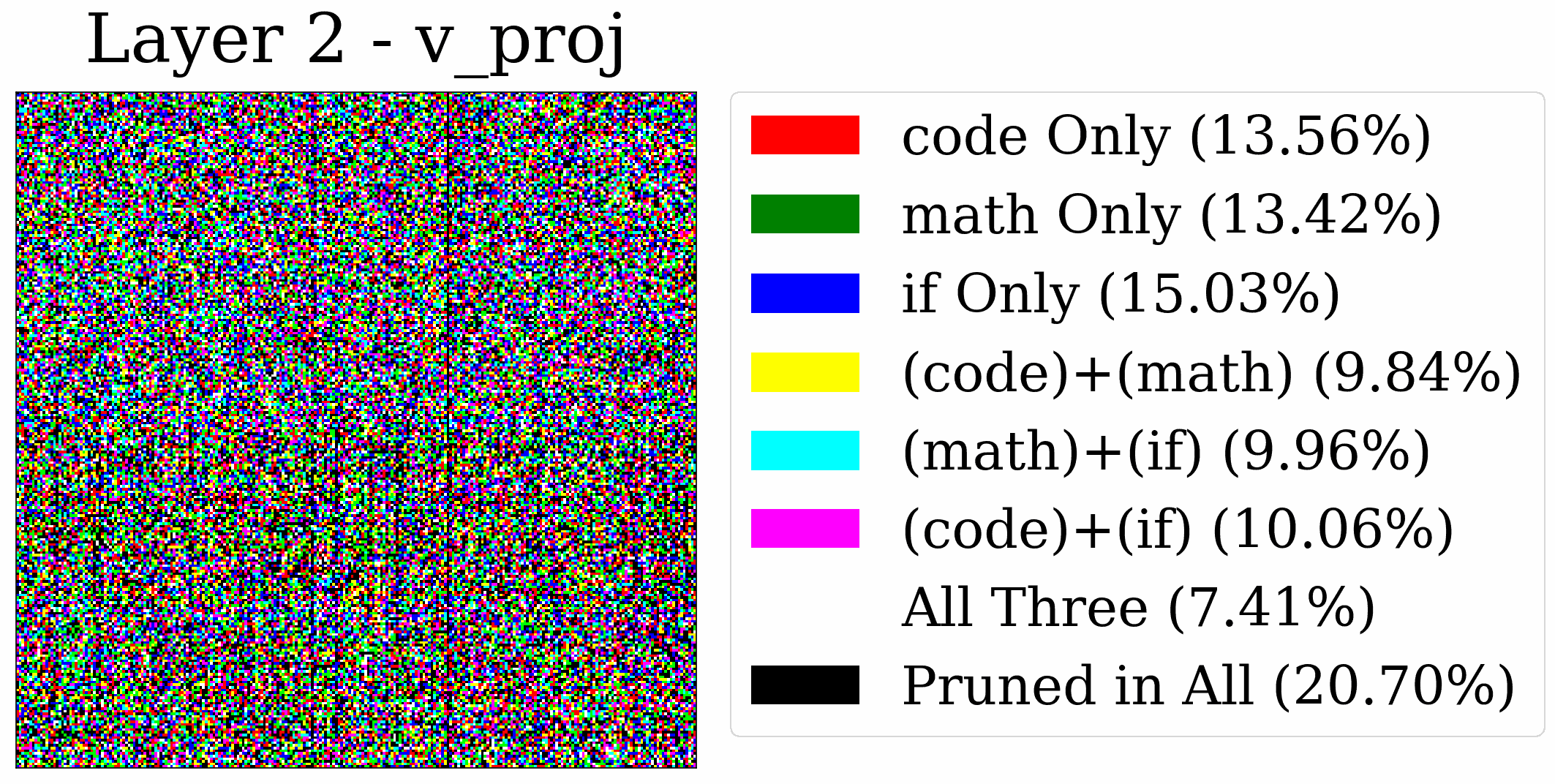}
    \end{subfigure}
    \hfill
    \begin{subfigure}[b]{0.24\textwidth}
        \includegraphics[width=\textwidth]{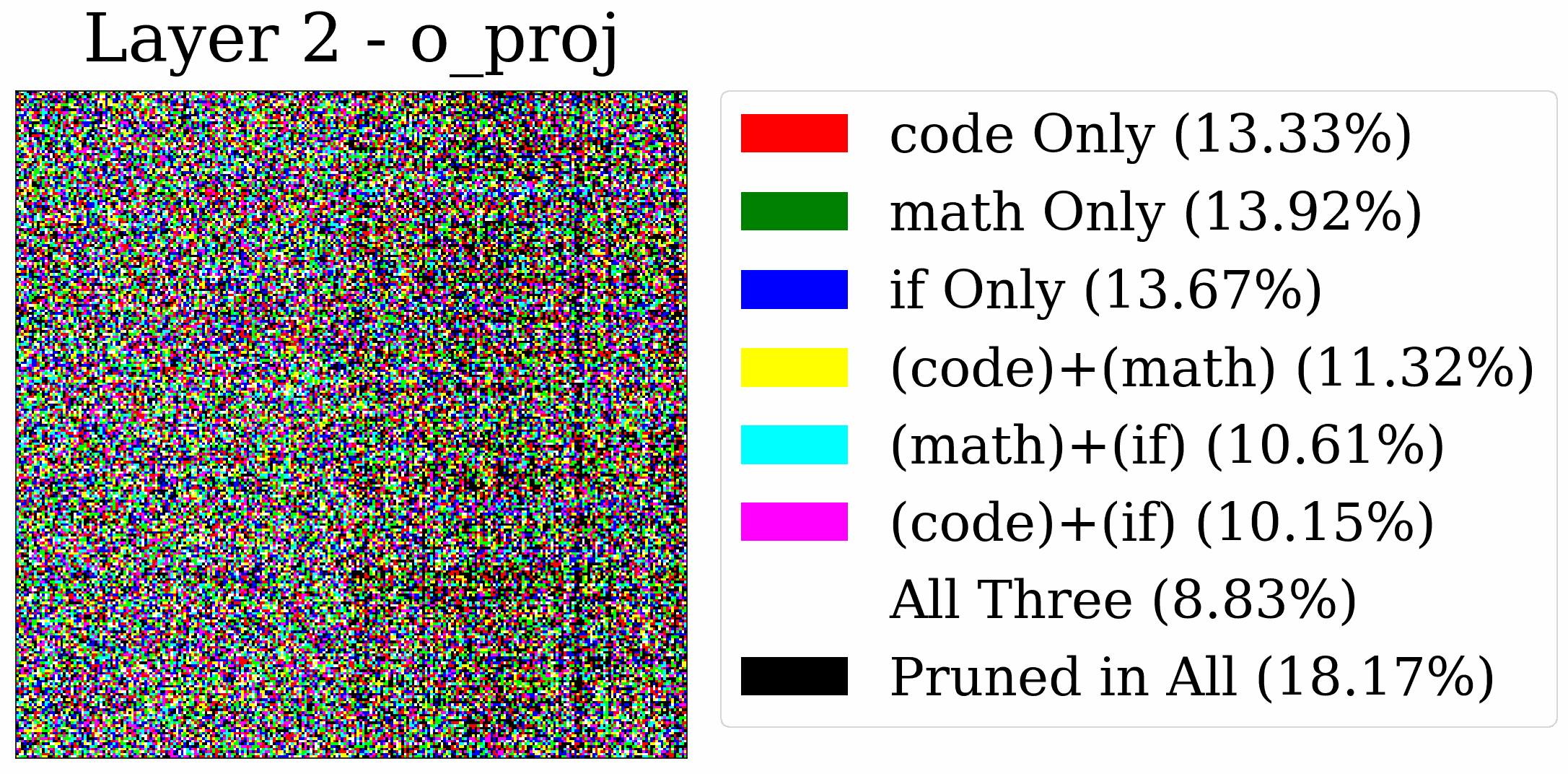}
    \end{subfigure}
    % Layer 15
    \begin{subfigure}[b]{0.24\textwidth}
        \includegraphics[width=\textwidth]{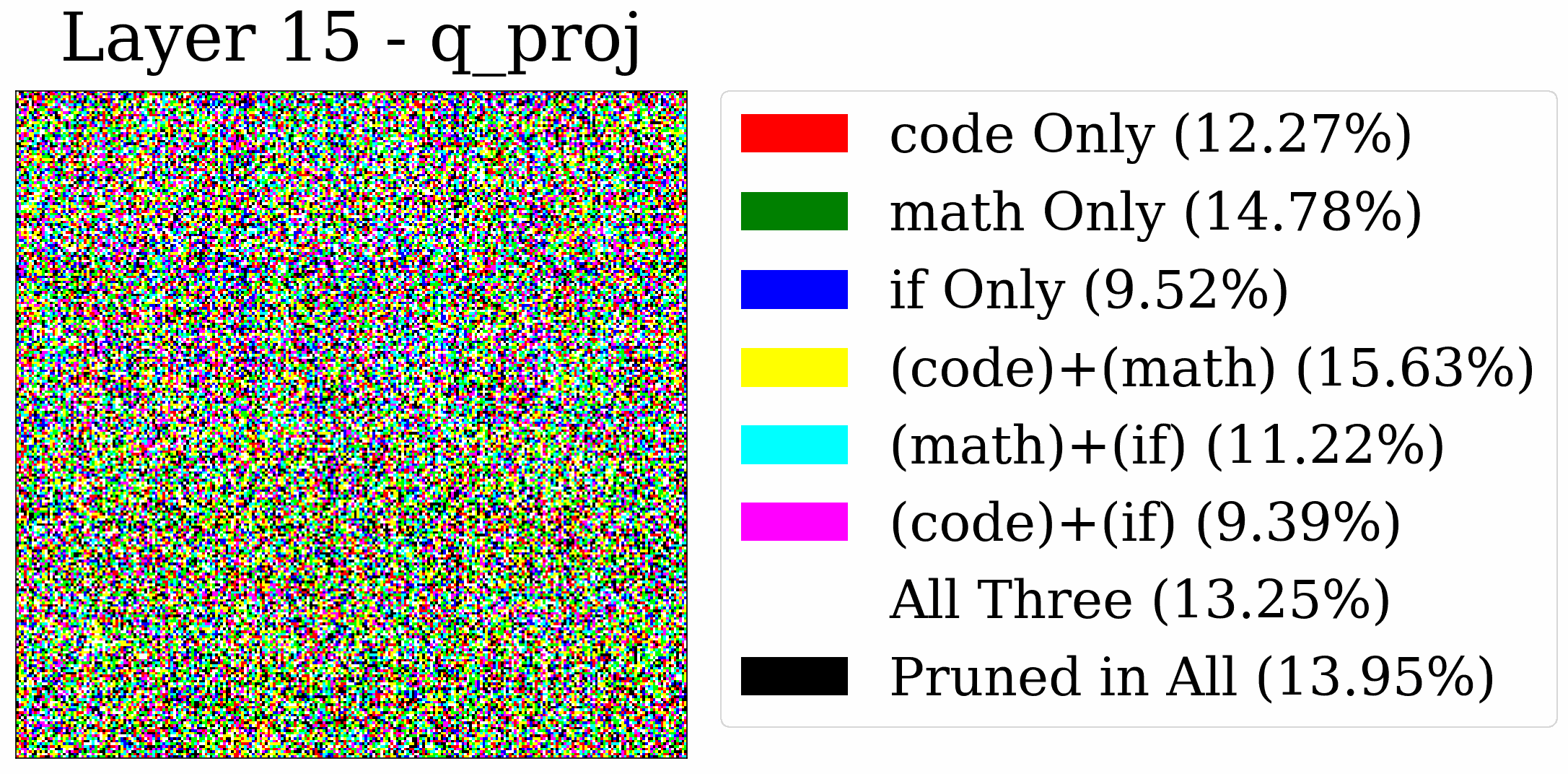}
    \end{subfigure}
    \hfill
    \begin{subfigure}[b]{0.24\textwidth}
        \includegraphics[width=\textwidth]{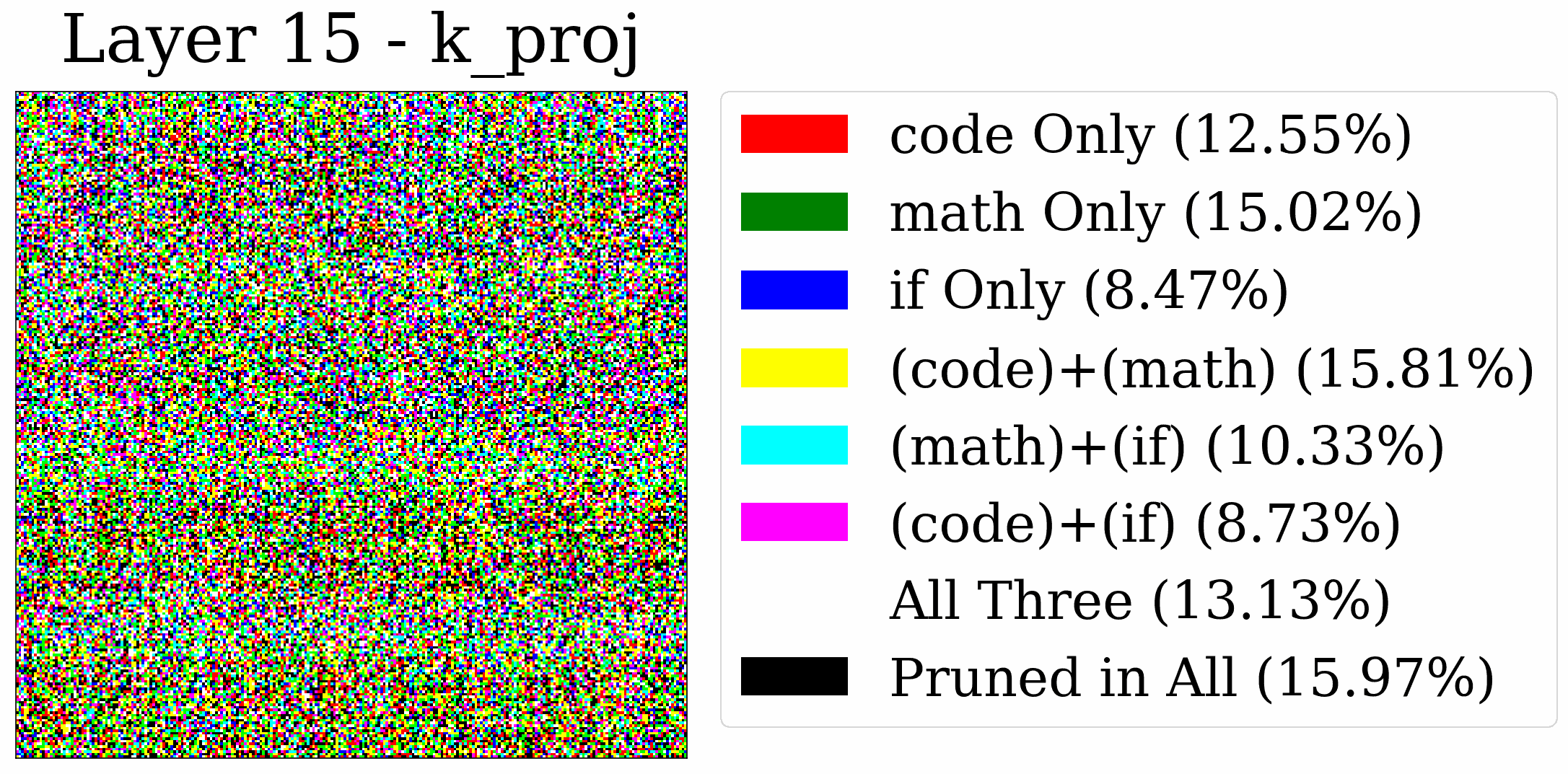}
    \end{subfigure}
    \hfill
    \begin{subfigure}[b]{0.24\textwidth}
        \includegraphics[width=\textwidth]{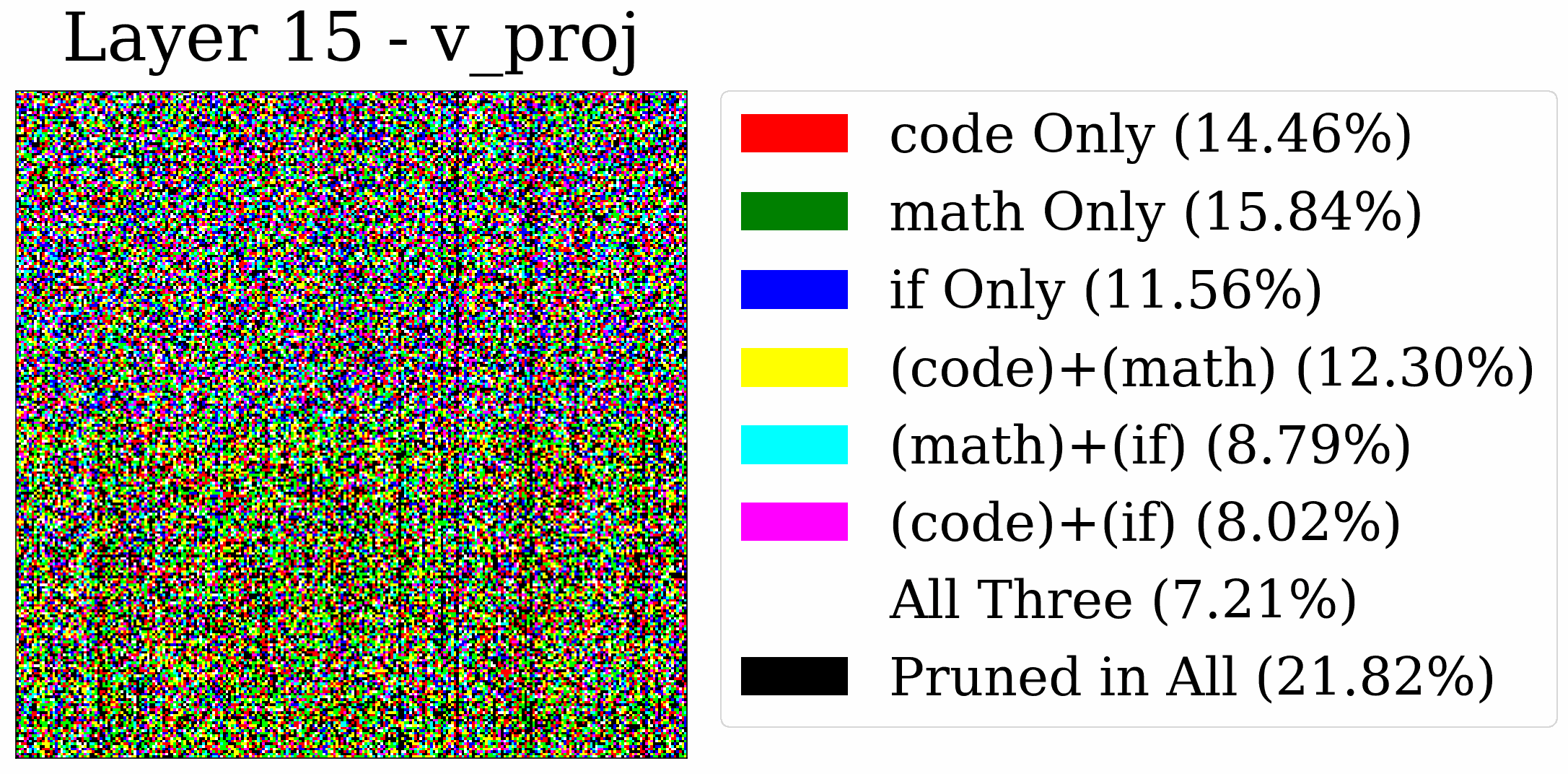}
    \end{subfigure}
    \hfill
    \begin{subfigure}[b]{0.24\textwidth}
        \includegraphics[width=\textwidth]{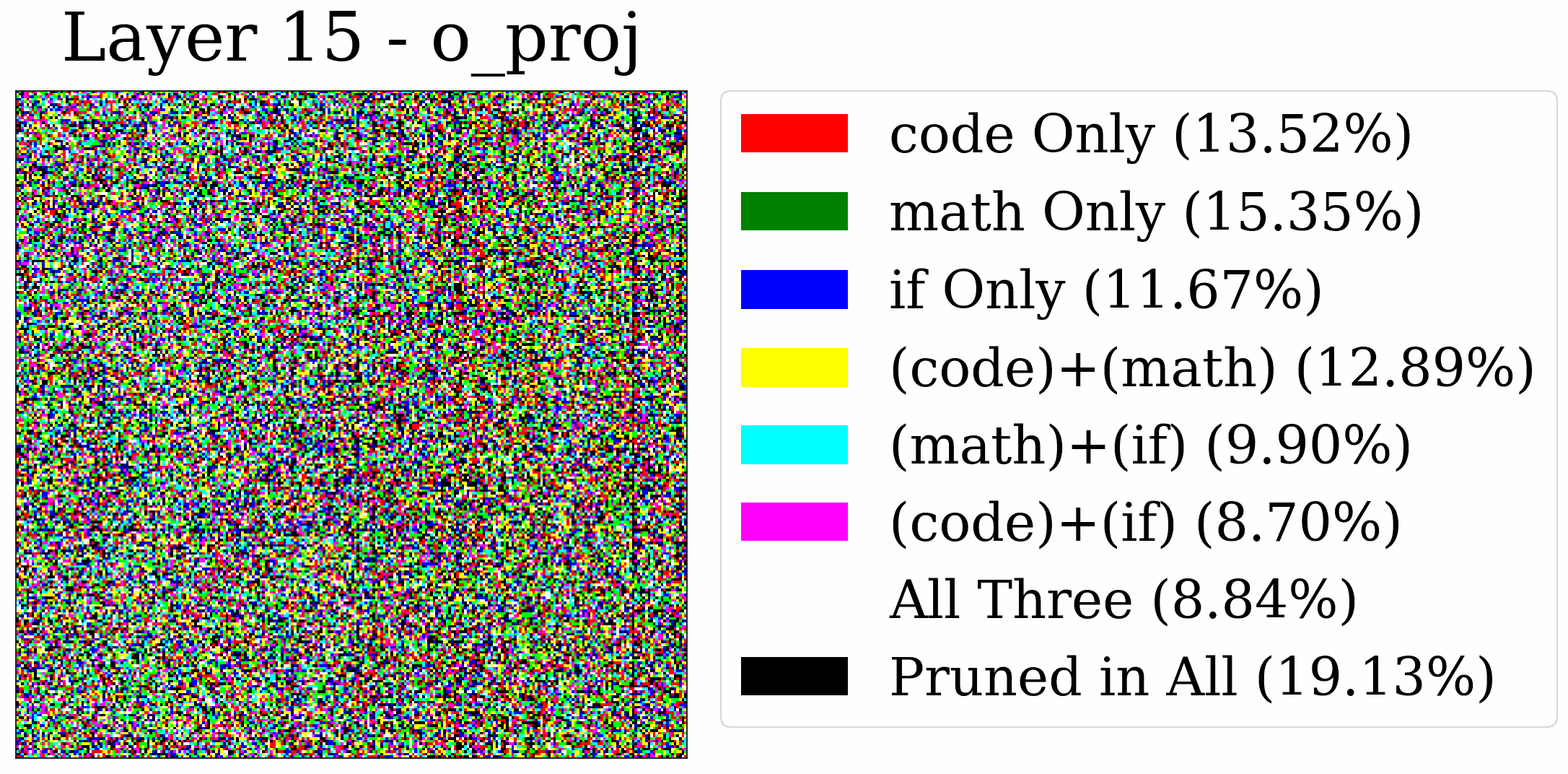}
    \end{subfigure}
    % Layer 18
    \begin{subfigure}[b]{0.24\textwidth}
        \includegraphics[width=\textwidth]{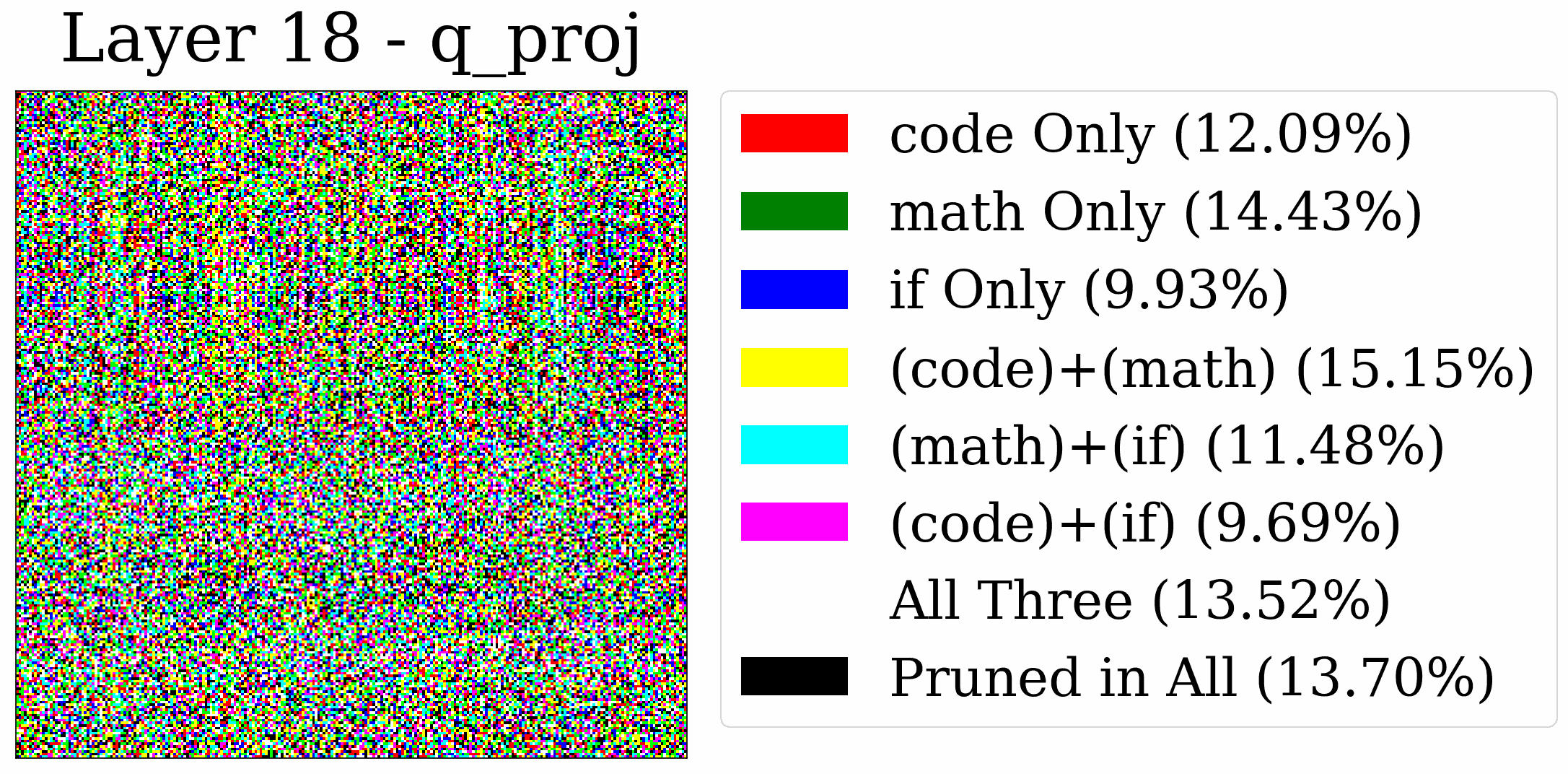}
    \end{subfigure}
    \hfill
    \begin{subfigure}[b]{0.24\textwidth}
        \includegraphics[width=\textwidth]{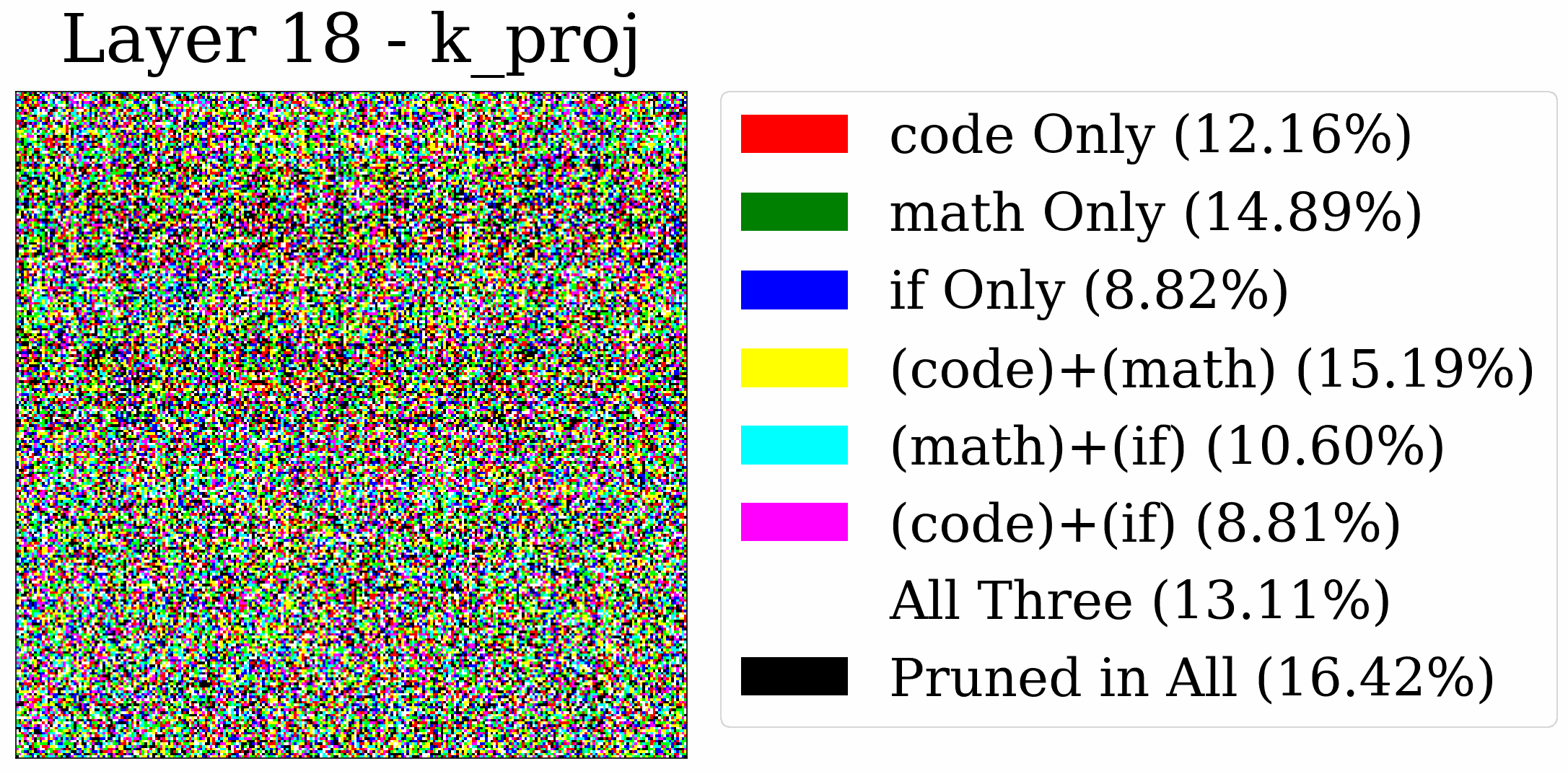}
    \end{subfigure}
    \hfill
    \begin{subfigure}[b]{0.24\textwidth}
        \includegraphics[width=\textwidth]{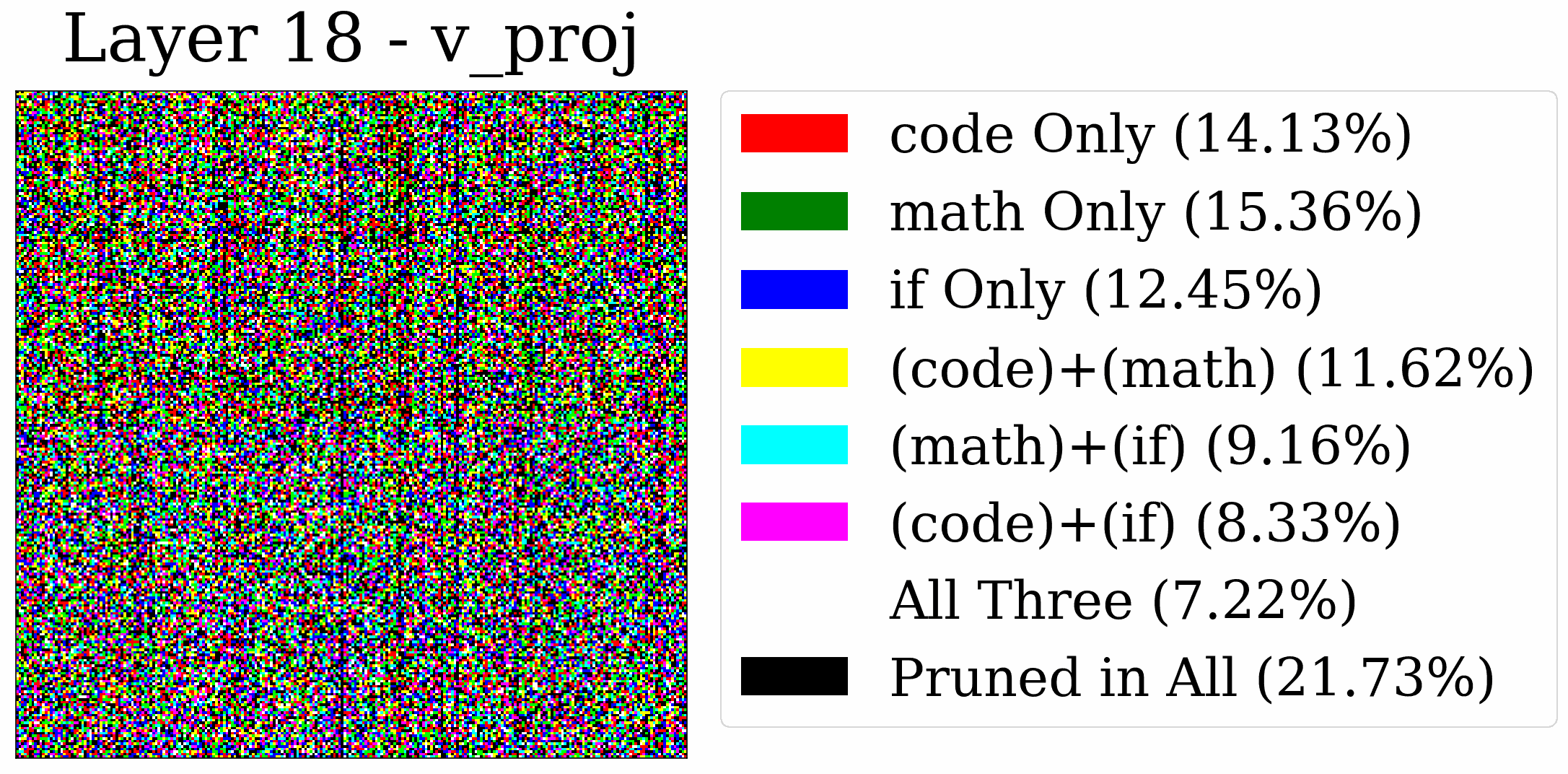}
    \end{subfigure}
    \hfill
    \begin{subfigure}[b]{0.24\textwidth}
        \includegraphics[width=\textwidth]{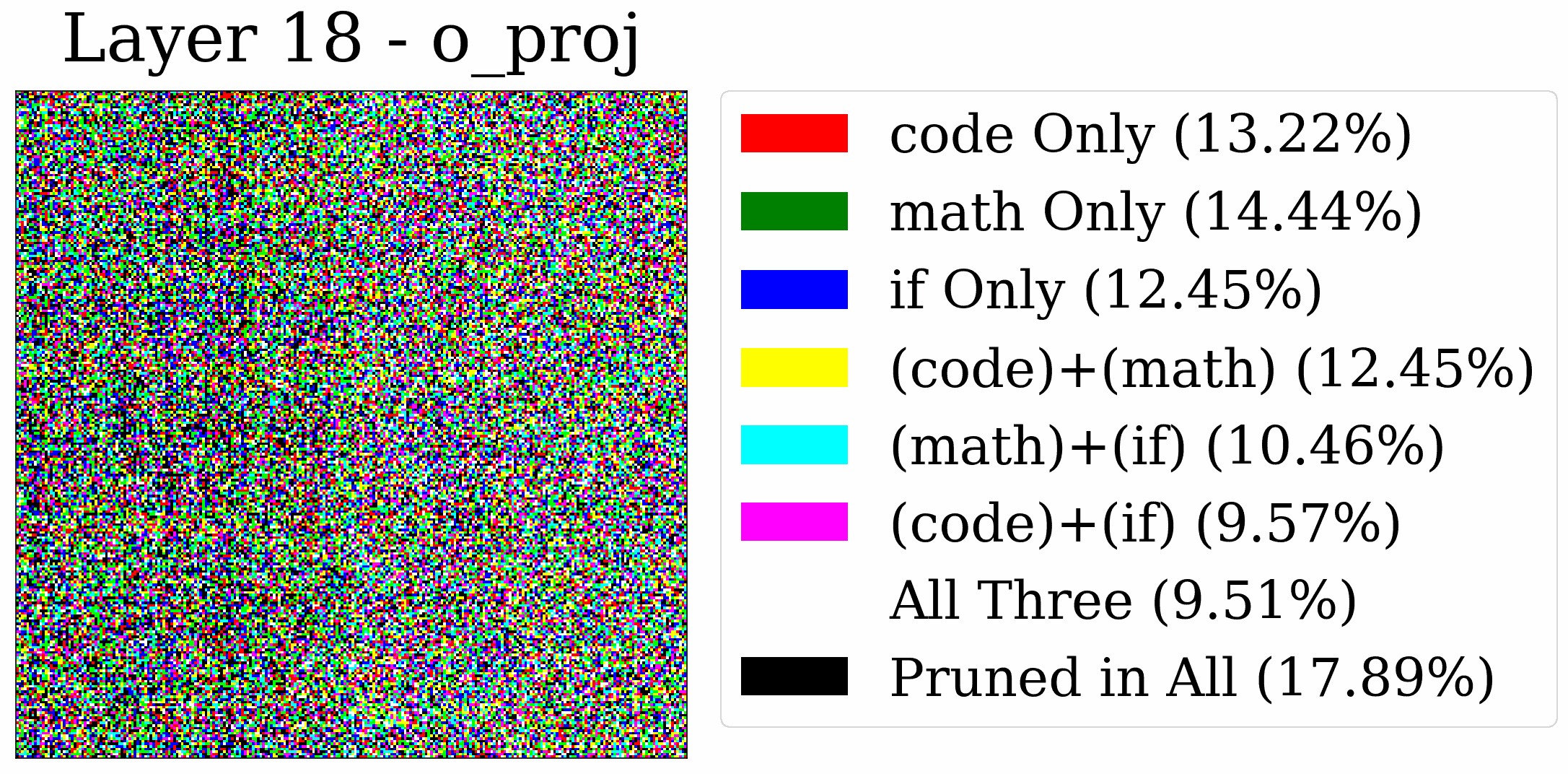}
    \end{subfigure}
    % Layer 29
    \begin{subfigure}[b]{0.24\textwidth}
        \includegraphics[width=\textwidth]{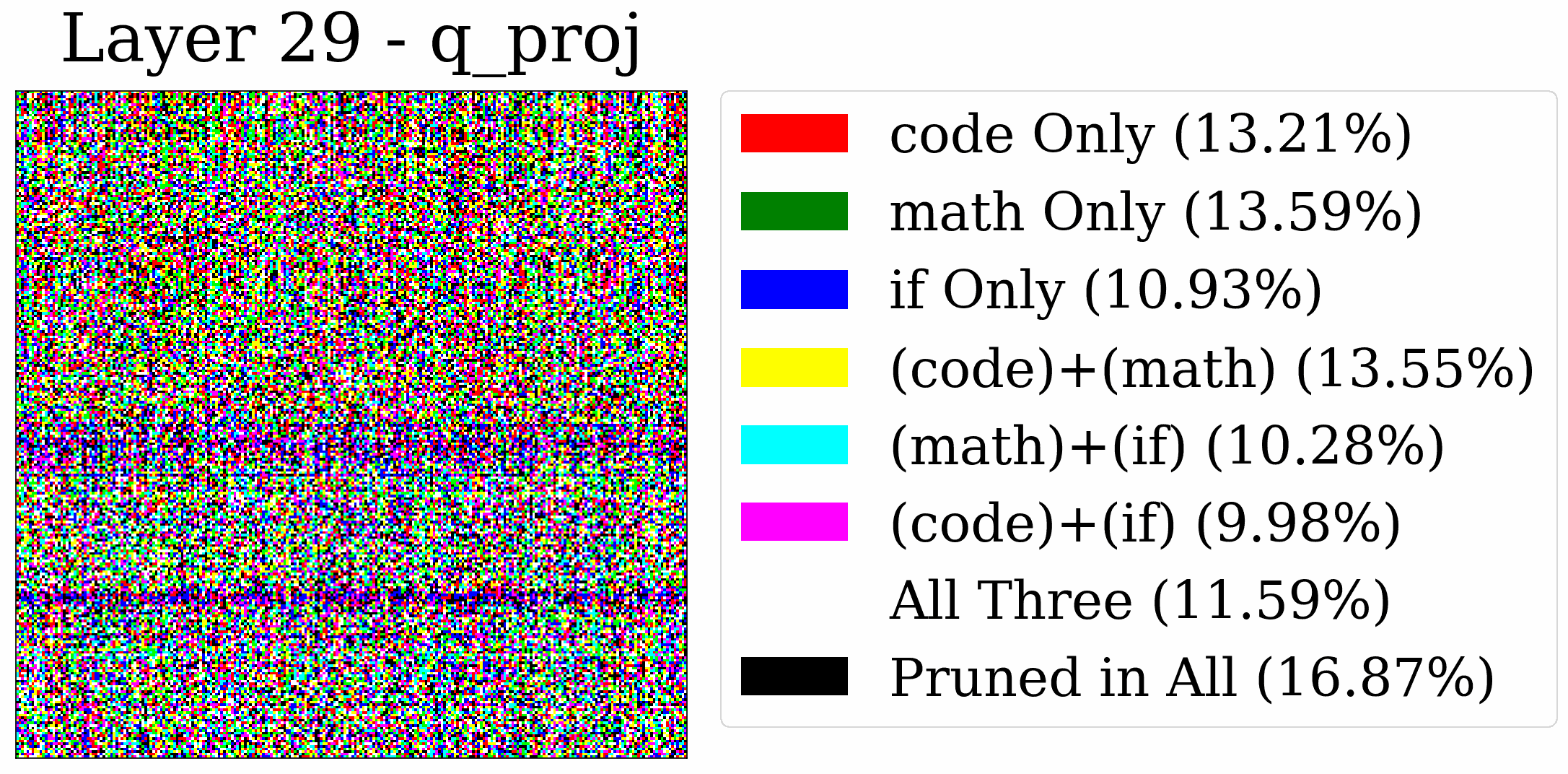}
    \end{subfigure}
    \hfill
    \begin{subfigure}[b]{0.24\textwidth}
        \includegraphics[width=\textwidth]{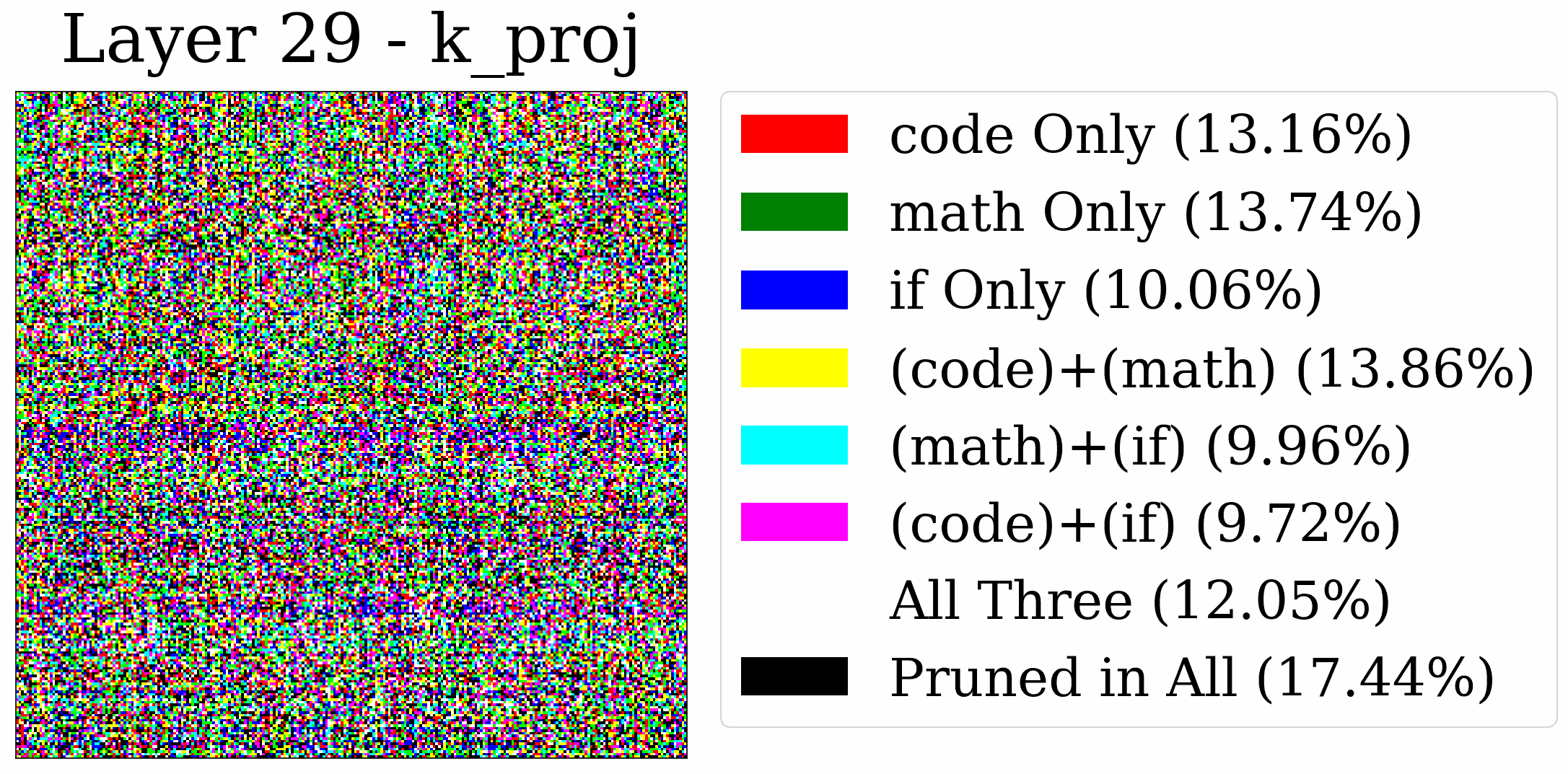}
    \end{subfigure}
    \hfill
    \begin{subfigure}[b]{0.24\textwidth}
        \includegraphics[width=\textwidth]{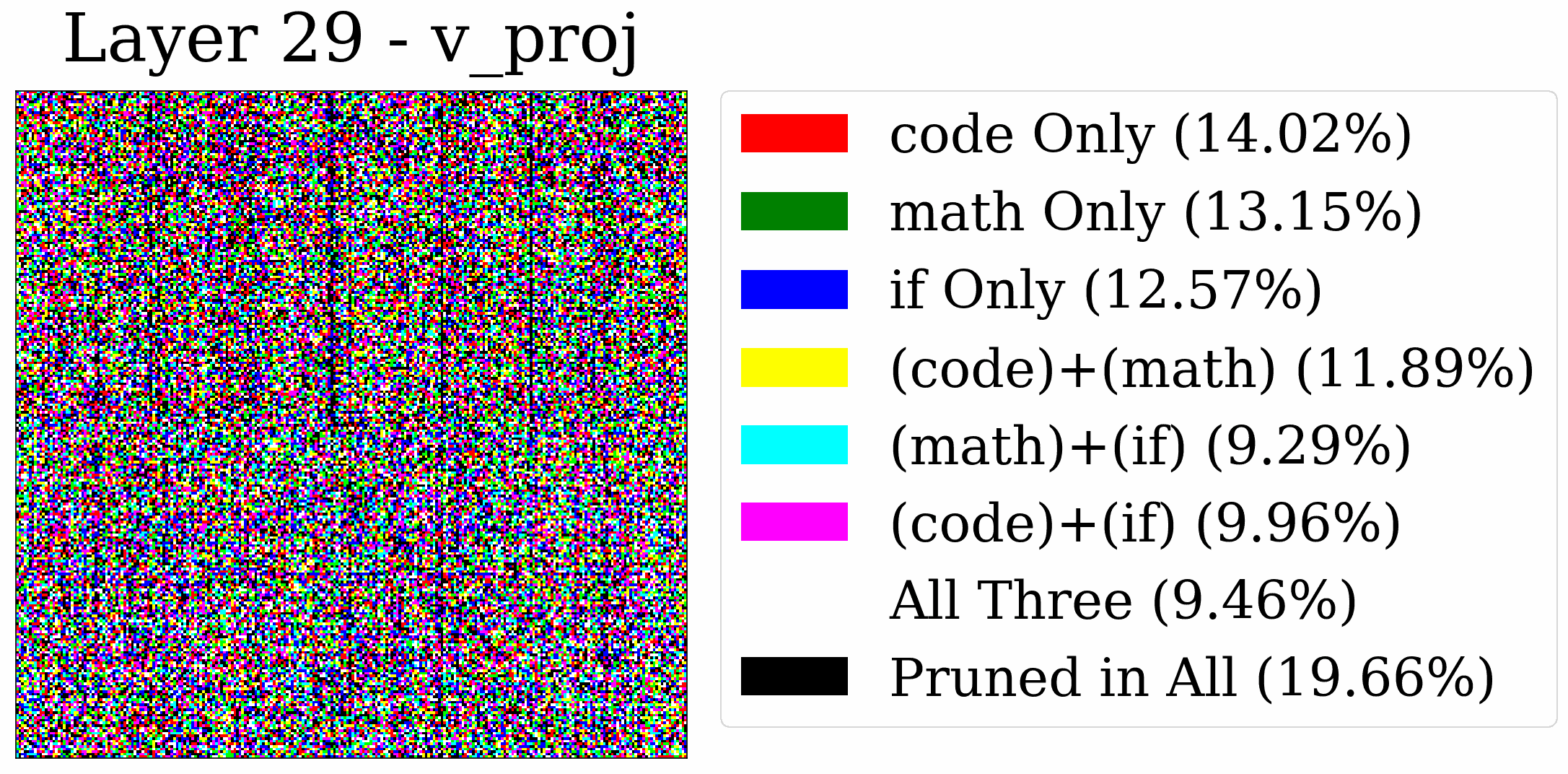}
    \end{subfigure}
    \hfill
    \begin{subfigure}[b]{0.24\textwidth}
        \includegraphics[width=\textwidth]{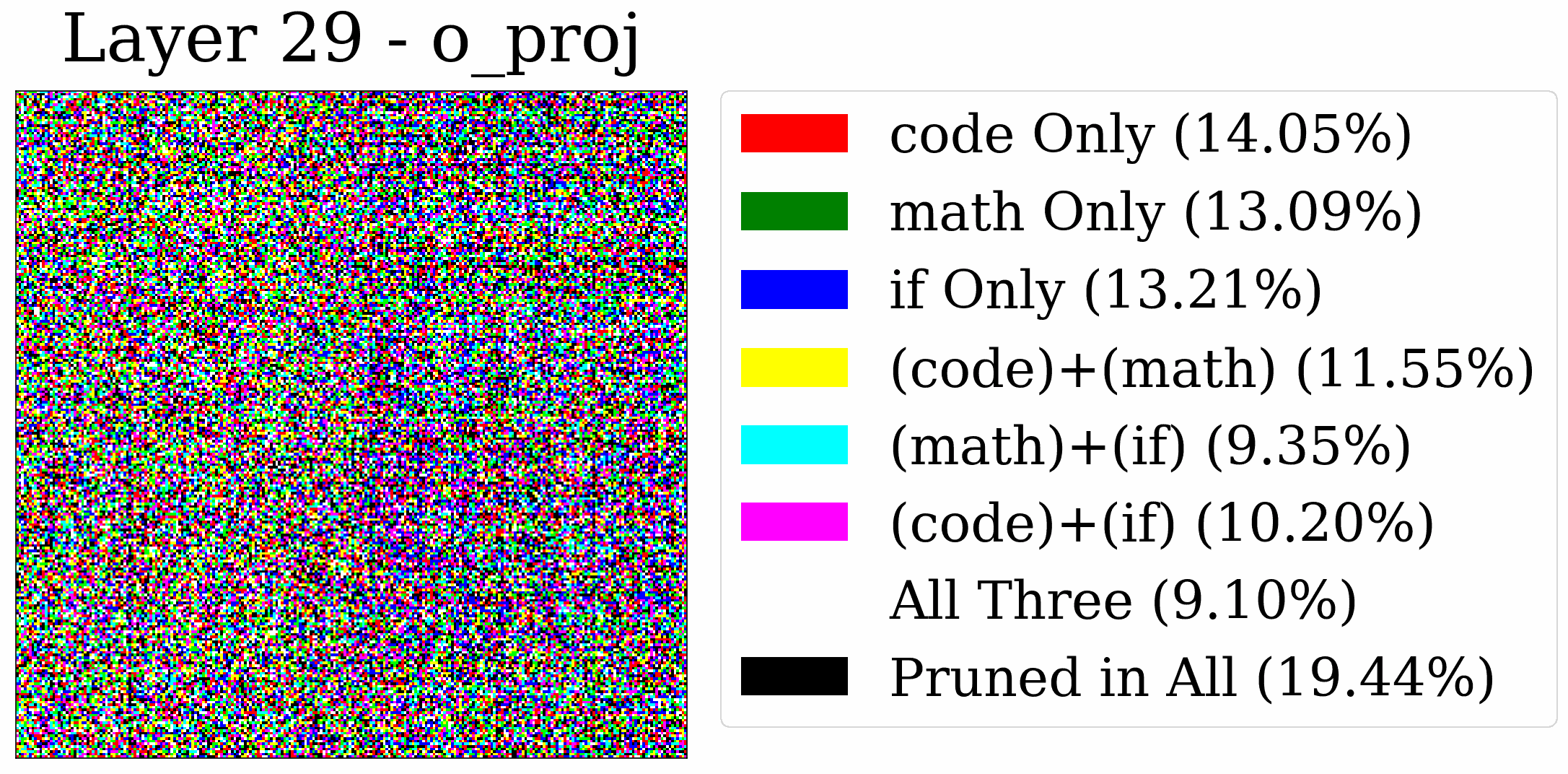}
    \end{subfigure}
    % Layer 30
    \begin{subfigure}[b]{0.24\textwidth}
        \includegraphics[width=\textwidth]{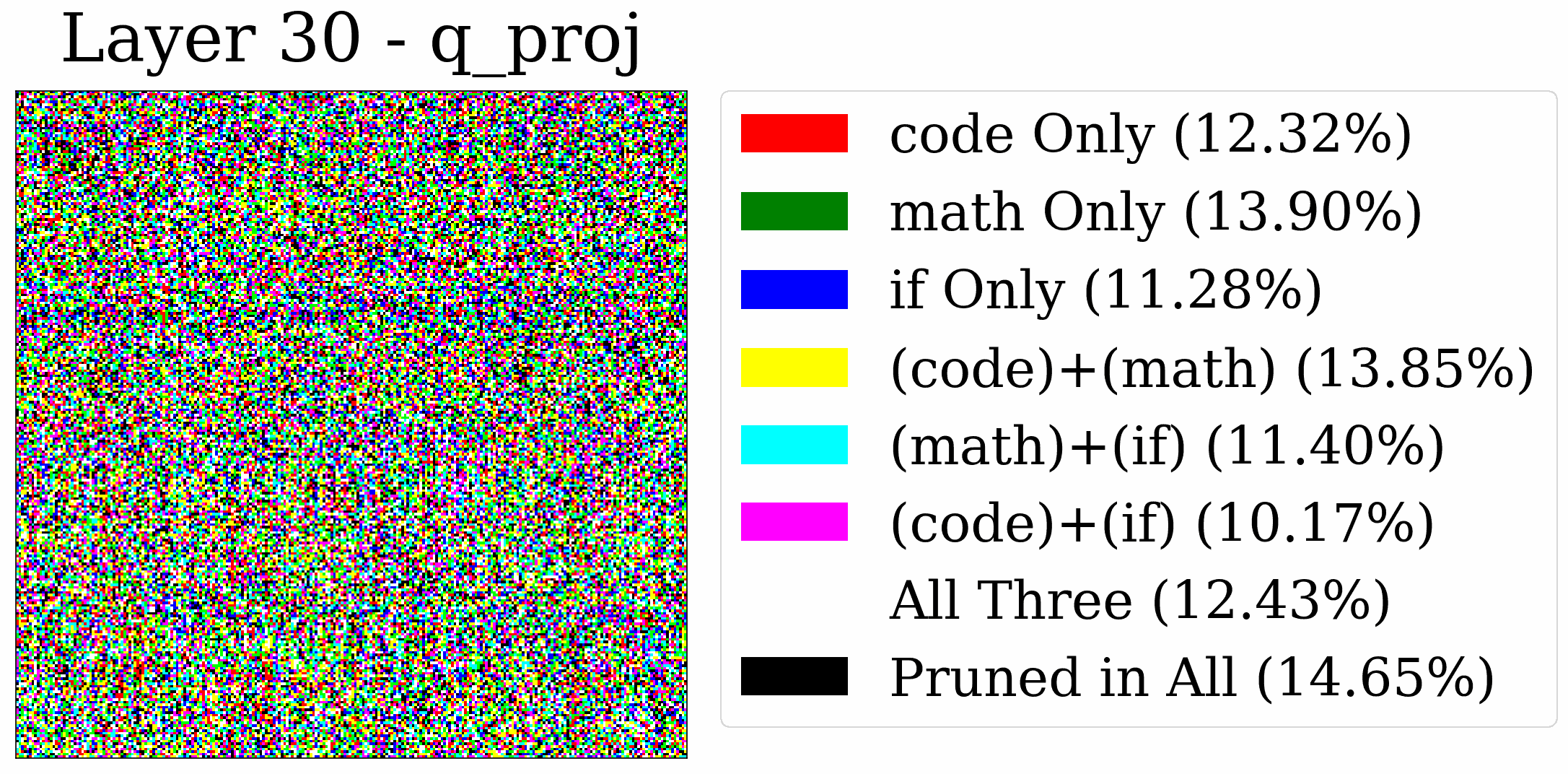}
    \end{subfigure}
    \hfill
    \begin{subfigure}[b]{0.24\textwidth}
        \includegraphics[width=\textwidth]{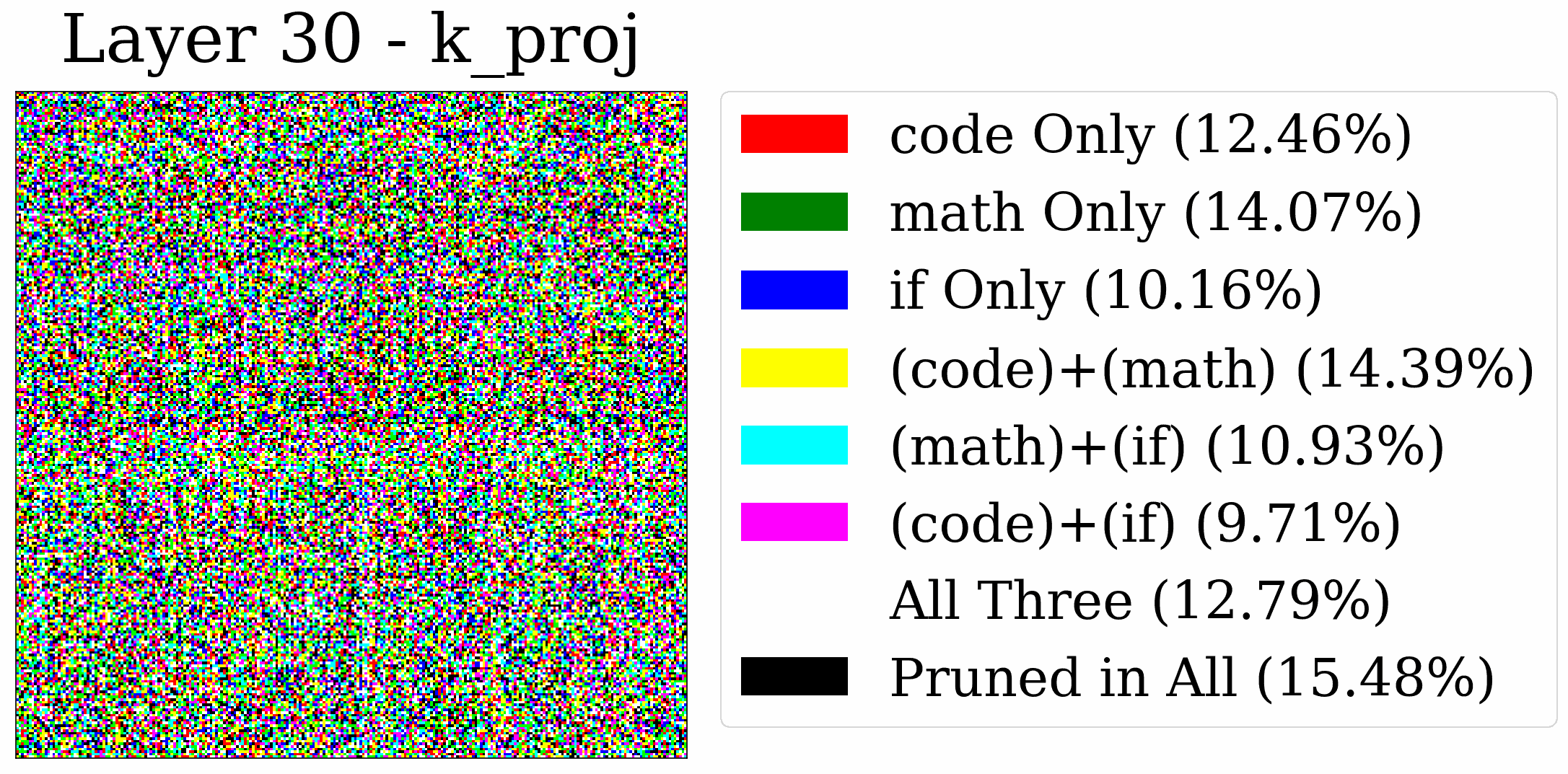}
    \end{subfigure}
    \hfill
    \begin{subfigure}[b]{0.24\textwidth}
        \includegraphics[width=\textwidth]{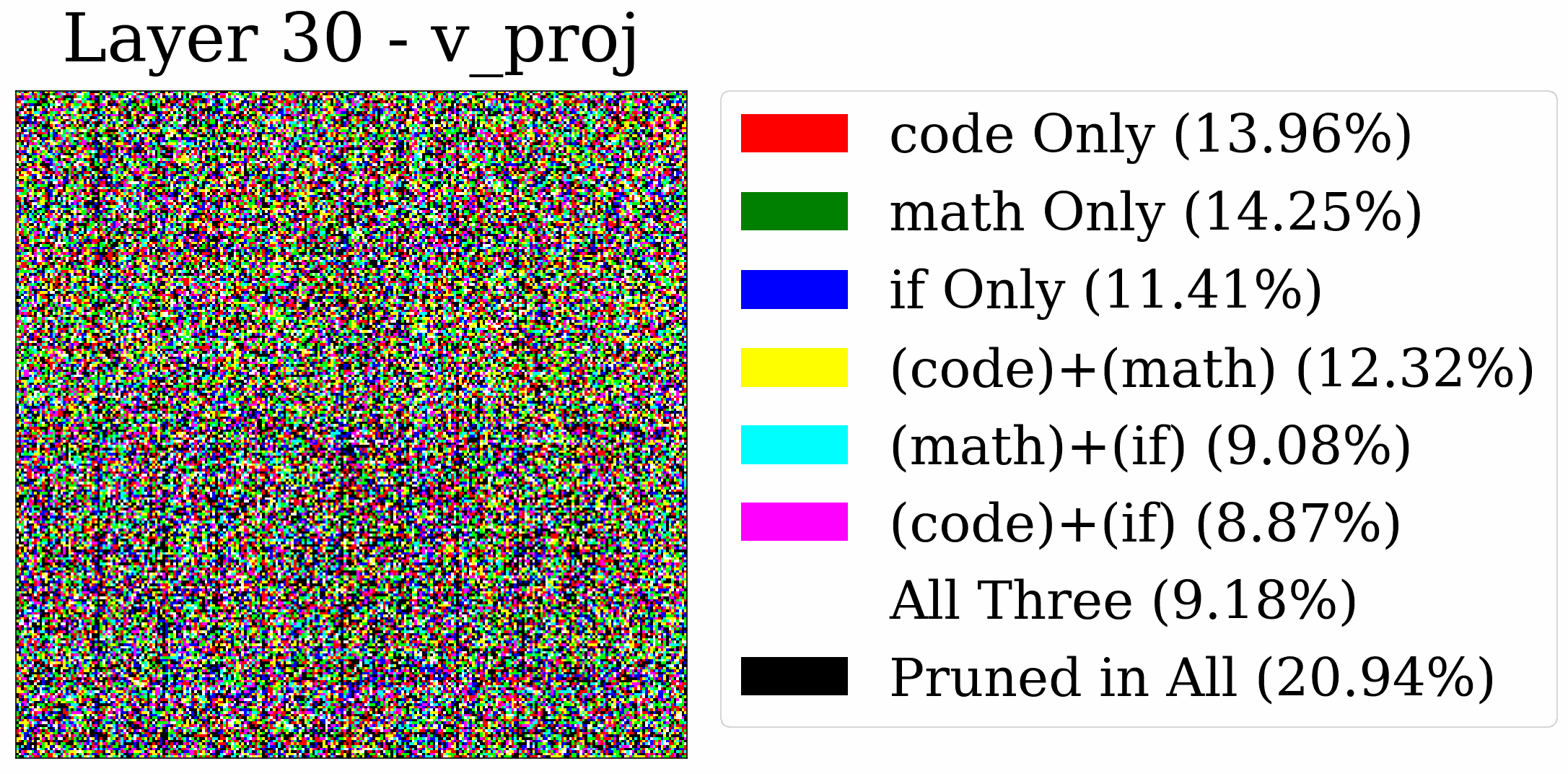}
    \end{subfigure}
    \hfill
    \begin{subfigure}[b]{0.24\textwidth}
        \includegraphics[width=\textwidth]{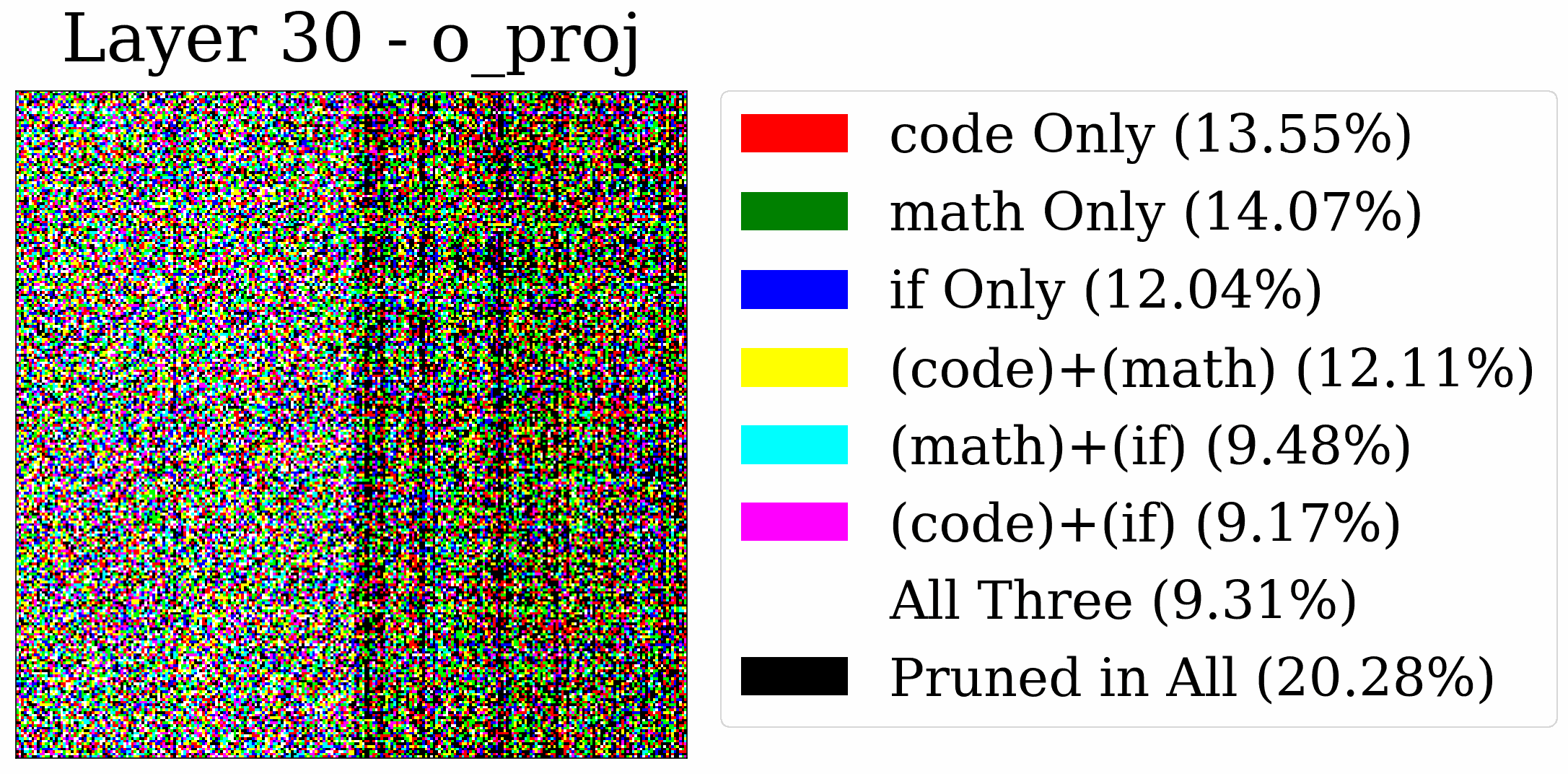}
    \end{subfigure}
    \caption{3-way magnitude-based comparison for attention components across layers 1, 2, 15, 18, 29, and 30. Columns show $W_q$, $W_k$, $W_v$, and $W_o$.}
    \label{fig:3way_mag_attention_layers}
\end{figure}

\begin{figure}[!ht]
    \centering
    % Layer 1
    \begin{subfigure}[b]{0.32\textwidth}
        \includegraphics[width=\textwidth]{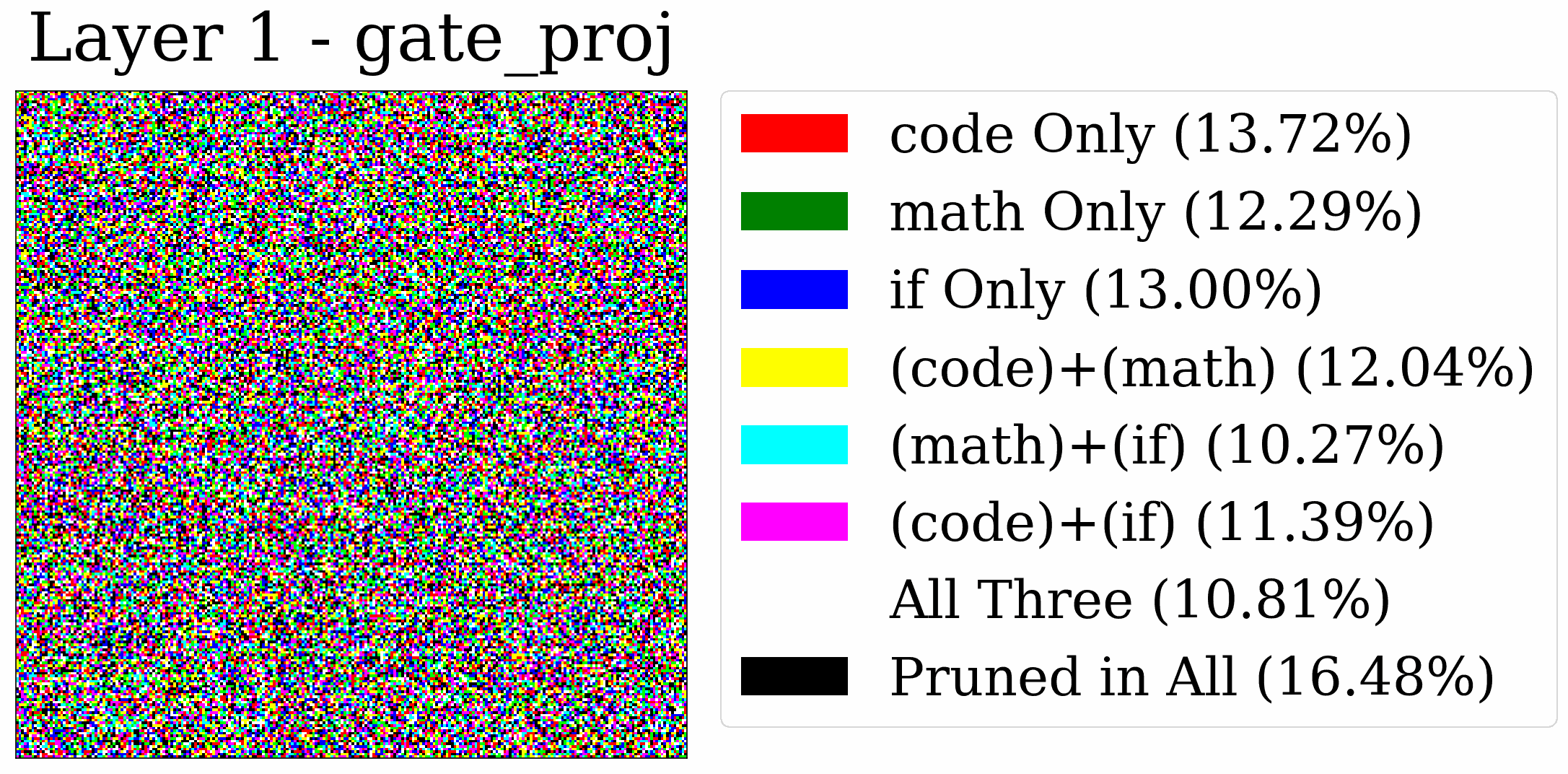}
    \end{subfigure}
    \hfill
    \begin{subfigure}[b]{0.32\textwidth}
        \includegraphics[width=\textwidth]{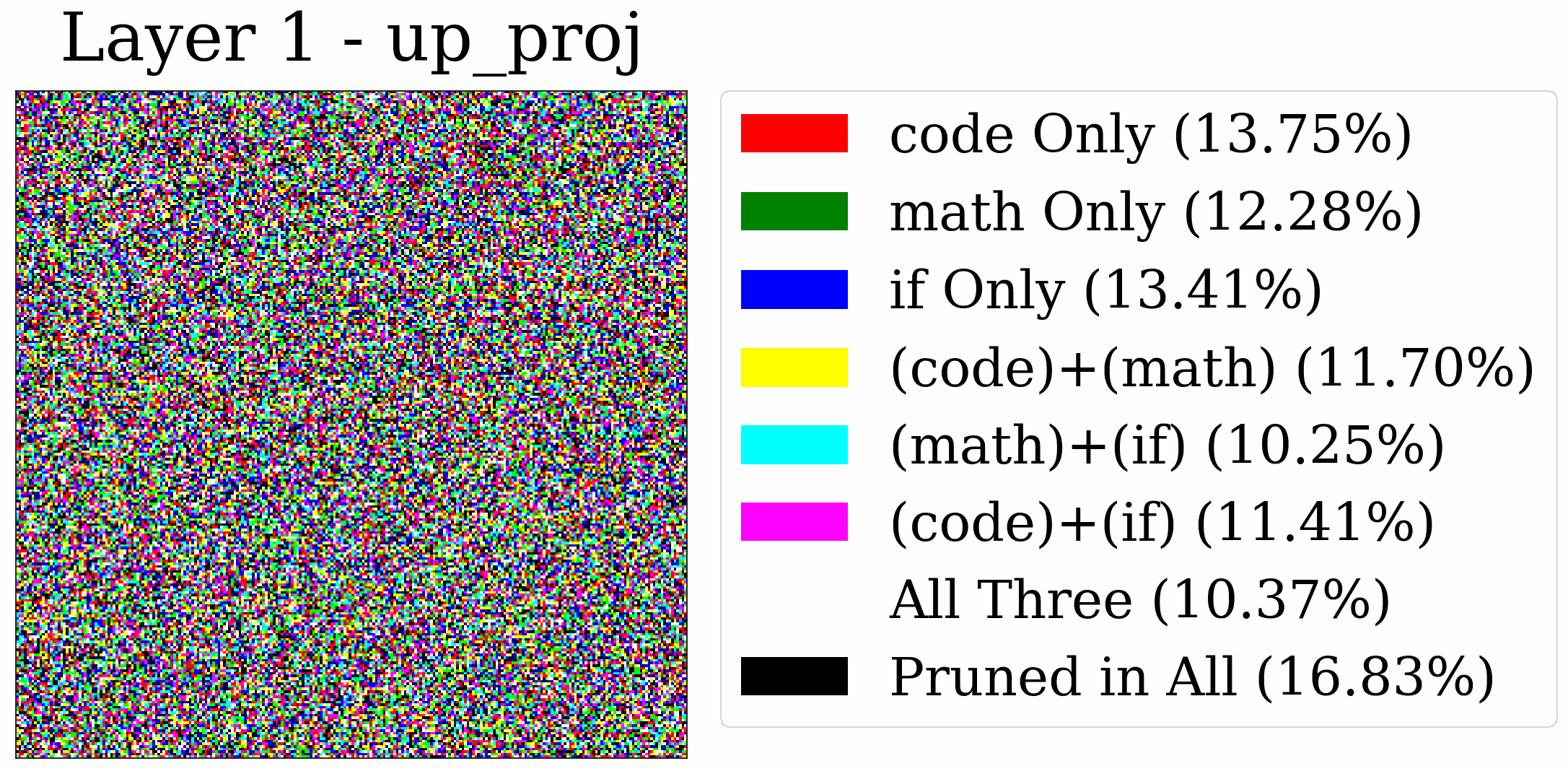}
    \end{subfigure}
    \hfill
    \begin{subfigure}[b]{0.32\textwidth}
        \includegraphics[width=\textwidth]{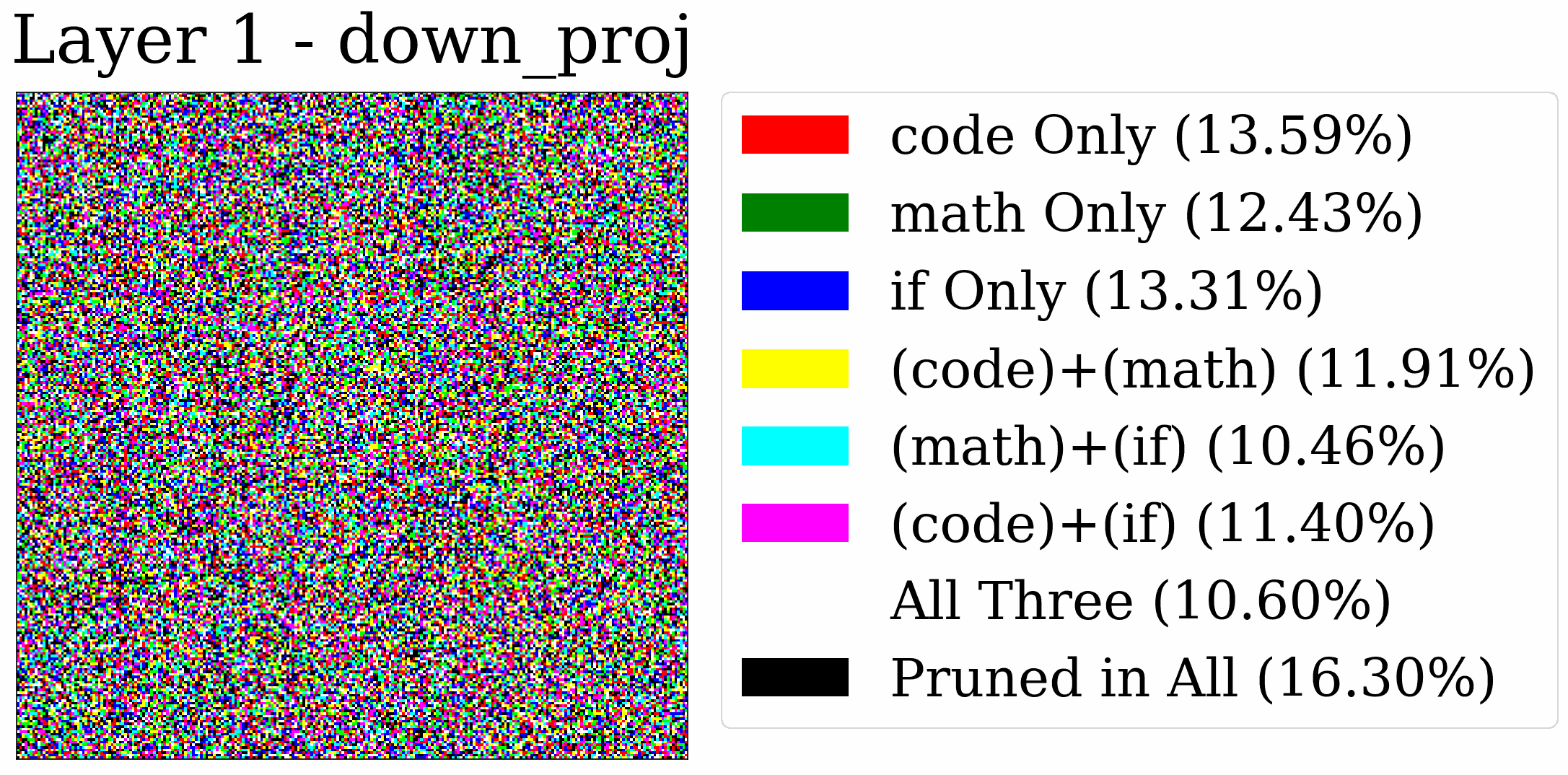}
    \end{subfigure}
    % Layer 2
    \begin{subfigure}[b]{0.32\textwidth}
        \includegraphics[width=\textwidth]{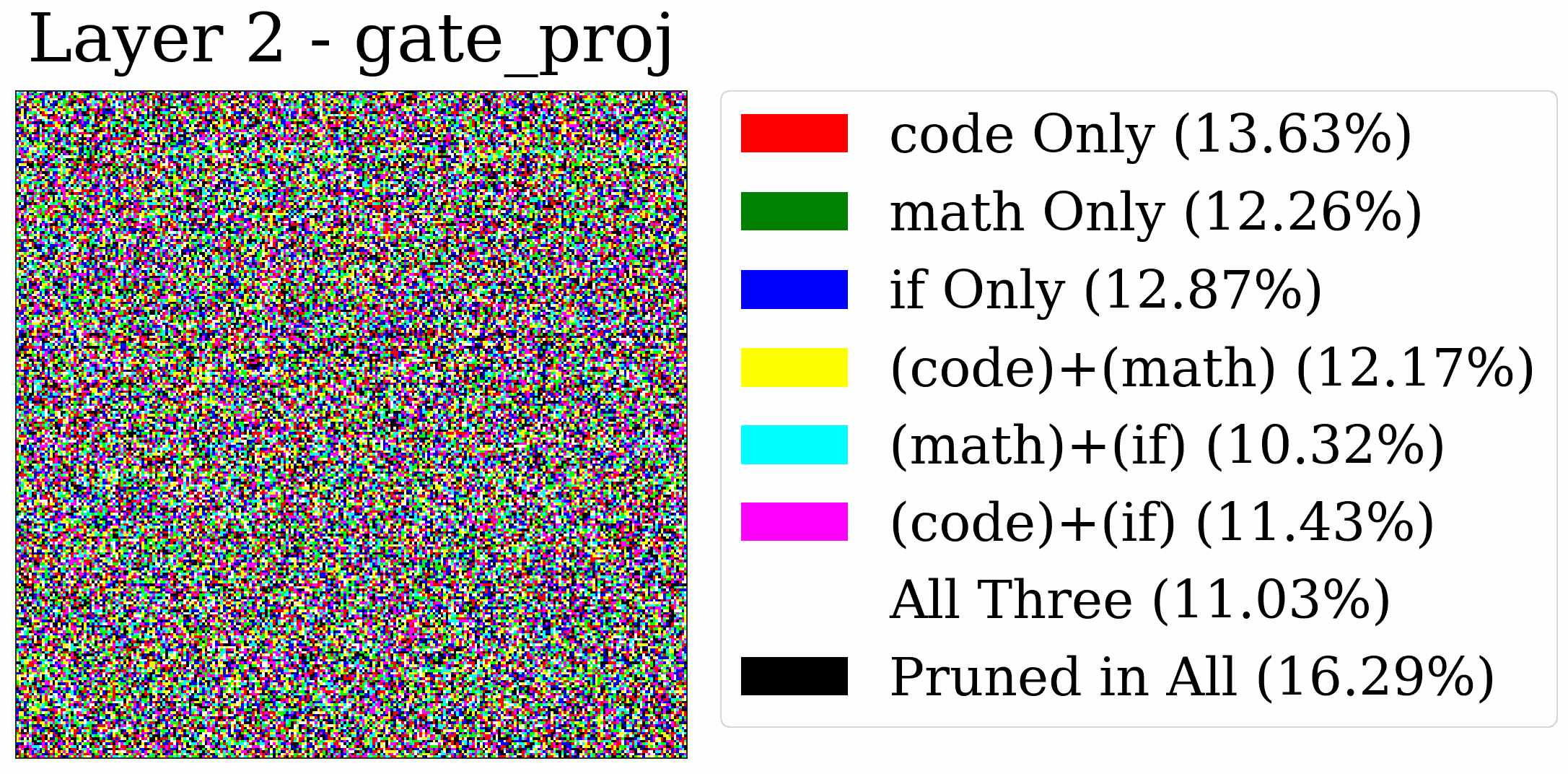}
    \end{subfigure}
    \hfill
    \begin{subfigure}[b]{0.32\textwidth}
        \includegraphics[width=\textwidth]{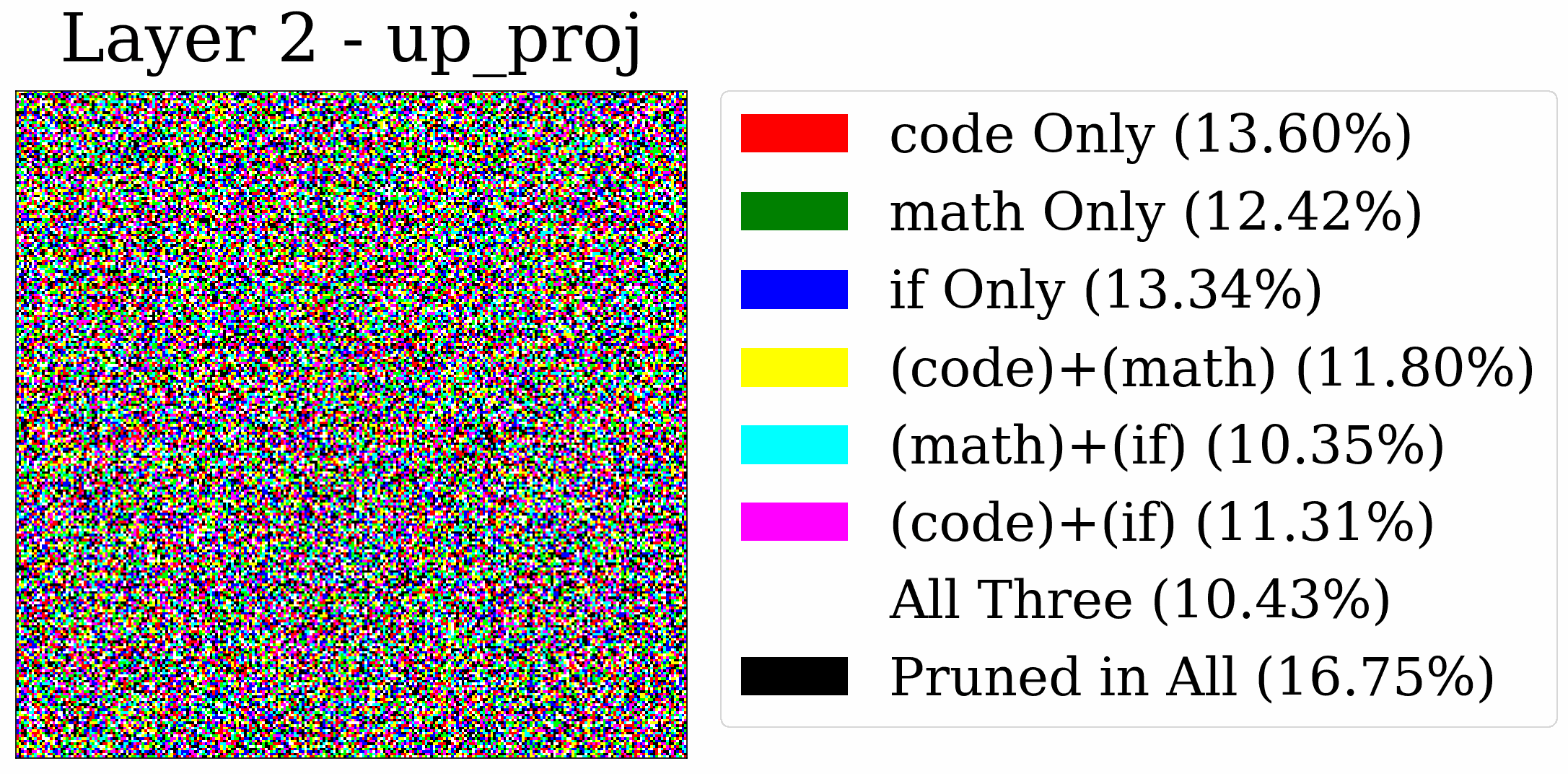}
    \end{subfigure}
    \hfill
    \begin{subfigure}[b]{0.32\textwidth}
        \includegraphics[width=\textwidth]{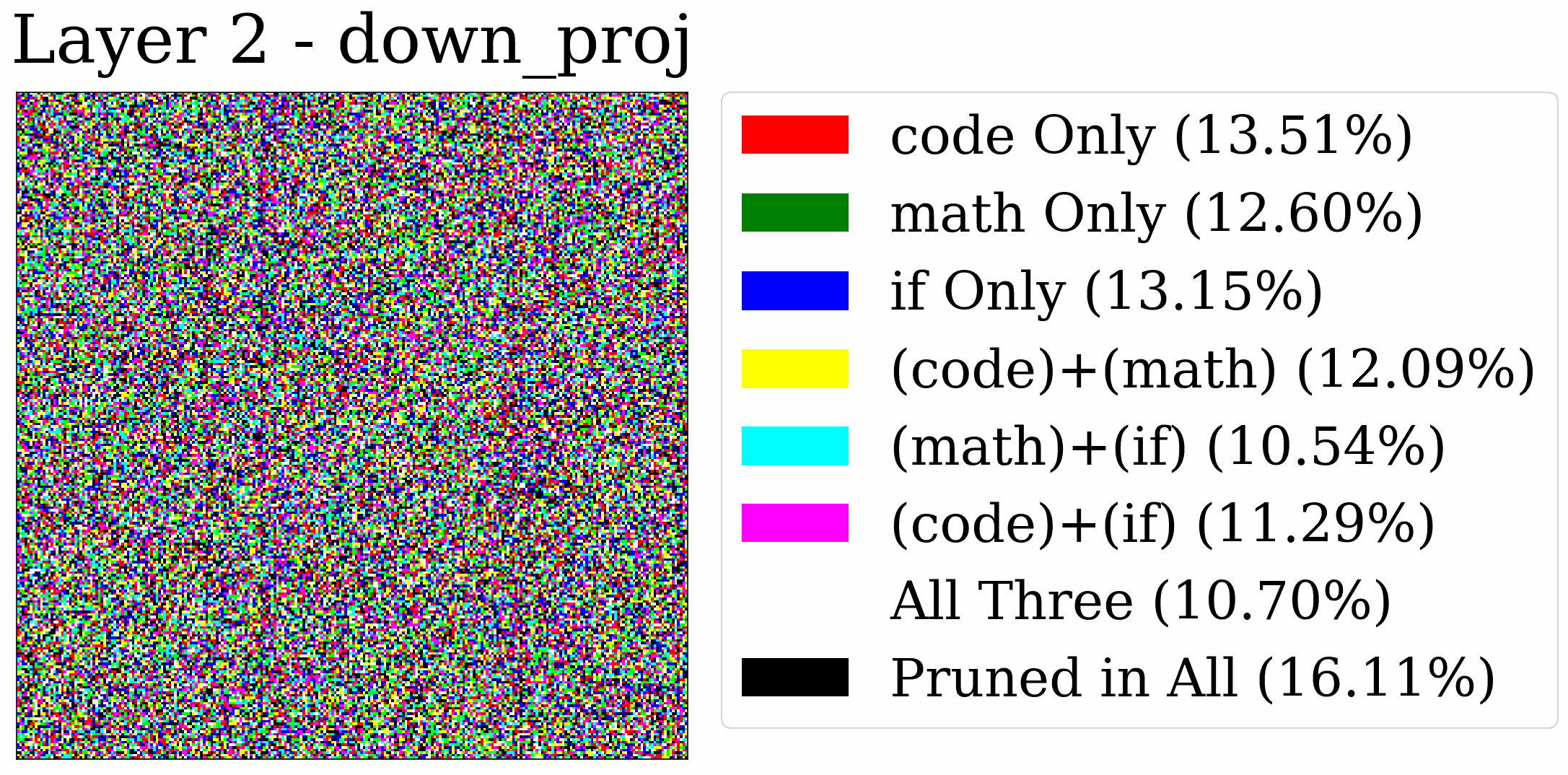}
    \end{subfigure}
    % Layer 15
    \begin{subfigure}[b]{0.32\textwidth}
        \includegraphics[width=\textwidth]{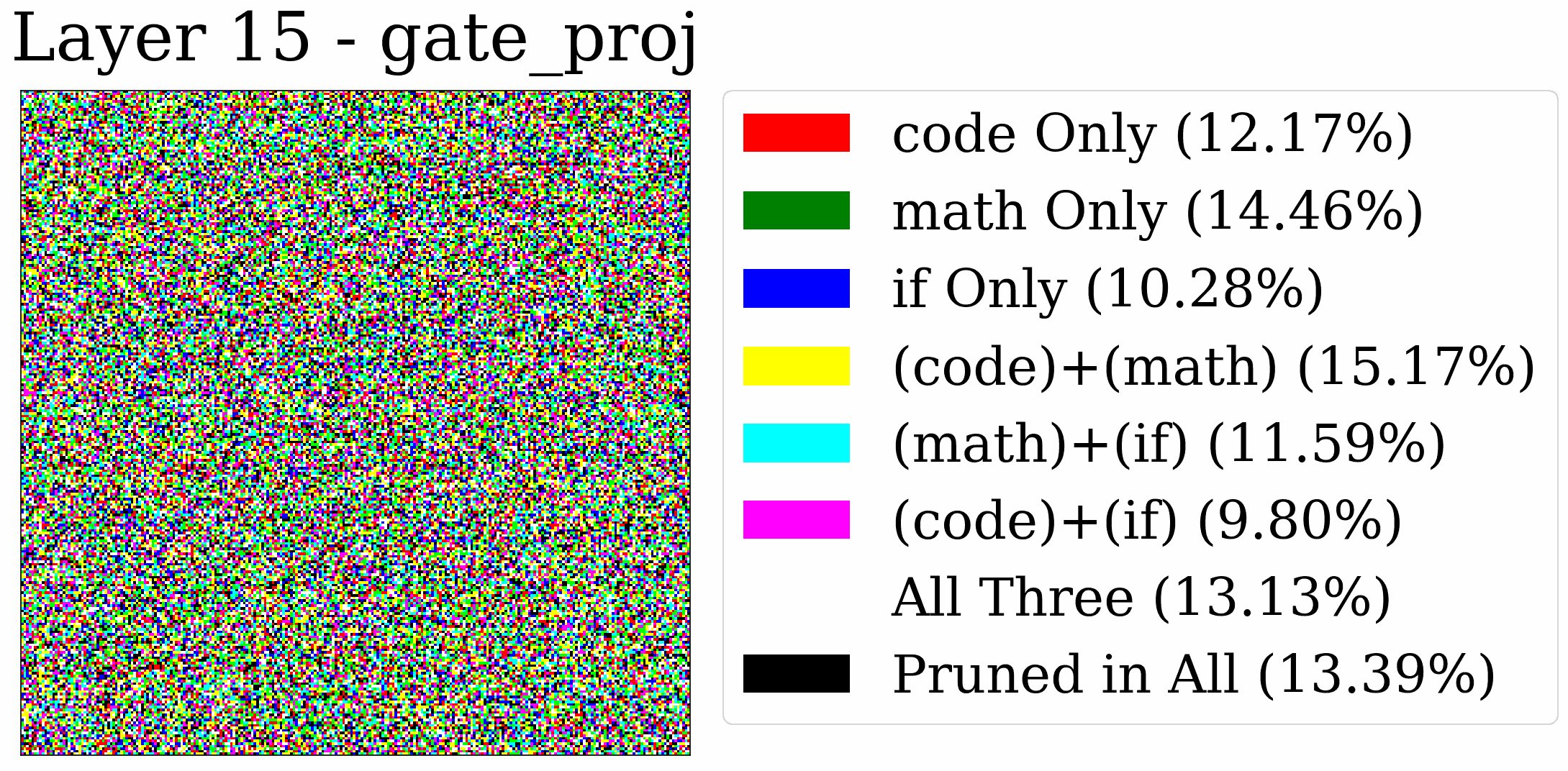}
    \end{subfigure}
    \hfill
    \begin{subfigure}[b]{0.32\textwidth}
        \includegraphics[width=\textwidth]{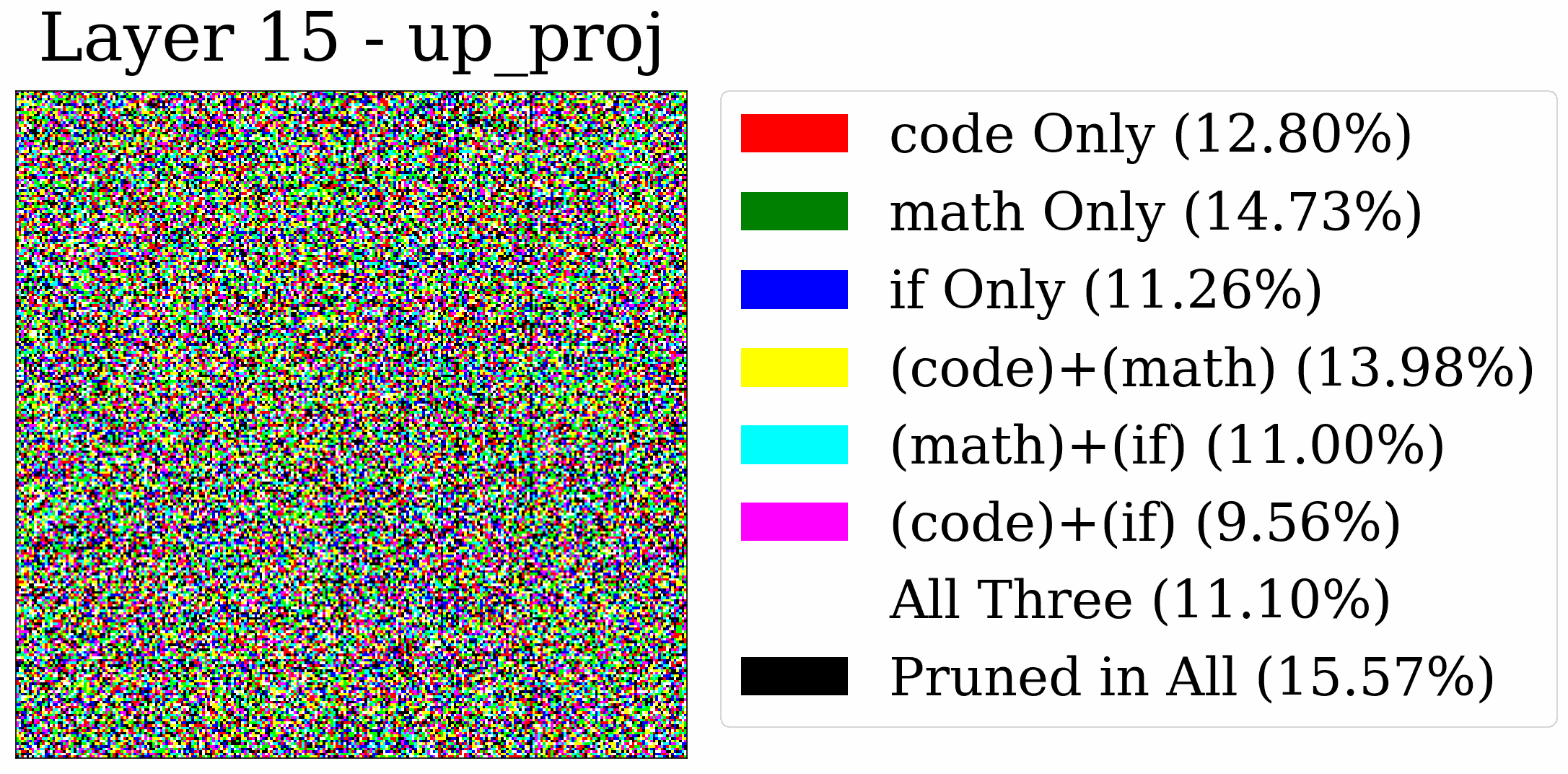}
    \end{subfigure}
    \hfill
    \begin{subfigure}[b]{0.32\textwidth}
        \includegraphics[width=\textwidth]{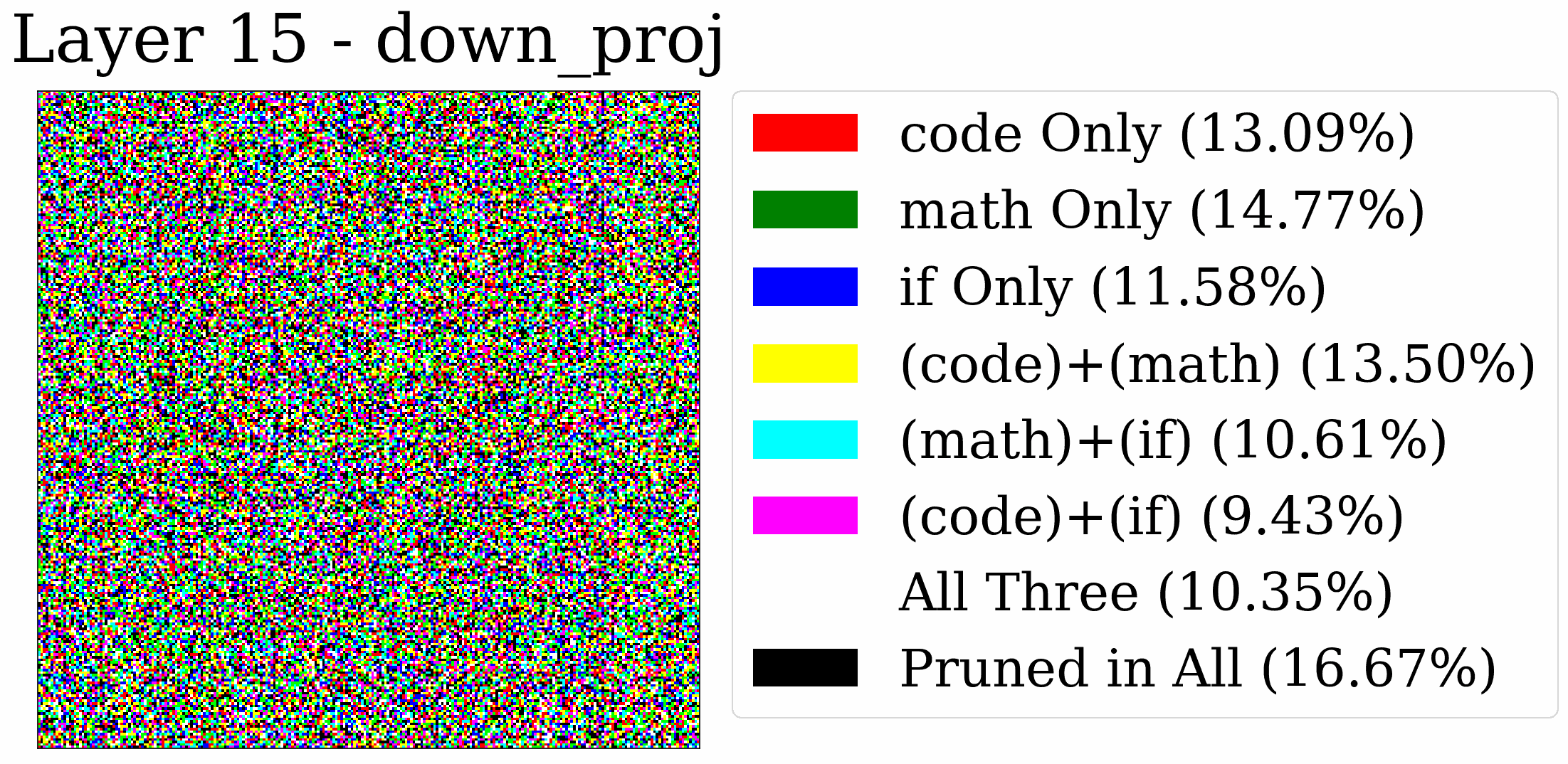}
    \end{subfigure}
    % Layer 18
    \begin{subfigure}[b]{0.32\textwidth}
        \includegraphics[width=\textwidth]{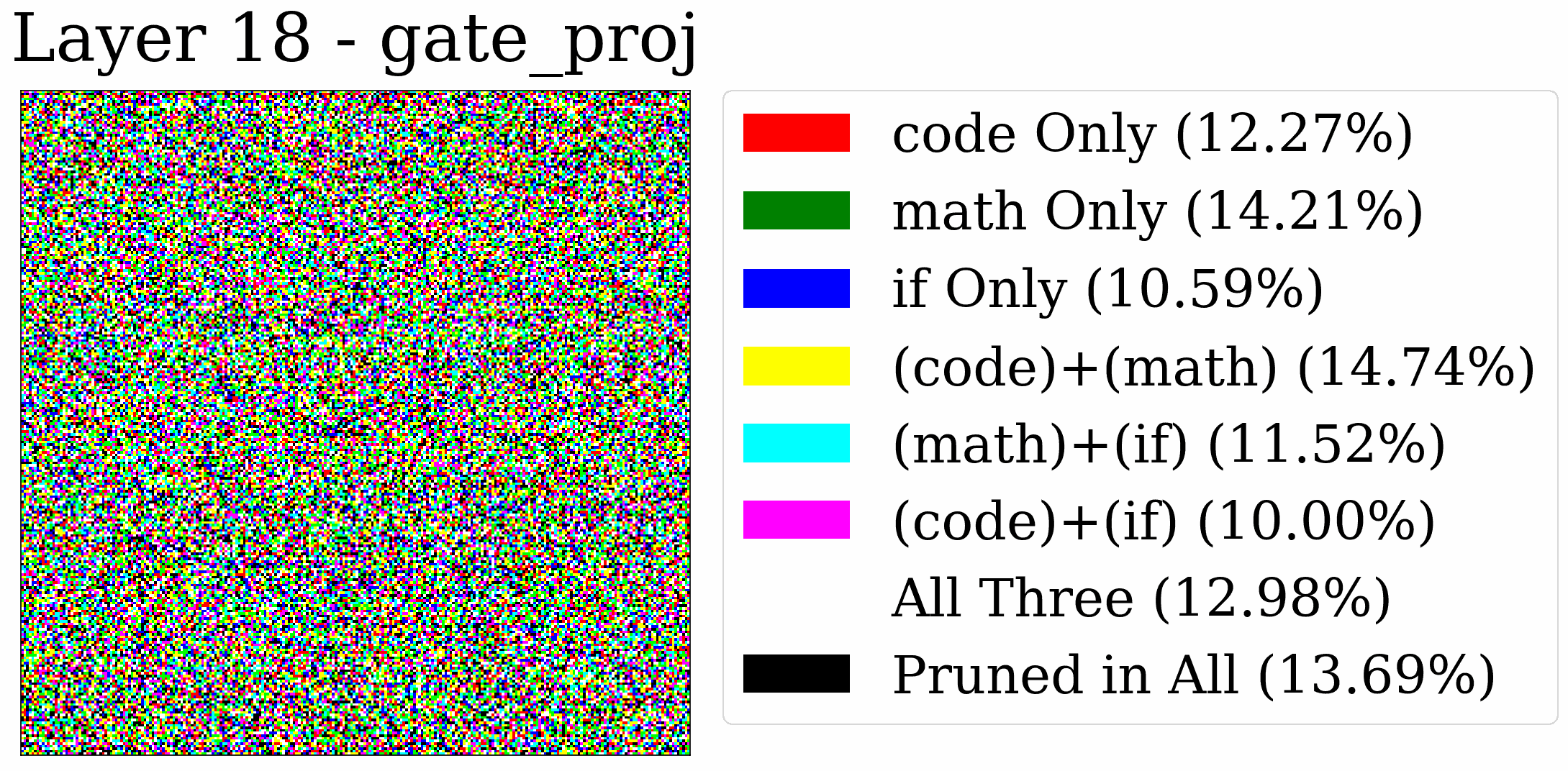}
    \end{subfigure}
    \hfill
    \begin{subfigure}[b]{0.32\textwidth}
        \includegraphics[width=\textwidth]{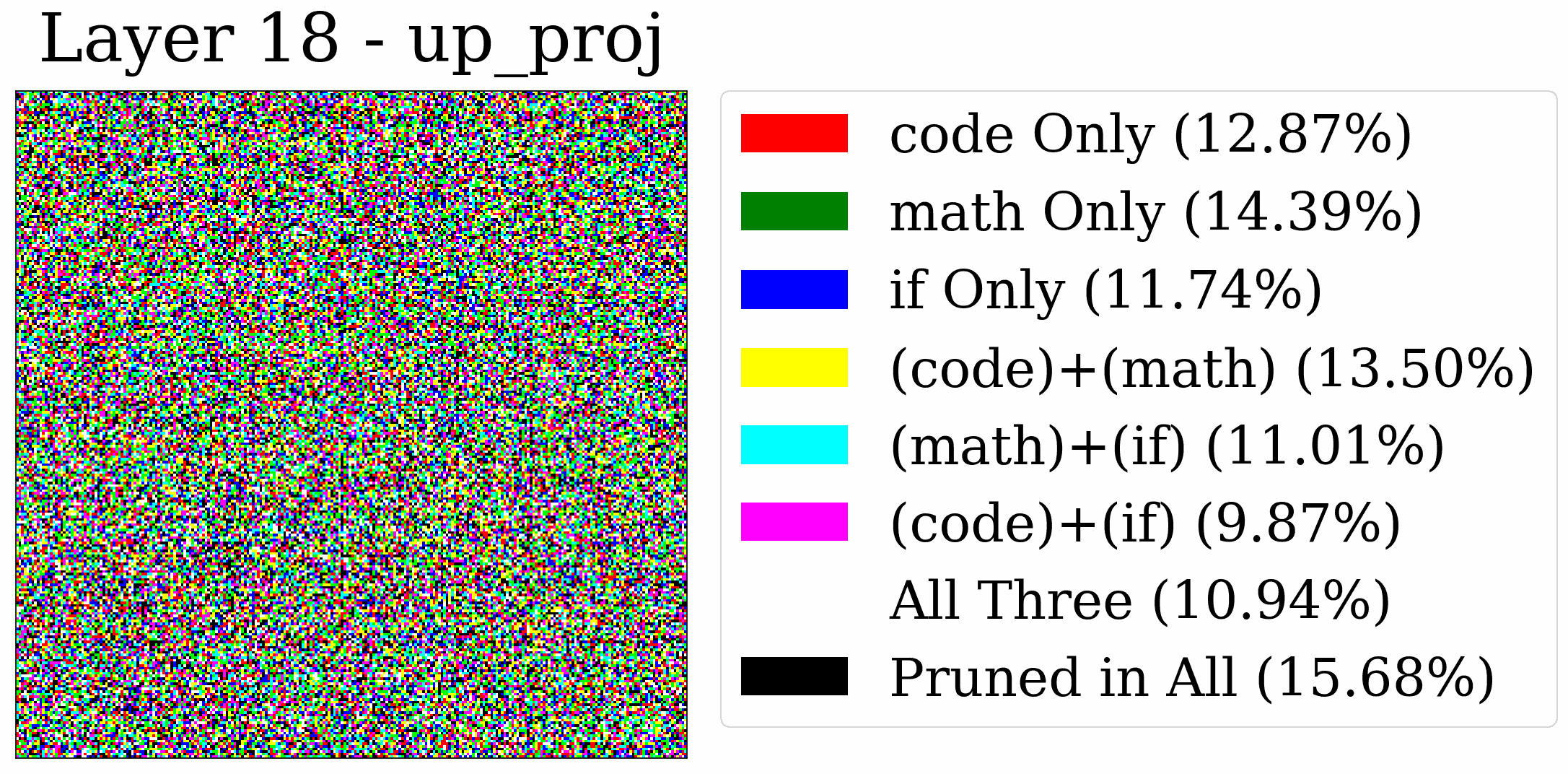}
    \end{subfigure}
    \hfill
    \begin{subfigure}[b]{0.32\textwidth}
        \includegraphics[width=\textwidth]{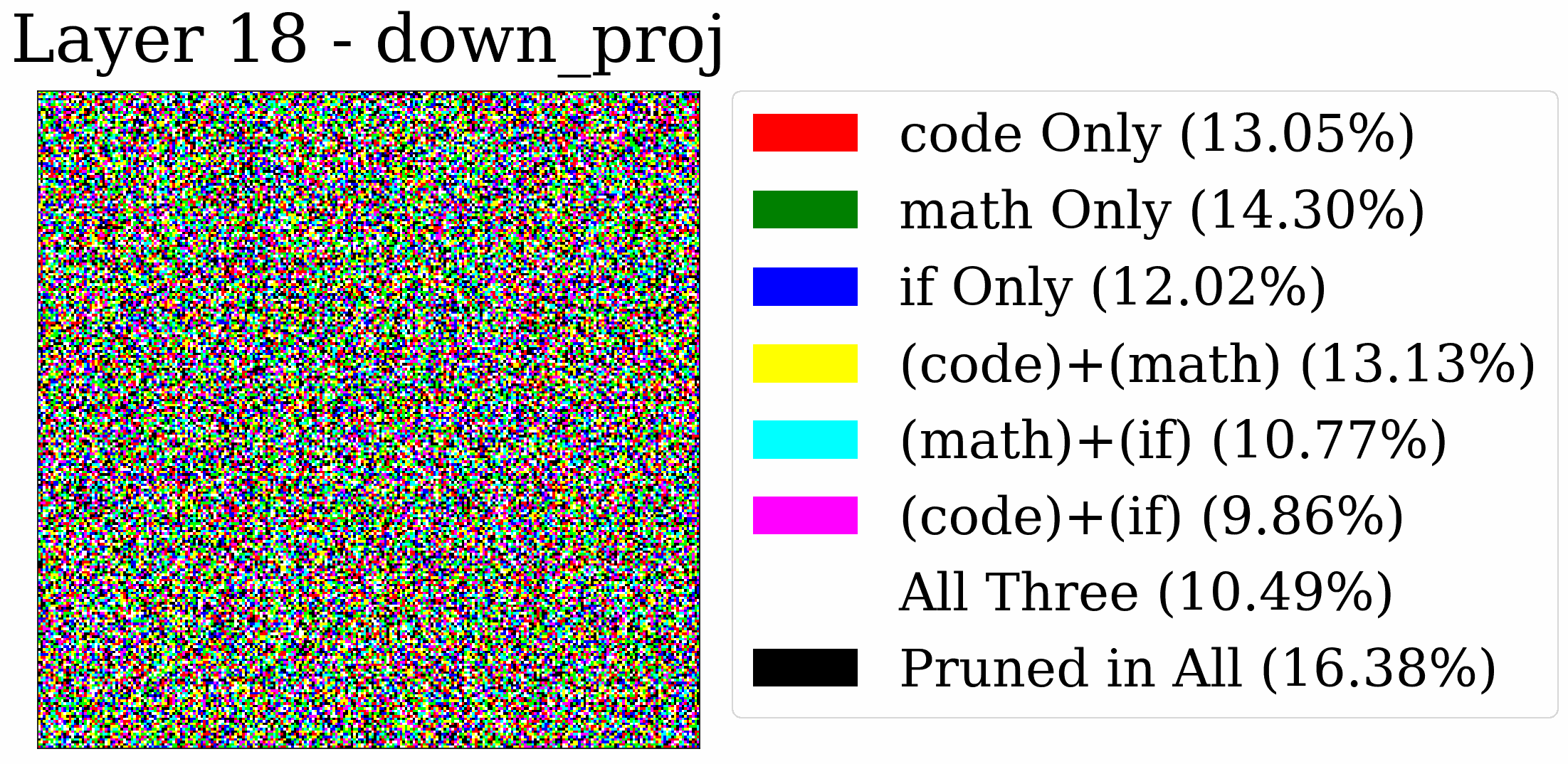}
    \end{subfigure}
    % Layer 29
    \begin{subfigure}[b]{0.32\textwidth}
        \includegraphics[width=\textwidth]{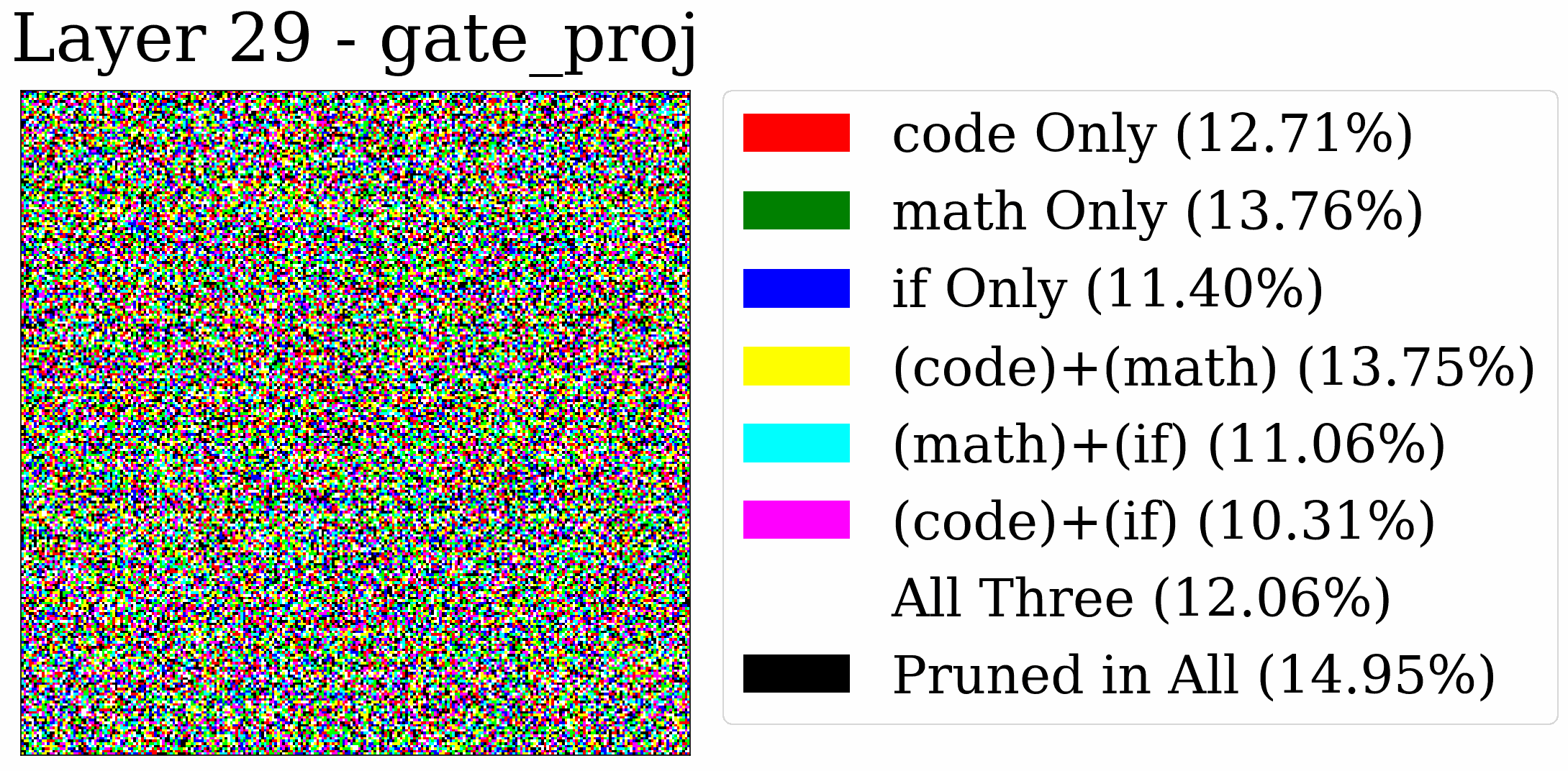}
    \end{subfigure}
    \hfill
    \begin{subfigure}[b]{0.32\textwidth}
        \includegraphics[width=\textwidth]{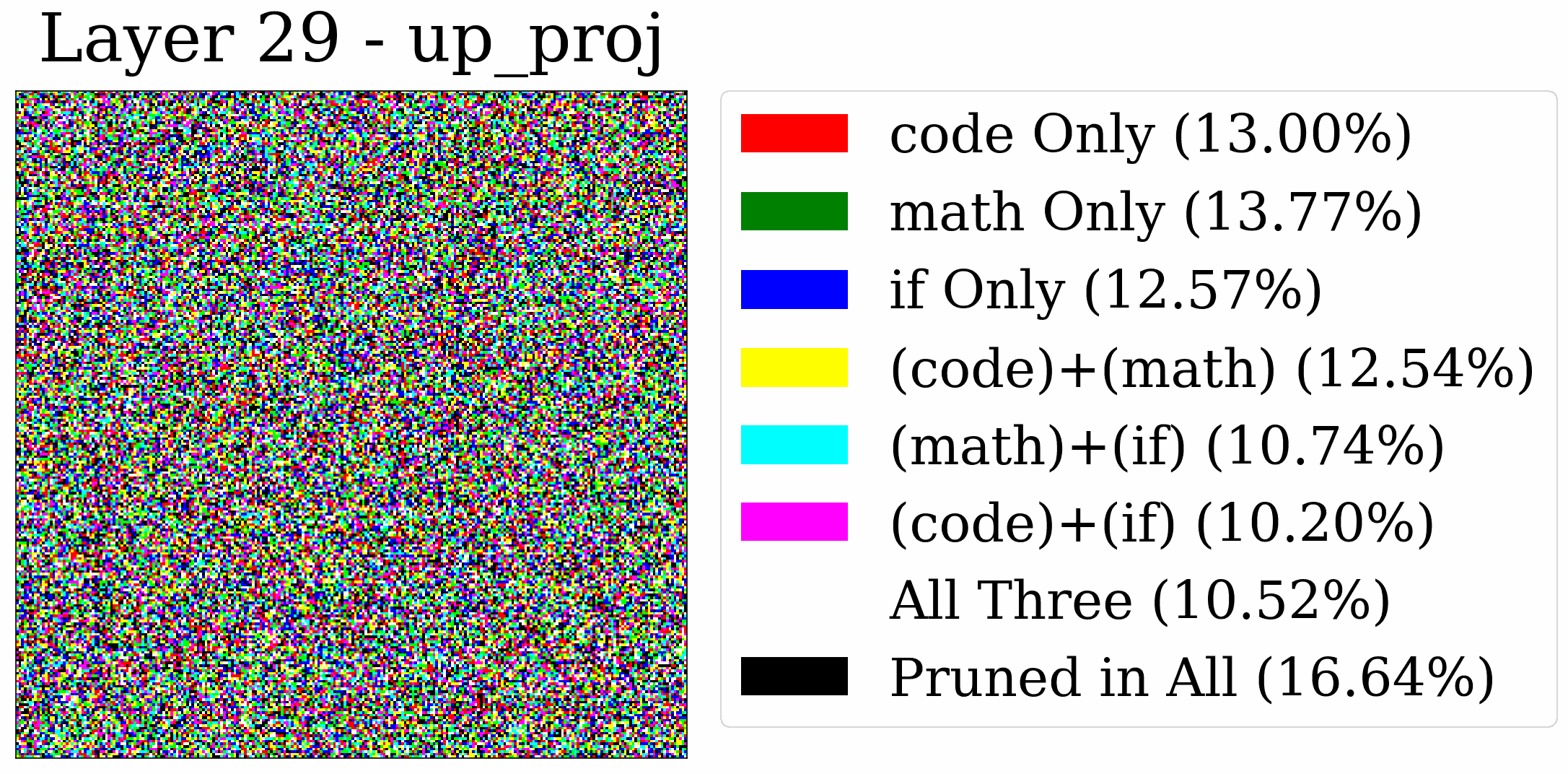}
    \end{subfigure}
    \hfill
    \begin{subfigure}[b]{0.32\textwidth}
        \includegraphics[width=\textwidth]{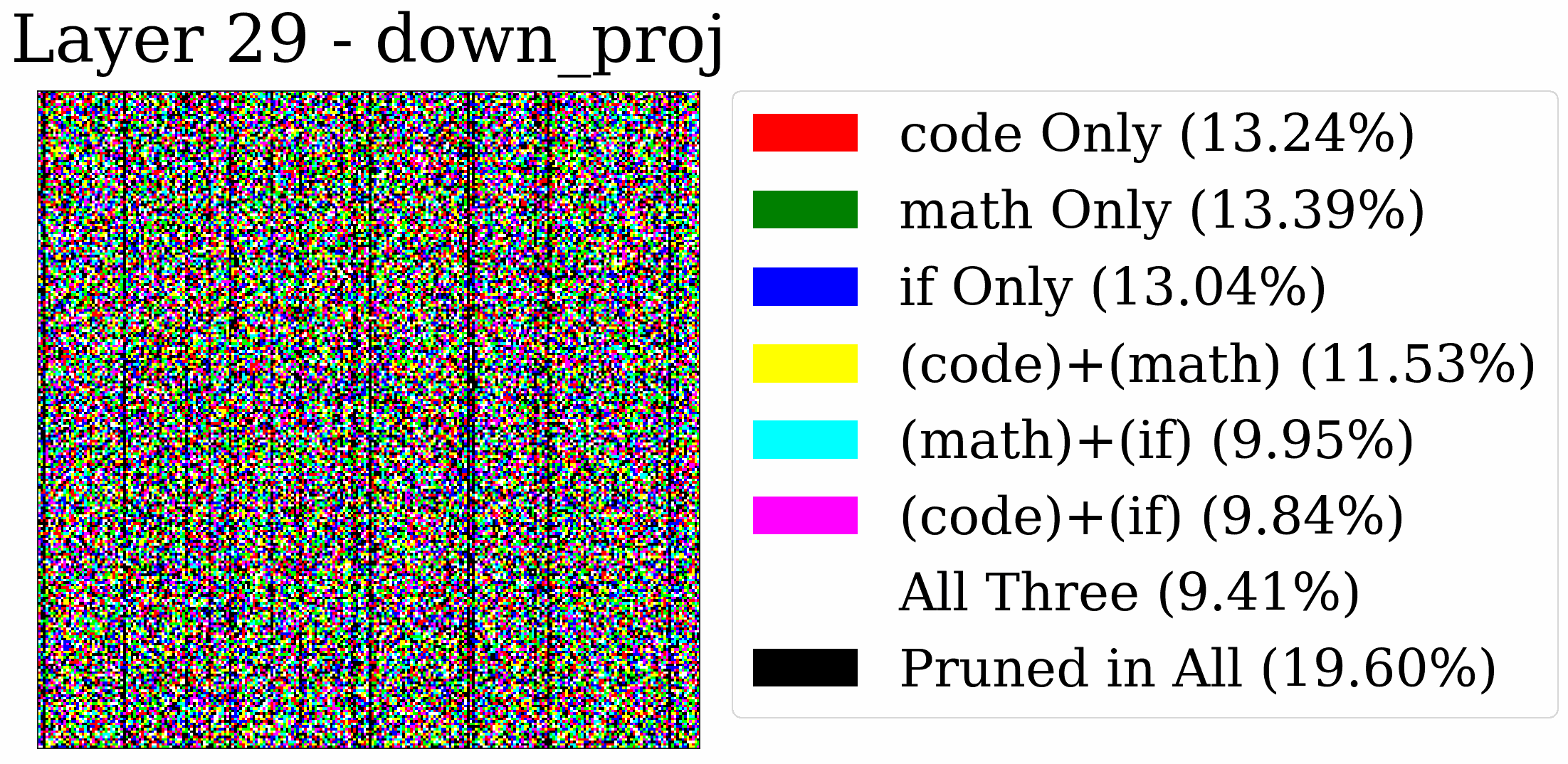}
    \end{subfigure}
    % Layer 30
    \begin{subfigure}[b]{0.32\textwidth}
        \includegraphics[width=\textwidth]{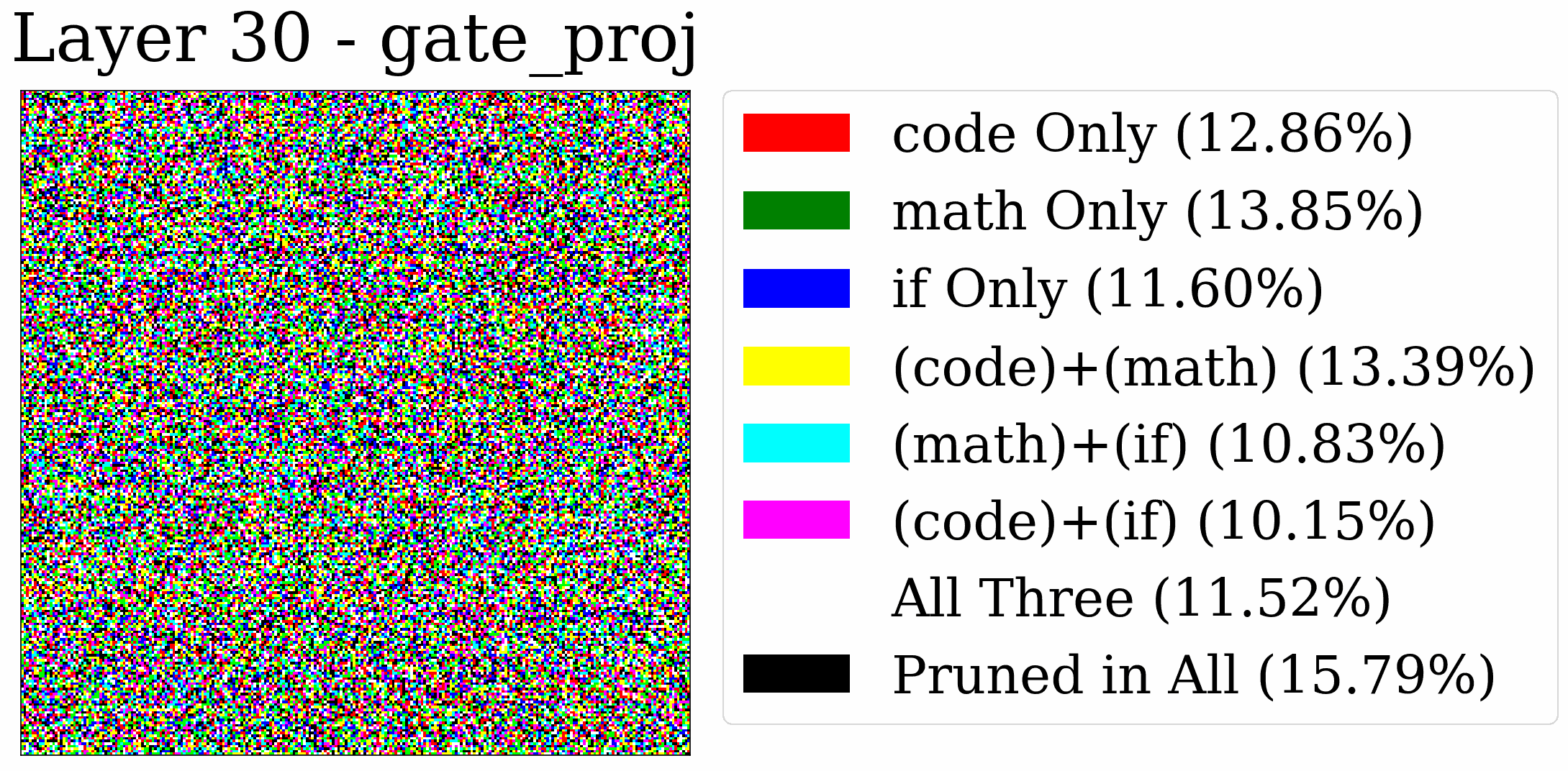}
    \end{subfigure}
    \hfill
    \begin{subfigure}[b]{0.32\textwidth}
        \includegraphics[width=\textwidth]{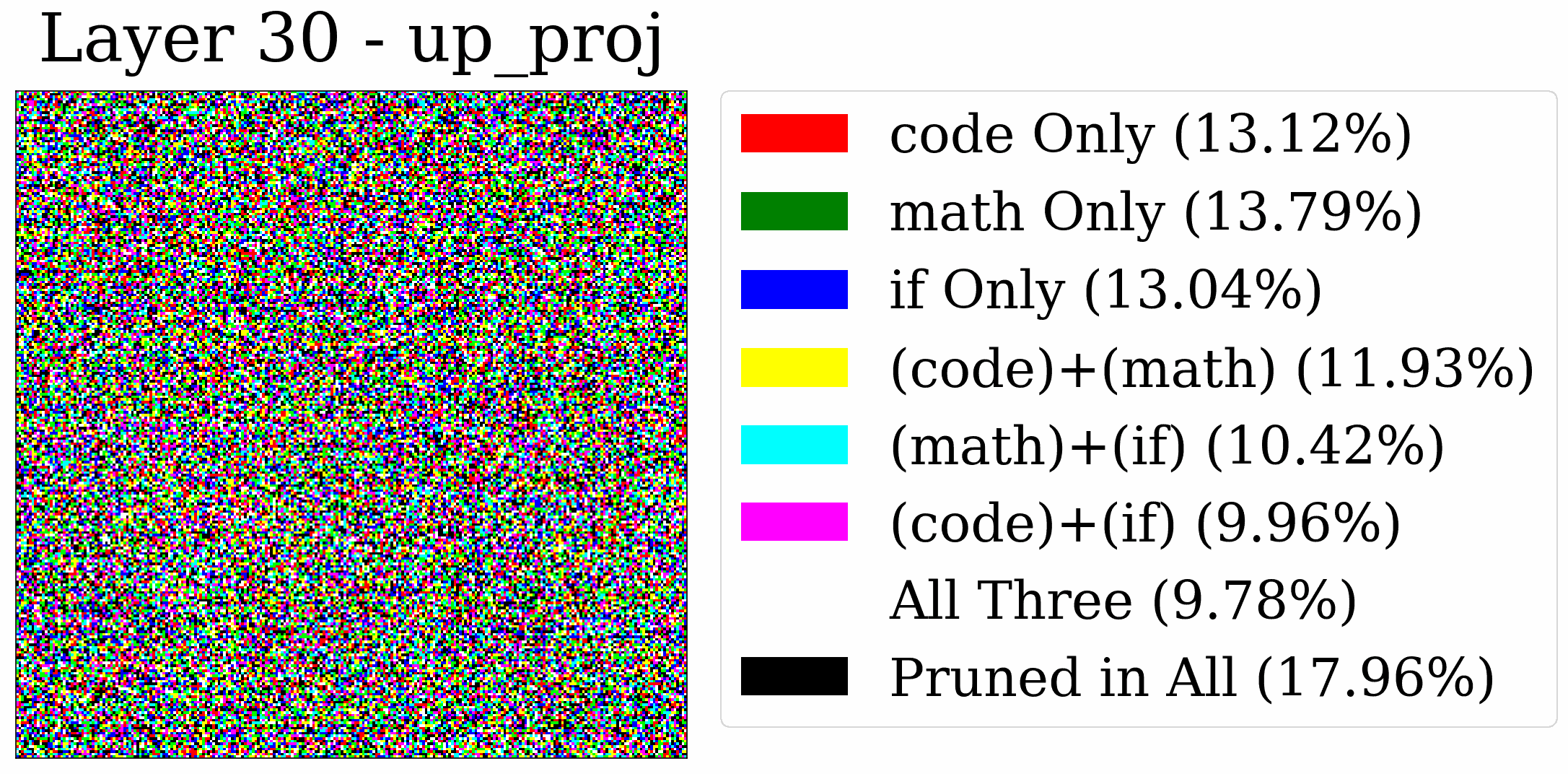}
    \end{subfigure}
    \hfill
    \begin{subfigure}[b]{0.32\textwidth}
        \includegraphics[width=\textwidth]{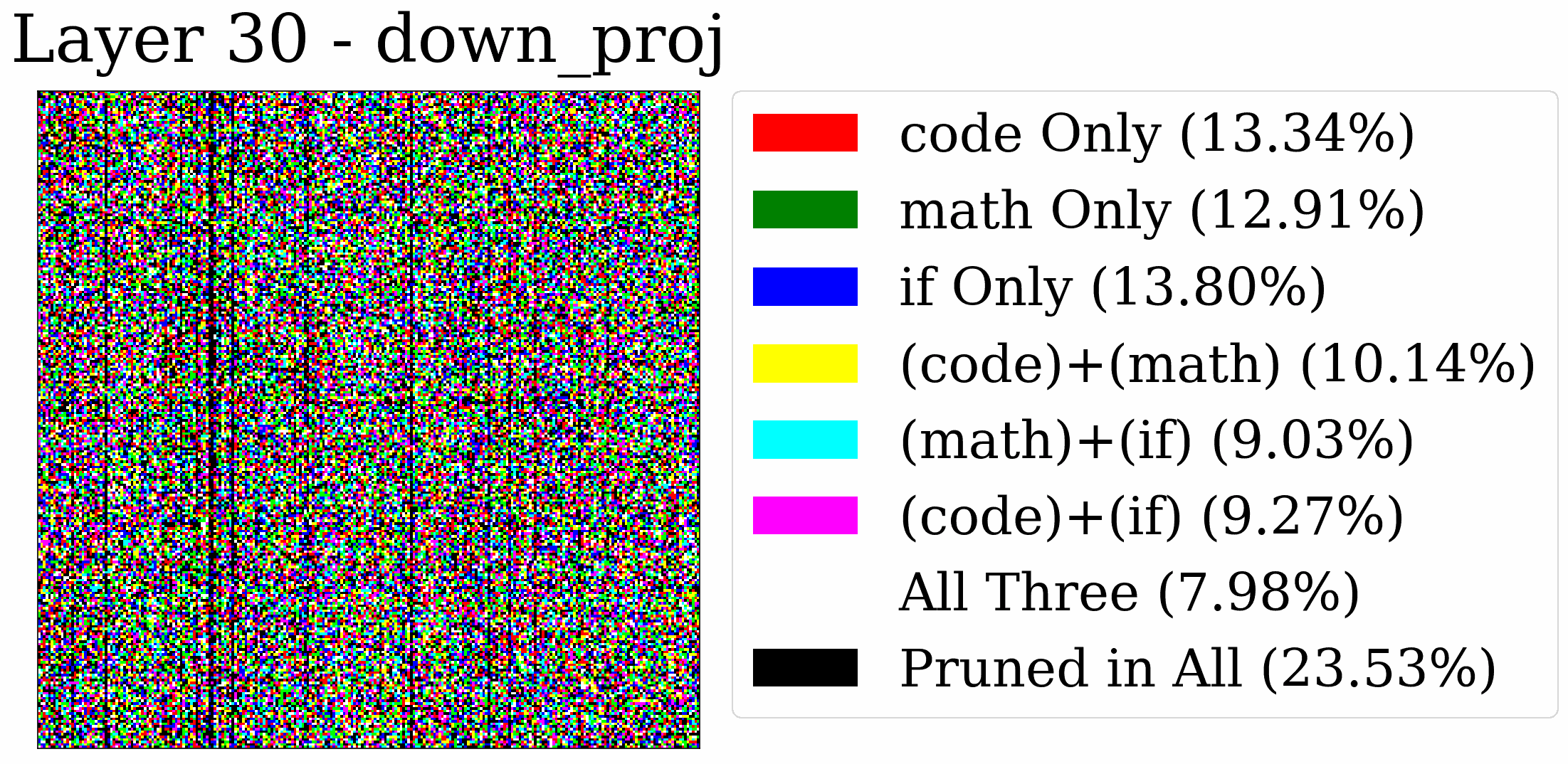}
    \end{subfigure}
    \caption{3-way magnitude-based comparison for FFN components across layers 1, 2, 15, 18, 29, and 30. Columns show $W_{gate}$, $W_{up}$, and $W_{down}$.}
    \label{fig:3way_mag_ffn_layers}
\end{figure}

\begin{figure}[!ht]
    \centering
    % Layer 1
    \begin{subfigure}[b]{0.32\textwidth}
        \includegraphics[width=\textwidth]{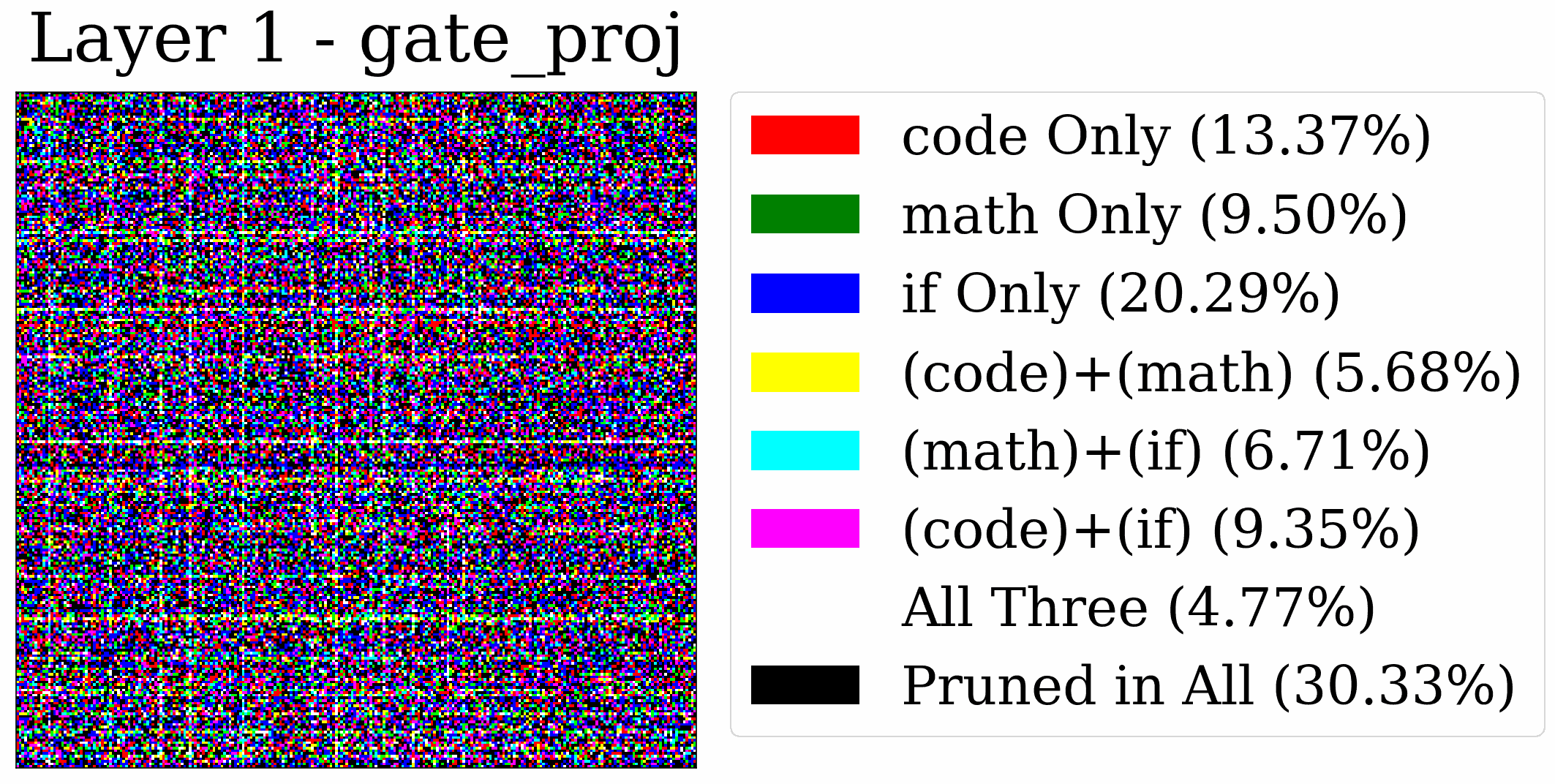}
    \end{subfigure}
    \hfill
    \begin{subfigure}[b]{0.32\textwidth}
        \includegraphics[width=\textwidth]{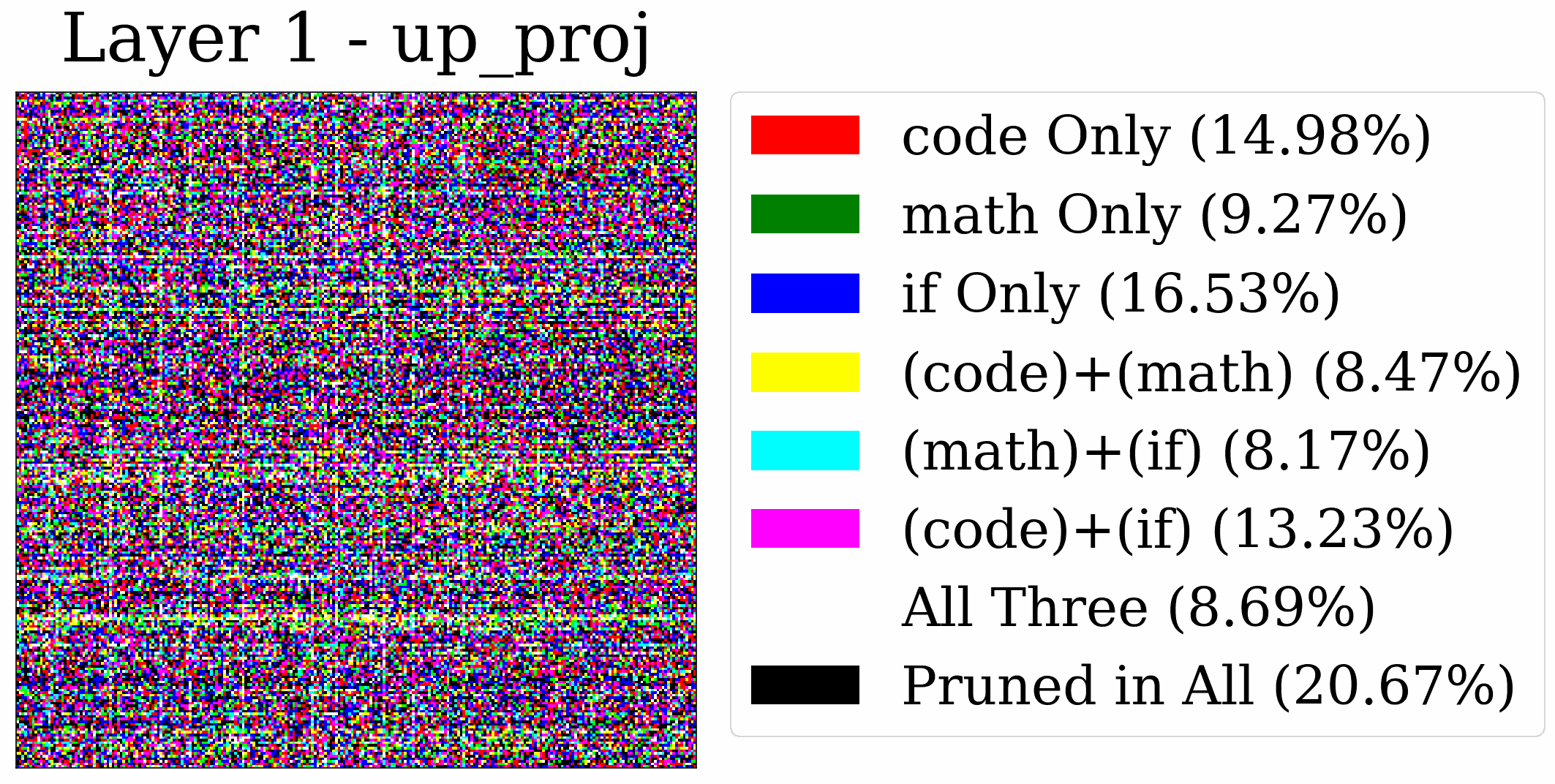}
    \end{subfigure}
    \hfill
    \begin{subfigure}[b]{0.32\textwidth}
        \includegraphics[width=\textwidth]{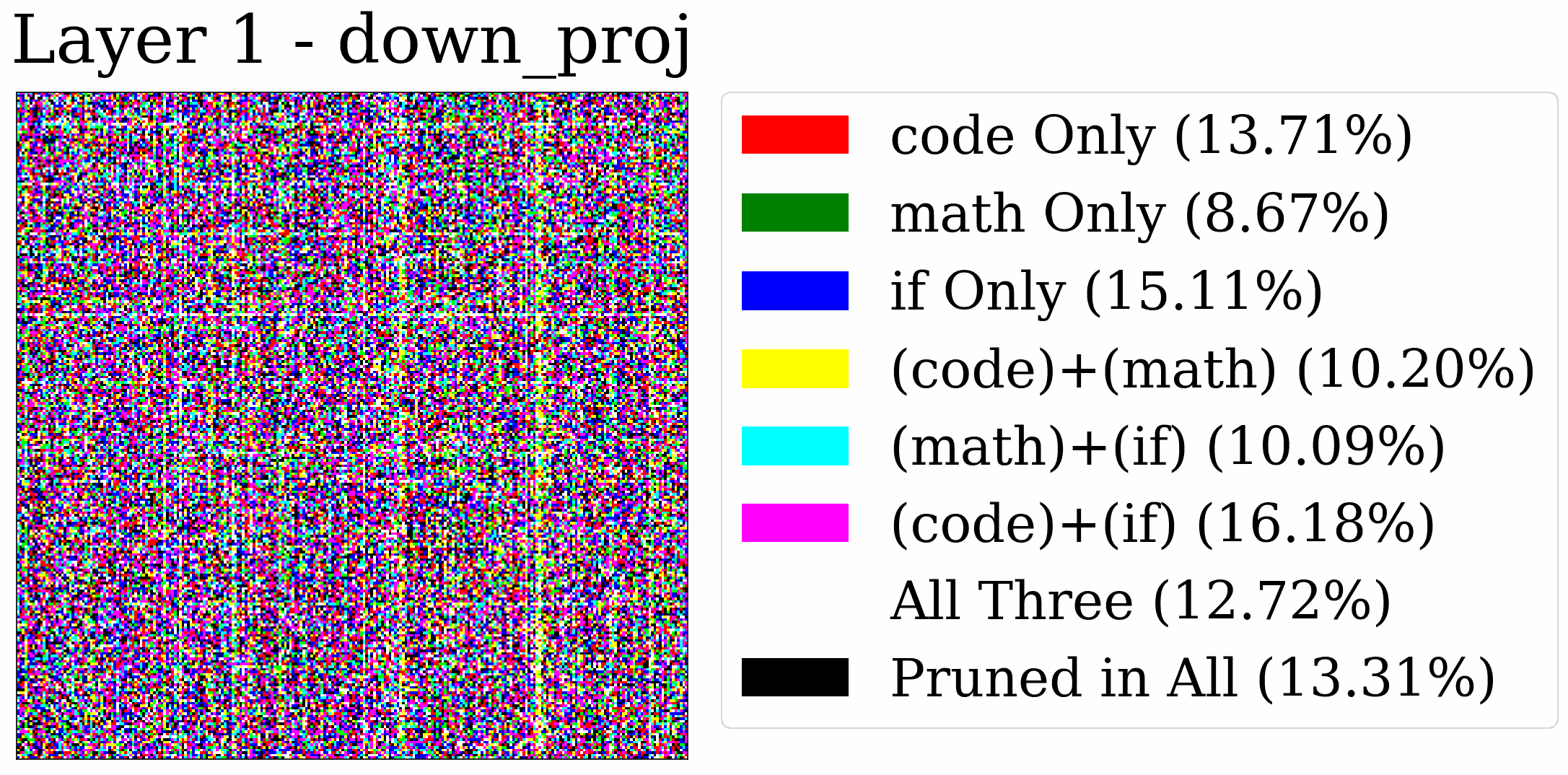}
    \end{subfigure}
    % Layer 2
    \begin{subfigure}[b]{0.32\textwidth}
        \includegraphics[width=\textwidth]{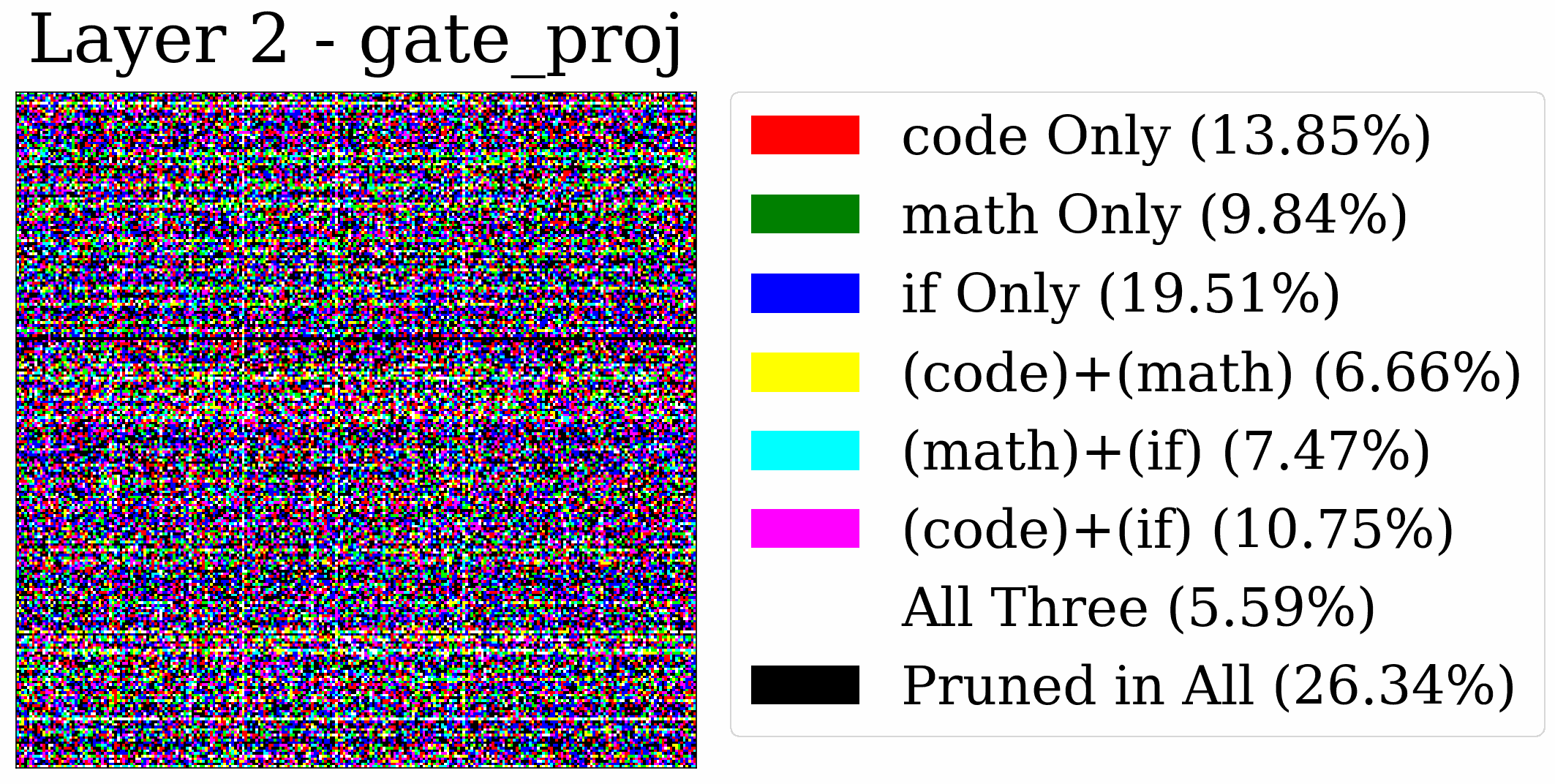}
    \end{subfigure}
    \hfill
    \begin{subfigure}[b]{0.32\textwidth}
        \includegraphics[width=\textwidth]{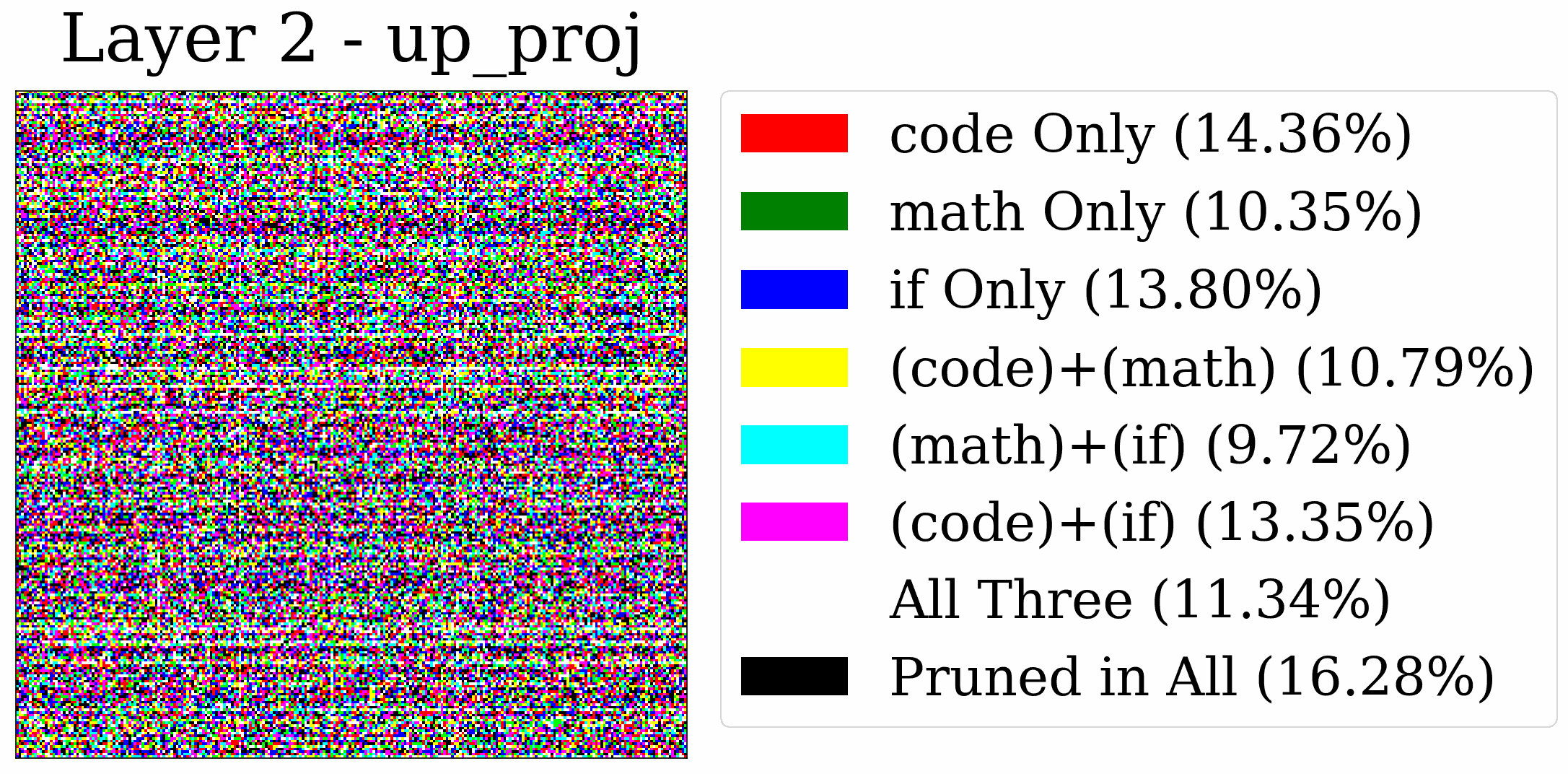}
    \end{subfigure}
    \hfill
    \begin{subfigure}[b]{0.32\textwidth}
        \includegraphics[width=\textwidth]{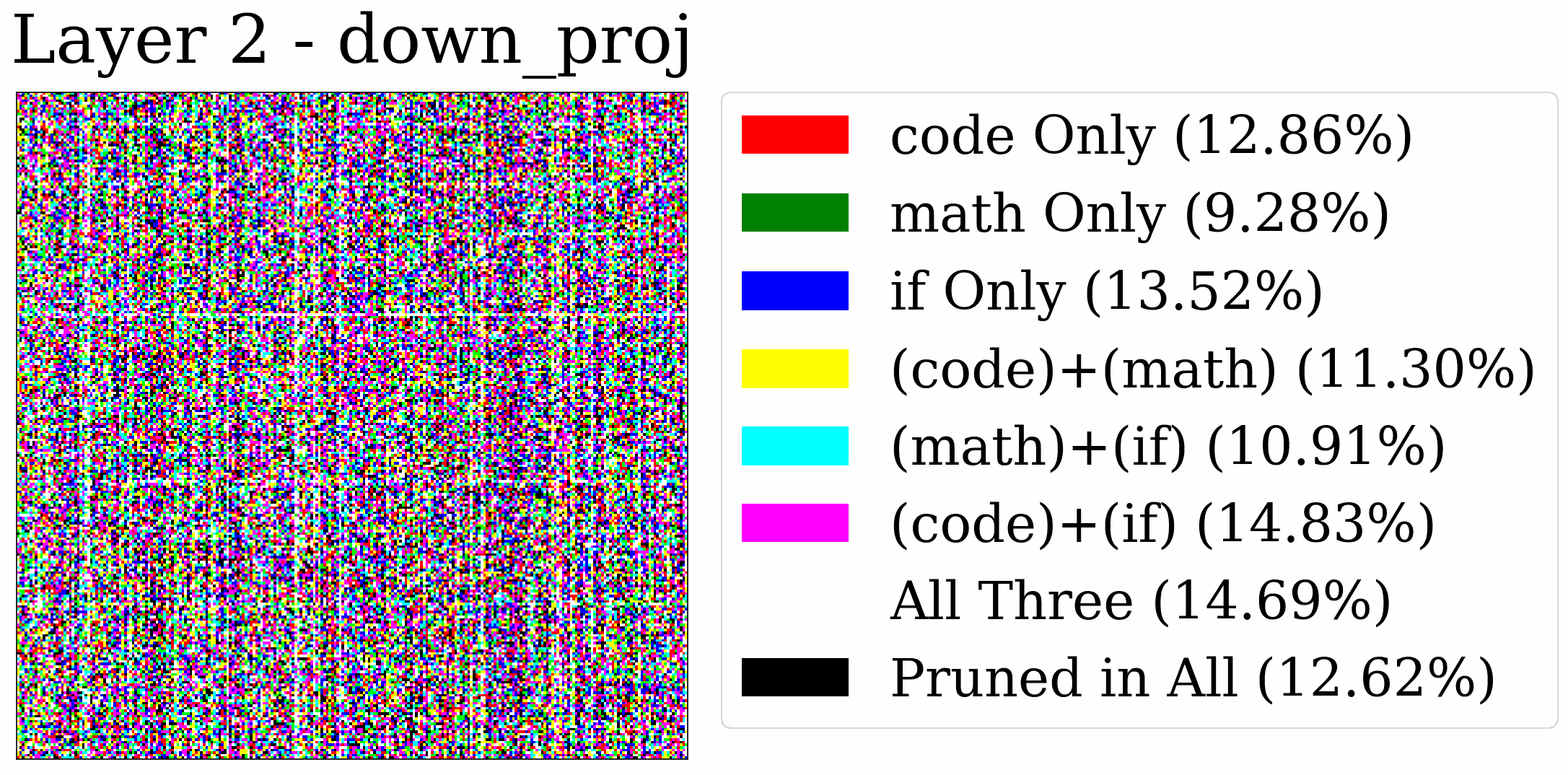}
    \end{subfigure}
    % Layer 15
    \begin{subfigure}[b]{0.32\textwidth}
        \includegraphics[width=\textwidth]{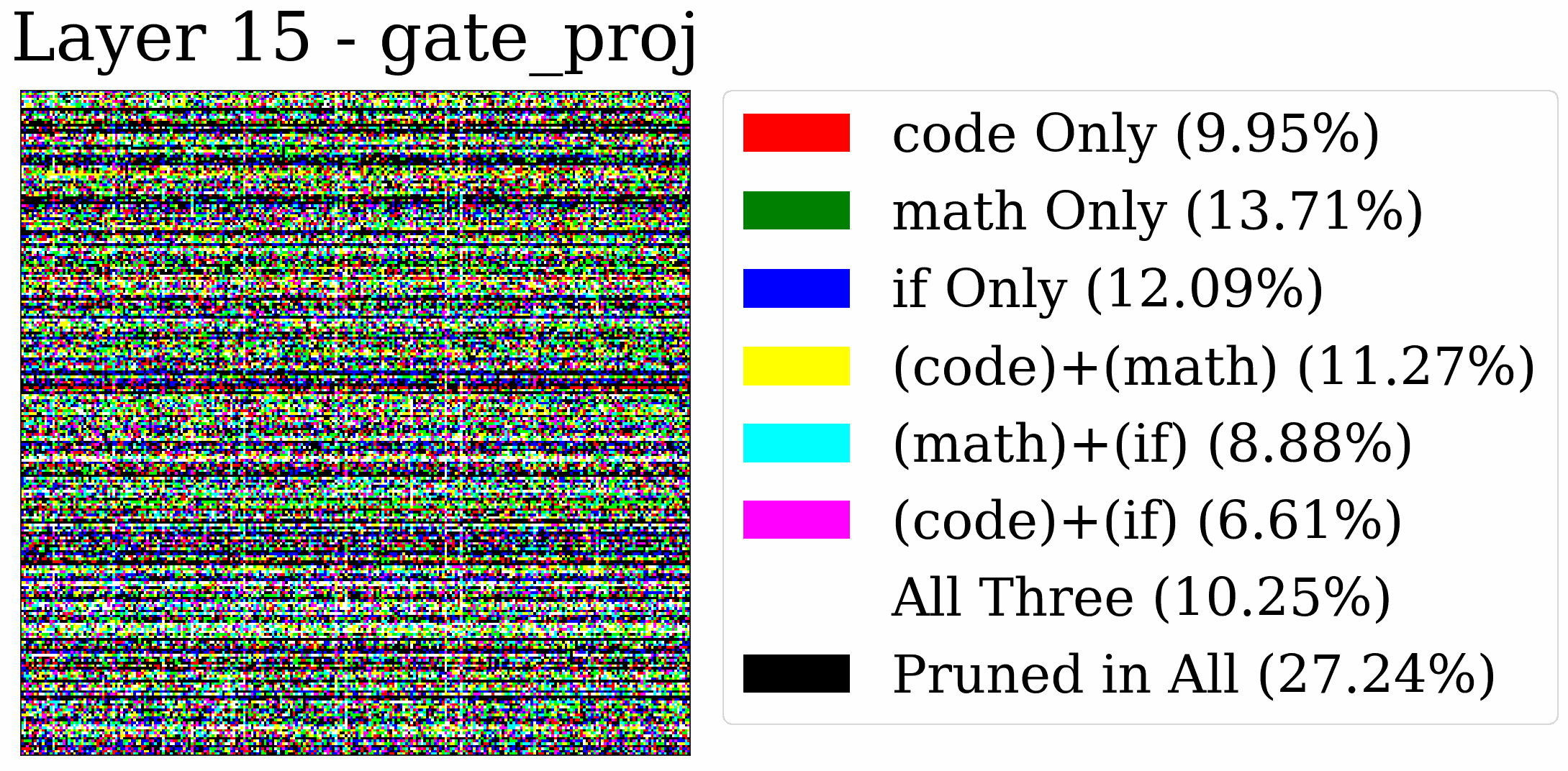}
    \end{subfigure}
    \hfill
    \begin{subfigure}[b]{0.32\textwidth}
        \includegraphics[width=\textwidth]{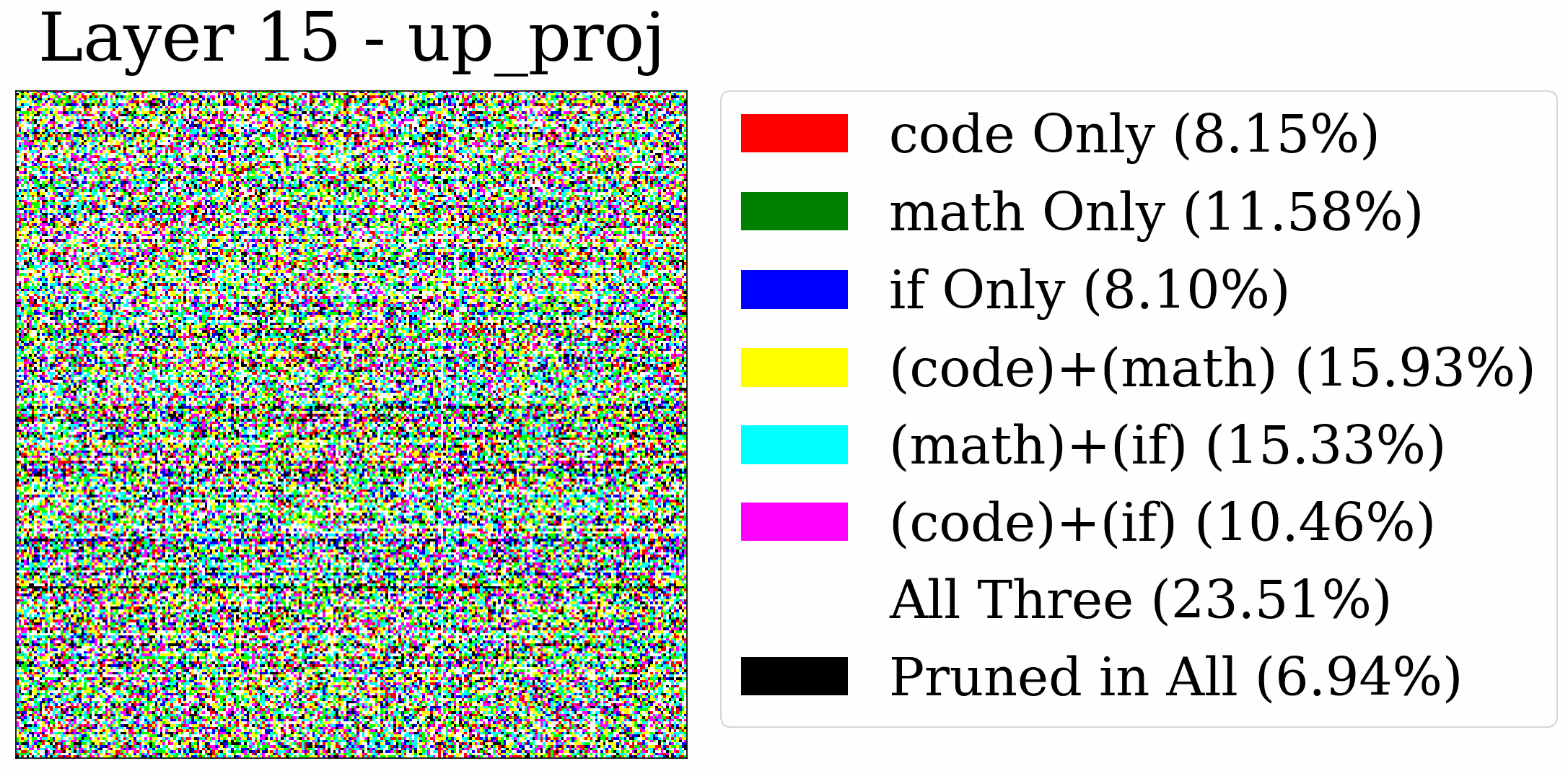}
    \end{subfigure}
    \hfill
    \begin{subfigure}[b]{0.32\textwidth}
        \includegraphics[width=\textwidth]{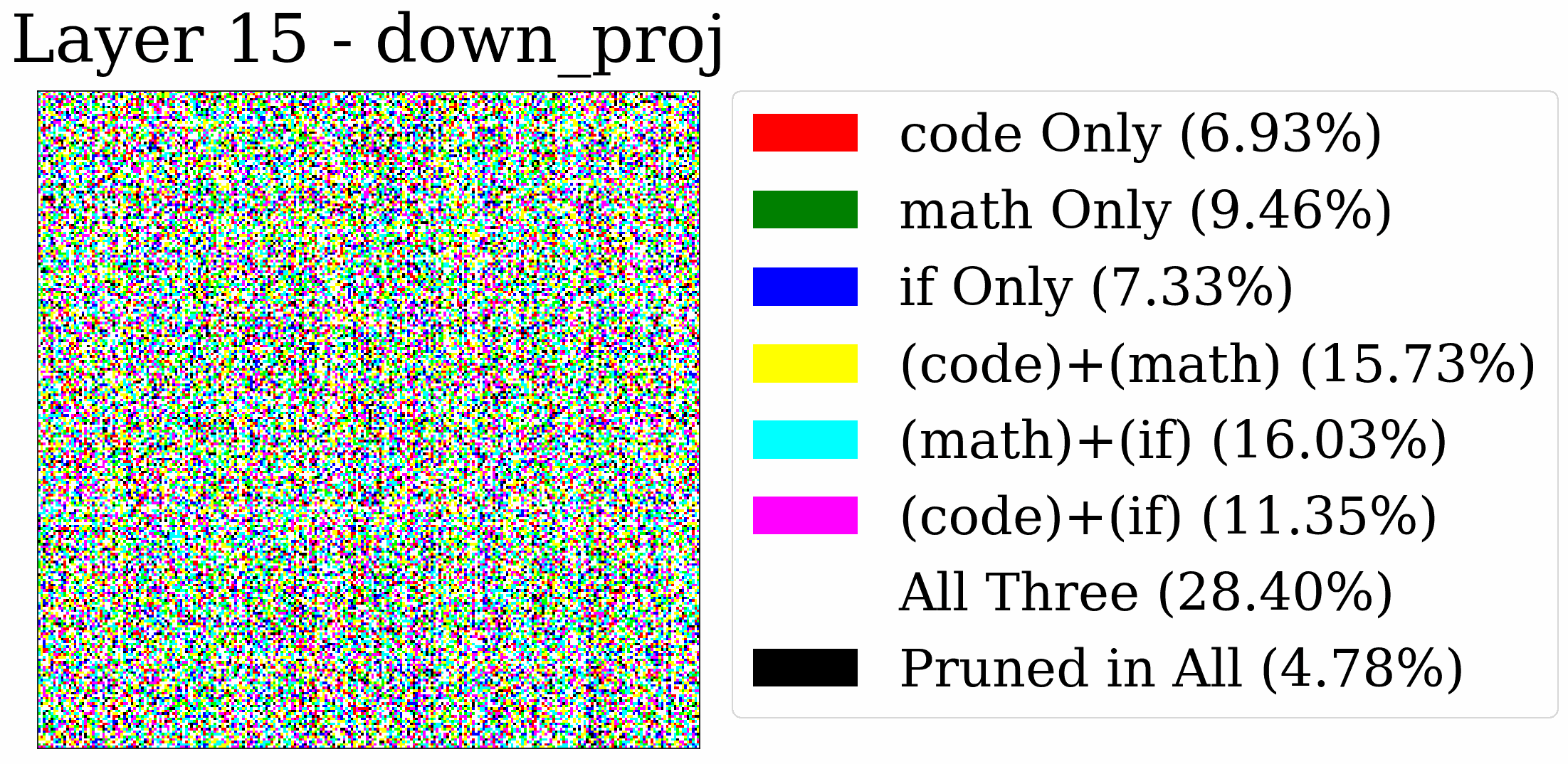}
    \end{subfigure}
    % Layer 18
    \begin{subfigure}[b]{0.32\textwidth}
        \includegraphics[width=\textwidth]{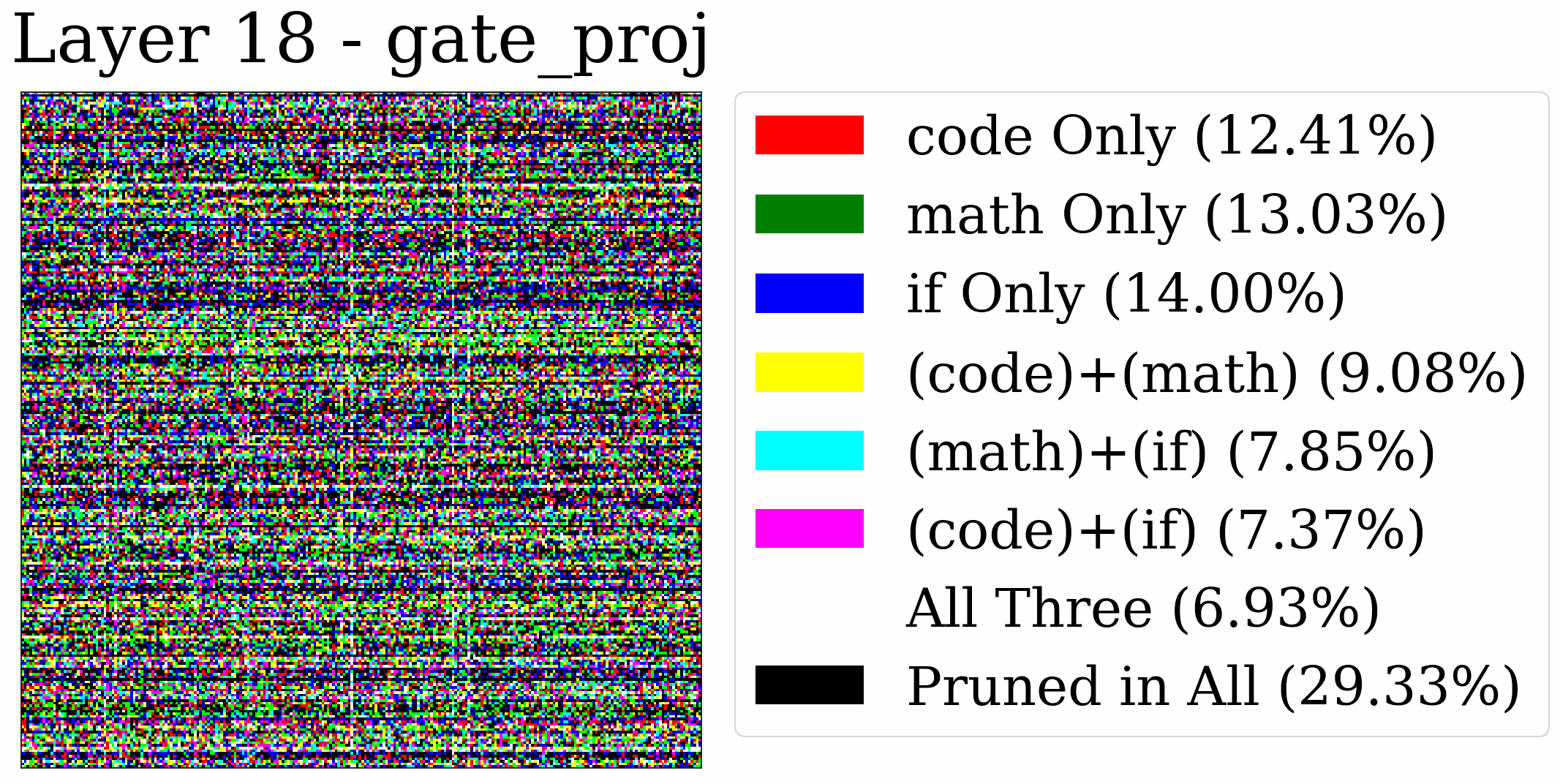}
    \end{subfigure}
    \hfill
    \begin{subfigure}[b]{0.32\textwidth}
        \includegraphics[width=\textwidth]{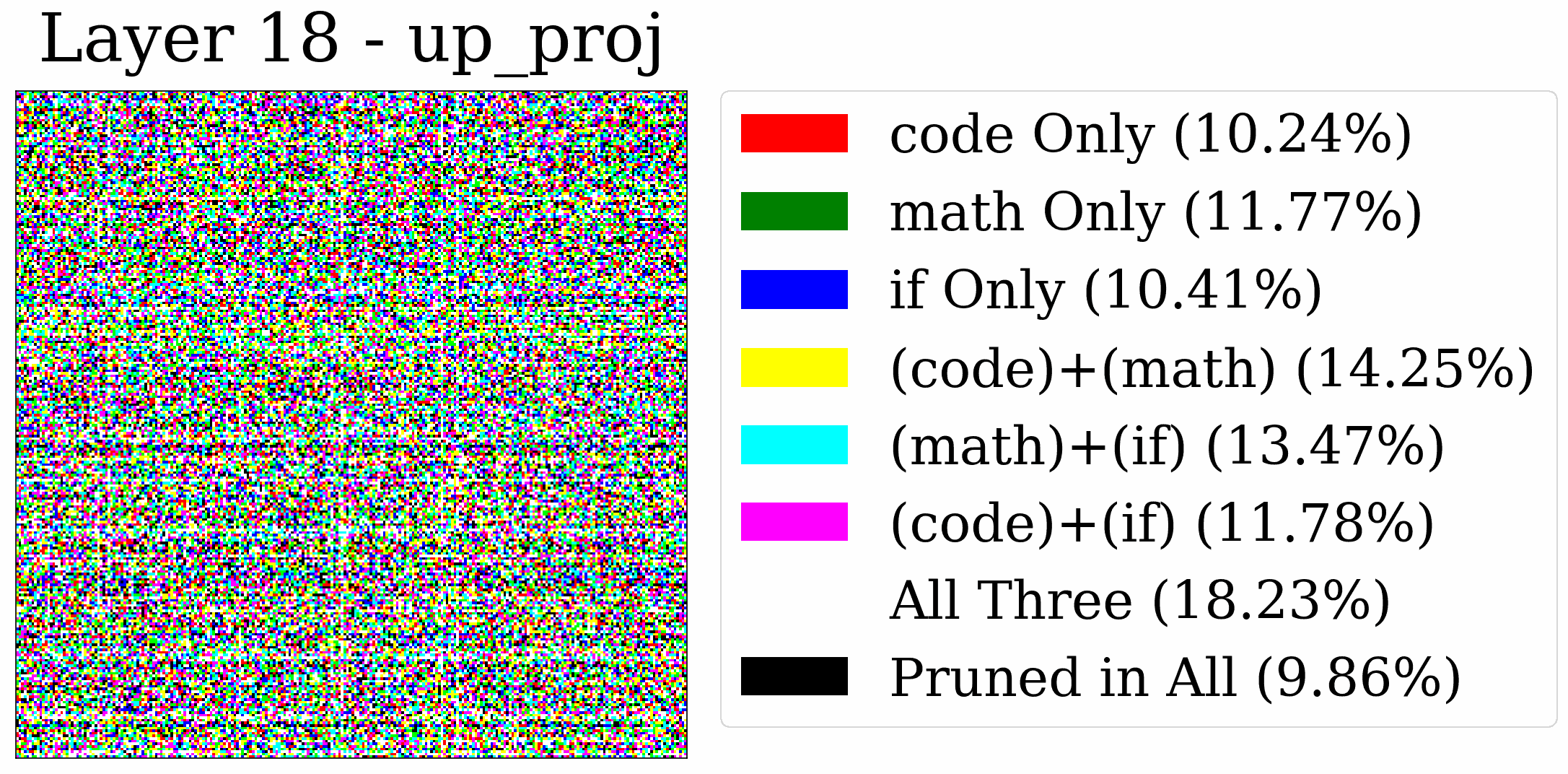}
    \end{subfigure}
    \hfill
    \begin{subfigure}[b]{0.32\textwidth}
        \includegraphics[width=\textwidth]{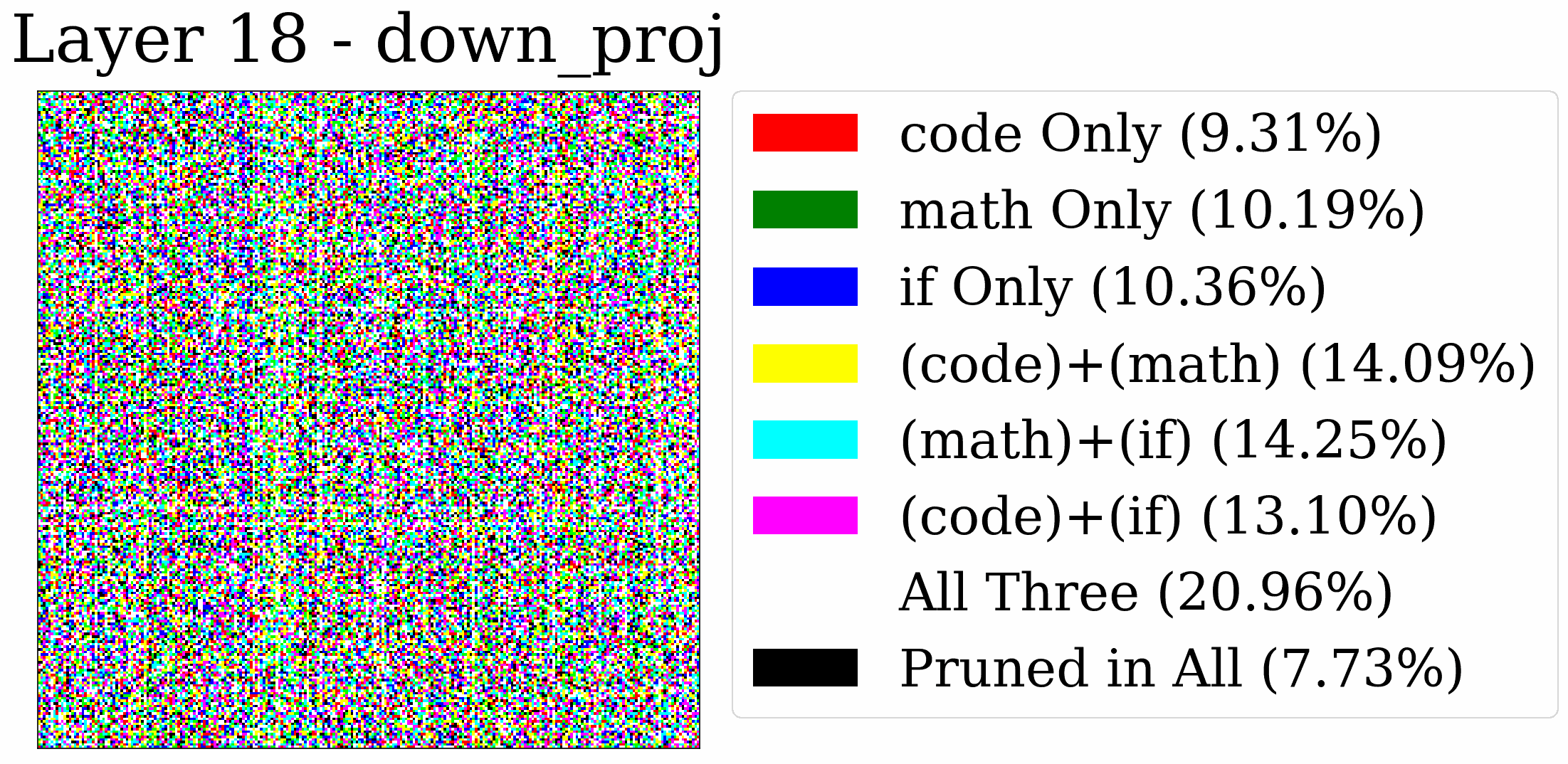}
    \end{subfigure}
    % Layer 29
    \begin{subfigure}[b]{0.32\textwidth}
        \includegraphics[width=\textwidth]{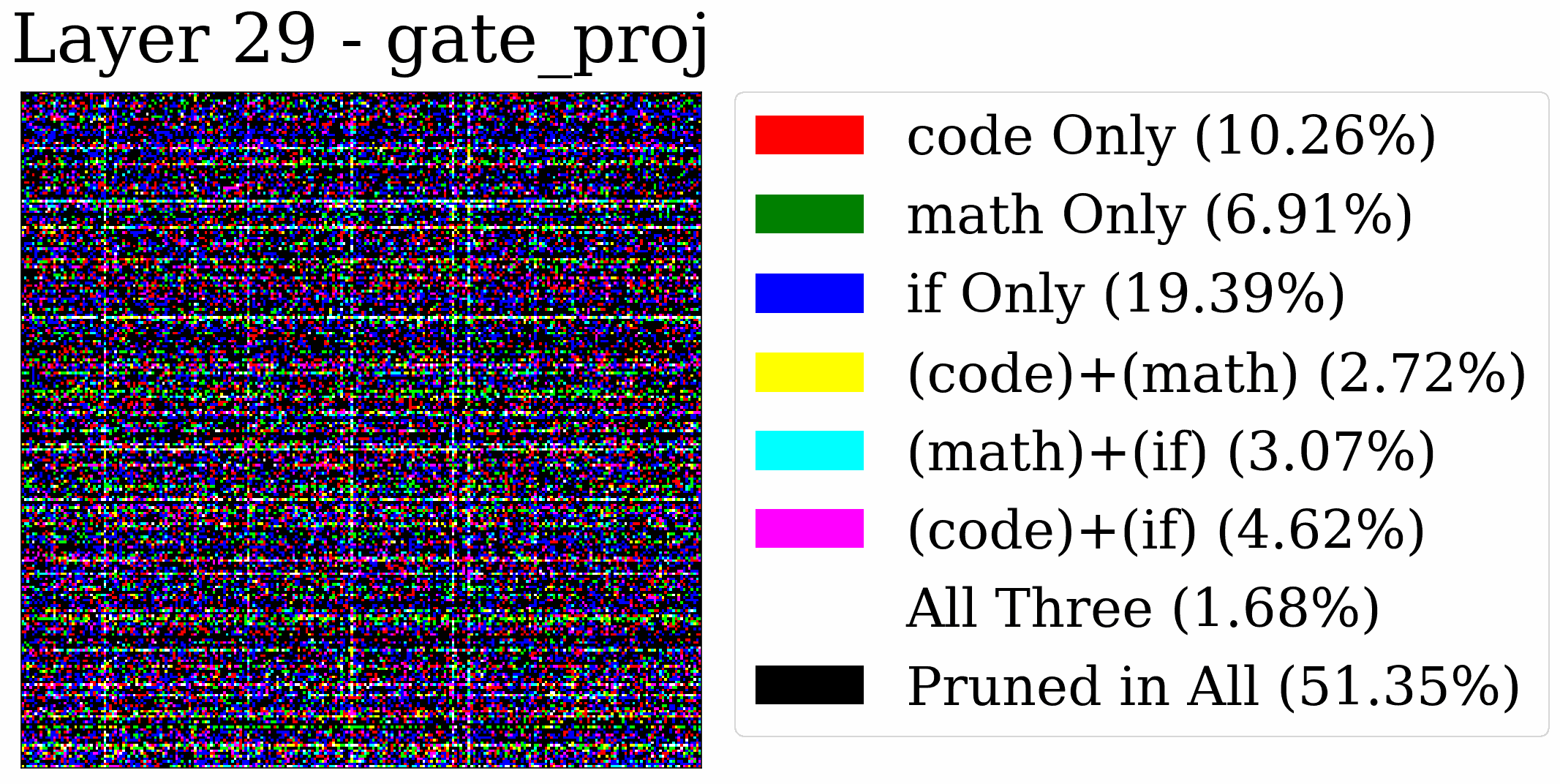}
    \end{subfigure}
    \hfill
    \begin{subfigure}[b]{0.32\textwidth}
        \includegraphics[width=\textwidth]{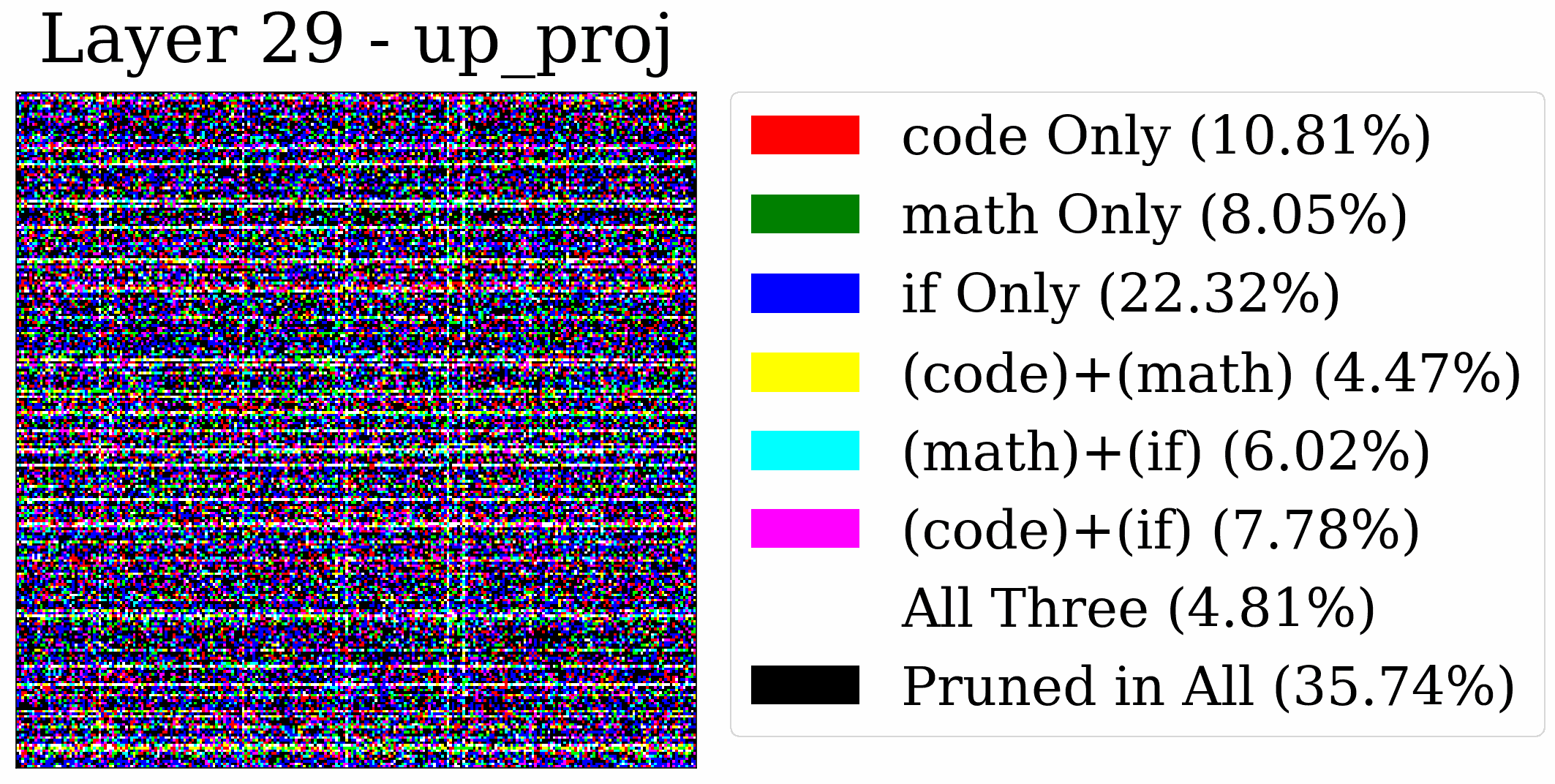}
    \end{subfigure}
    \hfill
    \begin{subfigure}[b]{0.32\textwidth}
        \includegraphics[width=\textwidth]{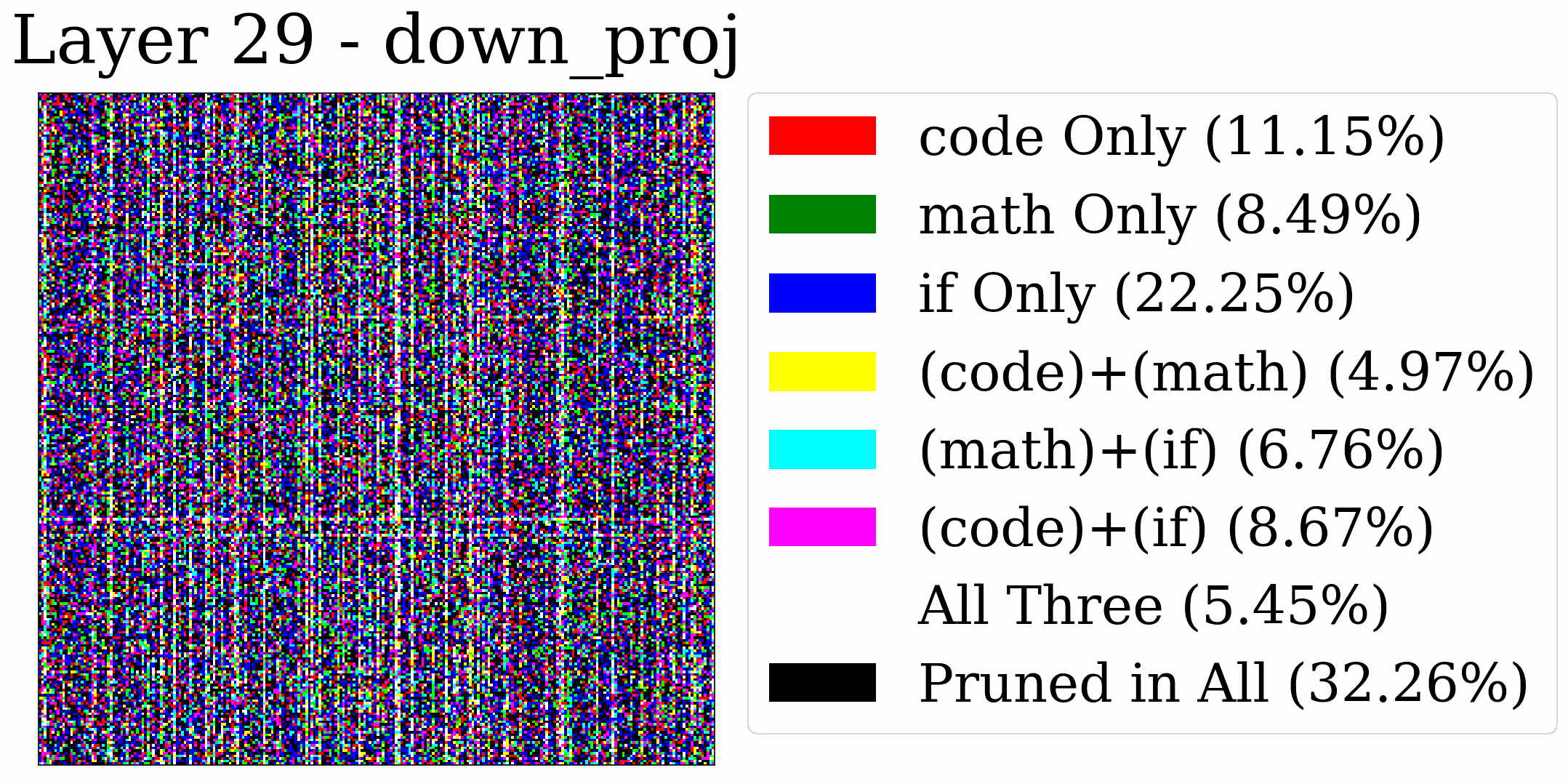}
    \end{subfigure}
    % Layer 30
    \begin{subfigure}[b]{0.32\textwidth}
        \includegraphics[width=\textwidth]{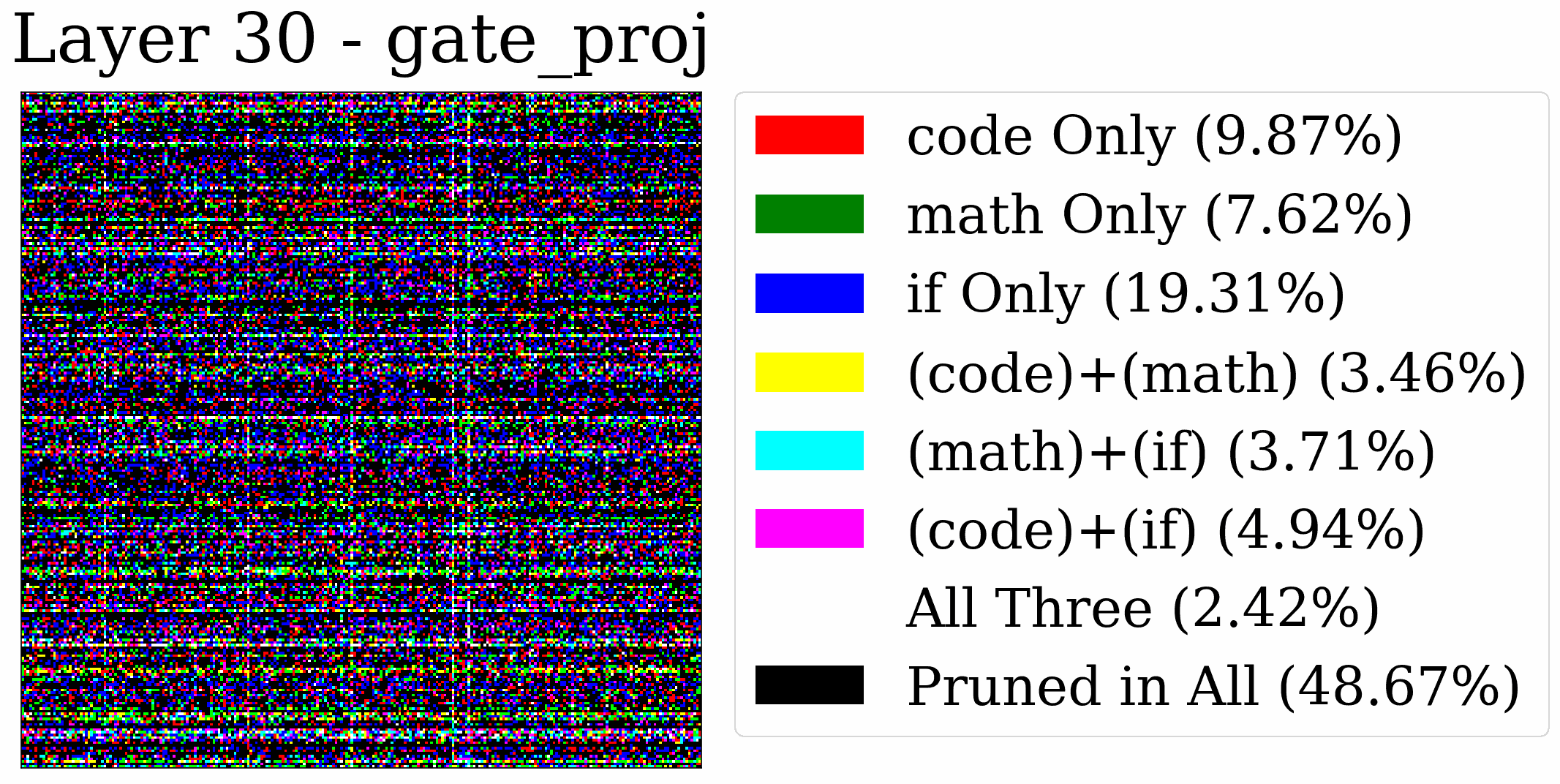}
    \end{subfigure}
    \hfill
    \begin{subfigure}[b]{0.32\textwidth}
        \includegraphics[width=\textwidth]{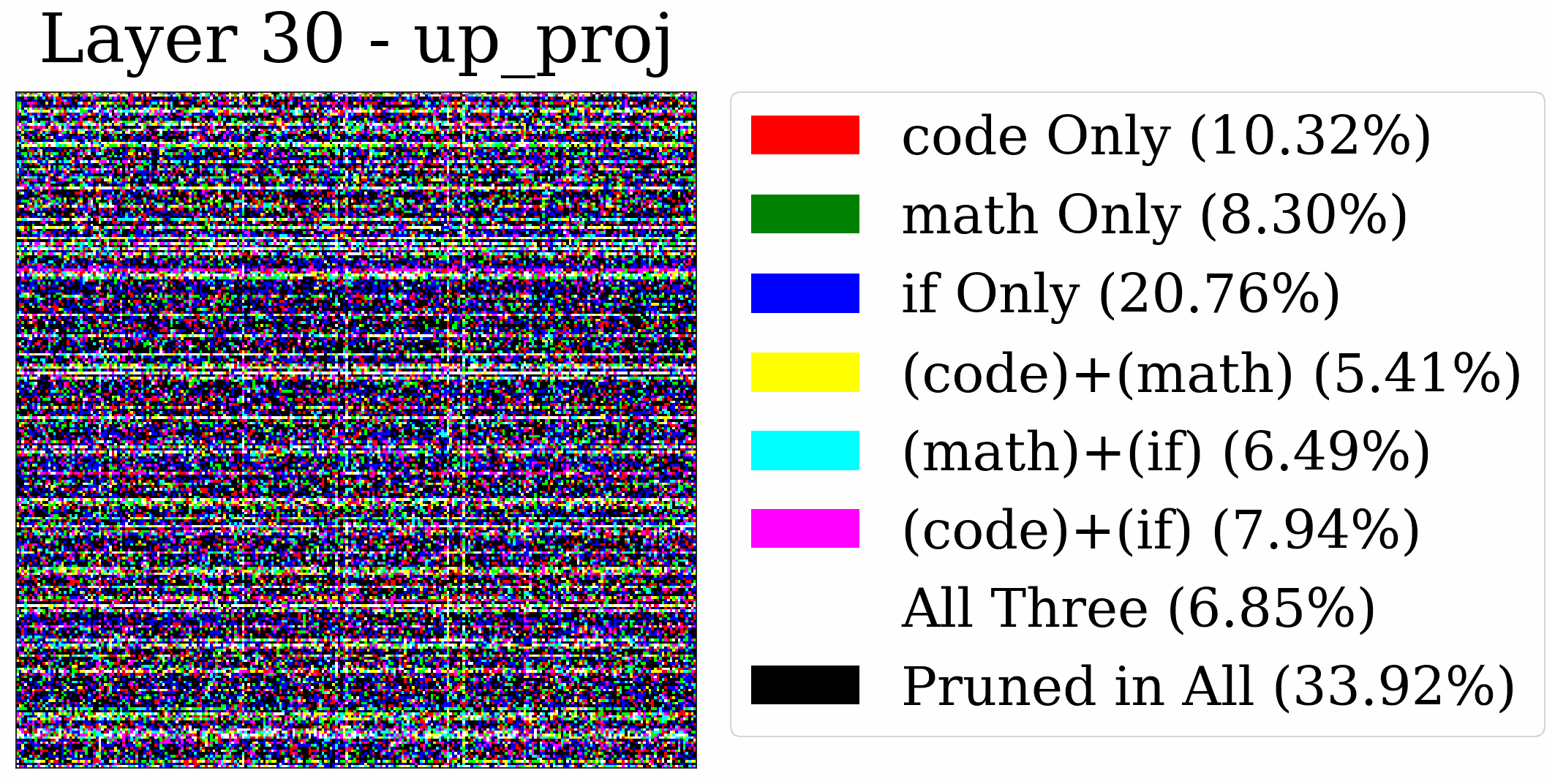}
    \end{subfigure}
    \hfill
    \begin{subfigure}[b]{0.32\textwidth}
        \includegraphics[width=\textwidth]{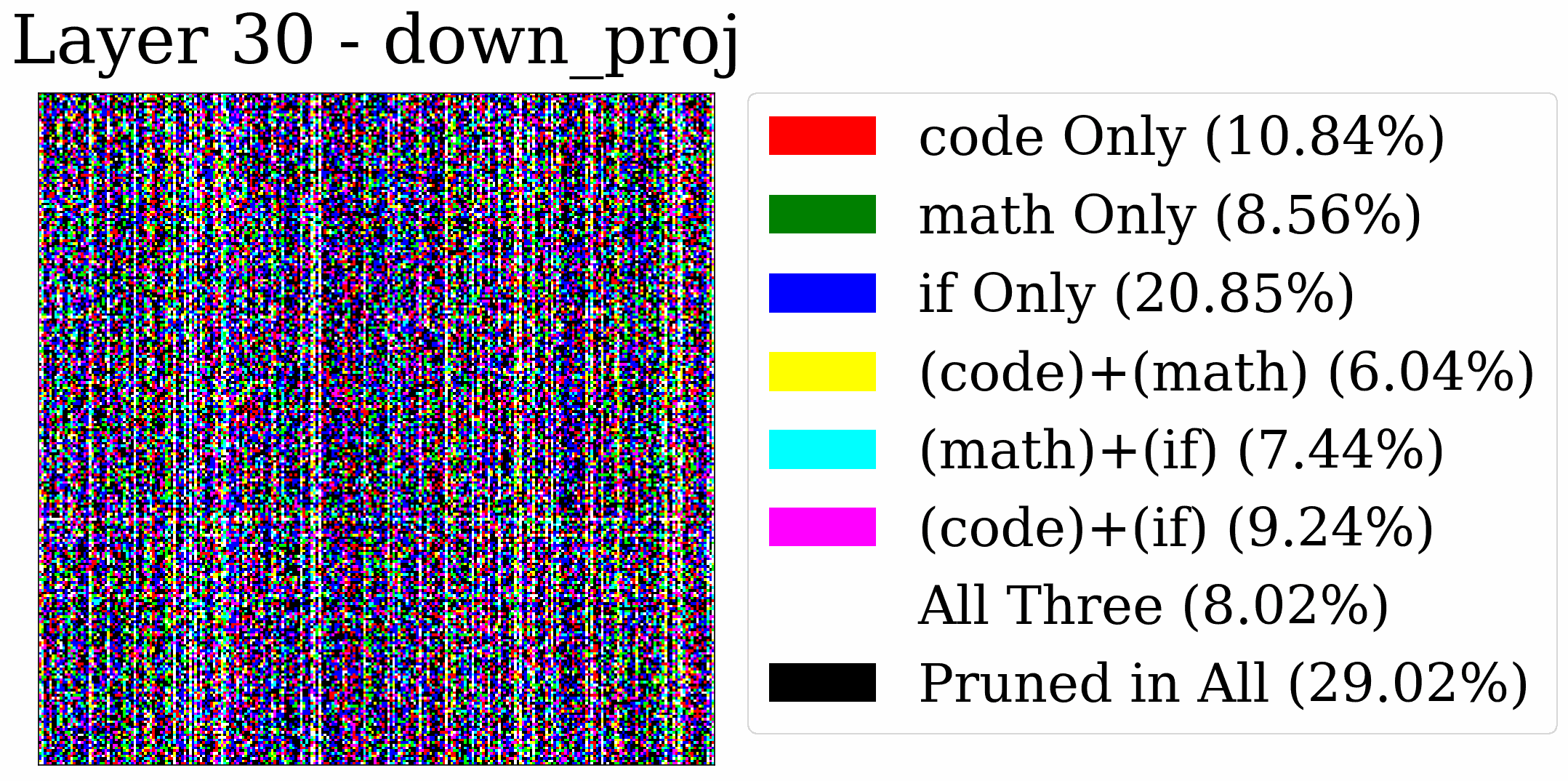}
    \end{subfigure}
    \caption{3-way FFG comparison for FFN components across layers 1, 2, 15, 18, 29, and 30. Columns show $W_{gate}$, $W_{up}$, and $W_{down}$.}
    \label{fig:3way_ffg_ffn_layers}
\end{figure}

\begin{figure}[ht]
  \centering
  % Row 1: Math model
  \begin{subfigure}[b]{0.32\textwidth}
    \includegraphics[width=\textwidth]{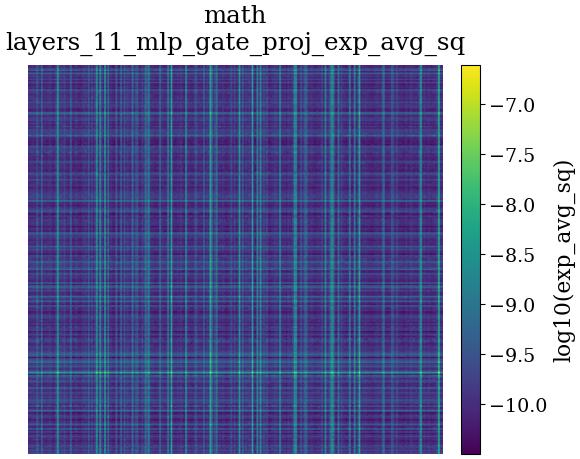}
  \end{subfigure}
  \hfill
  \begin{subfigure}[b]{0.32\textwidth}
    \includegraphics[width=\textwidth]{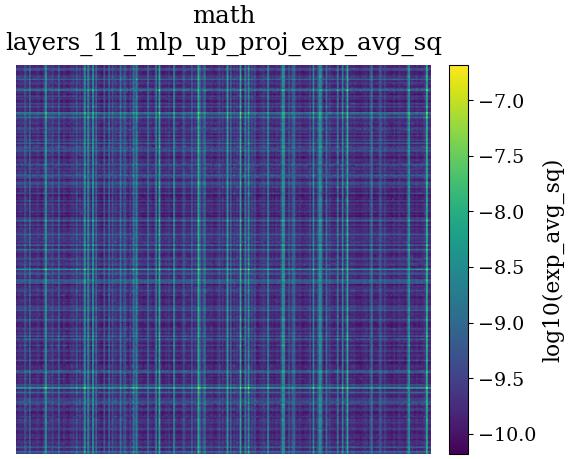}
  \end{subfigure}
  \hfill
  \begin{subfigure}[b]{0.32\textwidth}
    \includegraphics[width=\textwidth]{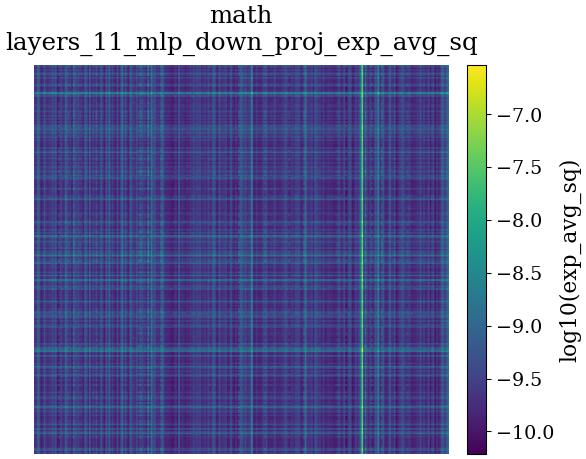}
  \end{subfigure}
  % Row 2: Code model
  \begin{subfigure}[b]{0.32\textwidth}
    \includegraphics[width=\textwidth]{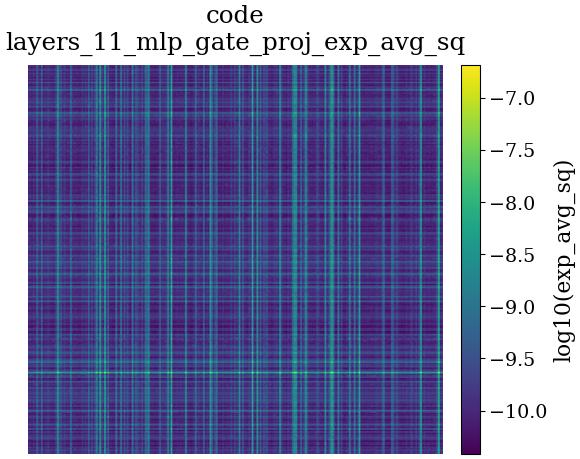}
  \end{subfigure}
  \hfill
  \begin{subfigure}[b]{0.32\textwidth}
    \includegraphics[width=\textwidth]{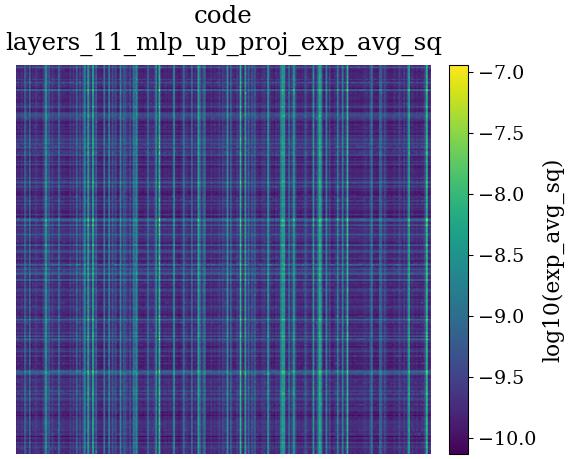}
  \end{subfigure}
  \hfill
  \begin{subfigure}[b]{0.32\textwidth}
    \includegraphics[width=\textwidth]{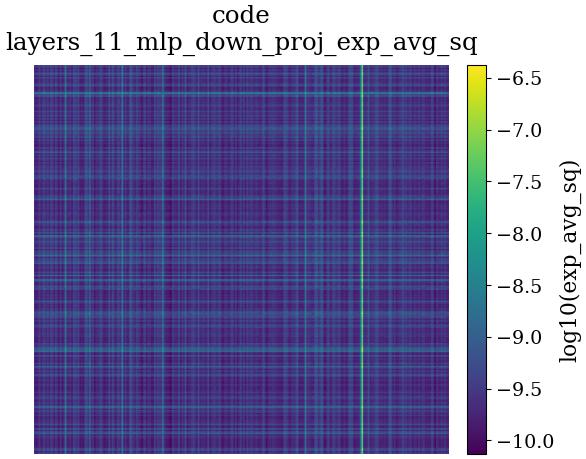}
  \end{subfigure}
  % Row 3: Precise IF model
  \begin{subfigure}[b]{0.32\textwidth}
    \includegraphics[width=\textwidth]{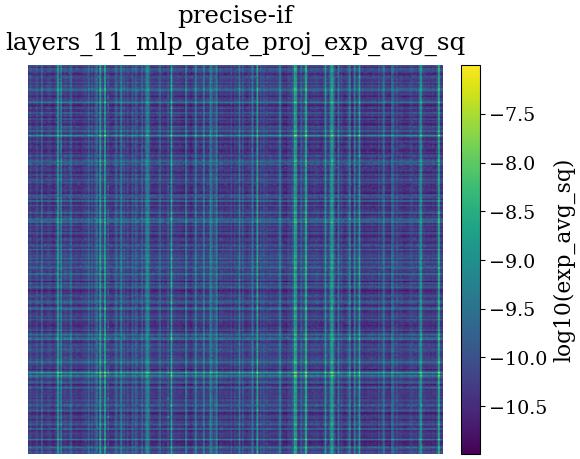}
  \end{subfigure}
  \hfill
  \begin{subfigure}[b]{0.32\textwidth}
    \includegraphics[width=\textwidth]{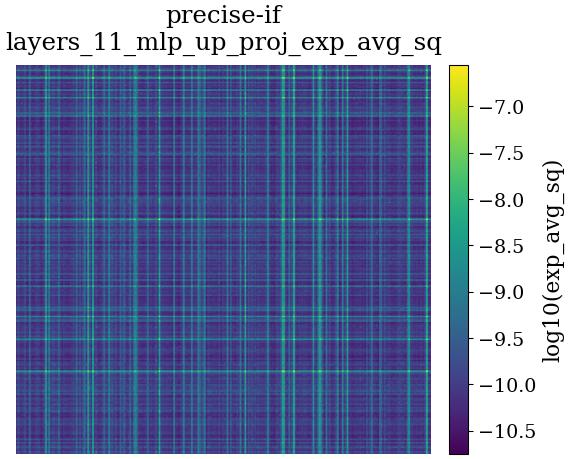}
  \end{subfigure}
  \hfill
  \begin{subfigure}[b]{0.32\textwidth}
    \includegraphics[width=\textwidth]{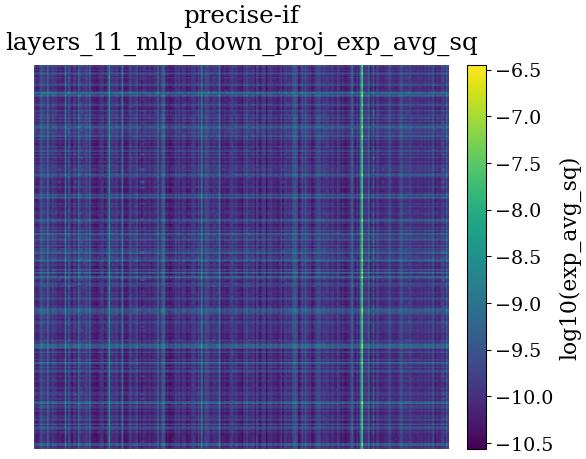}
  \end{subfigure}
  % Row 4: Max-Min ratio
  \begin{subfigure}[b]{0.32\textwidth}
    \includegraphics[width=\textwidth]{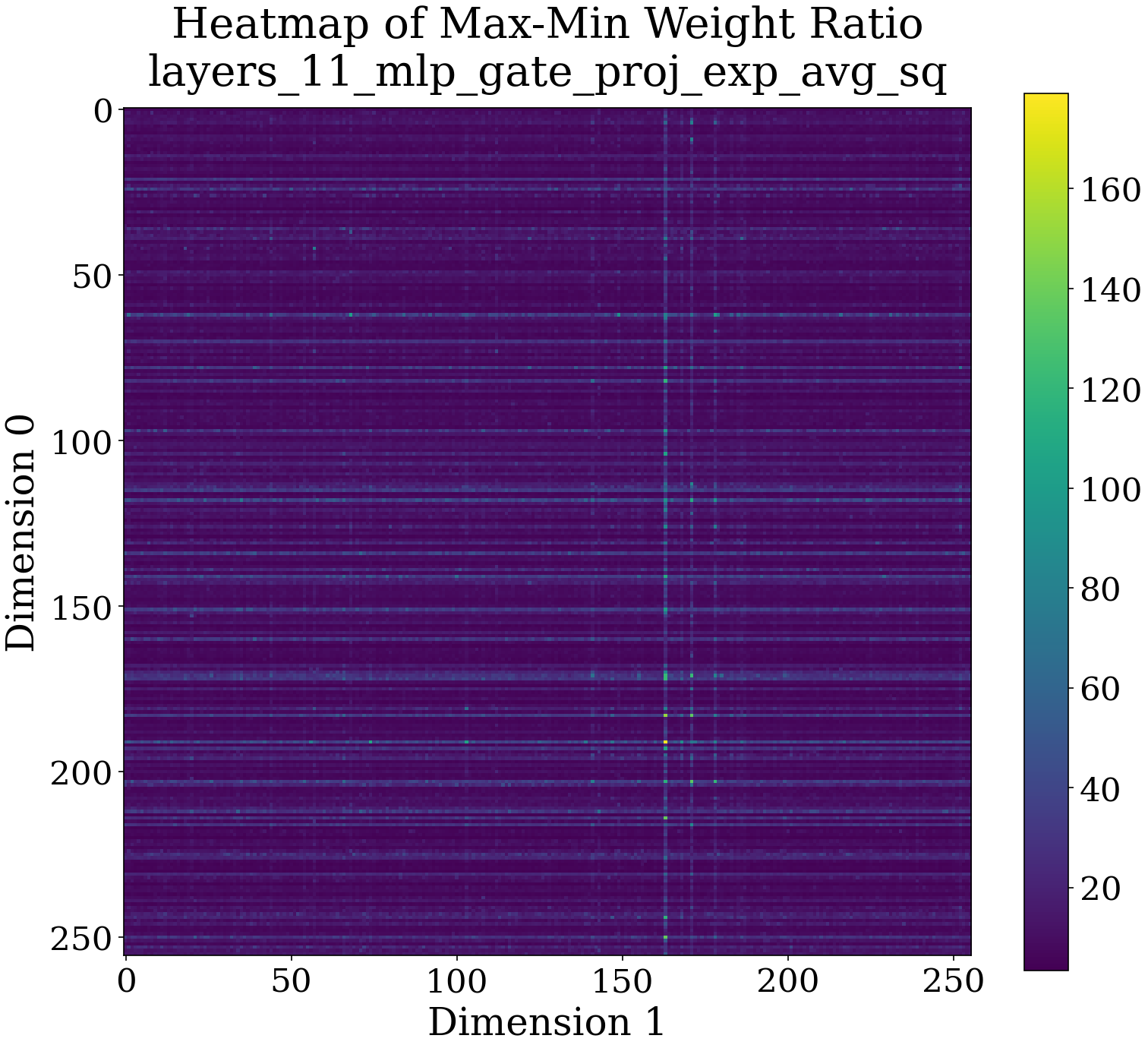}
  \end{subfigure}
  \hfill
  \begin{subfigure}[b]{0.32\textwidth}
    \includegraphics[width=\textwidth]{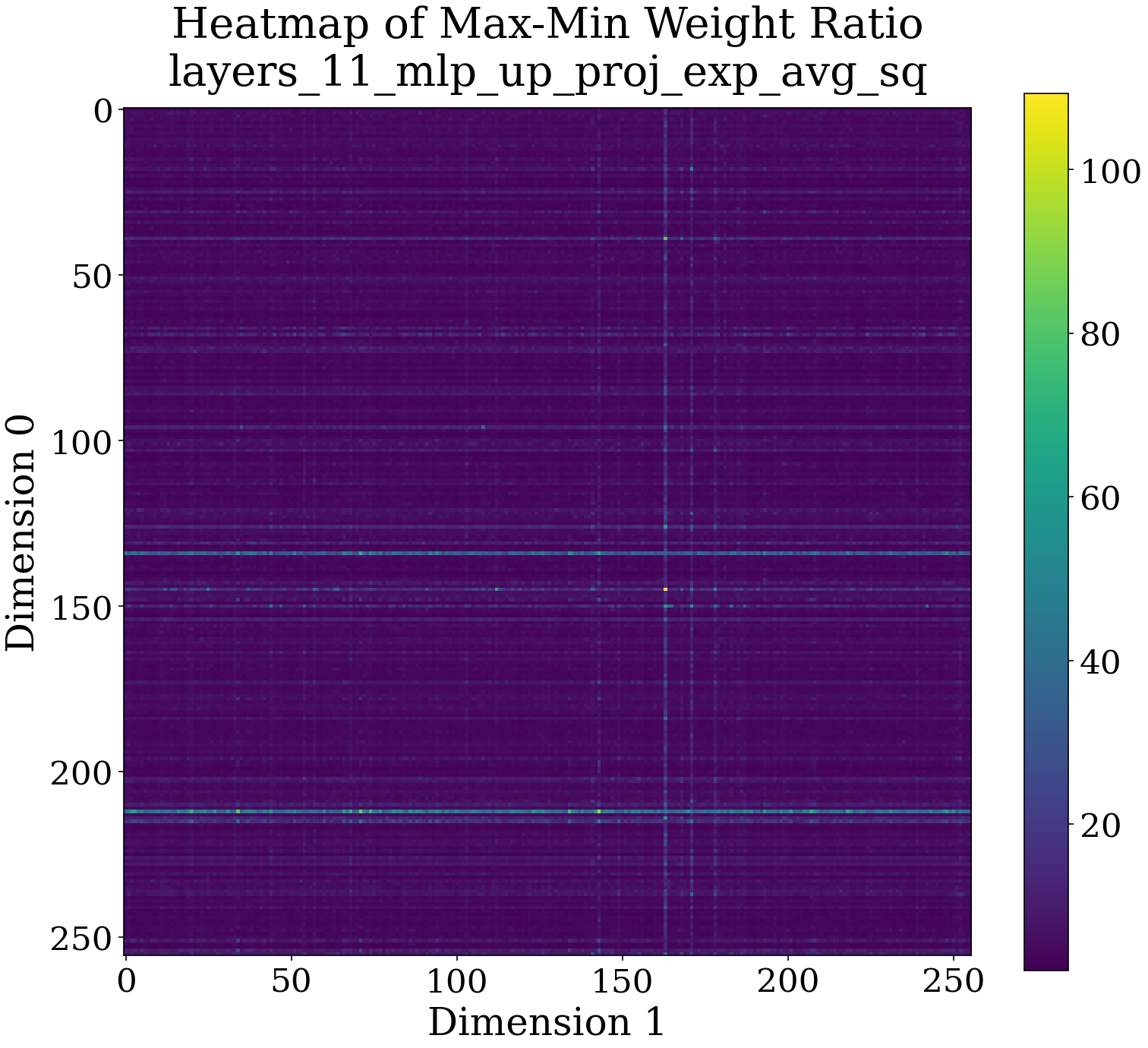}
  \end{subfigure}
  \hfill
  \begin{subfigure}[b]{0.32\textwidth}
    \includegraphics[width=\textwidth]{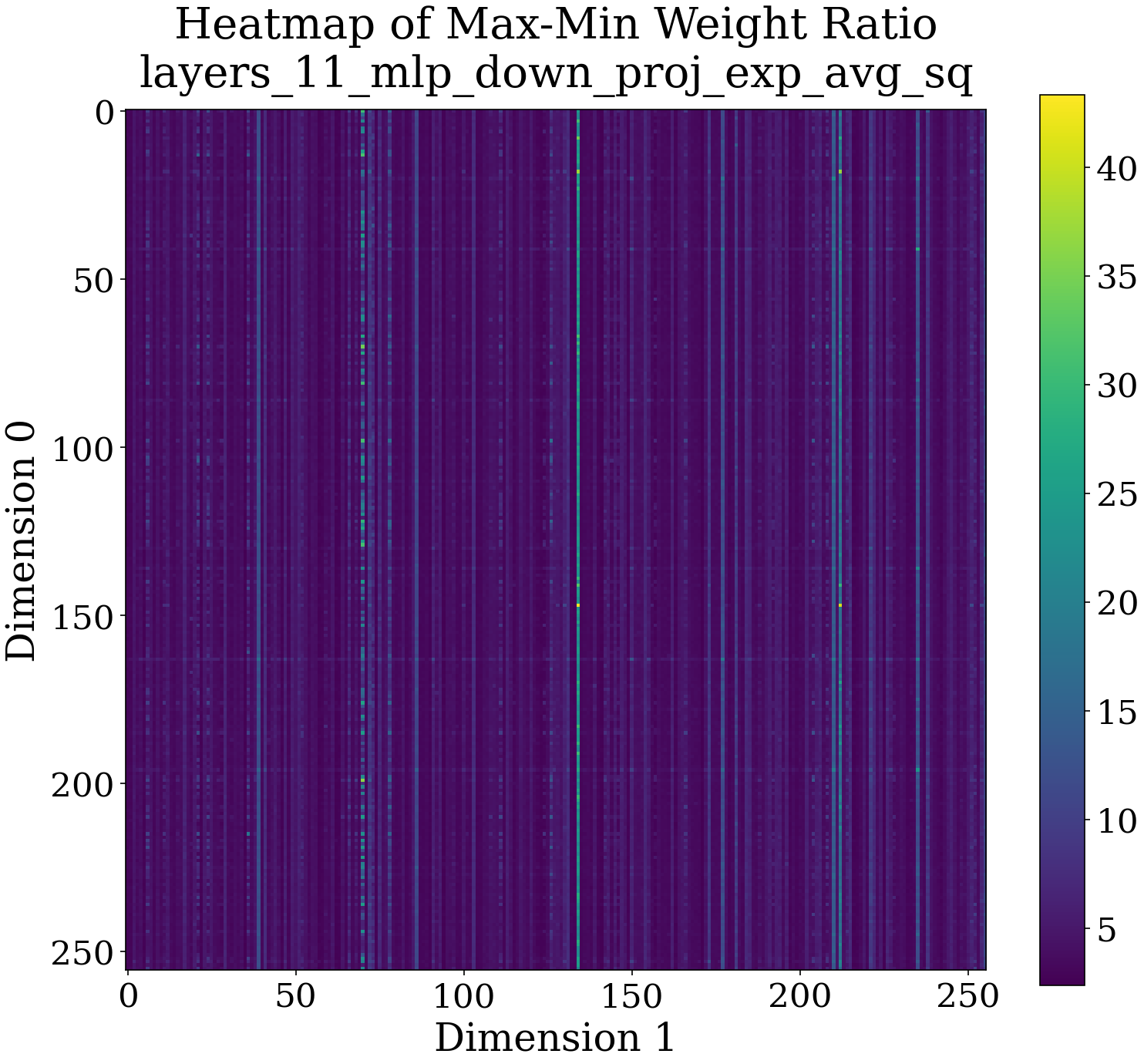}
  \end{subfigure}
  \caption{
    \textbf{Shared curvature geometry in FFN layers across specialist models.} 
    Log-scaled heatmaps of the square root of the second-moment Adam statistics for layer 11 feed-forward network projection weights. Rows represent: Math specialist, Code specialist, Precise IF specialist, and Max-Min ratio across all models (top to bottom). Columns show $W_{gate}$, $W_{up}$, and $W_{down}$ (left to right). The structural similarity persists even in FFN layers, reinforcing our finding that shared geometry is a model-wide phenomenon. The bottom row quantifies the variance across models, with darker regions indicating higher consensus in curvature patterns.
  }
  \label{fig:ffn_shared_curvature}
\end{figure}

\begin{figure}[ht]
  \centering
  % Row 1: Code model (Cosine LR)
  \begin{subfigure}[b]{0.32\textwidth}
    \includegraphics[width=\textwidth]{Figures/curvature_heatmaps/layers_11_mlp_gate_proj_exp_avg_sq_model_1_code_weights_heatmap.jpg}
  \end{subfigure}
  \hfill
  \begin{subfigure}[b]{0.32\textwidth}
    \includegraphics[width=\textwidth]{Figures/curvature_heatmaps/layers_11_mlp_up_proj_exp_avg_sq_model_1_code_weights_heatmap.jpg}
  \end{subfigure}
  \hfill
  \begin{subfigure}[b]{0.32\textwidth}
    \includegraphics[width=\textwidth]{Figures/curvature_heatmaps/layers_11_mlp_down_proj_exp_avg_sq_model_1_code_weights_heatmap.jpg}
  \end{subfigure}
  % Row 2: Code model (WSD LR)
  \begin{subfigure}[b]{0.32\textwidth}
    \includegraphics[width=\textwidth]{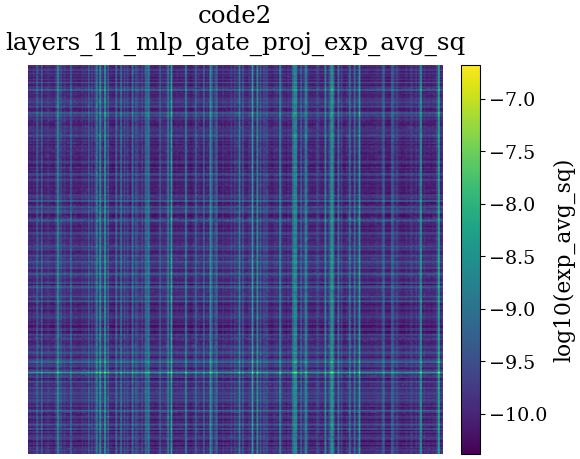}
  \end{subfigure}
  \hfill
  \begin{subfigure}[b]{0.32\textwidth}
    \includegraphics[width=\textwidth]{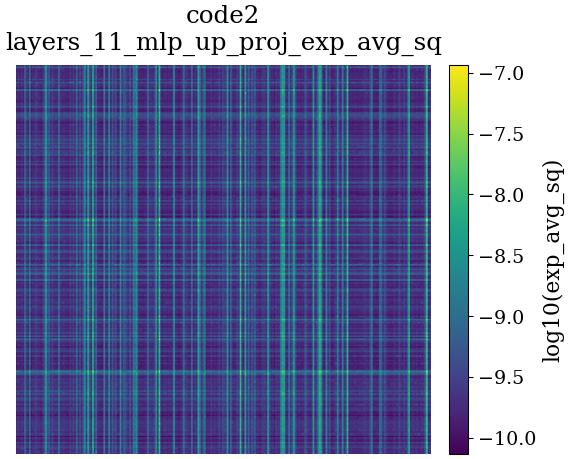}
  \end{subfigure}
  \hfill
  \begin{subfigure}[b]{0.32\textwidth}
    \includegraphics[width=\textwidth]{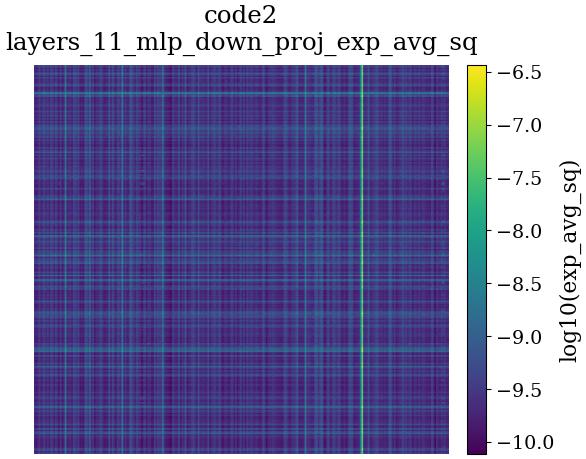}
  \end{subfigure}
  % Row 3: Max-Min ratio
  \begin{subfigure}[b]{0.32\textwidth}
    \includegraphics[width=\textwidth]{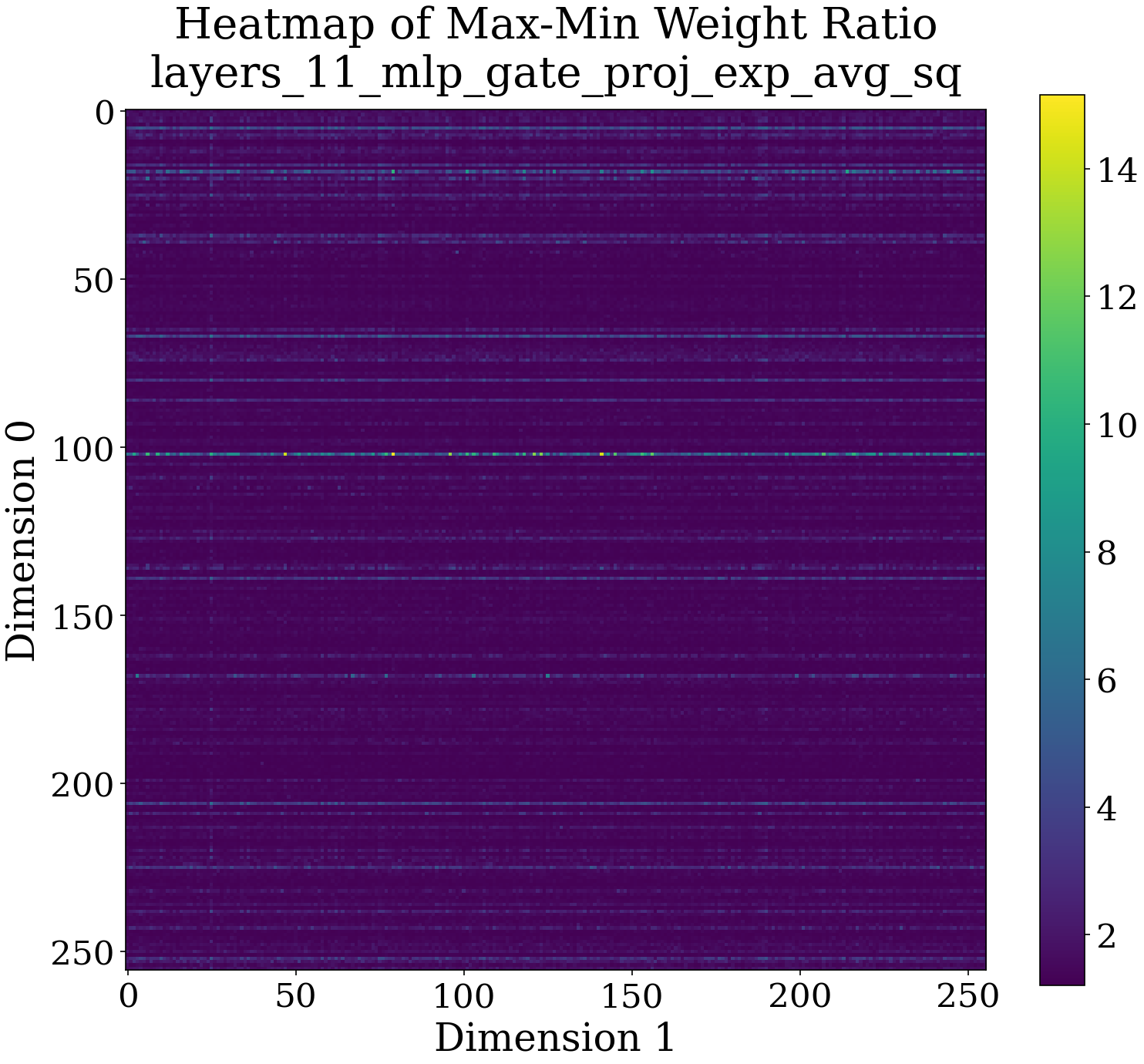}
  \end{subfigure}
  \hfill
  \begin{subfigure}[b]{0.32\textwidth}
    \includegraphics[width=\textwidth]{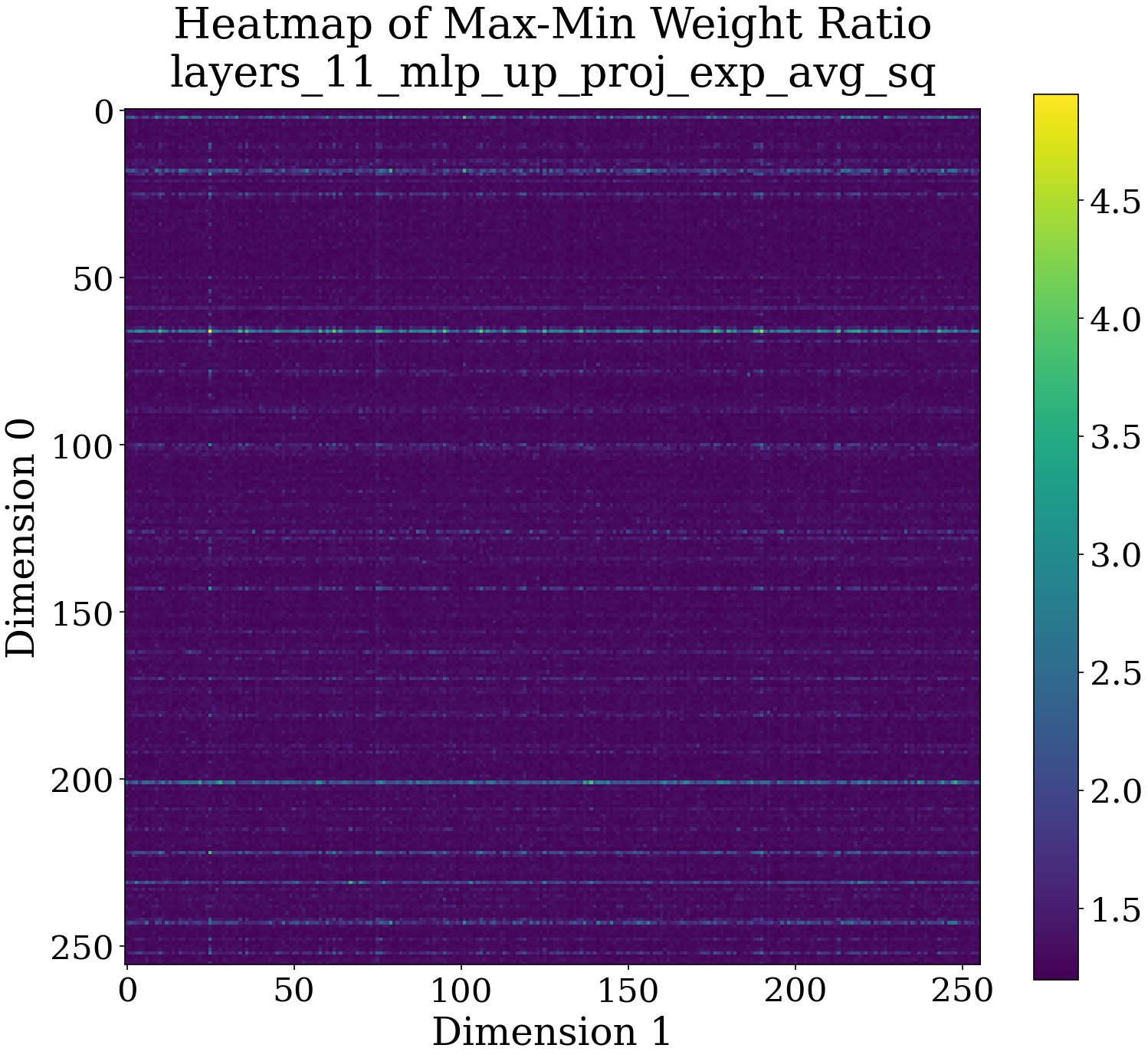}
  \end{subfigure}
  \hfill
  \begin{subfigure}[b]{0.32\textwidth}
    \includegraphics[width=\textwidth]{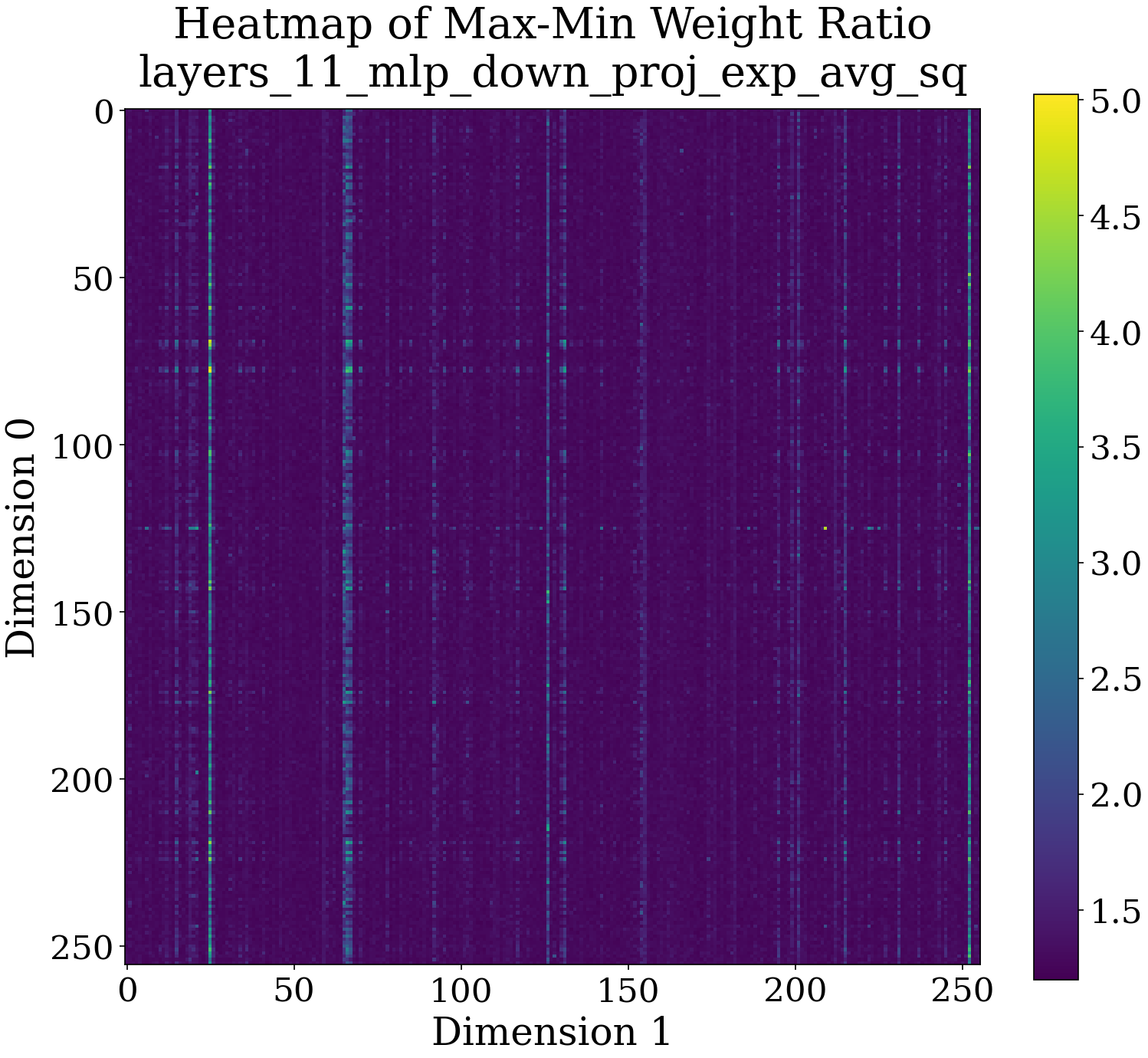}
  \end{subfigure}
  \caption{
    \textbf{Control Experiment: Shared curvature geometry in FFN layers for two Code models.} 
    Log-scaled heatmaps of the square root of the second-moment Adam statistics for layer 11 feed-forward network projection weights. Rows represent: Code specialist (Cosine LR), Code specialist (WSD LR), and Max-Min ratio across the two models (top to bottom). The near-perfect structural similarity and low max-min ratio provide a strong control for our main hypothesis.
  }
  \label{fig:ffn_shared_curvature_code_vs_code2}
\end{figure}

\section{Visualizing Structured Sparsity Masks}
\label{app:viz-structured-sparsity}

\paragraph{What the histograms show.}
For a weight matrix \(W\in\mathbb{R}^{d_{\text{out}}\times d_{\text{in}}}\) with binary mask \(M\in\{0,1\}^{d_{\text{out}}\times d_{\text{in}}}\),
we summarize mask structure via the row-wise and column-wise sparsities
\[
\rho^{\text{row}}_{i} \;=\; 1 - \frac{1}{d_{\text{in}}}\sum_{j=1}^{d_{\text{in}}} M_{ij},
\qquad
\rho^{\text{col}}_{j} \;=\; 1 - \frac{1}{d_{\text{out}}}\sum_{i=1}^{d_{\text{out}}} M_{ij},
\]
i.e., the fraction of zeros in each row/column (\(\text{sparsity}=1-\text{density}\)).
Each panel in Figs.~\ref{fig:ffg-q-hist} and \ref{fig:mag-q-hist} plots the histogram of
\(\{\rho^{\text{row}}_{i}\}_{i=1}^{d_{\text{out}}}\) (top) and \(\{\rho^{\text{col}}_{j}\}_{j=1}^{d_{\text{in}}}\) (bottom) for the self-attention \(q\)-projection of a single layer.
A spike near \(1.0\) indicates rows/columns that are almost entirely pruned.
All masks shown correspond to a global \(40\%\) density budget.

\paragraph{Key finding: FFG induces structured channel sparsity at the network edges.}
Under a single global density budget, FFG reallocates nonzeros across depth and weight types.
In the \(q\)-projection, early (layers 0–1) and late (layers 29–30) blocks display pronounced structure:
their histograms concentrate near \(\rho\simeq 1.0\), revealing many rows/columns that are nearly all zeros (Fig.~\ref{fig:ffg-q-hist}).
A similar pattern is observed for the \(k\)-projection (not shown), indicating that FFG often eliminates entire input/output channels in these attention blocks rather than scattering zeros uniformly.

\paragraph{Contrast with magnitude pruning.}
For the same global budget, magnitude pruning yields weaker row/column structure:
its histograms are centered around moderate sparsities with limited mass near~1.0 (Fig.~\ref{fig:mag-q-hist}).
Thus, FFG is not merely more sparse; it is selectively sparse at the level of entire channels.

\begin{figure}[t]
    \centering
    \includegraphics[width=.46\linewidth]{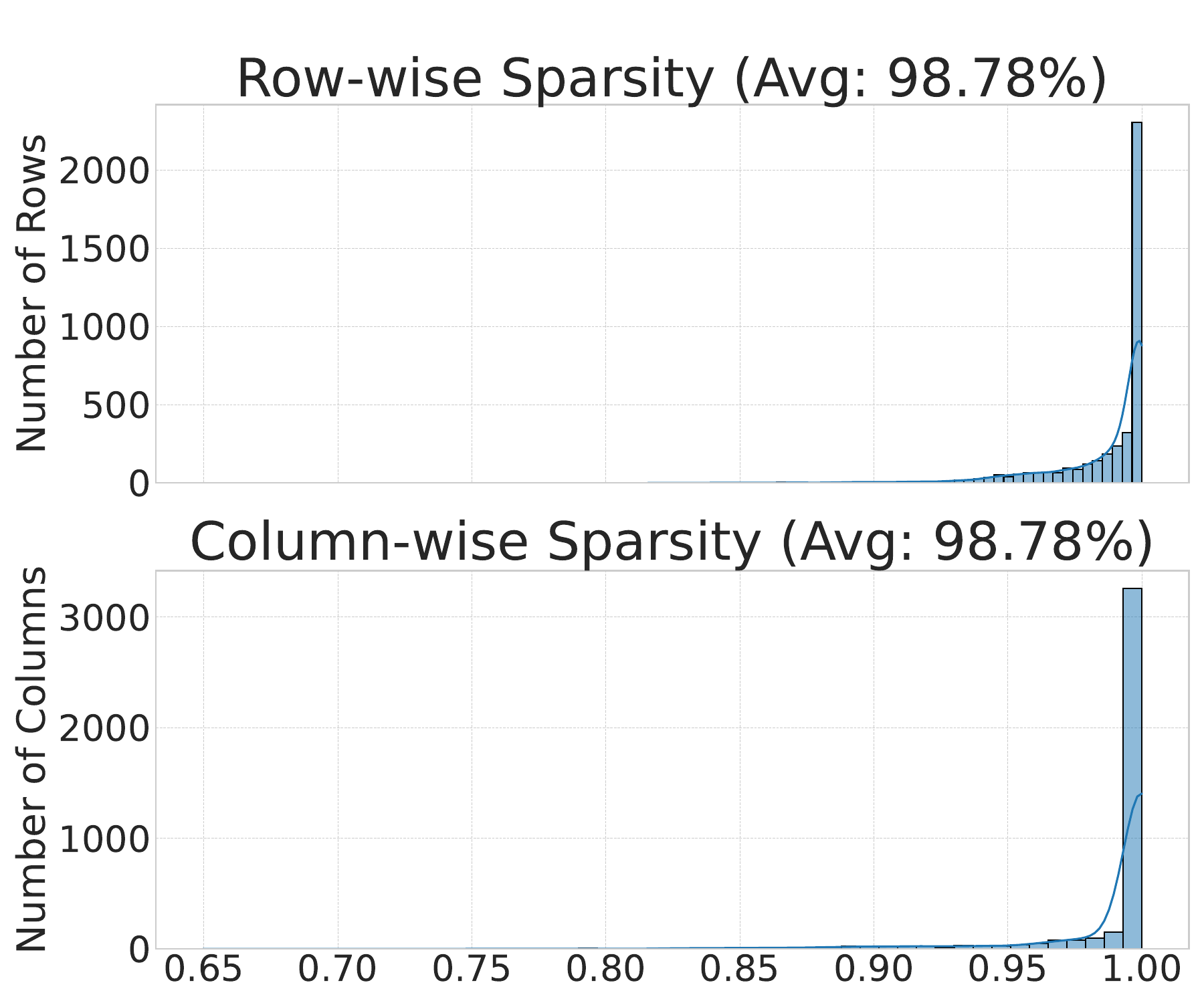}%
    \hfill
    \includegraphics[width=.46\linewidth]{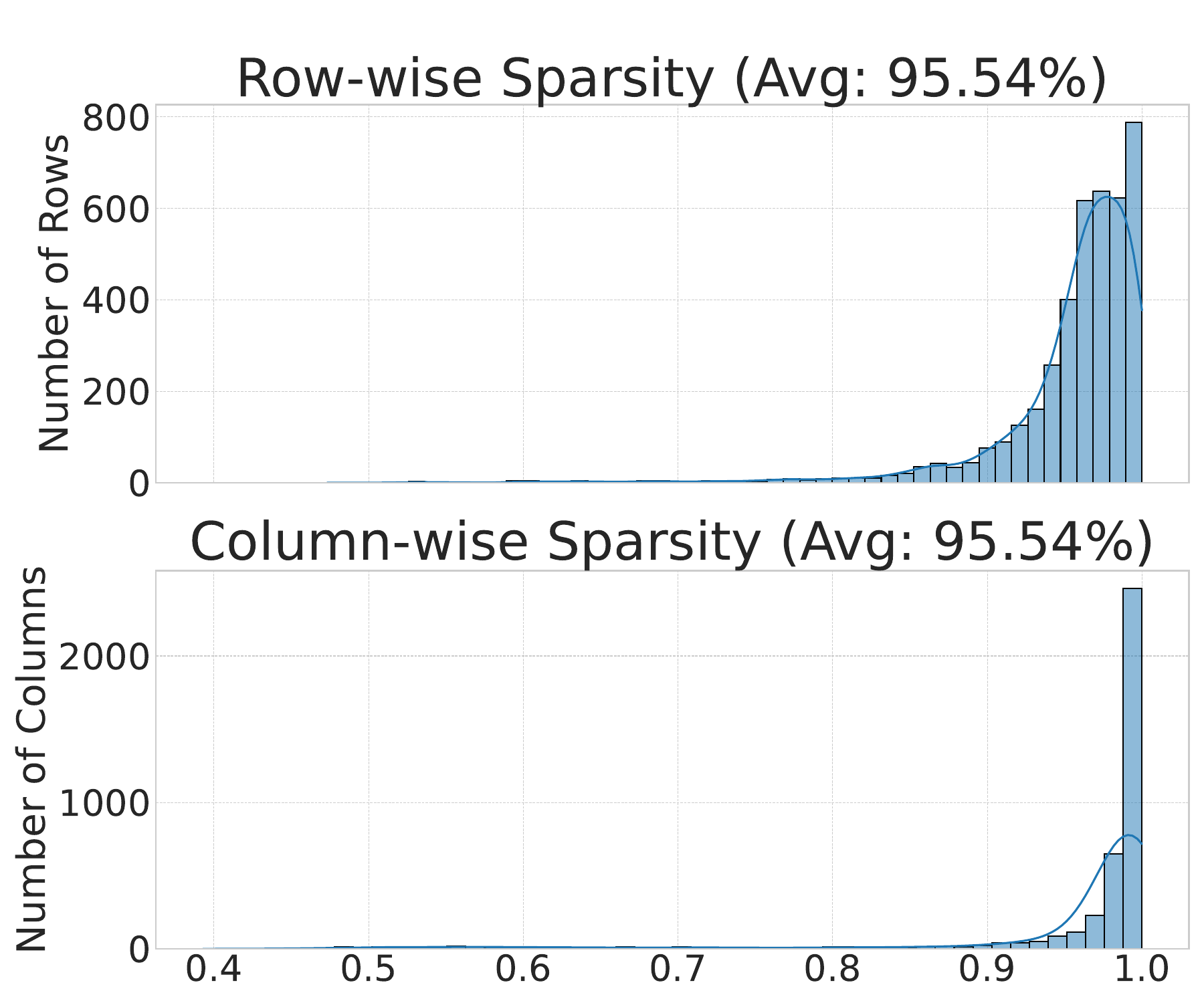}\\[0.25em]
    \includegraphics[width=.46\linewidth]{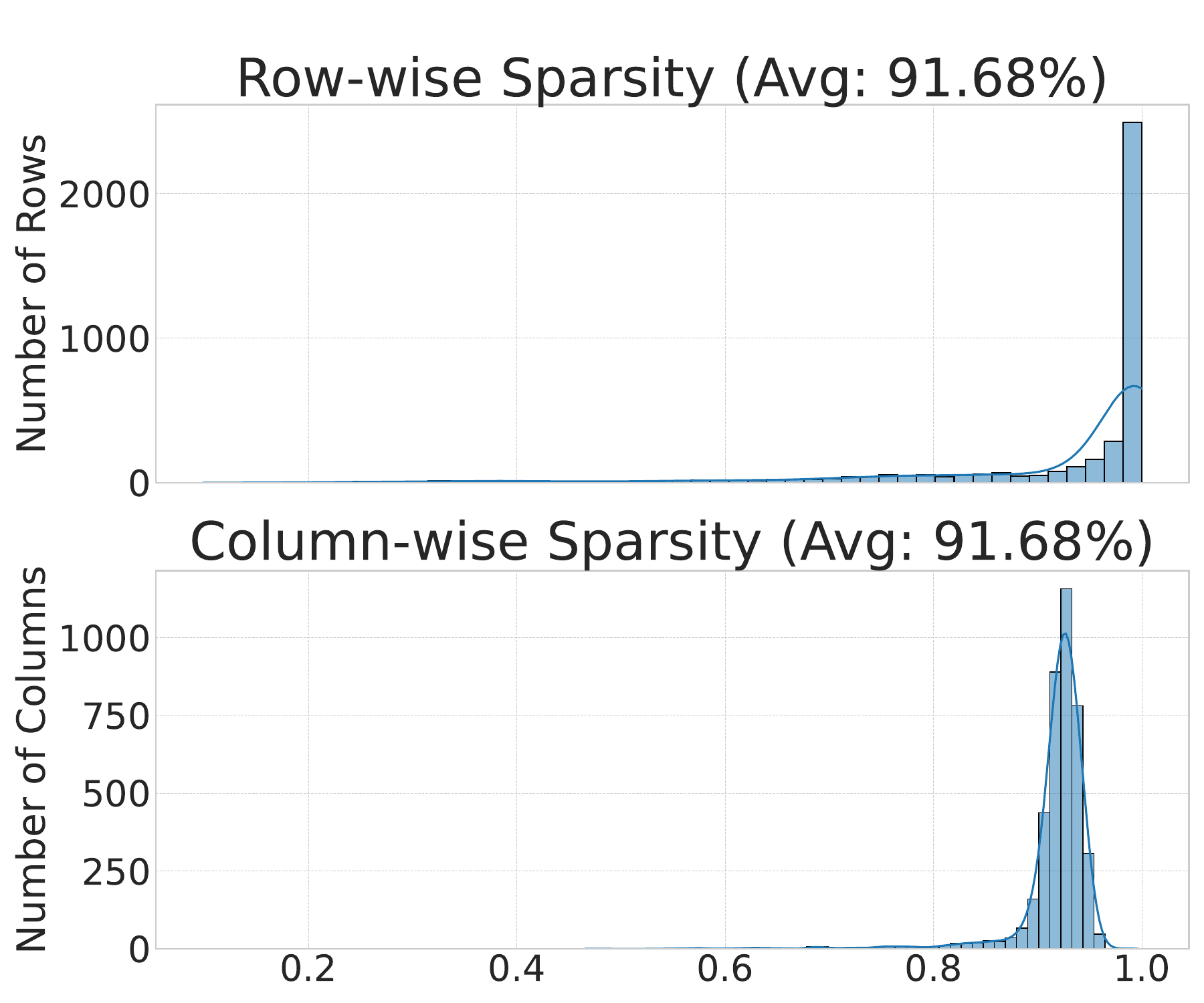}%
    \hfill
    \includegraphics[width=.46\linewidth]{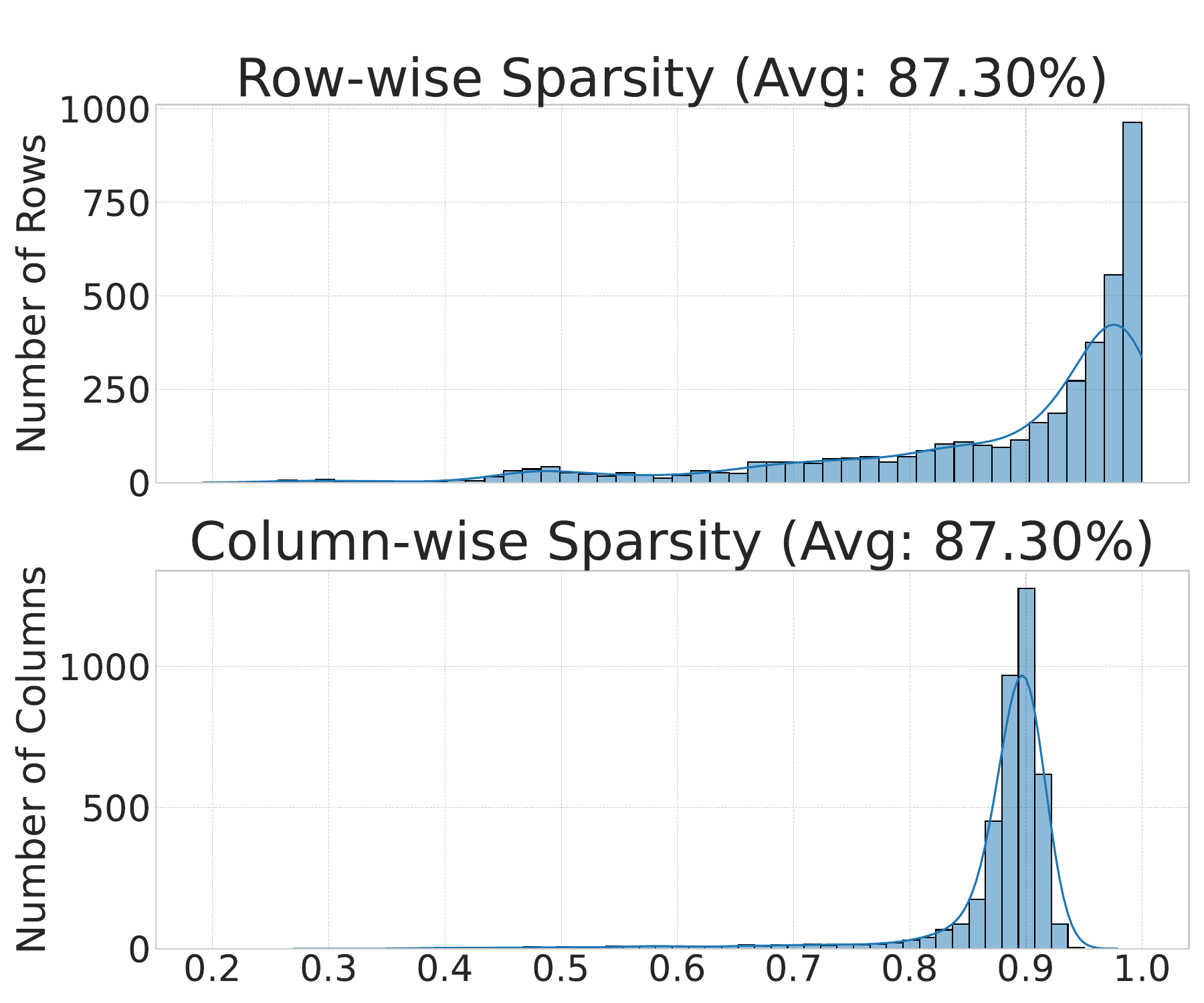}
    \caption{\textbf{FFG masks exhibit structured sparsity in the self-attention \(q\)-projection.}
    Row-wise (top) and column-wise (bottom) sparsity histograms for layers
    0, 1, 29, and 30 (left-to-right, top-to-bottom). Note the concentration near \(\rho\approx 1.0\),
    indicating that many rows/columns are almost entirely pruned.}
    \label{fig:ffg-q-hist}
\end{figure}

\begin{figure}[t]
    \centering
    \includegraphics[width=.46\linewidth]{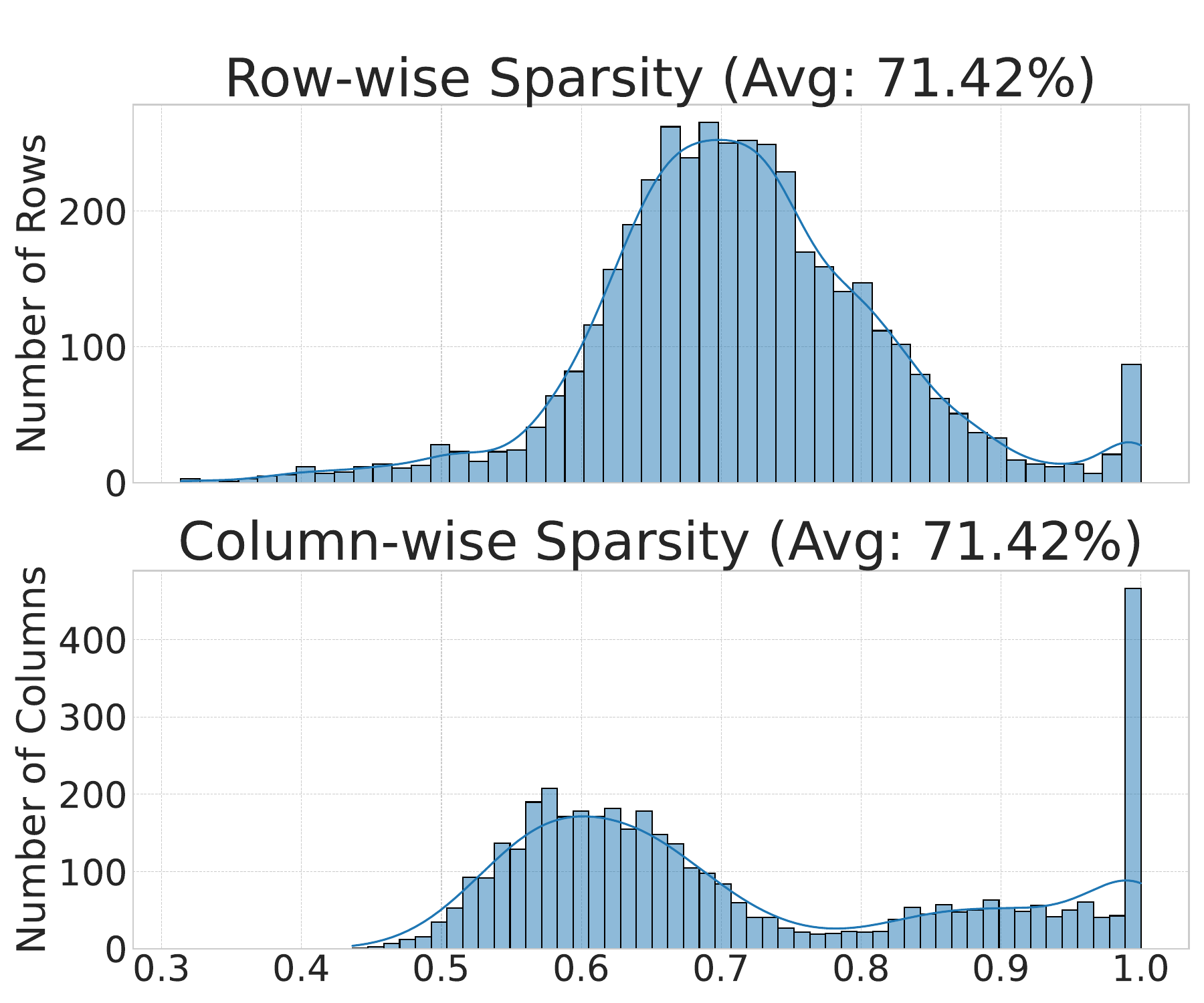}%
    \hfill
    \includegraphics[width=.46\linewidth]{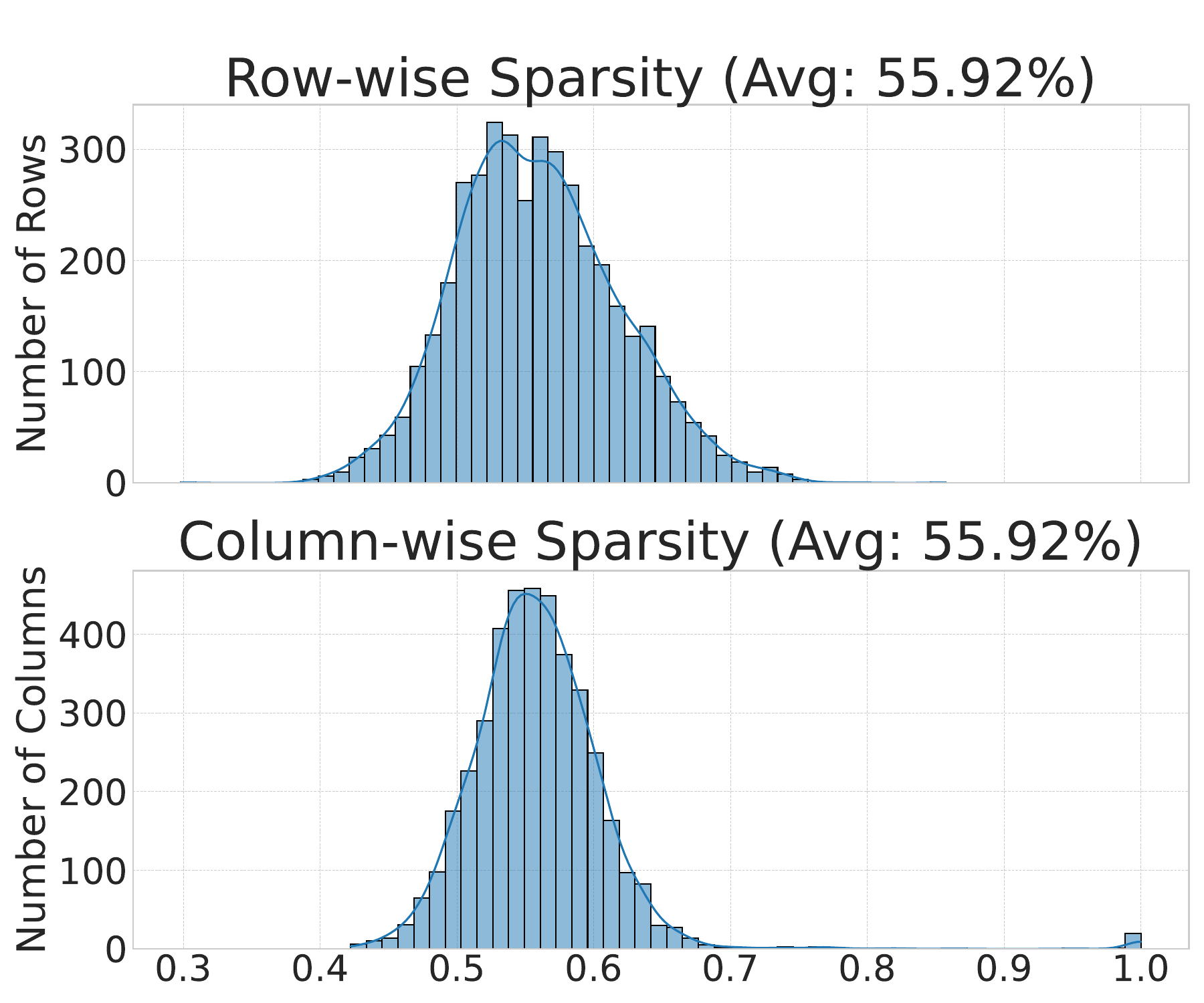}\\[0.25em]
    \includegraphics[width=.46\linewidth]{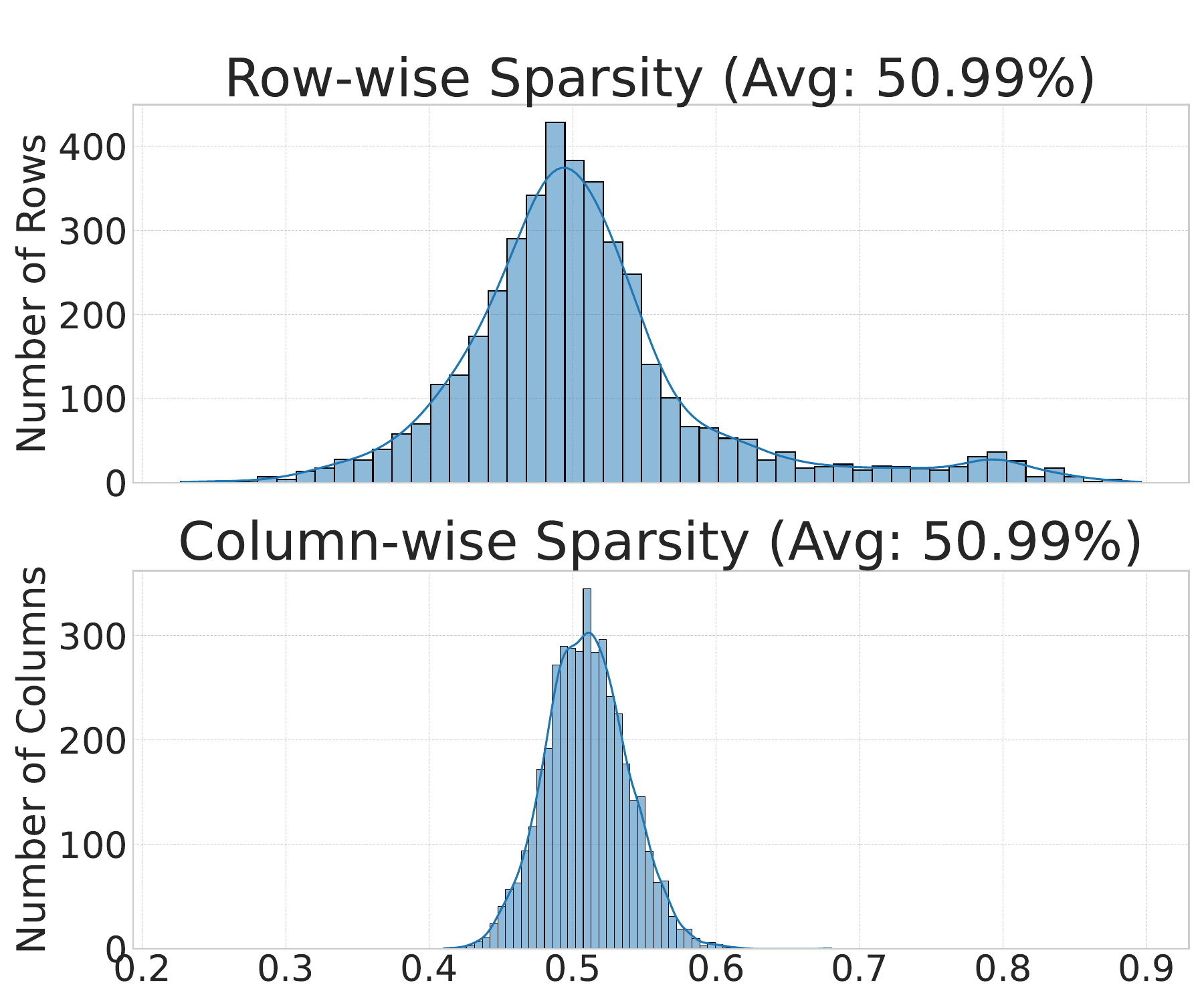}%
    \hfill
    \includegraphics[width=.46\linewidth]{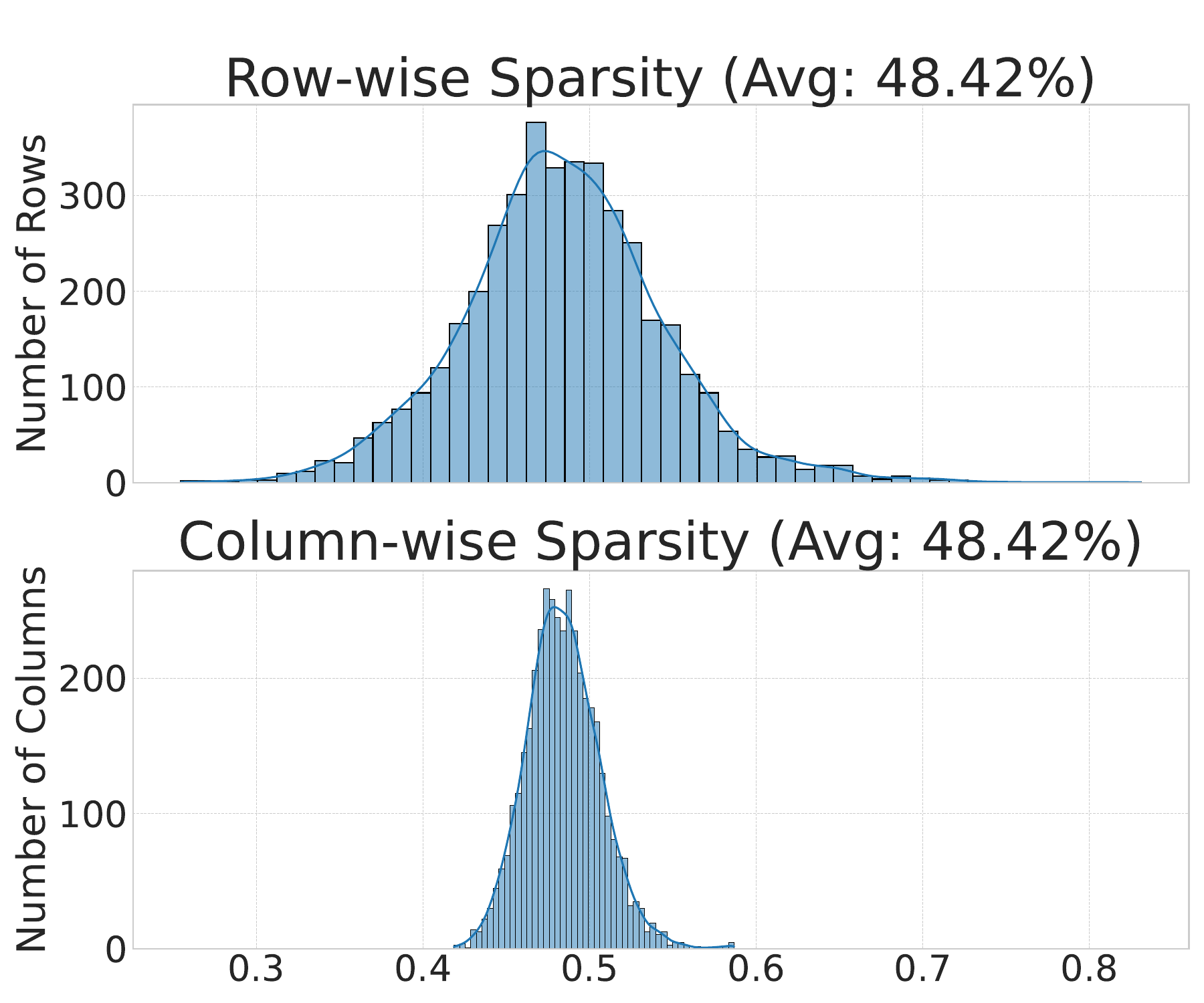}
    \caption{\textbf{Magnitude pruning produces weaker row/column structure.}
    Histograms for the same layers and weight type as Fig.~\ref{fig:ffg-q-hist} show mass centered at moderate sparsities and far less concentration near~1.0.}
    \label{fig:mag-q-hist}
\end{figure}

\end{document}